\documentclass{article}
\usepackage[top=1in, bottom=1in, left=1in, right=1in]{geometry}

\usepackage[utf8]{inputenc} %
\usepackage[T1]{fontenc}     %
\usepackage{hyperref}       %
\usepackage{xurl}            %
\usepackage{booktabs}       %
\usepackage{amsfonts}       %
\usepackage{nicefrac}       %
\usepackage{microtype}      %
\usepackage{xcolor}         %
 
\usepackage{amsthm,amsmath,amssymb,amsfonts}
\usepackage{tikz}
\usetikzlibrary{math}
\usepackage{mathrsfs}
\usepackage{dsfont}
\usepackage{cleveref}
\usepackage[boxruled]{algorithm2e}
\usepackage{enumitem,mathtools,pifont}
\usepackage[framemethod=TikZ]{mdframed}
\usepackage{fancybox,graphicx,subfigure} %

\usepackage{array}
\usepackage{multirow}

\SetAlgorithmName{Algorithm}{algorithm}{List of Algorithms}
\SetKwInput{kwInit}{Input}
\SetKwInput{kwOutp}{Output}
\SetKwFor{For}{For}{do}{end}
\SetKwFor{ForEach}{For each}{do}{end}
\newcommand{\commentalg}[1]{{\small\em\textcolor{gray}{//{#1}}}\hspace{-3mm}}

\newcommand{\yesnum}{\addtocounter{equation}{1}\tag{\theequation}}
\newcommand{\tagnum}[1]{\addtocounter{equation}{1}{\tag{#1)\ \ (\theequation}}}
\makeatletter
\newcommand{\customlabel}[2]{%
\protected@write \@auxout {}{\string \newlabel {#1}{{#2}{\thepage}{#2}{#1}{}} }%
\hypertarget{#1}{}
}
\makeatother

\mdfdefinestyle{FrameBox}{ linecolor=white, outerlinewidth=0pt, roundcorner=5pt,
innertopmargin=5pt, innerbottommargin=5pt,
innerrightmargin=10pt, innerleftmargin=10pt, backgroundcolor=white}

\mdfdefinestyle{FrameBox2}{ linecolor=black, outerlinewidth=0pt, roundcorner=5pt,
innertopmargin=5pt, innerbottommargin=5pt,
innerrightmargin=10pt, innerleftmargin=10pt, backgroundcolor=white}

\mdfdefinestyle{FrameQuestion2}{ linecolor=white, outerlinewidth=0pt,
roundcorner=5pt,
innertopmargin=5pt, innerbottommargin=0pt,
innerrightmargin=16pt, innerleftmargin=16pt, backgroundcolor=white}

\newtheorem{assumption}{Assumption}
\newtheorem{theorem}{Theorem}[section]
\newtheorem{lemma}[theorem]{Lemma}
\newtheorem{definition}[theorem]{Definition}

\newtheorem{corollary}[theorem]{Corollary}

\newtheorem{proposition}[theorem]{Proposition}
\newtheorem{fact}[theorem]{Fact}

\newtheorem{problem}[]{Problem}
\newtheorem*{theorem*}{Theorem}
\newtheorem{remark}[theorem]{Remark}
\newtheorem{example}[theorem]{Example}

\crefname{equation}{Equation}{Equations}
\crefname{figure}{Figure}{Figures}
\crefname{table}{Table}{Tables}
\crefname{section}{Section}{Sections}
\crefname{appendix}{Section}{Sections}
\crefname{algorithm}{Algorithm}{Algorithms}
\crefname{assumption}{Assumption}{Assumptions}
\crefname{theorem}{Theorem}{Theorems}
\crefname{lemma}{Lemma}{Lemmas}
\crefname{definition}{Definition}{Definitions}
\crefname{conjecture}{Conjecture}{Conjectures}
\crefname{corollary}{Corollary}{Corollaries}
\crefname{construction}{Construction}{Constructions}
\crefname{claim}{Claim}{Claims}
\crefname{observation}{Observation}{Observations}
\crefname{proposition}{Proposition}{Propositions}
\crefname{fact}{Fact}{Facts}
\crefname{question}{Question}{Questions}
\crefname{problem}{Problem}{Problems}
\crefname{remark}{Remark}{Remarks}
\crefname{example}{Example}{Examples}

\newcommand{\white}[1]{\textcolor{white}{#1}}

\newcommand{\Ham}{Hamming}

\newcommand{\N}{\mathbb{N}}

\newcommand{\R}{\mathbb{R}}

\newcommand{\cA}{\mathcal{A}}

\newcommand{\cD}{\mathcal{D}}
\newcommand{\cE}{\ensuremath{\mathcal{E}}}
\newcommand{\cF}{\mathcal{F}}

\newcommand{\cL}{\mathcal{L}}

\newcommand{\cP}{\mathcal{P}}

\newcommand{\cX}{\mathcal{X}}

\newcommand{\wt}[1]{\smash{\widetilde{#1}}}
\newcommand{\nfrac}{\nicefrac}
\newcommand{\sfrac}[2]{#1/#2}
\newcommand{\st}{\ensuremath{\mathrm{s.t.}}}

\newcommand{\eps}{\varepsilon}
\renewcommand{\epsilon}{\varepsilon}

\newcommand{\argmin}{\operatornamewithlimits{argmin}}

\newcommand{\Ex}{\operatornamewithlimits{\mathbb{E}}}

\newcommand{\poly}{\mathop{\mbox{\rm poly}}}
\newcommand{\polylog}{\mathop{\mbox{\rm polylog}}}

\def\abs#1{\left| #1 \right|}
\def\sabs#1{| #1 |}

\newcommand{\sinparen}[1]{(#1)}
\newcommand{\sinbrace}[1]{\{#1\}}
\newcommand{\sinsquare}[1]{[#1]}

\newcommand{\inparen}[1]{\left(#1\right)}
\newcommand{\inbrace}[1]{\left\{#1\right\}}
\newcommand{\insquare}[1]{\left[#1\right]}
\newcommand{\inangle}[1]{\left\langle#1\right\rangle}

\newcommand{\ceil}[1]{\left\lceil#1\right\rceil}

\newcommand{\zo}{\{0,1\}}

\newcommand{\negsp}{\hspace{-0.5mm}}

\newcommand{\sexp}[1]{{\hbox{\tiny$($}}#1{\hbox{\tiny$)$}}}

\newcommand{\wh}[1]{\widehat{#1}}

\newcommand{\hD}{\widehat{D}}

\newcommand{\hS}{\smash{\widehat{S}}}
\newcommand{\hX}{\widehat{X}}
\newcommand{\hY}{\widehat{Y}}
\newcommand{\hZ}{\smash{\widehat{Z}}}

\newcommand{\hx}{\widehat{x}}
\newcommand{\hy}{\widehat{y}}
\newcommand{\hz}{\widehat{z}}

\newcommand{\evG}{\mathscr{G}}

\newcommand{\evJ}{\mathscr{J}}

\newcommand{\folder}{./figs/}

\newcommand{\akm}{{\bf AKM}}

  \newcommand{\prog}[1]{Program~\eqref{#1}}
  \newcommand{\errtolerant}{Program~\ref{prog:errtolerant}}

  \newcommand{\Stackrel}[2]{\stackrel{\mathmakebox[\widthof{\ensuremath{#2}}]{#1}}{#2}}

  \newcommand{\midsepremove}{\aboverulesep = 0mm \belowrulesep = 0mm}
  \newcommand{\midsepdefault}{\aboverulesep = 0.605mm \belowrulesep = 0.984mm}

  \newif\ifconf
  \conffalse

  \newif\ifexpand
  \expandtrue

  \newcommand{\leftmarginINTERNAL}{18pt}
  \ifconf
  \renewcommand{\leftmarginINTERNAL}{18pt}
  \else
  \renewcommand{\leftmarginINTERNAL}{\leftmargin}
  \fi
  \newcommand{\itemsepINTERNAL}{0pt}
  \ifconf
  \renewcommand{\itemsepINTERNAL}{0pt}
  \else
  \renewcommand{\itemsepINTERNAL}{\itemsep}
  \fi

  \title{\bf Fair Classification with Adversarial Perturbations}

  \author{L. Elisa Celis \\ Yale University \and Anay Mehrotra \\ Yale University \and Nisheeth K. Vishnoi \\ Yale University}

\begin{document}

  \maketitle

  \begin{abstract}

    We study  fair classification in the presence of an omniscient adversary that, given an $\eta$, is allowed to choose an arbitrary $\eta$-fraction of the training samples and arbitrarily perturb  their protected attributes.
    The motivation comes from settings in which protected attributes can be incorrect due to strategic misreporting, malicious actors, or errors in imputation; and prior approaches that make stochastic or independence assumptions on  errors may not satisfy their guarantees in this adversarial setting.
    Our main contribution is an optimization framework to learn fair classifiers in this adversarial setting that comes with provable guarantees on accuracy and fairness.
    Our framework works with multiple and non-binary protected attributes, is designed for the large class of linear-fractional fairness metrics, and can also handle perturbations besides protected attributes.
    We  prove  near-tightness of our framework's guarantees for natural hypothesis classes: no algorithm can have significantly better accuracy and any algorithm with better fairness must have lower accuracy.
    Empirically, we evaluate the classifiers produced by our framework for statistical rate on real-world and synthetic datasets for a family of adversaries.
  \end{abstract}

  \newpage
  \setcounter{tocdepth}{2}
  \tableofcontents
  \addtocontents{toc}{\protect\setcounter{tocdepth}{2}}

  \newpage

  \section{Introduction}\label{sec:intro}

  It is increasingly common to deploy classifiers to assist in decision-making in applications such as criminal recidivism~\cite{northpointe2012compas}, credit lending~\cite{dedman1988color}, and predictive policing~\cite{hvistendahl2016can}.
  Hence, it is imperative to ensure that these classifiers are fair with respect to protected attributes such as gender and race.
  Consequently, there has been extensive work on approaches for fair classification~\cite{hardt2016equality, fish2016confidence, goh2016satisfying, chouldechova2017fair, ZafarVGG17, zafar17, menon2018the, DworkIKL18, goel2018non, AgarwalBD0W18, celis2019classification}.
  At a high level, a classifier $f$ is said to be ``fair'' with respect to a protected attribute $Z$ if it has a similar ``performance'' with respect to a given metric on different protected groups defined by $Z$.
  Given a fairness metric and a hypothesis class $\cF$, fair classification frameworks consider the problem of finding a classifier $f^\star\in \cF$ that   maximizes accuracy constrained to being fair with respect to the given fairness metric (and $Z$)~\cite{barocas-hardt-narayanan}.
  To specify  fairness constraints, these approaches need  protected attributes of  training data to be known.
  However,  protected attributes can be erroneous
  for various reasons; there could be uncertainties during  data collection or data cleaning process~\cite{council2004eliminating,saundersAccuracyOfRecordedEthnicInfo}, or the attributes could be strategically misreported~\cite{luh2019not}.
  Further, protected attributes may be missing entirely, as is often the case for racial and ethnic information in healthcare~\cite{council2004eliminating}
  or when data is scraped from the internet as with many image datasets~\cite{deng2009imagenet,lfw_data,fddbTech}.
  In these cases, protected attributes can be ``imputed''
  ~\cite{coldman1988classification,KallusMZ20,ChenKMSU19}, but this
  can also introduce errors~\cite{BuolamwiniG18}; further, imputation by machine-learning-based methods is known to be fragile to imperceptible changes in the inputs~\cite{goodfellow2014explaining} and to have correlated errors across samples~\cite{muthukumar2018understanding}.
  Perturbations in protected attributes, regardless of origin, have been shown to have adverse effects on fair classifiers, affecting their performance on both accuracy and fairness metrics; see e.g., \cite{ChenKMSU19,bagdasaryan2019differential,SolansB020}.

  Towards addressing this problem, several recent works have developed  fair classification algorithms for various models of errors in the protected attributes.
  \cite{LamyZ19} consider an extension of the ``mutually contaminated learning model''~\cite{scott2013classification} where, instead of observing samples from the ``true'' joint distribution,  distributions of observed group-conditional distributions are stochastic mixtures of their true counterparts.
  \cite{awasthi2020equalized} consider a binary protected attribute and  Bernoulli  perturbations that are independent of the labels (and of each other).
  \cite{celis2020fairclassification} consider the setting where each sample's protected attribute is independently flipped to a different value with a known probability.
  \cite{wang2020robust} considers two approaches to deal with perturbations.
  In their ``soft-weights'' approach, %
  they assume perturbations follow a fixed distribution and one has access to an auxiliary data containing independent draws of both the true and perturbed protected attributes.
  In their distributionally robust approach, %
  for each protected group, its feature and label distributions in the true data and the perturbed data are a known total variation distance away from each other.
  Finally, in an independent work, \cite{konstantinov2021fairness} study fair classification under  the Malicious noise model~\cite{valiant1984theory, KearnsL93} in which a fraction of the training samples are chosen uniformly at random, and can then be perturbed arbitrarily.

  \ifconf{\bf Our perturbation model.}\else\paragraph{Our perturbation model.}\fi

  We extend this line of work by studying fair classification under the following worst-case adversarial perturbation model:
  Given an $\eta>0$, after the training samples are independently drawn from a true distribution $\cD$, the adversary with unbounded computation power sees all the samples and can use this information to choose any $\eta$-fraction  of the samples and perturb their protected attributes arbitrarily.
  This model is a straightforward adaptation of the perturbation model of \cite{Hamming1950error}  to the fair classification setting and we refer to it as the $\eta$-\Ham{} model.
  Unlike  perturbation models studied before, this model can capture settings where the perturbations are strategic or arbitrarily correlated as can arise in the data collection stage or during imputation of the protected attributes, %
  and in which the errors cannot be ``estimated'' using auxiliary data.
  In fact, under this perturbation model, the  classifiers outputted by prior works can violate the fairness constraints by a large amount {or have an accuracy that is significantly lower than the accuracy of $f^\star$; see \ifexpand\cref{sec:empirics:synthetic,sec:prior_work_comparison}\else\cref{sec:empirics,sec:prior_work_comparison}\fi.}
  Taking these perturbed samples, a fairness metric $\Omega$, and a desired fairness threshold $\tau$ as input, the goal is to learn a classifier $f$ with the maximum accuracy with respect to the true distribution $\cD$ subject to having a fairness value, $\Omega_\cD(f)$, of at least $\tau$ with respect to the true distribution $\cD$.

  \ifconf{\bf Our contributions.}\else\paragraph{Our contributions.}\fi

  We present an optimization framework (\cref{def:errtolerantprogram}) that outputs fair classifiers for the $\eta$-\Ham{} model and comes with provable guarantees on accuracy and fairness (\cref{thm:main_result}).
  Our framework works for multiple and non-binary protected attributes, and the large class of linear-fractional fairness metrics (that capture most fairness metrics studied in the literature); see \cref{def:performance_metrics} and \cite{celis2019classification}.
  The framework provably outputs a classifier whose accuracy is within $2\eta$ of the accuracy of $f^\star$ and which violates the fairness constraint by at most $O(\nfrac\eta\lambda)$  additively (\cref{thm:main_result}), under the mild assumption that the {``performance'' of $f^\star$ on each protected group} is larger than a known constant $\lambda>0$ (\cref{asmp:1}).
  \cref{asmp:1} is drawn from the work of  \cite{celis2020fairclassification} for fair classification with stochastic perturbations.
  While it is not clear if the assumption is necessary in their model,
  we show that \cref{asmp:1} is necessary for fair classification in the  $\eta$-\Ham{} model:
  If $\lambda$ is not bounded away from $0$, then no algorithm can give a {non-trivial guarantee} on {\em both} accuracy and fairness value of the output classifier (\cref{thm:no_stable_classifier}).
  Moreover, we prove the near-tightness of our framework's guarantee under \cref{asmp:1}:
  No algorithm can guarantee to output a classifier with accuracy closer than $\eta$ to that of $f^\star$ and any algorithm that  violates the fairness constraint by less than $\nfrac{\eta}{(20\lambda)}$ additively has an accuracy at most $\nfrac{19}{20}$ (\cref{thm:imposs_under_assumption_sr_2,thm:imposs_under_assumption_sr}).
  Finally, we also extend our framework's guarantees to the Nasty Sample Noise model (\cref{sec:main_result:for_stronger_model}).
  The Nasty Sample Noise model is a generalization of the $\eta$-Hamming model, which was studied by	\cite{bshouty2002pac} in the context of PAC learning (without any fairness considerations), where the adversary can choose any $\eta$-fraction of the samples, and can arbitrarily perturb both their labels and features.

  We implement our framework for logistic loss function with linear classifiers and evaluate its performance on COMPAS~\cite{Angwin2016a}, Adult~\cite{adult}, and a synthetic dataset (\cref{sec:empirics}).
  We generate perturbations of these datasets  admissible in the $\eta$-\Ham{} model
  and compare the performance of our approach to several  baselines~\cite{LamyZ19,awasthi2020equalized,wang2020robust,celis2020fairclassification,konstantinov2021fairness}
  with statistical rate and false-positive rate as fairness metrics.\footnote{Let $q_\ell(f, {\rm SR})$ (respectively $q_\ell(f, {\rm FPR})$) be the fraction of positive predictions (respectively false-positive predictions) by $f$ in the $\ell$-th protected group.
  $f$'s statistical rate (respectively false-positive rate) is the ratio of the minimum value to the maximum value of $q_\ell(f, {\rm SR})$ (respectively $q_\ell(f, {\rm FPR})$) over all protected groups.
  }
  On the synthetic dataset, we compare against a method developed for fair classification under stochastic  perturbations~\cite{celis2020fairclassification} and demonstrate the comparative strength of the $\eta$-\Ham{} model;
  our results show that \cite{celis2020fairclassification}'s framework achieves a significantly lower accuracy than our framework for the same  statistical rate.
  Empirical results on COMPAS and Adult show that the classifier output by our framework can attain better statistical rate and false-positive rate than the accuracy maximizing classifier on the true distribution, with a small loss in accuracy.
  Further, our framework has a similar (or better) fairness-accuracy trade-off compared to all baselines we consider in a variety of settings, and \mbox{is not dominated by any other approach (\cref{fig:adversarial_noise_main_body,fig:fpr_main_results,fig:adult:adversarial_noise_appendix}).}

  \ifconf{\bf Techniques.}\else\paragraph{Techniques.}\fi
  The starting point of our optimization framework (\cref{def:errtolerantprogram}) is the ``standard'' optimization program for fair classification in the {\em absence} of any perturbations: Given a fairness metric $\Omega$ and a desired fairness threshold $\tau$ as input, find $f^\star\in \cF$ that maximizes the accuracy on the {\em given data} $\hS$ constrained to a fairness value at least $\tau$ on the given data.
  However, when $\hS$ is given to us by an $\eta$-\Ham{} adversary, this standard program, which imposes the fairness constraints with respect to the perturbed data $\hS$, may output a classifier with an accuracy/fairness-value worse than that of $f^\star$ when measured with respect to $\cD$.
  But, observe that the difference in accuracies of a classifier when measured with respect to the given data $\hS$ and data sampled from $\cD$ is at most $\eta$.
  Thus, if $f^\star\in \cF$ is feasible for the standard optimization program, this observation (used twice) implies that the accuracy of the output classifier measured with respect to $\cD$ is within $2\eta$ of the accuracy of $f^\star$ measured with respect $\cD$  (\cref{eq:2eta_proof}).
  However, without any modifications, the classifier output by the standard optimization program could still have a fairness value much lower than $\tau$ with respect to $\cD$ (see \cref{example:ham_adv_can_reduce_fairness}).
  To bypass this, we introduce the notion of $s$-stability that  allows us to lower bound the fairness value of a classifier with respect to $\cD$ given its fairness value on $\hS$.
  Roughly, $f\in \cF$ is said to be $s$-stable with respect to a fairness metric if for any $\hS$ that is generated by an $\eta$-\Ham{} adversary, the ratio of fairness value of $f$ with respect to $\cD$ and with respect to $\hS$ is between $s$ and $s^{-1}$ (see  \cref{def:stable_class}).
  It follows that any $s$-stable classifier that has fairness value $\tau^\prime>0$ with respect to  $\hS$, has fairness value at least $s\cdot \tau^\prime$ with respect to  $\cD$.
  Hence, an optimization program that ensures that all feasible classifiers  are $s$-stable (for a suitable choice of $s$) and have fairness value at least $\tau^\prime>0$ with respect to $\hS$, comes with a guarantee that any feasible classifier has a fairness value at least $s\cdot \tau^\prime$ (with respect to $\cD$).
  If such an optimization program could further ensure that $f^\star$ is feasible for it, then by arguments presented above, the classifier output by this optimization program would satisfy required guarantees on both fairness and accuracy (\cref{lem:opt_solution_satisfies_guarantees}).
  The issue is that, to directly enforce $s$-stability, one needs to compute the fairness values of classifiers with respect to $\cD$, but this is not possible in the absence of samples from $\cD$.
  We overcome this by present a ``proxy'' constraint on the classifier (\cref{eq:lower_bound}) that involves  only $\hS$ and ensures that any classifier that satisfies it is $s$-stable.
  Moreover, $f^\star$ satisfies this constraint under \cref{asmp:1}.
  Overall, modifying \prog{prog:target_fair} to include this constraint (\cref{eq:lower_bound}) with a suitable value of $s$, and setting an appropriate fairness threshold $\tau$ so that $f^\star$ remains feasible, leads us to our framework.

  \section{Related Work}
  In this section, we situate this paper in relation to lines of work which also consider fair classification with perturbed protected attributes; additional related work (e.g., on fair classification in the absence of protected attributes) are presented in \cref{sec:other_related_work}.

  \cite{LamyZ19} give a framework which comes with provable guarantees on the accuracy and fairness value of output classifiers for a binary protected attribute and either statistical rate or equalized-odds fairness metrics.
  \cite{awasthi2020equalized} identify conditions on the distribution of perturbations under which the post-processing algorithm of \cite{hardt2016equality} improves the fairness value of the accuracy-maximizing classifier with respect to equalized-odds on the true distribution with a binary protected attribute.
  \cite{wang2020robust} consider a non-binary protected attribute.
  In their ``soft-weights'' approach, they give provable guarantees on the accuracy (with respect to  $f^\star$) and fairness value of the output classifier  {\em in expectation} and
  in their distributionally robust approach, they give provable guarantees on the fairness value of the output classifiers.\footnote{\cref{sec:comparison_to_wang} gives an example where \cite{wang2020robust}'s distributionally robust approach outputs a classifier whose accuracy is arbitrarily close to $\nfrac12$.}
  \cite{celis2020fairclassification} give provable guarantees on the accuracy and fairness value of output classifiers for multiple non-binary protected attributes and the class of linear-fractional metrics.
  All of the aforementioned works~\cite{LamyZ19,awasthi2020equalized,wang2020robust,celis2020fairclassification} consider stochastic perturbation models, which are weaker than the model considered in this paper.
  Further, compared  to \cite{LamyZ19,awasthi2020equalized}, our approach (and that of \cite{celis2020fairclassification}) can handle multiple categorical protected attributes and multiple linear-fractional metrics (which include statistical rate and can ensure equalized-odds constraints).
  Compared to \cite{awasthi2020equalized,wang2020robust}, our work (and those of \cite{LamyZ19,celis2020fairclassification}) give provable guarantees on the accuracy (with respect to  $f^\star$) and fairness value of output classifiers {\em with high probability}.
  In another related work, \cite{konstantinov2021fairness} give an algorithm for a binary protected attribute which, under the realizable assumption (i.e., assuming there exists a classifier with perfect accuracy), outputs a classifier with  guarantees on accuracy and fairness value with respect to the true-positive rate fairness metric.
  They study the Malicious noise model, which can modify a uniformly randomly selected subset of samples arbitrarily; this is weaker than the Nasty Sample Noise  model~\cite{bshouty2002pac,Auer16}, and hence, than the model considered in this paper.
  Further, our framework works without the realizable assumption (i.e., in the agnostic setting), can handle multiple and non-binary protected attributes, and can ensure fairness with respect to multiple linear-fractional metrics (which include true-positive rate).

  Another line of work has studied PAC learning in the presence of adversarial (and stochastic) perturbations in the data, without considerations of fairness~\cite{KearnsL93,AngluinL87,bshouty2002pac,cesa1999sample, auer1998line};
  see also \cite{Auer16}.
  In particular, \cite{bshouty2002pac} study PAC learning (without fairness constraints) under the Nasty Sample Noise model.
  They use the empirical risk minimization framework (see, e.g., \cite{shalev2014understanding}) run on the perturbed samples to output a classifier.
  Our framework \errtolerant{} finds empirical risk minimizing classifiers that satisfy fairness constraints on the perturbed data, and that are also ``stable'' for the given fairness metric.
  While both frameworks show that the accuracy of the respective output classifiers is within $2\eta$ of the respective optimal classifiers when the data is unperturbed, the optimal classifiers can be quite different.
  For instance, while \cite{bshouty2002pac}'s framework is guaranteed to output a classifier with high accuracy, it can perform poorly on fairness metrics; see \cref{sec:empirics,sec:bshouty_poor}.

  \section{Model}\label{sec:model}
  Let the data domain be $D\coloneqq \cX\times\zo\times [p],$ where  $\cX$ is the set of non-protected features, $\zo$ is the set of binary labels, and $[p]$ is the set of $p$ protected attributes.
  Let $\cD$ be a distribution over $D$.
  Let $\cF\subseteq \zo^{\cX\times [p]}$ be a hypothesis class of binary classifiers.
  For  $f\in \cF$, let $${\rm Err}_\cD(f)\coloneqq \Pr\nolimits_{(X,Y,Z)\sim \cD}[f(X,Z)\neq Y]$$ denote $f$'s predictive error on draws from $\cD$.
  In the vanilla classification problem, the learner $\cL$'s goal is to find a classifier with  minimum error, i.e., to solve
  $$\min\nolimits_{f\in \cF} {\rm Err}_\cD(f).$$
  In the fair classification problem, the learner is restricted to pick classifiers that have a ``similar performance'' conditioned on $Z=\ell$ for all $\ell\in [p]$.
  We consider the following  class of metrics.
  \begin{definition}[\bf Linear/linear-fractional metrics \cite{celis2019classification}]\label{def:performance_metrics}
    Given  $f\in \cF$ and two events $\cE(f)$ and $\cE^\prime(f)$, that can depend on $f$,
    define the performance of $f$ on $Z=\ell$ ($\ell\in [p]$) as
    $q_{\ell}(f) \coloneqq \Pr\nolimits_\cD[\cE(f)\mid \cE^\prime(f), Z=\ell].$
    If $\cE^\prime$ depends on $f$, then $q_{\ell}(f)$ is said to be linear-fractional, otherwise linear.
  \end{definition}
  \noindent \cref{def:performance_metrics} captures most of the performance metrics considered in the literature.
  For instance, for $\cE\coloneqq (f=1)$ and $\cE^\prime\coloneqq \emptyset$, we get statistical rate (a linear metric).\footnote{We overload the notation $f$ to denote both the classifier as well as its prediction, and the terms, statistical rate and false-positive rate, to refer to both the linear/linear-fractional metric $q$ and the resulting fairness metric $\Omega$.}
  For $\cE\coloneqq (f=1)$ and $\cE^\prime\coloneqq (Y=0)$, we get false-positive rate (also a linear metric).
  For $\cE\coloneqq (Y=0)$ and $\cE^\prime\coloneqq (f=1)$, we get  false-discovery rate (a linear-fractional metric).
  Given a performance metric $q$, the corresponding fairness metric is defined as
  \begin{align*}
    \Omega_\cD(f)\coloneqq \frac{\min_{\ell\in [p]} q_\ell(f)}{\max_{\ell\in [p]} q_\ell(f)}.\yesnum
  \end{align*}
  When $\cD$ is the empirical distribution over samples $S$, we use $\Omega(f,S)$ to denote $\Omega_\cD(f)$.
  The goal of fair classification, given a fairness metric $\Omega$ and a threshold $\tau\in(0,1]$, is to (approximately) solve:
  \begin{align*}
    \min\nolimits_{f\in \cF} {\rm Err}_\cD(f) \quad \st,\quad \Omega_{\cD}(f)\geq \tau.\yesnum\label{prog:target_fair}
  \end{align*}
  \noindent If samples from $\cD$ are available, then one could try to solve this program.
  However, as discussed in \cref{sec:intro}, we do not have access to the {\em true} protected attribute $Z$, but instead only see a perturbed version, $\hZ\in [p]$, generated by the following adversary.

  \ifexpand\paragraph{$\eta$-\Ham{} model.}\else{\bf $\eta$-\Ham{} model.}\fi
  Given an $\eta\in [0,1]$, let $\cA(\eta)$ denote the set of all adversaries in the $\eta$-\Ham{} model.
  Any adversary $A\in \cA(\eta)$ is a randomized algorithm with {\em unbounded} computation resources that knows the true distribution $\cD$ and the algorithm of the learner $\cL$.
  In this model, the learner $\cL$ queries $A$ for $N\in \N$ samples from $\cD$ {\em exactly once}.
  On receiving the request, $A$ draws $N$ independent samples $S\coloneqq {\sinbrace{(x_i,y_i,z_i)}_{i\in [N]}}$ from $\cD$,
  then $A$ uses its knowledge of $\cD$ and $\cL$ to choose an arbitrary $\eta\cdot N$ samples ($\eta\in [0,1]$) and perturb their protected attribute arbitrarily to generate $\hS\coloneqq {\sinbrace{(x_i,y_i,\hz_i)}_{i\in [N]}}$.
  Finally, $A$ gives these perturbed samples $\hS$ to $\cL$.\noindent

  \ifexpand\paragraph{Learning model.}\else{\bf Learning model.}\fi
  Given $\hS$ and the $\eta$, the learner $\cL$ would like to (approximately) solve Program~\eqref{prog:target_fair}.
  \begin{definition}[\textbf{($\eps,\nu$)-learning}]\label{def:learning_model}
    Given bounds on error $\eps\in (0,1)$ and constraint violation $\nu\in (0,1)$,
    a learner $\cL$ is said to $(\eps, \nu)$-learn a hypothesis class $\cF\subseteq \zo^{\cX\times [p]}$ with perturbation rate $\eta\in [0,1]$
    and confidence $\delta\in (0,1)$
    if for all
    \begin{itemize}[itemsep=1pt]
      \item
      distributions $\cD$ over $\cX\times \zo \times [p]$ and
      \item
      adversaries $A\in \cA(\eta)$,
    \end{itemize}
    there exists a threshold $N_0(\eps,\nu,\delta,\eta)\in \N$, such that
    with probability at least $1-\delta$ over the draw of $N\geq N_0(\eps,\nu,\delta,\eta)$ iid samples $S\sim \cD$,
    given $\eta$ and the perturbed samples $\hS\coloneqq A(S)$,
    $\cL$ outputs $f\in \cF$ that satisfies
    \begin{align*}
      {\rm Err}_\cD(f) - {\rm Err}_\cD(f^\star) \leq \eps\quad \text{and}\quad \Omega_{\cD}(f) \geq \tau -\nu,
    \end{align*}
    where $f^\star$ is the optimal solution of Program~\eqref{prog:target_fair} (i.e., $f^\star\coloneqq \argmin\nolimits_{f\in \cF} {\rm Err}_\cD(f),\ \st,\ \Omega_{\cD}(f) \geq \tau$).
  \end{definition}
  \noindent Given a finite number of perturbed samples,
  \cref{def:learning_model} requires the learner to output a classifier that violates the fairness constraints additively by at most $\nu$ and that has a predictive error at most $\eps$ smaller than that of $f^\star$, with probability at least $1-\delta$.
  Like PAC learning~\cite{valiant1984theory}, for a given hypothesis class $\cF$, \cref{def:learning_model} requires the learner to succeed on all distributions $\cD$.

  \begin{problem}[\bf Fair classification with adversarial perturbations]\label{problem1}
    Given
    a hypothesis class $\cF\subseteq \zo^{\cX\times [p]}$,
    a fairness metric $\Omega$,
    a threshold $\tau\in [0,1]$,
    a perturbation rate $\eta\in [0,1]$, and
    perturbed samples $\wh{S}$,
    the goal is to ($\eps,\nu$)-learn $\cF$ for small $\eps,\nu\in (0,1)$.
  \end{problem}

  \section{Theoretical Results}\label{sec:theory}
  \ifexpand\subsection{An Optimization Framework with Provable Guarantees}\fi

  In this section, we present our results on learning fair classifiers under the $\eta$-\Ham{} model.
  Our optimization framework (\errtolerant{}) is a careful modification of Program~\eqref{prog:target_fair}.
  The main difficulty is that, unlike Program~\eqref{prog:target_fair}, it only has access to the perturbed samples $\hS$,
  and the ratio of a classifier's fairness with respect to the true distribution $\cD$ and with respect to $\hS$ can be arbitrarily small (see \cref{example:ham_adv_can_reduce_fairness} in \cref{sec:ham_captures}).
  To overcome this, our framework ensures that all feasible classifiers are ``stable'' (\cref{def:stable_class}).
  Then, as mentioned in \cref{sec:intro}, imposing the fairness constraint on $\hS$ guarantees (approximate) fairness on the true distribution $\cD$.
  The accuracy guarantee follows by ensuring that the optimal solution of Program~\eqref{prog:target_fair},
  $f^\star\in \cF$,
  is feasible for our framework.
  To ensure this, we require \cref{asmp:1} that also appeared in \cite{celis2020fairclassification}.
  \begin{assumption}\label{asmp:1}
    {There is a known constant $\lambda>0$ such that
    $\min_{\ell\in [p]}\Pr\nolimits_{\cD}[\cE(f^\star), \cE^\prime(f^\star), Z=\ell]\geq \lambda.$}
  \end{assumption}
  \noindent It can be shown that this assumption implies that  $\lambda$ is also a lower bound on the performances $q_1(f^\star),\dots, q_p(f^\star)$ that depend on $\cE$ and $\cE^\prime$.
  We expect $\lambda$ to be a non-vanishing positive constant in  applications.
  For example, if $q$ is statistical rate, the minority protected group makes at least 20\% of the population (i.e., $\min_{\ell\in [p]}\Pr_\cD[Z=\ell]\geq 0.2$), and for all $\ell\in [p]$, $\Pr[f^\star=1\mid Z=\ell]\geq \nfrac12$, then $\lambda\geq 0.1$.
  In practice, $\lambda$ is not known exactly, but it can be set based on the context (e.g., see \cref{sec:empirics} and \cite{celis2020fairclassification}).
  We show that \cref{asmp:1} is necessary for the $\eta$-\Ham{} model (see \cref{thm:no_stable_classifier}).
  \begin{definition}[\bf Error-tolerant program]\label{def:errtolerantprogram}
    Given a fairness metric $\Omega$ and corresponding events $\cE$ and $\cE^\prime$ (as in \cref{def:performance_metrics}),
    a perturbation rate $\eta\in [0,1]$, and constants $\lambda,\Delta\in (0,1]$, %
    we define the error-tolerant program for perturbed samples $\hS$, whose empirical distribution is $\smash{\hD}$, as
    \begin{align*}
      \min\nolimits_{f\in \cF}&\quad {\rm Err}_{\hD}(f),\quad\tagnum{ErrTolerant}\customlabel{prog:errtolerant}{(ErrTolerant)}\\
      \st,&\quad
      \Omega_{}(f, \hS) \geq \tau \cdot \inparen{\frac{1- (\eta+\Delta)/\lambda}  {  1 + (\eta+\Delta)/\lambda}}^2,
      \yesnum\label{eq:fairness_constraint_in_err_tol}\\
      &\quad
      \text{$\forall\ \ell\in [p]$,\ } \Pr\nolimits_{\hD}\insquare{\cE(f),\cE^\prime(f), \hZ=\ell} \geq \lambda - \eta - \Delta.\yesnum\label{eq:lower_bound}
    \end{align*}
  \end{definition}
  \noindent $\Delta$ acts as a relaxation parameter in \errtolerant{}, which can be fixed in terms of the other parameters; see \cref{thm:main_result}.
  Equation~\eqref{eq:fairness_constraint_in_err_tol} ensures all feasible classifiers satisfy fairness constraints with respect to the perturbed samples $\hS$.
  Equation~\eqref{eq:lower_bound} ensures that all feasible classifiers are ($1-O(\nfrac{\eta}{\lambda})$)-stable (see \cref{def:stable_class}).
  As mentioned in \cref{sec:intro}, this suffices to ensure that all feasible classifiers are fair with respect to $S$.
  Finally, to ensure the accuracy guarantee the thresholds in the right-hand side of Equations~\eqref{eq:fairness_constraint_in_err_tol} and \eqref{eq:lower_bound} are carefully tuned to ensure that $f^\star$ is feasible for \errtolerant{}; see \cref{lem:opt_solution_satisfies_guarantees}.
  %
  We refer the reader to the proof overview of \cref{thm:main_result} at the end of this section for further discussion of \errtolerant{}.

  Before presenting our result we require the definition of the Vapnik–Chervonenkis (VC) dimension.
  \begin{definition}
    Given a finite set $A$, define the collection of subsets $\cF_A\coloneqq \{\{a\in A\mid f(a)=1\}\mid f\in \cF\}$.
    We say that $\cF$ shatters a set $B$ if $|\cF_{B}|=2^{|B|}$.
    The VC dimension of $\cF$, ${\rm VC}(\cF)\in \N$,  {is the largest integer such that there exists a set $C$ of size ${\rm VC}(\cF)$ that is shattered by $\cF$.}
  \end{definition}

  \noindent  Our first result bounds the accuracy and fairness of an optimal solution $f_{\rm ET}$ of \errtolerant{} for any hypothesis class $\cF$ with a finite VC dimension using $O({\rm VC}(\cF))$ samples.
  \begin{theorem}[\textbf{Main result}]\label{thm:main_result}
    Suppose \cref{asmp:1} holds with constant $\lambda>0$ and $\cF$ has VC dimension $d\in\N$.
    Then, for all perturbation rates $\eta\in (0,\nfrac\lambda{2})$,
    fairness thresholds $\tau\in (0,1]$,
    bounds on error $\eps> 2\eta$ and constraint violation $\nu >  \nfrac{8\eta\tau}{(\lambda-2\eta)}$,
    and confidence parameters $\delta\in (0,1)$
    with probability at least $1-\delta$,
    the optimal solution $f_{\rm ET}\in \cF$ of \errtolerant{} with parameters $\eta$, $\lambda$, and $\Delta\coloneqq O\inparen{\min\inbrace{\eps-2\eta, \nu-\nfrac{8\eta\tau}{(\lambda-2\eta)},\lambda-2\eta}}$, and
    $N=\poly(d, \nfrac{1}{\Delta},  \log(\nfrac{p}\delta))$
    perturbed samples from the $\eta$-\Ham{} model satisfies
    \begin{align*}
      {\rm Err}_\cD(f_{\rm ET}) - {\rm Err}_\cD(f^\star) \leq \eps\quad \text{and}\quad \Omega_{\cD}(f_{\rm ET}) \geq \tau -\nu.
    \end{align*}
  \end{theorem}
  \noindent Thus, \cref{thm:main_result} shows that any procedure that outputs $f_{\rm ET}$, given with a sufficiently large number of perturbed samples, ($\eps,\nu$)-learns $\cF$ for any $\eps>2\eta$ and $\nu=O(\nfrac{(\eta\cdot \tau)}{\lambda})$.
  \ifexpand\par\else\fi
  \cref{thm:main_result} can be  extended to provably satisfy multiple linear-fractional metrics (at the same time) and work for multiple non-binary protected attributes; see \cref{thm:multiple_fairness_metrics} in \cref{sec:multiple_fairness_metrics}.
  Moreover,	    \cref{thm:main_result} also holds for the Nasty Sample Noise model.
  The proof of this result is implicit in the proof of \cref{thm:main_result};
  we present the details in \cref{sec:main_result:for_stronger_model}.
  Finally, \errtolerant{} only requires an estimate of one parameter, $\lambda$. (Since  $\eta$ is known, $\tau$ is fixed by the user, and $\Delta$ can be set in terms of the other parameters.)
  If for each $\ell\in [p]$, we also have estimates of
  $\lambda_\ell\coloneqq\Pr\nolimits_{\cD}\insquare{\cE(f^\star), \cE^\prime(f^\star), Z=\ell}$ and $\gamma_\ell \coloneqq \Pr\nolimits_\cD[\cE^\prime(f^\star), Z=\ell]$,
  then we can use this information to ``tighten'' \errtolerant{} to the following program:
  \begin{align*}
    \min\nolimits_{f\in \cF}\quad  &{\rm Err}_{\hD}(f),\quad
    \tagnum{ErrTolerant+}\customlabel{prog:errtolerant_extnd}{(ErrTolerant+)}\\
    \st,\quad
    &\Omega_{}(f, \hS) \geq \tau\cdot s,\\
    &\text{$\forall\ \ell\in [p]$,\ } \Pr\nolimits_{\hD}\insquare{\cE(f),\cE^\prime(f), \hZ=\ell} \geq \lambda_\ell - \eta - \Delta.
  \end{align*}
  where the scaling parameter $s\in [0,1]$ is the solution of the following optimization program
  \begin{align*}
    \min_{\eta_1,\eta_2,\dots,\eta_p\geq 0}\ \min_{\ell,k\in [p]}\frac{1-\sfrac{\eta_\ell}{\lambda_\ell}}{1+\sfrac{(\eta_k-\eta_\ell)}{\gamma_\ell}}
    \cdot
    \frac{1 +\sfrac{(\eta_\ell-\eta_k)}{\gamma_k}}{1+\sfrac{\eta_\ell}{\lambda_k}},
    \quad \st, \quad \sum_{\ell\in [p]}\eta_\ell \leq \eta+\Delta.\yesnum\label{eq:...}
  \end{align*}
  If the classifiers in $\cF$ do not use the protected attributes for prediction, then we can show that Program~\ref{prog:errtolerant_extnd}  has a fairness guarantee of $(1-s)+\nfrac{4\eta\tau}{(\lambda-2\eta)}$ (which is always smaller than $\nfrac{8\eta\tau}{(\lambda-2\eta)}$) and an accuracy guarantee of $2\eta.$
  We prove this result in \cref{sec:guarantees_on_errTol_plus}.
  Thus, in applications where one can estimate $\lambda_1,\dots,\lambda_p$ and $\gamma_1,\dots,\gamma_p$,  Program~\ref{prog:errtolerant_extnd} offers better fairness guarantee than \errtolerant{} (up to constants).
  The proof of \cref{thm:main_result} appears in \cref{sec:proofof:thm:main_result}.

  As for computing $f_{\rm ET}$, note that in general, \errtolerant{} is a nonconvex optimization problem.
  In our simulations, we use the standard solver SLSQP in SciPy~\cite{scipy} to heuristically find $f_{\rm ET}$; see \cref{sec:more_empirical_results:implementation_details}.
  Theoretically, for any arbitrarily small $\alpha>0$, the techniques from \cite{celis2019classification} can be used to find an $f\in \cF$ that has the optimal objective value for \errtolerant{} and that additively violates its fairness {constraint \eqref{eq:fairness_constraint_in_err_tol}} by at most {$\alpha$ by solving a set of $O(\nfrac1{(\lambda\alpha}))$ convex programs; details appear in \cref{sec:efficient_algorithms_for_errtol}.}

  \ifexpand\subsection{Impossibility Results}\else{\bf Impossibility results.} \fi
  We now present  results complementing the guarantees of \cref{thm:main_result}.
  \begin{theorem}[\textbf{No algorithm can guarantee high accuracy {\em and} fairness without \cref{asmp:1}}]\label{thm:no_stable_classifier}
    For all perturbation rates $\eta\in (0,1]$,
    confidence parameters $\delta\in [0,\nfrac12)$,
    thresholds $\tau\in (\nfrac12,1)$, and
    bounds
    $$\eps<\frac12 \quad\text{and}\quad \nu<\tau-\frac12$$
    if the fairness metric is statistical rate, 
    then it is impossible to ($\eps,\nu$)-learn any hypothesis class $\cF\subseteq \zo^{\cX\times [p]}$ that shatters a set of 6 points of the form $\sinbrace{x_A,x_B,x_C}\times[2]\subseteq \cX\times [p]$ for some distinct $x_A,x_B,x_C\in \cX$.
  \end{theorem}

  \noindent Suppose that $\tau= 0.8$, say to encode the 80\%
  Then, \cref{thm:no_stable_classifier} shows that for any $\eta>0$, any $\cF$ satisfying the condition in \cref{thm:no_stable_classifier} is not ($\eps,\nu$)-learnable for any $\eps<\nfrac12$ and $\nu<\tau-\nfrac12=\nfrac3{10}$.
  Intuitively, the condition on $\cF$ avoids ``simple'' hypothesis classes.
  It is similar to the conditions considered by works on PAC learning with adversarial perturbations~\cite{bshouty2002pac,KearnsL93}, and holds for common hypothesis classes such as decision-trees and SVMs (\cref{rem:our_assumption_on_CF} in \cref{sec:ham_captures}).
  Thus, even if $\eta$ is vanishingly small, without additional assumptions, any $\cF$ satisfying mild assumptions is not ($\eps,\nu$)-learnable
  for any $\eps<\nfrac12$ and $\nu<\nfrac3{10}$, justifying \cref{asmp:1}.
  The proof of \cref{thm:no_stable_classifier} appears in \cref{sec:proofof:thm:no_stable_classifier}.
  \begin{theorem}[\textbf{Fairness guarantee of \cref{thm:main_result} is optimal up to a constant factor}] \label{thm:imposs_under_assumption_sr}
    For any perturbation rates $\eta\in (0,1]$,
    confidence parameter $\delta\in [0,\nfrac12)$, and
    (known) constant $\lambda\in (0,\nfrac14]$,
    if the fairness metric is statistical rate and $\tau=1$,
    \mbox{then given the promise that \cref{asmp:1} holds with constant $\lambda$,
    for any} 
    \begin{align*}
      \eps<\frac{1}4 - \frac{2\eta}{5}\quad \text{and}\quad v<\frac{\eta}{10\lambda}\cdot (1-4\lambda)-O\inparen{{\frac{\eta^2}{\lambda^2}}}
    \end{align*}
    it is impossible to ($\eps,\nu$)-learn any hypothesis class $\cF\subseteq \zo^{\cX\times [p]}$ that shatters a set of 10 points of the form $\sinbrace{x_A,x_B,x_C,x_D,x_E}\times[2]\subseteq \cX\times [p]$ for some distinct $x_A,x_B,x_C,x_D,x_E\in \cX.$
  \end{theorem}

  \noindent Suppose that $\lambda<\nfrac18$ and $\eta<\nfrac12$, then \cref{thm:imposs_under_assumption_sr} shows that for any $\eta>0$,
  any learner $\cL$ that has a fairness guarantee $\nu<\nfrac{\eta}{(20\lambda)} - O\sinparen{{\nfrac{\eta^2}{\lambda^2}}}$,
  must have a poor error bound, of at least $\eps\geq \nfrac14 - \nfrac{2\eta}5\geq \nfrac1{20}$, to ($\eps,\nu$)-learn any $\cF$ that satisfies a mild assumption.
  When $\nfrac{\eta}{\lambda}$ is small, this shows that any learner with a fairness guarantee $\nu=o(\nfrac{\eta}{\lambda})$ must have an error guarantee at least $\nfrac14 - \nfrac{2\eta}5\gg 2\eta$.
  Thus, \cref{thm:imposs_under_assumption_sr} shows that one cannot improve the fairness guarantee in \cref{thm:main_result} by more than a constant amount without deteriorating the error guarantee from $2\eta$ to $\nfrac14 - \nfrac{2\eta}5$.
  Like \cref{thm:no_stable_classifier}, the condition on $\cF$ in \cref{thm:imposs_under_assumption_sr} avoids ``simple'' hypothesis classes and holds for common hypothesis (\cref{rem:our_assumption_on_CF} in \cref{sec:ham_captures}).
  The proof of \cref{thm:imposs_under_assumption_sr} appears in \cref{sec:proofof:thm:imposs_under_assumption_sr}.

  \begin{theorem}[\textbf{Accuracy guarantee of \cref{thm:main_result} is optimal up to a factor of 2}] \label{thm:imposs_under_assumption_sr_2}
    For any perturbation rates $\eta\in (0,\nfrac12]$,
    confidence parameter $\delta\in [0,\nfrac12)$, and
    (known) constant $\lambda\in (0,\nfrac14]$,
    if the fairness metric is statistical rate and $\tau=1$,
    \mbox{then given the promise that \cref{asmp:1} holds with constant $\lambda$,
    for any}
    $$\eps<\eta \quad\text{and}\quad v<\tau$$
    it is impossible to ($\eps,\nu$)-learn any hypothesis class $\cF\subseteq \zo^{\cX\times [p]}$ that shatters a set of 6 points of the form $\sinbrace{x_A,x_B,x_C}\times[2]\subseteq \cX\times [p]$ for some distinct $x_A,x_B,x_C\in \cX.$
  \end{theorem}
  \noindent Suppose $\lambda<\nfrac14$ and $\eta<\nfrac12$, then \cref{thm:imposs_under_assumption_sr_2} proves that for any $\eps<\eta$, no algorithm can ($\eps,\nu$)-learn any hypothesis class $\cF$ satisfying mild assumptions.
  Thus, the accuracy guarantee in \cref{thm:main_result} is optimal up to a factor of 2.
  Like our other impossibility results, the condition on $\cF$ avoids ``simple'' hypothesis classes and holds for 
  common hypothesis classes (\cref{rem:our_assumption_on_CF} in \cref{sec:ham_captures}).
  \mbox{The proof of \cref{thm:imposs_under_assumption_sr_2} appears in \cref{sec:additional_imposs_result}.}
  \ifexpand\subsection{Proof Overview of \cref{thm:main_result}}\else {\bf Proof overview of \cref{thm:main_result}.} \fi
  We explain the key ideas behind \errtolerant{} and how they connect with the proof of \cref{thm:main_result}.
  Our goal is to construct error-tolerant constraints using perturbed samples $\hS$ such that the classifier $f_{\rm ET}$, that has the smallest error on $\hS$ subject to satisfying these constraints, has accuracy $2\eta$-close to that of $f^\star$ and that additively violates the fairness constraints by at most $O(\nfrac\eta\lambda).$

  \ifexpand
  \paragraph{Step 1: Lower bound on the accuracy of $f_{\rm ET}$.}
  \else
  \noindent {\em Step 1: Lower bound on the accuracy of $f_{\rm ET}$.}
  \fi
  This step relies on \cref{lem:overview1}.
  \begin{lemma}\label{lem:overview1}
    For any bounded function $g\colon \zo^2\times[p]\to [0,1]$, $\delta,\Delta\in(0,1)$,
    and adversaries $A\in \cA(\eta)$,
    given $N =\poly(\nfrac{1}{\Delta},{\rm VC}(\cF),\log{1/\delta})$
    true samples $S\sim \cD$
    and corresponding perturbed samples $A(S) \coloneqq \{(x_i,y_i,\hz_i)\}_{i\in [N]}$, with probability at least $1-\delta$, it holds that
    $$\forall\ f\in \cF,\quad  \abs{\frac1N  \sum\nolimits_{i\in[N]} g(f(x_i,\hz_i), y_i, \hz_i) - \Ex\nolimits_{(X,Y,Z)\sim\cD}\insquare{g(f(X,Z), Y, Z)}} \leq \Delta+\eta.$$
  \end{lemma}
  \noindent The proof of \cref{lem:overview1} follows from generalization bounds for bounded functions (e.g., see~\cite{shalev2014understanding}) and because the  $\eta$-\Ham{} model perturbs at most $\eta\cdot N$ samples.
  Let $g$ be the 0-1 loss (i.e., $g(\wt{y},y,z)\coloneqq \mathds{I}\insquare{\wt{y}\neq y}$),
  then for all $f\in \cF$, \cref{lem:overview1} shows that the error of $f$ on samples drawn from $\cD$ and samples in $\hS$ are close:
  $$\abs{{\rm Err}_\cD(f)-{\rm Err}(f,\hS)}\leq\Delta+\eta.$$
  Thus, intuitively, minimizing ${\rm Err}(f,\hS)$ could be a good strategy to minimize ${\rm Err}_\cD(f)$.
  Then, if  $f^\star$ is feasible for \errtolerant{}, we can bound the error of $f_{\rm ET }$:
  Since $f_{\rm ET}$ is optimal for \errtolerant{}, its error on $\hS$ is at most the error of $f^\star$ on $\hS$.
  Using this and applying \cref{lem:overview1} we get that
  \begin{align*}
    {\rm Err}_\cD(f_{\rm ET})
    \ {\leq}\    {\rm Err}(f_{\rm ET}, \hS)+\eta+\Delta
    \ \leq\  {\rm Err}(f^\star, \hS)+\eta+\Delta
    \ {\leq}\   {\rm Err}_\cD(f^\star)+2(\eta+\Delta).\yesnum\label{eq:2eta_proof}
  \end{align*}
  \ifexpand
  \noindent {\bf Step 2: Lower bound on the fairness of $f_{\rm ET}$.}
  \else
  \noindent {\em Step 2: Lower bound on the fairness of $f_{\rm ET}$.}
  \fi
  One could try to bound the fairness of $f_{\rm ET}$ using the same approach as Step 1, i.e., show that for all $f\in \cF$:
  $$\abs{\Omega_\cD(f)-\Omega(f,\hS)}\leq O\inparen{\frac{\eta}{\lambda}}.$$
  Then ensuring that $f$ has a high fairness on $\hS$ implies that it also has high fairness on $S$ (up to an $O(\nfrac{\eta}{\lambda})$ factor).
  However, such a bound does not hold for any $\cF$ satisfying mild assumptions (see \cref{example:ham_adv_can_reduce_fairness}).
  The first idea is to prove a similar (in fact, stronger multiplicative) bound on a specifically chosen {\em subset} of $\cF$ (consisting of ``stable'' classifiers).
  Toward this, we define:
  \begin{definition}\label{def:stable_class}
    A classifier $f\in \cF$ is said to be $s$-stable for fairness metric $\Omega$, if for all adversaries $A\in \cA(\eta)$ and confidence parameters $\delta\in (0,1)$, given $\polylog(\nfrac1\delta)$ samples $S\sim\cD$, with probability at least $1-\delta$,
    it holds that $\frac{\Omega_\cD(f)}{\Omega(f,\wh{S})}\in \insquare{s,\frac{1}{s}}$, where $\hS\coloneqq A(S).$
  \end{definition}
  \noindent If an $s$-stable classifier $f$ has fairness $\tau$ on $\hS$, then it has a fairness at least $\tau\cdot s$ on $\cD$ with high probability.
  Thus, if we have a condition such that any feasible $f\in \cF$ satisfying this condition is $s$-stable, then any classifier satisfying this condition and the fairness constraint, $\Omega(\cdot, \hS)\geq \nfrac{\tau}{s}$, must have a fairness at least $\tau$ on $\cD$ with high probability.
  The key idea is coming up such constraints.
  \begin{lemma}\label{lem:suff_cond_for_stable}
    Any classifier $f\in \cF$ that satisfies
    $$\min_{\ell\in [p]}\Pr\nolimits_{\cD}\insquare{\cE(f), \cE^\prime(f), \hZ=\ell}\geq \lambda+\eta+\Delta,$$
    is $\sinparen{\frac{1- (\eta+\Delta)/\lambda}  {  1 + (\eta+\Delta)/\lambda}}^2$-stable for fairness metric $\Omega$ (defined by events $\cE$ and $\cE^\prime$).
  \end{lemma}

  \ifexpand
  \paragraph{Step 3: Requirements for the error-tolerant program.}
  \else
  \noindent {\em Step 3: Requirements for the error-tolerant program.}
  \fi
  Building on Steps 1 and 2, we prove: %
  \begin{lemma}\label{lem:opt_solution_satisfies_guarantees}
    If the following conditions hold then,
    $${\small {\rm Err}_\cD(f_{\rm ET}) - {\rm Err}_\cD(f^\star) \leq 2\eta} \quad \text{and}\quad {\small \Omega_{\cD}(f_{\rm ET}) \geq \tau - O\inparen{\frac{\eta}{\lambda}}},$$
    \begin{enumerate}[itemsep=\itemsepINTERNAL,leftmargin=\leftmarginINTERNAL]
      \item[(C1)] $f^\star$ is feasible for \errtolerant{},
      \item[(C2)] all $f\in \cF$ feasible for \errtolerant{} are $s$-stable for $s=1-O(\nfrac{\eta}{\lambda})$, and
      \item[(C3)] satisfy $\Omega(f,\hS)\geq \tau\cdot (1-O(\nfrac{\eta}{\lambda}))$.
    \end{enumerate}
  \end{lemma}
  \noindent Thus, it suffices to find error-tolerant constraints that satisfy conditions (C1) to (C3).
  Condition (C3) can be satisfied by adding the constraint $\Omega(\cdot,\hS)\geq \tau^\prime$, for $\tau^\prime=\tau\cdot (1-O(\nfrac{\eta}{\lambda})).$
  From \cref{lem:suff_cond_for_stable}, condition (C2) follows by using the constraint in $\min_{\ell\in [p]}\Pr\nolimits_{\cD}\insquare{\cE(f), \cE^\prime(f), \hZ=\ell}\geq \lambda^\prime$, for $\lambda^\prime \geq \Theta(\lambda).$
  It remains to pick $\tau^\prime$ and $\lambda^\prime$ such that condition (C1) also holds.
  The tension in setting $\tau^\prime$ and $\lambda^\prime$ is that if they are too large then condition (C1) does not hold, and if they are too small, then conditions (C2) and (C3) do not hold.
  In the proof we show that
  \begin{align*}
    \tau^\prime\coloneqq \tau\cdot {\inparen{\frac{1- (\eta+\Delta)/\lambda}  {  1 + (\eta+\Delta)/\lambda}}^2}\quad \text{and}\quad \lambda^\prime\coloneqq \lambda-\eta-\Delta
  \end{align*}
  suffice to satisfy conditions (C1) to (C3) (this is where we use \cref{asmp:1}).

  Overall the main technical idea is to identify the notion of $s$-stable classifiers and sufficient conditions for a classifier to be $s$-stable; combining these conditions with the fairness constraints on $\hS$, ensures that $f_{\rm ET}$ has high fairness on $S$, and carefully tuning the thresholds so that $f^\star$ is likely to be feasible for \errtolerant{} ensures that $f_{\rm ET}$ has an accuracy close to $f^\star$.

  \ifexpand
  \subsection{Proof Overviews of \cref{thm:no_stable_classifier,thm:imposs_under_assumption_sr,thm:imposs_under_assumption_sr_2}}
  \else
  {\bf Proof overviews of \cref{thm:no_stable_classifier,thm:imposs_under_assumption_sr,thm:imposs_under_assumption_sr_2}.}
  \fi
  Our proofs are inspired by \cite[Theorem  1]{KearnsL93} and \cite[Theorem 1]{bshouty2002pac} which consider PAC learning with adversarial corruptions.
  In \cref{thm:no_stable_classifier,thm:imposs_under_assumption_sr,thm:imposs_under_assumption_sr_2}, for some $\eps,\nu\in [0,1]$, the goal is to show that given samples perturbed by an $\eta$-\Ham{} adversary,  under some additional assumptions, no learner $\cL$ can output a classifier that has accuracy  $\eps$-close to the accuracy of $f^\star$ sand that additively violates the fairness constraints by at most $\nu$.
  Say a classifier $f\in\cF$ is ``good'' if it satisfies these required guarantees.
  The approach is to construct two or more distributions $\cD_1,\cD_2,\dots,\cD_m$ that satisfy the following conditions:
  \ifexpand
  \begin{enumerate}[itemsep=\itemsepINTERNAL,leftmargin=\leftmarginINTERNAL]
    \item[(C1)]
    For any $\ell,k$, given a iid draw $S$ from $\cD_\ell$, an $\eta$-\Ham{} adversary can add perturbations such that with high probability $\hS$ is distributed according to iid samples from $\cD_k.$
  \end{enumerate}
  \else
  (C1) For any $\ell,k$, given a iid draw $S$ from $\cD_\ell$, an $\eta$-\Ham{} adversary can add perturbations such that with high probability $\hS$ is distributed according to iid samples from $\cD_k.$
  \fi
  Thus $\cL$, who only sees $\hS$, with high probability, cannot identify the original distribution of $S$ and is forced to output a classifier that is good for all $\cD_1,\dots,\cD_m$.
  The next condition ensures that this is not possible.
  \ifexpand
  \begin{enumerate}[itemsep=\itemsepINTERNAL,leftmargin=\leftmarginINTERNAL]
    \item[(C2)]
    No classifier $f\in \cF$ is good for all $\cD_1,\dots,\cD_m$, and for each $\cD_i$ ($i\in [m]$), there is at least one good classifier $f_i\in \cF$. (The latter-half ensures that the fairness and accuracy requirements are not vacuously satisfied.)
  \end{enumerate}
  \else
  (C2) No classifier $f\in \cF$ is good for all $\cD_1,\dots,\cD_m$, and for each $\cD_i$ ($i\in [m]$), there is at least one good classifier $f_i\in \cF$. (The latter-half ensures that the fairness and accuracy requirements are not vacuously satisfied.)
  \fi
  Thus, for every $\cL$ there is a distribution in $\cD_1,\dots,\cD_m$ for which $\cL$ outputs a bad classifier.
  (Note that even if the learner is randomized, it must fail with probability at least $\nfrac1m$.)
  Finally, the assumptions on $\cF$ ensure that condition (C2) is satisfiable.
  For instance, if $\cF$ has less than $m$ hypothesis, then condition (C2) cannot be satisfied.

  The key idea in the proofs is to come up with distributions satisfying the above conditions.
  \cite{KearnsL93,bshouty2002pac} follow the same outline in the context of PAC learning, however, as we also consider fairness constraints, our constructions end up being very different from their constructions.
  Our constructions are specific to the statistical rate fairness metric.
  However, one can still apply the general approach outlined above to other fairness metrics by constructing a suitable set of distributions.
  Full details appear in \cref{sec:proofof:thm:no_stable_classifier,sec:proofof:thm:imposs_under_assumption_sr,sec:additional_imposs_result}.

  \newcommand{\chkv}{{\bf CHKV}}
  \newcommand{\kl}{{\bf KL}}
  \newcommand{\lmzv}{{\bf LMZV}}
  \newcommand{\errtol}{{\bf Err-Tol}}
  \newcommand{\AKM}{{\bf AKM}}
  \newcommand{\uncons}{{\bf Uncons}}

  \section{Empirical Results}\label{sec:empirics}
  We implement our approach using the logistic loss function with linear classifiers and evaluate its performance on real world and synthetic data.
  \ifexpand
  \paragraph{Metrics and baselines.}
  \else
  {\bf Metrics and baselines.}
  \fi
  The selection of an appropriate fairness metric is context-dependent and beyond the scope of this work~\cite{srivastava2019mathematical};
  for illustrative purposes we (arbitrarily) consider the statistical rate (SR) and compare
  an implementation of our framework (Program~\ref{prog:errtolerant_extnd}), \errtol{}{\bf ,} with state-of-the-art fair classification frameworks for statistical rate under stochastic perturbations: \lmzv{} \cite{LamyZ19} and \chkv{} \cite{celis2020fairclassification}.
  \lmzv{} and \chkv{} take parameters $\delta_L,\tau\in [0,1]$ as input; these parameters control the desired fairness, where decreasing $\delta_L$ or increasing $\tau$ increases the desired fairness.
  We also compare against \kl{} \cite{konstantinov2021fairness}, which controls for true-positive rate (TPR) in the presence of a Malicious adversary, and \AKM{} \cite{awasthi2020equalized} that is the post-processing method of \cite{hardt2016equality} and controls for equalized-odds fairness constraints.
  We also compare against the optimal unconstrained classifier, \uncons{}\textbf{;} this is the same as \cite{bshouty2002pac}'s {algorithm for PAC-learning in the Nasty Sample Noise Model without fairness constraints. }
  We provide additional comparisons using our framework with false-positive rate as the fairness metric with additional baselines and using the Adult data~\cite{adult} in \cref{sec:more_empirical_results}.

  \ifexpand
  \paragraph{Implementation details.}
  \else
  {\bf Implementation details.}
  \fi
  We use a randomly generated 70-30 train ($S$) test ($T$) split of the data, and generate the perturbed data $\hS$ from $S$ for a (known) perturbation rate $\eta$.
  We train each algorithm on $\hS$, and report the accuracy (acc) and statistical rate (SR) of the output classifiers on the (unperturbed) test data $T$.
  \errtol{} is given the perturbation rate $\eta$ and uses the SLSQP solver in SciPy~\cite{scipy} to solve Program~\ref{prog:errtolerant_extnd}.
  To advantage the baselines in our comparison, we provide them with even more information as needed by their approaches:
  \lmzv{} and \chkv{} are given group-specific perturbation rates: for each $\ell\in[p]$, $\eta_\ell \coloneqq \Pr_D[\hZ\neq Z\mid Z=\ell]$, and \kl{} is given $\eta$ and
  for each $\ell\in [p]$, the probability $\Pr_D[Z=\ell,Y=1]$; where $D$ is the empirical distribution of $S$.
  \errtol{} implements Program~\ref{prog:errtolerant_extnd} which requires estimates of $\lambda_\ell$ and $\gamma_\ell$ for all $\ell \in [p]$.
  As a heuristic, we set $\gamma_\ell=\lambda_\ell\coloneqq \Pr\nolimits_{\hD}[Z=\ell]$, where $\smash{\hD}$ is the empirical distribution of $\hS$.
  We find that these estimates suffice, and expect that a more refined approach would only improve the performance of \errtol{}\textbf{.}

  \ifexpand
  \paragraph{Adversaries.}
  \else
  {\bf Adversaries.}
  \fi
  We consider two $\eta$-Hamming adversaries (which we call $A_{\rm TN}$ and $A_{\rm FN}$); %
  each one computes the ``optimal fair classifier'' $f^\star$, which has the highest accuracy (on $S$) subject to having statistical rate at least $\tau$ on $S$.
  $A_{\rm TN}$ considers the set of all true negatives of $f^\star$ that have protected attribute $Z=1$, selects the $\eta \cdot |S|$ samples that are furthest from the decision boundary of $f^\star$, and perturbs their protected attribute to $\hZ=2$.
  $A_{\rm FN}$ is similar, except that it considers the set of false negatives of $f^\star$.
  Both adversaries try to increase the performance of $f^\star$ on $Z=1$ in $\hS$ by removing the samples that $f^\star$ predicts as negative; thus, increasing $f^\star$'s statistical rate.
  The adversary's hope is that choosing samples far from the decision boundary would (falsely) give the appearance of a high statistical rate on $\hS$.
  This would make a fair classification framework output unfair classifiers with higher accuracy.
  Note that these are not intended to be ``worst-case'' adversaries; as \errtol{} comes with provable guarantees, we expect it to perform well against other adversaries while other approaches may have even poorer performance.

  \setlength\fboxsep{0pt}
  \setlength{\tabcolsep}{2pt}
  \begin{table*}[t]
    \centering
    \caption{
    \ifconf\small\fi
    {\em Simulation on synthetic data:}
    We run \chkv{} and \errtol{} with $\tau=0.8$ on synthetic data and report their average accuracy (acc) and statistical rate (SR) with standard deviation in parentheses.
    The result shows that prior approaches can fail to satisfy their guarantees under the $\eta$-Hamming model.
    \label{table:result_on_synthetic_data}
    }
    \vspace{2mm}
    \newcommand{\PreserveBackslash}[1]{\let\temp=\\#1\let\\=\temp}
    \newcolumntype{C}[1]{>{\PreserveBackslash\centering}p{#1}}
    \label{tab:sr_results}
    \footnotesize
    \begin{tabular}{p{2.8cm}p{1.7cm}p{1.7cm}p{1.7cm}p{1.7cm}p{1.7cm}p{1.7cm}}
      \toprule
      \midsepremove{}
      & acc ($\eta$=0\%) & SR ($\eta$=0\%) & acc ($\eta$=3\%) & SR ($\eta$=3\%)  & acc ($\eta$=5\%) & SR ($\eta$=5\%)\\
      \midrule
      {\textbf{Unconstrained}}&  1.00 (.001) & 	 .799 (.001) & 1.00 (.000) & 	 .799 (.002) &  1.00 (.001) & 	 .800 (.001) \\
      {\chkv{}\hspace{0.1mm} ($\tau{=}.8$)}\hspace{-3mm} &  1.00 (.001) &  .800 (.002) & .859 (.143) & .787 (.015) & .799 (.139) & .795 (.049)\\
      {\errtol{}\hspace{0.1mm} ($\tau{=}.8$)} & {.985 (.065)} & .800 (.001)  & 1.00 (.001) & .799 (.002)  & .999 (.002) & .799 (.004) \\
      \bottomrule
    \end{tabular}
    \midsepdefault{}
  \end{table*}

  \ifexpand
  \subsection{Simulation on Synthetic Data}\label{sec:empirics:synthetic}
  \else
  {\bf Simulation on synthetic data.}
  \fi
  We first show empirically that perturbations by the $\eta$-Hamming adversary can be prohibitively disruptive for methods that attempt to correct for stochastic noise.
  We consider synthetic data with 1,000 samples from two equally-sized protected groups; each sample has a binary protected attribute, two continuous features $x_1,x_2\in \R$, and a binary label.
  Conditioned on the protected attribute, ($x_1,x_2$) are independent draws from a mixture of 2D Gaussians (see \cref{fig:synthetic_dataset}).
  This distribution and the labels are such that a) one group has a higher likelihood of a positive label than the other, and b) \uncons{} {has a near-perfect accuracy ($>99\%$) and a statistical rate of $0.8$ on $S$.
  Similar to \uncons{}\textbf{,} we consider a fairness constraint of $\tau = 0.8$.
  Thus, in the absence of noise, this is an ``easy case:'' where \uncons{} satisfies the fairness constraints.
  We generate $\hS$ using $A_{\rm TN}$, %
  and compare against \chkv{}\textbf{,} which was developed for correcting stochastic perturbations.\footnote{
  We also attempted to compare against \akm{}, \kl{}, and \lmzv{}.
  But they did not converge to $f^\star$ even on the unperturbed synthetic data, and hence, we did not include these results as it would be an unfair comparison.
  }
  \ifexpand
  \paragraph{Results.}
  \else
  {\em Results.}
  \fi
  The fairness and statistical rate averaged over 50 iterations are reported in \cref{tab:sr_results} as a function of the perturbation $\eta$.
  At $\eta=0$, both \chkv{} and \errtol{}  nearly-satisfy the fairness constraint (${\rm SR}\geq\negsp{} 0.79$) and have a near-perfect accuracy (${\rm acc}\geq\negsp{} 0.98$).
  However, as $\eta$ increases, while \chkv{} retains the same statistical rate ($\sim 0.8$), it loses a significant amount of accuracy ($\sim 20\%$).
  In contrast, \errtol{} has high accuracy and fairness    (${\rm acc}\geq\negsp{} 0.99$ and ${\rm SR}\geq\negsp{} 0.79$) for all $\eta$ considered.
  Hence, this shows that stochastic approaches may fail to satisfy their guarantees under the $\eta$-\Ham{} model.

  \ifexpand
  \subsection{Simulations on Real-World Data}\label{sec:empirics:real_world}
  \else
  {\bf Simulations on real-world data.}
  \fi
  In this simulation, we show that our framework can outperform each baseline with respect to the accuracy-fairness trade-off under  perturbations from the adversaries we consider,
  and does not under-perform compared to baselines under perturbations from either adversary.
  The COMPAS data {in \cite{ibm_aif360}} contains 6,172 samples with 10 binary features and a label that is 1 if the individual did not recidivate and 0 otherwise; the statistical rate of \uncons{} on COMPAS is $0.78$.
  We take gender (coded as binary) as the protected attribute, and set the fairness constraint on the statistical rate to be $\tau=0.9$ for \errtol{} and all baselines.
  We consider both adversaries $A_{\rm TN}$ and $A_{\rm FN}$, and a perturbation rate of $\eta=3.5\%$, as $3.5\%$ is roughly the smallest value for $\eta$ necessary to ensure that the optimal fair classifier $f^\star$ for \mbox{$\tau=0.9$ (on $S$) has a statistical rate less than $0.78$ on $\hS$.}

  \renewcommand{\folder}{./figures/adversarial-noise-SR-final-100-iter}
  \begin{figure}[t!]
    \vspace{-1mm}
    \centering
    \includegraphics[width=0.8\linewidth, trim={0cm 0cm 0cm 0cm},clip]{\folder/../legend-medium.pdf}
    \par\vspace{-2.5mm}
    {
    \begin{tikzpicture}
      \node (image) at (0,-0.17) {\includegraphics[width=0.43\linewidth, trim={1.4cm 0cm 1.5cm 1cm},clip]{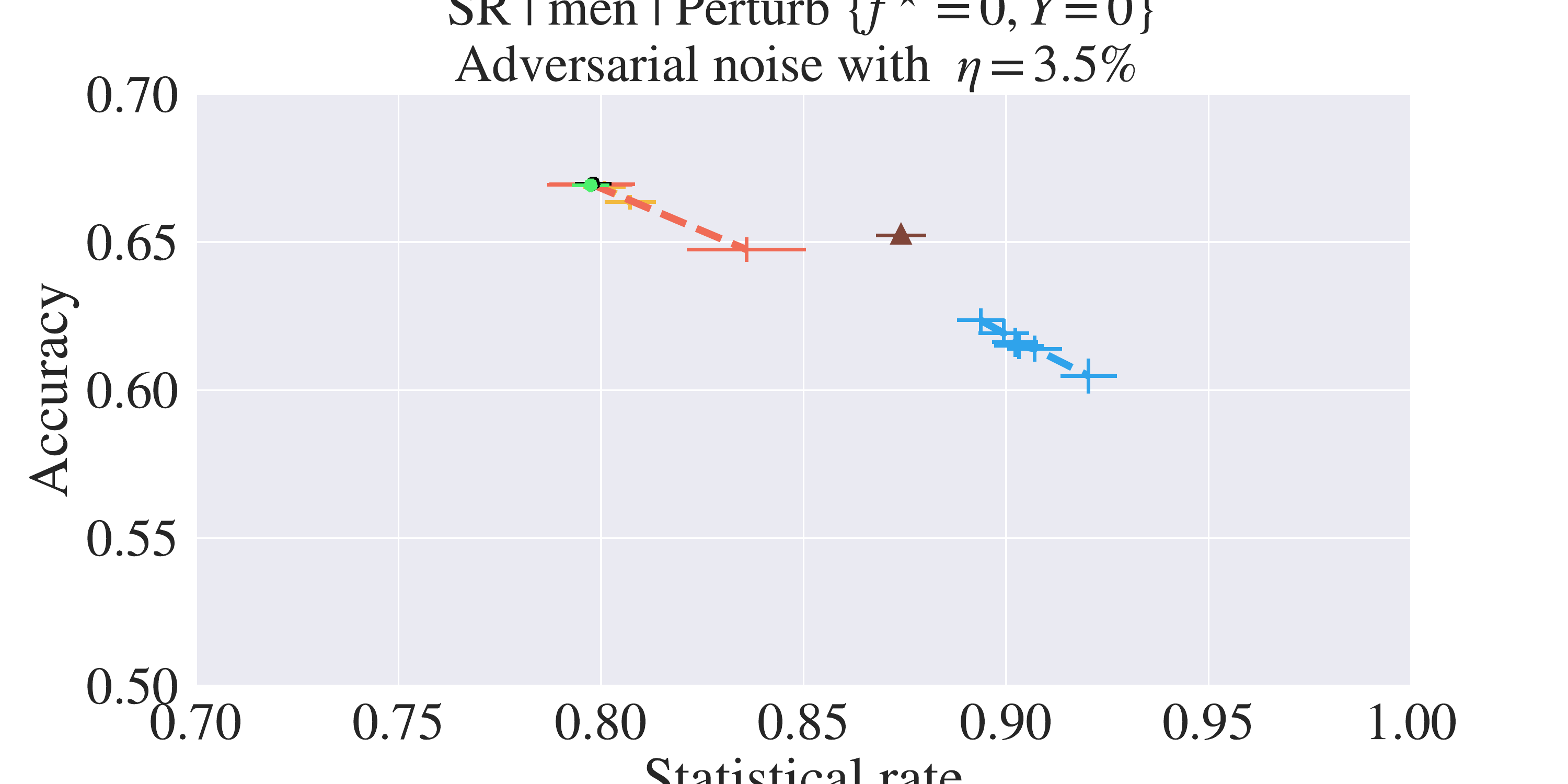}};
      \node[rotate=90, fill=white] at (-3.7,0) {\white{\large aaaaaaaaaaaaa}};
      \node[rotate=90] at (-3.9, 0) {Accuracy};
      \draw[draw=white, fill=white] (-2, -2.2) rectangle ++(4,0.35);
      \draw[draw=white, fill=white] (-2, 1.5) rectangle ++(4,0.2);
      \node[rotate=0, fill=white] at (0.05, -2.2)  {Statistical rate};
      \node[rotate=0] at (0.05, -2.8)  {(a) $A_{\rm TN}$};
      \node[rotate=0] at (2.35,-2.2) {\textit{\small(more fair)}};
      \node[rotate=0] at (-2.25,-2.2) {\textit{\small(less fair)}};
    \end{tikzpicture}
    \begin{tikzpicture}
      \node (image) at (0,-0.17) {\includegraphics[width=0.43\linewidth, trim={1.4cm 0cm 1.5cm 1cm},clip]{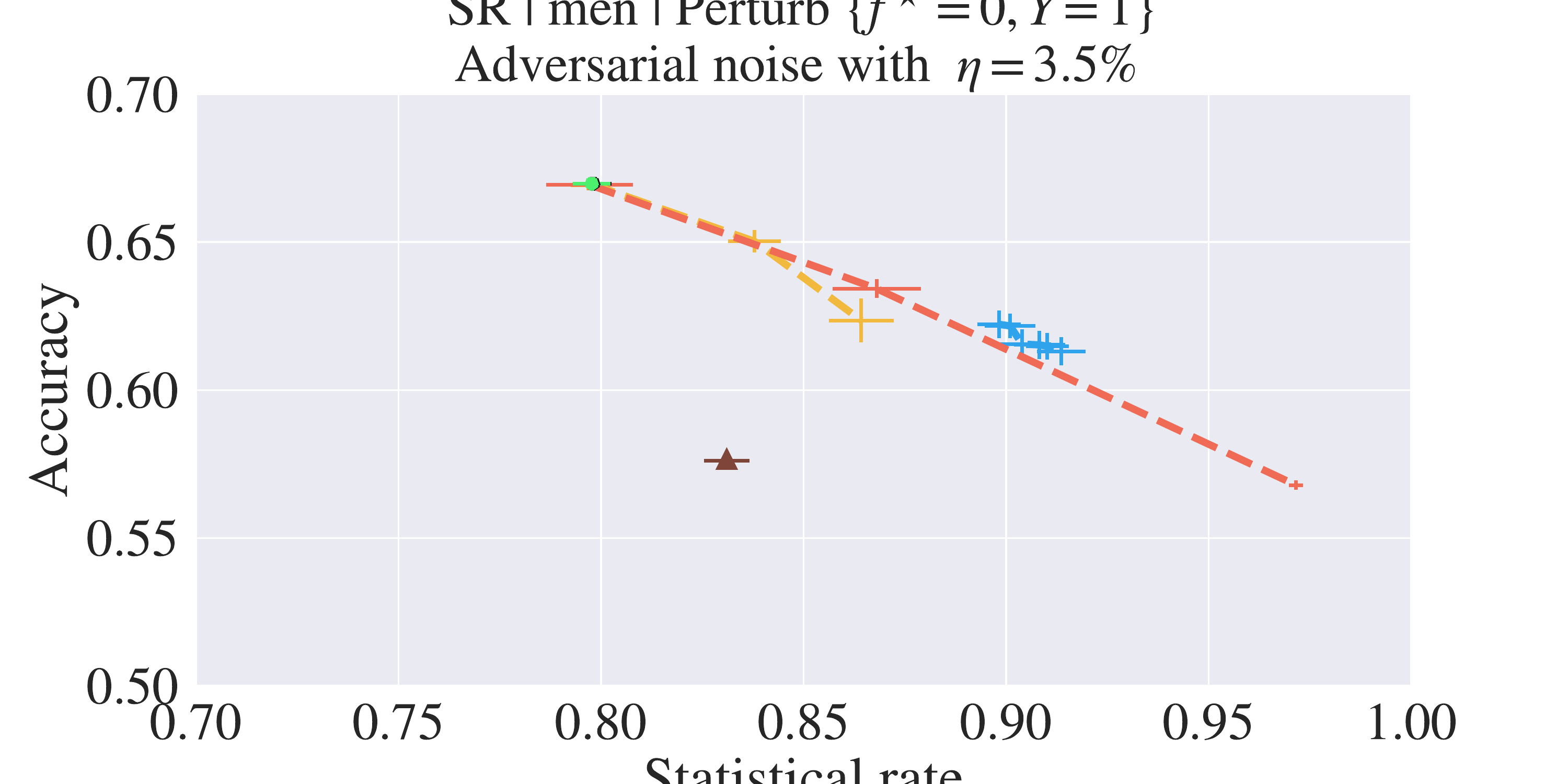}};
      \node[rotate=90, fill=white] at (-3.7,0) {\white{\large aaaaaaaaaaaaa}};
      \node[rotate=90] at (-3.9, 0) {Accuracy};
      \draw[draw=white, fill=white] (-2, -2.2) rectangle ++(4,0.35);
      \draw[draw=white, fill=white] (-2, 1.5) rectangle ++(4,0.2);
      \node[rotate=0, fill=white] at (0.05, -2.2)  {Statistical rate};
      \node[rotate=0] at (0.05, -2.8)  {(a) $A_{\rm FN}$};
      \node[rotate=0] at (2.35,-2.2) {\textit{\small(more fair)}};
      \node[rotate=0] at (-2.25,-2.2) {\textit{\small(less fair)}};
      \end{tikzpicture} %
      }
      \vspace{-3mm}
      \caption{
      \ifconf\small\fi
      {\em {Simulations on COMPAS data:}}
      Perturbed data is generated using adversary $A_{\rm TN}$ (a) and $A_{\rm FN}$ (b) as described in \cref{sec:empirics} with $\eta=3.5\%$.
      All algorithms are run on the perturbed data varying the fairness parameters ($\tau\in [0.7,1]$ and $\delta_L\in [0,0.1]$).
      The $y$-axis depicts accuracy and the $x$-axis depicts statistical rate (SR); both values are computed over the unperturbed test set.
      We observe that for both adversaries     our approach \errtol{}\textbf{,} attains a better fairness than the unconstrained classifier with a natural trade-off in accuracy.
      Further, \errtol{} achieves a better fairness-accuracy trade-off than each baseline on at least one of (a) or (b).  Error bars represent the standard error of the mean.
      }
      \vspace{-2mm}
      \label{fig:adversarial_noise_main_body}
    \end{figure}

    \ifexpand
    \paragraph{Results.}
    \else
    {\em Results.}
    \fi
    The accuracy and statistical rate (SR) of \errtol{} and baselines for $\tau\in [0.7,1]$ and $\delta_L\in [0,0.1]$  and averaged over 100 iterations are reported in \cref{fig:adversarial_noise_main_body}.
    For both adversaries, \errtol{} attains a better statistical rate than the unconstrained classifier (\uncons{}) for a small trade-off in accuracy.
    For adversary $A_{\rm TN}$ (\cref{fig:adversarial_noise_main_body}(a)), \uncons{} has statistical rate ($0.80$) and accuracy ($0.67$).
    In contrast, \errtol{} achieves high statistical rate ($0.92$) with a trade-off in accuracy {($0.60$)}.
    In comparison, \akm{} has a higher accuracy ($0.65$) but a lower statistical rate ($0.87$), and
    other baselines have an even lower statistical rate ($\leq 0.84$) with accuracy comparable to \akm{}\textbf{.}
    For adversary $A_{\rm FN}$ (\cref{fig:adversarial_noise_main_body}(b)), \uncons{} has statistical rate {($0.80$)} and accuracy ($0.67$),  while \errtol{} has a significantly higher SR {($0.91$)} and accuracy {($0.61$)}.
    This significantly outperforms \akm{} which has statistical rate {($0.83$)} and accuracy {($0.58$)}.
    \lmzv{} achieves the highest statistical rate ($0.97$) with a natural reduction in accuracy to {($0.57$)}.
    In this case, \errtol{} has similar accuracy to statistical rate trade-off as \lmzv{}{\bf ,} but achieves a lower maximum statistical rate ($0.91$).
    Meanwhile, \errtol{} has a significantly higher statistical rate trade-off than \chkv{} at the same accuracy.
    We further evaluate our framework under stochastic perturbations in \cref{sec:more_empirical_results} (specifically, against the perturbation model of~\cite{celis2020fairclassification}) and observe similar statistical rate and accuracy \mbox{trade-offs as approaches~\cite{celis2020fairclassification,LamyZ19} tailored for stochastic perturbations.}

    \begin{remark}[\textbf{Range of fairness parameters in the simulation}]
      Among baselines, \akm{}\textbf{,} \kl{}\textbf{,} and \uncons{} do not take the desired-fairness value as input, so they appear as points in \cref{fig:adversarial_noise_main_body}.
      For all other methods (\chkv{}\textbf{,} \errtol{}\textbf{,} and \lmzv{}), we vary the fairness parameters starting from the tightest constraints (i.e., $\tau=1$ and $\delta_L=0$) and relax the constraints until all algorithms’ achieved statistical rate matches the achieved statistical rate of the unconstrained classifier (this happens around $\tau=0.7$ and $\delta_L=0.1$).
      We do not relax the fairness parameters further because the resulting problem is equivalent to the unconstrained classification problem. (This is because the unconstrained classifier, which has the highest accuracy, satisfies the fairness constraints for $\tau\leq 0.7$ and $\delta_L\geq 0.1$).
    \end{remark}

    \section{Proofs}

    \subsection{{Proof of \texorpdfstring{\cref{thm:main_result}}{Theorem 4.3} }}\label{sec:proofof:thm:main_result}

    Recall that we assume \cref{asmp:1} holds with constant $\lambda>0$ and the VC dimension of $\cF$ is finite, say $d\in\N$.
    Our goal is to prove that for all perturbation rates $\eta\in (0,\nfrac\lambda{2})$,
    fairness thresholds $\tau\in (0,1]$,
    bounds on error $\eps> 2\eta$ and bounds on constraint violation $\nu >  \nfrac{8\eta\tau}{(\lambda-2\eta)}$,
    and confidence parameters $\delta\in (0,1)$,
    given sufficiently many perturbed samples from an $\eta$-\Ham{} adversary,
    with probability at least $1-\delta$, it holds that
    \begin{align*}
      {\rm Err}_\cD(f_{\rm ET}) - {\rm Err}_\cD(f^\star) &\leq \eps
      \quad\text{and}\quad
      \Omega_{\cD}(f_{\rm ET}) \geq \tau -\nu,
    \end{align*}
    where $f_{\rm ET}$ is an optimal solution of \errtolerant{}.
    In the proof, we set the following parameters
    \begin{align*}
      \Delta_0 \coloneqq \min\inbrace{2\eta-\eps, \frac{8\eta\tau}{\lambda-2\eta} - \nu}
      \quad\text{and}\quad
      \delta_0 \coloneqq \frac{\delta}{2p+1}. \yesnum\label{eq:define_params}
    \end{align*}
    Precisely, we set $$\Delta\coloneqq \Delta_0\cdot \frac{(\lambda-2\eta)^2}{32\cdot 3}$$ in \errtolerant{} and require at least $N$ samples (perturbed by an $\eta$-\Ham{} adversary), where $N$ satisfies
    \begin{align*}
      N=\Theta\inparen{\frac{1}{\Delta^2\cdot(\lambda-2\eta)^4}\cdot\inparen{d\log\inparen{\frac{d}{\Delta^2\cdot(\lambda-2\eta)^4}}+\log\inparen{\frac{p}{\delta_0}}}}.\yesnum
    \end{align*}
    Note that these values satisfy the requirements $$\Delta=O\inparen{\min\inbrace{\eps-2\eta,\nu-\frac{8\eta\tau}{\lambda-2\eta},\lambda-2\eta}}\quad \text{and}\quad N=\poly\inparen{d, \frac{1}{\Delta},  \log\inparen{\frac{p}\delta}}.$$
    \begin{remark}
      All probabilities and expectations in all proofs are with respect to the draw of $(X,Y,Z)$.
      Given a distribution $\cP$, we use $\Pr_\cP[\cdot]$ to denote $\Pr\nolimits_{(X,Y,Z)\sim\cP}[\cdot]$ and $\Ex_\cP[\cdot]$ to denote $\Ex\nolimits_{(X,Y,Z)\sim\cP}[\cdot]$.
      If $\cP$ denotes the distribution of perturbed samples, then we use $\Pr_\cP[\cdot]$ to denote $\Pr\nolimits_{(X,Y,\hZ)\sim\cP}[\cdot]$ and $\Ex_\cP[\cdot]$ to denote $\Ex\nolimits_{(X,Y,\hZ)\sim\cP}[\cdot]$; The difference between the two will be clear from context.
    \end{remark}

    \subsubsection{Preliminaries: Generalization Bound}
    We use \cref{lem:conc} in the proof of \cref{thm:main_result}.
    See \cite[Section 28.1]{shalev2014understanding} for a proof.
    \begin{lemma}[\bf Concentration of mean of bounded functions]\label{lem:conc}
      For any bounded function $g\colon \zo\times\zo\times[p]\to [0,1]$ and constants $\Delta,\delta_0 \in (0,1)$,
      given $N \geq \Theta\big(\Delta^{-2}\cdot\big({\rm VC}(\cF)\cdot\log\inparen{\nfrac{{\rm VC}(\cF)}{\Delta}}+\log\inparen{\nfrac{1}{\delta_0}}\big)\big)$ samples $S$ iid from $\cD$,
      with probability at least $1-\delta_0$, it holds that
      \begin{align*}
        \forall\ f\in\cF, \quad \abs{\Ex\nolimits_{D}\insquare{g(f(X,Z), Y, Z)} - \Ex\nolimits_{\cD}\insquare{g(f(X,Z), Y, Z)}} \leq \Delta,
      \end{align*}
      where $D$ is the empirical distribution of $S$.
    \end{lemma}
    \subsubsection{{Step 1: Lower Bound on the Accuracy of \texorpdfstring{$f_{\rm ET}$}{f\_ET}}}
    This step relies on \cref{coro:whp_bounding_power}, which is the formal version of \cref{lem:overview1}.
    We use \cref{coro:whp_bounding_power} in the proof of \cref{lem:conditionalAccuracyGuarantee} to lower bound the accuracy of $f_{\rm ET}$.
    \begin{lemma}[\bf Bound on difference in means of bounded functions on $\cD$ and on $\hS$]\label{coro:whp_bounding_power}
      For any bounded function $g\colon \zo\times\zo\times[p]\to [0,1]$ and constants $\Delta,\delta_0 \in (0,1)$,
      given $N \geq \Theta\big(\Delta^{-2}\cdot\big({\rm VC}(\cF)\cdot\log\inparen{\nfrac{{\rm VC}(\cF)}{\Delta}}+\log\inparen{\nfrac{1}{\delta_0}}\big)\big)$ samples $S$ iid from $\cD$,
      and corresponding perturbed samples $A(S) \coloneqq \{(x_i,y_i,\hz_i)\}_{i\in [N]}$, with probability at least $1-\delta$, it holds that
      \begin{align*}
        \forall\ f\in\cF, \quad \abs{\Ex\nolimits_{\hD}\insquare{g(f(X,\hZ), Y, \hZ)} - \Ex\nolimits_{\cD}\insquare{g(f(X,Z), Y, Z)}} \leq \Delta+\eta,
      \end{align*}
      where $\hD$ is the empirical distribution of $\hS$.
    \end{lemma}
    \begin{proof}
      Let $S\coloneqq \inbrace{(x_i,z_i,y_i)}_{i\in [N]}$ and $\hS\coloneqq \inbrace{(x_i,\hz_i,y_i)}_{i\in [N]}$.
      Using the triangle inequality for absolute value
      \ifconf
      \begin{align*}
        &\abs{\Ex\nolimits_{\hD}\insquare{g(f(X,\hZ), Y,\hZ)} - \Ex\nolimits_{\cD}\insquare{g(f(X,Z), Y, Z)}}\\
        &\qquad\qquad\qquad  \leq \abs{\Ex\nolimits_{D}\insquare{g(f(X,Z), Y,Z)} - \Ex\nolimits_{\cD}\insquare{g(f(X,Z), Y, Z)}}\\
        &\qquad\qquad\qquad + \hspace{0.45mm} \abs{\Ex\nolimits_{\hD}\insquare{g(f(X,\hZ), Y,\hZ)} - \Ex\nolimits_{D}\insquare{g(f(X,Z), Y,Z)}}.\yesnum
        \label{eq:eq_1_in_whp_bounding_power}
      \end{align*}
      \else
      \begin{align*}
        \abs{\Ex\nolimits_{\hD}\insquare{g(f(X,\hZ), Y,\hZ)} - \Ex\nolimits_{\cD}\insquare{g(f(X,Z), Y, Z)}}
        & \leq \abs{\Ex\nolimits_{D}\insquare{g(f(X,Z), Y,Z)} - \Ex\nolimits_{\cD}\insquare{g(f(X,Z), Y, Z)}}\\
        &+ \hspace{0.45mm} \abs{\Ex\nolimits_{\hD}\insquare{g(f(X,\hZ), Y,\hZ)} - \Ex\nolimits_{D}\insquare{g(f(X,Z), Y,Z)}}.\yesnum
        \label{eq:eq_1_in_whp_bounding_power}
      \end{align*}
      \fi
      We can upper bound the first term in the right-hand side using \cref{lem:conc}.
      In particular, we have that with probability at least $1-\delta$, for all $f\in \cF$, it holds that
      \begin{align*}
        \abs{\Ex\nolimits_{D}\insquare{g(f(X,Z), Y,Z)} - \Ex\nolimits_{\cD}\insquare{g(f(X,Z), Y, Z)}}\leq \Delta.\yesnum\label{eq:eq_2_in_whp_bounding_power}
      \end{align*}
      Further, we can upper bound the second term in the right-hand side of Equation~\eqref{eq:eq_1_in_whp_bounding_power} for all $f\in \cF$ as follows
      \begin{align*}
        &\abs{\Ex\nolimits_{\hD}\insquare{g(f(X,\hZ), Y,\hZ)} - \Ex\nolimits_{\hD}\insquare{g(f(X,\hZ), Y,\hZ)}}\\
        &\quad =\quad \frac1N\cdot \abs{\sum\nolimits_{i\in[N]} g(f(x_i, \hz_i), y_i,\hz_i) -g(f(x_i, z_i), y_i, z_i)}\\
        &\quad =\quad  \frac1N\cdot \abs{\sum\nolimits_{i\in[N]\colon z_i\neq \hz_i} g(f(x_i, \hz_i), y_i,\hz_i) - g(f(x_i, z_i), y_i, z_i)}
        \tag{For all $i\in [N]$, where $z_i=\hz_i$, $g(f(x_i, \hz_i), y_i,\hz_i) = g(f(x_i, z_i), y_i, z_i)$}
      \end{align*}
      \begin{align*}
        &\quad \leq\quad  \frac1N\cdot \abs{\sum\nolimits_{i\in[N]\colon z_i\neq \hz_i} 1}\tag{Using that $g$ is bounded by $0$ and 1}\\
        &\quad \leq\quad  \eta.\yesnum\label{eq:eq_3_in_whp_bounding_power}
      \end{align*}
      Since with probability at least $1-\delta$, both Equation~\eqref{eq:eq_2_in_whp_bounding_power} and \eqref{eq:eq_3_in_whp_bounding_power} hold for all $f\in \cF$, substituting them in Equation~\eqref{eq:eq_1_in_whp_bounding_power} gives us the required bound.
    \end{proof}
    \begin{lemma}\label{lem:conditionalAccuracyGuarantee}
      If $f^\star$ is feasible for \errtolerant{}, then it holds that:
      \begin{align*}
        {{\rm Err}_\cD(f_{\rm ET})-{\rm Err}_\cD(f^\star)}\leq 2\eta+\Delta.
      \end{align*}
    \end{lemma}
    \begin{proof}
      Let $g$ be the 0-1 loss (i.e., $g(\wt{y},y,z)\coloneqq \mathds{I}\insquare{\wt{y}\neq y}$),
      then for all $f\in \cF$, \cref{coro:whp_bounding_power} shows that the error of $f$ on samples drawn from $\cD$ and samples in $\hS$ are close:
      $$\abs{{\rm Err}_\cD(f)-{\rm Err}(f,\hS)}\leq\Delta+\eta.$$
      Since $f_{\rm ET}$ is optimal for \errtolerant{}, its error on $\hS$ is at most the error of $f^\star$ on $\hS$.
      Using this and applying \cref{coro:whp_bounding_power} we get that
      \begin{align*}
        {\rm Err}_\cD(f_{\rm ET})
        \qquad &\Stackrel{\rm\cref{coro:whp_bounding_power}}{\leq}\qquad    {\rm Err}(f_{\rm ET}, \hS)+\eta+\Delta\\
        &\leq\qquad  {\rm Err}(f^\star, \hS)+\eta+\Delta\\
        &\Stackrel{\rm\cref{coro:whp_bounding_power}}{\leq}\qquad   {\rm Err}_\cD(f^\star)+2(\eta+\Delta).
      \end{align*}
    \end{proof}
    %

    \subsubsection{{Step 2: Lower Bound on the Fairness of \texorpdfstring{$f_{\rm ET}$}{f\_ET}  }}\label{sec:main_result:step_2}
    In this step, we show that any $f\in \cF$ feasible for \errtolerant{} satisfies $\Omega_\cD(f)\geq \tau-\nu$ (\cref{lem:approx_feasible}).
    This relies on the notion of $s$-stability (\cref{def:stable_class:formal}) and the fact that any $f\in \cF$ feasible for \errtolerant{} is $s$-stable (for $s$ a function of $\eta$ and $\lambda$); \cref{lem:suff_cond_for_stable_app}.

    \begin{definition}[\bf $s$-stability]\label{def:stable_class:formal}
      Given a constant $s\in (0,1)$ and perturbation rate $\eta\in [0,1)$,
      a classifier $f\in \cF$ is said to be $s$-stable for fairness metric $\Omega$ with perturbation rate $\eta$,
      if for all adversaries $A\in \cA(\eta)$ and confidence parameters $\delta_0\in (0,1)$
      given $N =\polylog(\nfrac{1}{\delta_0})$
      samples $S$ iid from $\cD$
      and corresponding perturbed samples $\hS\coloneqq A(S)$,
      with probability at least $1-\delta_0$ (over draw of $S\sim\cD$),
      it holds that $$\frac{\Omega_\cD(f)}{\white{\widehat{S}}\Omega(f,\hS)\white{\widehat{S}}}\in \insquare{s,\frac{1}{s}}.$$
    \end{definition}
    \noindent If an $s$-stable classifier $f$ has fairness $\tau$ on $\hS$, then it has a fairness at least $\tau\cdot s$ on $\cD$ with high probability.
    Thus, if we have some constraint $C$ such that any feasible $f\in \cF$ satisfying constraint $C$ is $s$-stable, then any classifier satisfying constraint $C$ and the fairness constraint, $\Omega(\cdot, \hS)\geq \nfrac{\tau}{s}$, must have a fairness at least $\tau$ on $\cD$ with high probability.
    The key idea is coming up such a constraints.
    First, in \cref{lem:bound_on_fc}, we give such constraints which use the unperturbed protected attributes $Z$, later in \cref{lem:suff_cond_for_stable_app} we give constraints which only use the perturbed protected attributes $\hZ$.
    \begin{lemma}[\bf Sufficient condition for a stable classifier]\label{lem:bound_on_fc}
      For each $\alpha \in (0,1)$, any $f\in\cF$ satisfying
      \begin{align*}
        \min_{\ell\in [p]}\Pr\nolimits_{\cD}\insquare{\cE(f), \cE^\prime(f), Z=\ell}\geq \alpha,\yesnum\label{eq:assumption_on_f}
      \end{align*}
      is $\sinparen{\frac{1- (\eta+\Delta)/\alpha}  {  1 + (\eta+\Delta)/\alpha}}^2$-stable with respect to the fairness metric $\Omega$ defined by events $\cE$ and $\cE^\prime$.
    \end{lemma}
    \begin{proof}[Proof of \cref{lem:bound_on_fc}]
      Let $\hD$ be the empirical distribution of $\hS$.
      Let $\evJ$ be the event that for all $\ell\in[p]$ and all $f\in\cF$
      \begin{align*}
        \abs{\Pr\nolimits_{\hD}\insquare{ \cE(f), \cE^\prime(f), \hZ=\ell}  - \Pr\nolimits_{\cD}\insquare{\cE(f), \cE^\prime(f), Z_i=\ell }  } &< \eta+\Delta,\yesnum\label{eq:evJ_1}\\
        \abs{\Pr\nolimits_{\hD}\insquare{ \cE^\prime(f), \hZ=\ell}  - \Pr\nolimits_{\cD}\insquare{\cE^\prime(f), Z_i=\ell}  } &< \eta+\Delta.\yesnum\label{eq:evJ_2}
      \end{align*}
      Using \cref{coro:whp_bounding_power} with
      $g(\wt{y},y,z)\coloneqq \mathds{I}\insquare{ \cE(\wt{y}), \cE^\prime(\wt{y}),z=\ell}$, we get that with probability at least $1-\delta_0$,  Equation~\eqref{eq:evJ_1} holds for $f\in\cF$ and a particular $\ell\in[p]$.
      Similarly, using \cref{coro:whp_bounding_power} with $g(\wt{y},y,z)\coloneqq \mathds{I}\insquare{\cE(\wt{y}),z=\ell}$,
      we get that with probability at least $1-\delta_0$,
      Equation~\eqref{eq:evJ_2} holds for $f\in\cF$ and a particular $\ell\in[p]$.
      Applying the union bound over all $\ell\in [p]$, we get that
      \begin{align*}
        \Pr[\evJ] \geq 1 - 2p\delta_0.\yesnum\label{eq:prob_j}
      \end{align*}
      Suppose event $\evJ$ holds.
      Then, for any $\ell,k\in [p]$, we have
      \begin{align*}
        &\frac{ \Pr\nolimits_{\hD}\sinsquare{ \cE(f), \cE^\prime(f), \wh{Z}=\ell }}{\Pr\nolimits_{\hD}\sinsquare{ \cE^\prime(f), \wh{Z}=\ell }}
        \cdot \frac{\Pr\nolimits_{\hD}\sinsquare{ \cE^\prime(f), \wh{Z} = k }}{ \Pr\nolimits_{\hD}\sinsquare{ \cE(f), \cE^\prime(f), \wh{Z} = k }}\\
        &\quad \Stackrel{\eqref{eq:evJ_1},\eqref{eq:evJ_2}}{\geq}\qquad
        \frac{ \Pr\nolimits_{\cD}\sinsquare{\cE(f), \cE^\prime(f), Z=\ell } -\eta-\Delta}{\Pr\nolimits_{\cD}\sinsquare{\cE^\prime(f), Z=\ell }+\eta+\Delta}\cdot
        \frac{\Pr\nolimits_{\cD}\sinsquare{\cE^\prime(f), Z = k }-\eta-\Delta}{ \Pr\nolimits_{\cD}\sinsquare{\cE(f), \cE^\prime(f), Z = k } +\eta+\Delta}.\yesnum\label{eq:lowerbound}
      \end{align*}
      Using Equation~\eqref{eq:assumption_on_f}, we can lower bound the right-hand side of
      Equation~\eqref{eq:lowerbound} as follows
      \begin{align*}
        &\frac{ \Pr\nolimits_{\cD}\sinsquare{\cE(f), \cE^\prime(f), Z=\ell } -\eta-\Delta}{\Pr\nolimits_{\cD}\sinsquare{\cE^\prime(f), Z=\ell }+\eta+\Delta}\cdot
        \frac{\Pr\nolimits_{\cD}\sinsquare{\cE^\prime(f), Z = k }-\eta-\Delta}{ \Pr\nolimits_{\cD}\sinsquare{\cE(f), \cE^\prime(f), Z = k } +\eta+\Delta}\\
        &\qquad \Stackrel{\eqref{eq:assumption_on_f}}{\geq}\quad \frac{ \Pr\nolimits_{\cD}\sinsquare{\cE(f), \cE^\prime(f), Z=\ell } }{\Pr\nolimits_{\cD}\sinsquare{\cE^\prime(f), Z=\ell }}\cdot
        \frac{\Pr\nolimits_{\cD}\sinsquare{\cE^\prime(f), Z = k }}{ \Pr\nolimits_{\cD}\sinsquare{\cE(f), \cE^\prime(f), Z = k } }\cdot
        \inparen{\frac{1-\frac{1}\alpha\cdot(\eta+\Delta)}{1+\frac{1}\alpha\cdot(\eta+\Delta)}}^2.\yesnum\label{eq:lowerbound_2}
      \end{align*}
      Therefore, combining Equations \eqref{eq:lowerbound} and \eqref{eq:lowerbound_2}, we have that conditioned on event $\evJ$
      \begin{align*}
        \Omega(f,\hS)\quad & \Stackrel{}{=}\quad \min_{\ell,k\in[p]} \frac{ \Pr\nolimits_{\hD}\sinsquare{ \cE(f), \cE^\prime(f), \wh{Z}=\ell }}{\Pr\nolimits_{\hD}\sinsquare{ \cE^\prime(f), \wh{Z}=\ell }}
        \cdot \frac{\Pr\nolimits_{\hD}\sinsquare{ \cE^\prime(f), \wh{Z} = k }}{ \Pr\nolimits_{\hD}\sinsquare{ \cE(f), \cE^\prime(f), \wh{Z} = k }} \tag{Using the definition of the fairness metric}\\
        &\Stackrel{\eqref{eq:lowerbound},\eqref{eq:lowerbound_2}}{\geq}\quad  \inparen{\frac{1-\frac{1}\alpha\cdot(\eta+\Delta)}{1+\frac{1}\alpha\cdot(\eta+\Delta)}}^2 \cdot
        \min_{\ell,k\in[p]} \frac{ \Pr\nolimits_{\cD}\sinsquare{\cE(f), \cE^\prime(f), Z=\ell } }{\Pr\nolimits_{\cD}\sinsquare{\cE^\prime(f), Z=\ell }}\cdot
        \frac{\Pr\nolimits_{\cD}\sinsquare{\cE^\prime(f), Z = k }}{ \Pr\nolimits_{\cD}\sinsquare{\cE(f), \cE^\prime(f), Z = k } }\\
        &\Stackrel{}{=}\quad \Omega_{\cD}(f)\cdot \inparen{\frac{1-\frac{1}\alpha\cdot(\eta+\Delta)}{1+\frac{1}\alpha\cdot(\eta+\Delta)}}^2.\tag{Using the definition of the fairness metric}
      \end{align*}
      This completes the proof of the upper bound in \cref{def:stable_class:formal}.
      It remains to prove the lower bound in \cref{def:stable_class:formal}
      The proof of the lower bound in \cref{def:stable_class:formal} is analogous to the proof of the upper bound.
      For any $\ell,k\in [p]$, we have
      \begin{align*}
        &\frac{ \Pr\nolimits_{\hD}\sinsquare{ \cE(f), \cE^\prime(f), \wh{Z}=\ell }}{\Pr\nolimits_{\hD}\sinsquare{ \cE^\prime(f), \wh{Z}=\ell }}
        \cdot \frac{\Pr\nolimits_{\hD}\sinsquare{ \cE^\prime(f), \wh{Z} = k }}{ \Pr\nolimits_{\hD}\sinsquare{ \cE(f), \cE^\prime(f), \wh{Z} = k }}\\
        &\quad \Stackrel{\eqref{eq:evJ_1},\eqref{eq:evJ_2}}{\leq}\qquad
        \frac{ \Pr\nolimits_{\cD}\sinsquare{\cE(f), \cE^\prime(f), Z=\ell } +\eta+\Delta}{\Pr\nolimits_{\cD}\sinsquare{\cE^\prime(f), Z=\ell }-\eta-\Delta}\cdot
        \frac{\Pr\nolimits_{\cD}\sinsquare{\cE^\prime(f), Z = k }+\eta+\Delta}{ \Pr\nolimits_{\cD}\sinsquare{\cE(f), \cE^\prime(f), Z = k } -\eta-\Delta}.\yesnum\label{eq:upperbound}
      \end{align*}
      Using Equation~\eqref{eq:assumption_on_f}, we can upper bound the right-hand side of
      Equation~\eqref{eq:upperbound} as follows
      \begin{align*}
        &\frac{ \Pr\nolimits_{\cD}\sinsquare{\cE(f), \cE^\prime(f), Z=\ell } -\eta-\Delta}{\Pr\nolimits_{\cD}\sinsquare{\cE^\prime(f), Z=\ell }+\eta+\Delta}\cdot
        \frac{\Pr\nolimits_{\cD}\sinsquare{\cE^\prime(f), Z = k }-\eta-\Delta}{ \Pr\nolimits_{\cD}\sinsquare{\cE(f), \cE^\prime(f), Z = k } +\eta+\Delta}\\
        &\qquad \Stackrel{\eqref{eq:assumption_on_f}}{\leq}\quad \frac{ \Pr\nolimits_{\cD}\sinsquare{\cE(f), \cE^\prime(f), Z=\ell } }{\Pr\nolimits_{\cD}\sinsquare{\cE^\prime(f), Z=\ell }}\cdot
        \frac{\Pr\nolimits_{\cD}\sinsquare{\cE^\prime(f), Z = k }}{ \Pr\nolimits_{\cD}\sinsquare{\cE(f), \cE^\prime(f), Z = k } }\cdot
        \inparen{\frac{1+\frac{1}\alpha\cdot(\eta+\Delta)}{1-\frac{1}\alpha\cdot(\eta+\Delta)}}^2.\yesnum\label{eq:upperbound_2}
      \end{align*}
      Therefore, combining Equations \eqref{eq:lowerbound} and \eqref{eq:lowerbound_2}, we have that conditioned on event $\evJ$
      \begin{align*}
        \Omega_\cD(f)\quad & \Stackrel{}{=}\quad \min_{\ell,k\in[p]} \frac{ \Pr\nolimits_{\cD}\sinsquare{\cE(f), \cE^\prime(f), Z=\ell } }{\Pr\nolimits_{\cD}\sinsquare{\cE^\prime(f), Z=\ell }}\cdot
        \frac{\Pr\nolimits_{\cD}\sinsquare{\cE^\prime(f), Z = k }}{ \Pr\nolimits_{\cD}\sinsquare{\cE(f), \cE^\prime(f), Z = k } } \tag{Using the definition of the fairness metric}\\
        &\Stackrel{\eqref{eq:lowerbound},\eqref{eq:lowerbound_2}}{\geq}\quad  \inparen{\frac{1-\frac{1}\alpha\cdot(\eta+\Delta)}{1+\frac{1}\alpha\cdot(\eta+\Delta)}}^2 \cdot
        \min_{\ell,k\in[p]} \frac{ \Pr\nolimits_{\hD}\sinsquare{ \cE(f), \cE^\prime(f), \wh{Z}=\ell }}{\Pr\nolimits_{\hD}\sinsquare{ \cE^\prime(f), \wh{Z}=\ell }}
        \frac{\Pr\nolimits_{\hD}\sinsquare{ \cE^\prime(f), \wh{Z} = k }}{ \Pr\nolimits_{\hD}\sinsquare{ \cE(f), \cE^\prime(f), \wh{Z} = k }}
      \end{align*}
      \begin{align*}
        &\Stackrel{}{=}\quad \Omega_\cD(f,\hS)\cdot \inparen{\frac{1-\frac{1}\alpha\cdot(\eta+\Delta)}{1+\frac{1}\alpha\cdot(\eta+\Delta)}}^2. \tag{Using the definition of the fairness metric}
      \end{align*}
    \end{proof}
    \begin{corollary}\label{lem:suff_cond_for_stable_app}
      For all $\alpha\in (0,1)$, any classifier $f\in\cF$ satisfying
      \begin{align*}
        \min_{\ell\in [p]}\Pr\nolimits_{\cD}\insquare{\cE(f), \cE^\prime(f), \hZ=\ell}\geq \alpha,\yesnum\label{eq:boundboudn}
      \end{align*}
      is $\sinparen{\frac{1- (\eta+\Delta)/(\alpha-\eta-\Delta)}  {  1 + (\eta+\Delta)/(\alpha-\eta-\Delta)}}^2$-stable with respect to the fairness metric $\Omega$ defined by events $\cE$ and $\cE^\prime$.
    \end{corollary}
    \begin{proof}
      Suppose event $\evJ$ defined in \cref{lem:bound_on_fc} occurs.
      Then by the first condition of $\evJ$, \cref{eq:evJ_1}, we have that
      $$\abs{\Pr\nolimits_{\cD}\insquare{\cE(f), \cE^\prime(f), Z=\ell} - \Pr\nolimits_{\hD}\insquare{\cE(f), \cE^\prime(f), \hZ=\ell}}\leq \eta+\Delta.$$
      Combining this with \cref{eq:boundboudn}, we get that $$\Pr\nolimits_{\hD}\insquare{\cE(f), \cE^\prime(f), Z=\ell}\geq \alpha-\eta-\Delta.$$
      Then repeating the proof of \cref{lem:bound_on_fc} we get that, conditioned on $\evJ$, the following holds
      \begin{align*}
        \inparen{\frac{  1 + (\eta+\Delta)/(\alpha-\eta-\Delta)}{1- (\eta+\Delta)/(\alpha-\eta-\Delta)}}^2
        \leq \frac{\Omega_{\cD}(f)}{\Omega(f, \wh{S})}
        \leq \inparen{\frac{1- (\eta+\Delta)/(\alpha-\eta-\Delta)}  {  1 + (\eta+\Delta)/(\alpha-\eta-\Delta)}}^2
      \end{align*}
      Since $\evJ$ occurs with probability at least $1-2\delta_0>1-\delta$, we get that $f$ is $\sinparen{\frac{1- (\eta+\Delta)/(\alpha-\eta-\Delta)}  {  1 + (\eta+\Delta)/(\alpha-\eta-\Delta)}}^2$-stable.
    \end{proof}
    \noindent Setting $\alpha$ as $\lambda+\eta+\Delta$, we recover \cref{lem:suff_cond_for_stable} from \cref{lem:suff_cond_for_stable_app}.
    Now, we are ready to prove the main result in this step: \cref{lem:approx_feasible}.
    \begin{lemma}[\bf Any feasible solution of \errtolerant{} is approximately fair]\label{lem:approx_feasible}
      For all $\Delta,\delta_0 \in (0,1)$,
      given $N \geq \Theta\big(\Delta^{-2}\cdot\big({\rm VC}(\cF)\cdot\log\inparen{\nfrac{{\rm VC}(\cF)}{\Delta}}+\log\inparen{\nfrac{1}{\delta_0}}\big)$
      iid samples $S$ from $\cD$,
      and corresponding perturbed samples $A(S) \coloneqq \{(x_i,y_i,\hz_i)\}_{i\in [N]}$, with probability at least $1-2p\delta_0$,
      any $f\in \cF$ feasible for \errtolerant{} satisfies
      \begin{align*}
        \Omega_\cD(f) \geq \tau- \frac{8\eta\tau}{\lambda - 2\eta}- \Delta_0.
      \end{align*}
    \end{lemma}
    \begin{proof}
      Consider any classifier $f\in \cF$ that is feasible for \errtolerant{}.
      $f$ must satisfy
      \begin{align*}
        \Omega(f,\hS) &\geq \tau \cdot \inparen{\frac{1- (\eta+\Delta)/\lambda}  {  1 + (\eta+\Delta)/\lambda}}^2,\yesnum\label{eq:bound_ovf_1}\\
        \forall\ \ell\in[p],\quad \Pr\nolimits_{\hD}\insquare{ \cE(f), \cE^\prime(f), \hZ=\ell} &\geq \lambda - \eta - \Delta.
        \yesnum\label{eq:bound_ovf_2}
      \end{align*}
      Since $f$ satisfies \cref{eq:bound_ovf_2}, using \cref{lem:suff_cond_for_stable}, we get that $f$ is $\sinparen{\frac{1- (\eta+\Delta)/(\lambda-2\eta-2\Delta)}  {  1 + (\eta+\Delta)/(\lambda-2\eta-2\Delta)}}^2$-stable.
      Thus, with high probability, it holds that
      \begin{align*}
        \Omega_\cD(f)
        \quad &\geq\quad \Omega(f, \hS)\cdot
        \inparen{\frac{1- (\eta+\Delta)/(\lambda-2\eta-2\Delta)}  {  1 + (\eta+\Delta)/(\lambda-2\eta-2\Delta)}}^2\yesnum\label{eq:used_to_prove_guarantee_for_extension}\\
        &\Stackrel{\eqref{eq:bound_ovf_1}}{\geq}\quad \tau
        \cdot \inparen{\frac{1- (\eta+\Delta)/\lambda}  {  1 + (\eta+\Delta)/\lambda}}^2
        \cdot \inparen{\frac{1- (\eta+\Delta)/(\lambda-2\eta-2\Delta)}  {  1 + (\eta+\Delta)/(\lambda-2\eta-2\Delta)}}^2.
      \end{align*}
      We can lower bound the right-hand side of the above equation using \cref{fact:algebra}.
      \begin{fact}\label{fact:algebra}
        For all $\alpha\in (0,1]$ and $\eta,\Delta\in[0,1]$, $\inparen{\frac{1-\frac{1}\alpha\cdot(\eta+\Delta)}{1+\frac{1}\alpha\cdot(\eta+\Delta)}}^2\geq \inparen{1-\frac{4}\alpha \cdot(\eta+\Delta)}.$
      \end{fact}
      \noindent Using \cref{fact:algebra} twice we get that
      \begin{align*}
        \Omega_\cD(f)\ \
        &\Stackrel{}{\geq}\ \ \tau
        \cdot \inparen{{1- \frac{4(\eta+\Delta)}{\lambda}} }
        \cdot \inparen{{1- \frac{4(\eta+\Delta)}{\lambda-2\eta-2\Delta}}}\\
        &\Stackrel{}{>}\ \ \tau\cdot \inparen{1-\frac{4(\eta+\Delta)}{\lambda}  -\frac{4(\eta+\Delta)}{\lambda - 2\eta - 2\Delta}} \tag{Using that $\lambda> 2\eta+2\Delta$ and $\lambda,\eta,\Delta>0$}\\
        &\Stackrel{}{\geq}\ \ \tau\cdot \inparen{1 - \frac{8(\eta+\Delta)}{\lambda - 2\eta - 2\Delta}}\tag{Using that $\lambda> 2\eta+2\Delta$ and $\lambda,\eta,\Delta>0$}\\
        &\Stackrel{}{=}\ \ \tau\cdot \inparen{1 - \frac{8(\eta+\Delta)}{\lambda - 2\eta}\cdot  \frac{1}{1-\frac{2}{(\lambda - 2\eta)}\cdot\Delta}}\\
        &\Stackrel{}{\geq}\ \ \tau\cdot \inparen{1 - \frac{8(\eta+\Delta)}{\lambda - 2\eta}\cdot \inparen{1+\frac{4}{(\lambda - 2\eta)}\cdot\Delta}}
        \tag{Using that $\frac{2}{(\lambda - 2\eta)}\cdot\Delta\in [0,\frac12]$}\\
        &\Stackrel{}{=}\ \ \tau\cdot \inparen{1 - \frac{8\eta}{\lambda - 2\eta} - \frac{8\Delta}{\lambda - 2\eta} - \frac{32\eta\Delta}{(\lambda - 2\eta)^2}-\frac{32\Delta^2}{(\lambda - 2\eta)^2} }\\
        &\Stackrel{\eqref{eq:define_params}}{=}\ \ \tau\cdot \inparen{1 - \frac{8\eta}{\lambda - 2\eta} - \frac{\Delta_0}3 - \frac{\Delta_0}3 - \frac{\Delta_0}3}.
      \end{align*}
      Finally, from the discussion in the proof of \cref{lem:bound_on_fc}, it follows that for our choice of $N$, the above equation holds with probability at least $1-2p\delta_0$.
    \end{proof}

    \subsubsection{{Step 3: \texorpdfstring{$f^\star$}{f*} Is Feasible for \texorpdfstring{\errtolerant{}}{Program (ErrTolerant)} with High Probability}}
    In this step, we conclude the proof of \cref{thm:main_result}.
    It remains show that $f^\star$ is feasible for \errtolerant{}:
    The fairness guarantee follows from \cref{lem:approx_feasible}, and if $f^\star$ is feasible for \errtolerant{}, then the accuracy guarantee follows from \cref{lem:conditionalAccuracyGuarantee}.
    \begin{lemma}[\bf Structure of the optimal fair classifier]\label{coro:optf_feasible}
      If \cref{asmp:1} holds, then
      For all $\Delta,\delta_0 \in (0,1)$,
      given $N \geq \Theta\big(\Delta^{-2}\cdot\big({\rm VC}(\cF)\cdot\log\inparen{\nfrac{{\rm VC}(\cF)}{\Delta}}+\log\inparen{\nfrac{1}{\delta_0}}\big)$
      iid samples $S$ from $\cD$,
      and corresponding perturbed samples $A(S) \coloneqq \{(x_i,y_i,\hz_i)\}_{i\in [N]}$, with probability at least $1-\delta_0$,
      $f^\star$ is feasible for \errtolerant{}.
    \end{lemma}
    \begin{proof}
      Let $\evG$ be the event that for all $\ell\in[p]$ and $f\in\cF$
      \begin{align*}
        \abs{{\rm Err}_{\hD}(f)  - {\rm Err}_{\cD}(f)} &< \eta+\Delta.\yesnum\label{eq:evG_3}
      \end{align*}
      Using \cref{coro:whp_bounding_power} with $g(\wt{y},y,z)\coloneqq \mathds{I}\insquare{\wt{y}\neq y}$,
      shows that with probability at least $1-\delta_0$ Inequality~\eqref{eq:evG_3} holds for all $f\in\cF$.
      Thus,
      \begin{align*}
        \Pr[\evG] \geq 1 - \delta_0.\yesnum\label{eq:prob_g}
      \end{align*}
      Suppose events $\evG$ and $\evJ$ hold.
      To show that $f^\star$ is feasible for \errtolerant{}, we have to show that
      \begin{align*}
        \Omega(f^\star,\hS) &\geq \tau \cdot \inparen{1-\frac{4\eta}{\lambda} - \Delta_0},\yesnum\label{eq:bound_2}\\
        \forall\ \ell\in[p],\quad \Pr\nolimits_{\hD}\sinsquare{ \cE(f^\star), \cE^\prime(f^\star), \hZ=\ell} &\geq \lambda - \eta - \Delta_0.\yesnum\label{eq:bound_3}
      \end{align*}
      Since event $\evJ$ and \cref{asmp:1} hold, we can apply \cref{lem:bound_on_fc} to get that $\Omega(f^\star,\hS)\geq \Omega_\cD(f^\star) \cdot \sinparen{1-\frac{4\eta}{\lambda} - \Delta_0}.$
      Then, Equation~\eqref{eq:bound_2} holds since $\Omega_\cD(f^\star)\geq \tau$.
      Finally, Equation~\eqref{eq:bound_3} follows by using \cref{asmp:1} and Equation~\eqref{eq:evJ_1} in the definition of event $\evJ$.
    \end{proof}

    \begin{proof}[Proof of \cref{thm:main_result}]
      Let $\evJ$ be the event that the fairness guarantee in \cref{lem:approx_feasible} holds and $\evG$ be the event that $f^\star$ is feasible for \errtolerant{}.
      Using the union bound $\evJ$ and $\evG$, we get that
      \begin{align*}
        \Pr\nolimits_\cD[\evJ \land \evG] &\geq 1-(2p+1)\cdot \delta_0\tag{Using \cref{lem:approx_feasible} and \cref{coro:optf_feasible}}\\
        &\Stackrel{\eqref{eq:define_params}}{\geq} 1-\delta.
      \end{align*}
      Suppose events $\evJ$ and $\evG$ occur.
      By \cref{lem:approx_feasible}, we know that
      $$\Omega_\cD(f) \geq \tau- \frac{8\eta\tau}{\lambda - 2\eta}- \Delta.$$
      Since $$\Delta\leq \Delta_0\leq \frac{8\eta}{\lambda-2\eta}-\frac\nu{\tau},$$
      it follows that
      $$\Omega_\cD(f) \geq \tau- \nu.$$
      \noindent By \cref{coro:optf_feasible}, we know that $f^\star$ is feasible for \errtolerant{}.
      From \cref{lem:conditionalAccuracyGuarantee} it follows that $${{\rm Err}_\cD(f_{\rm ET})-{\rm Err}_\cD(f^\star)}\leq 2\eta+\Delta.$$
      Since $\Delta\leq \eps-2\eta,$
      it follows that
      \begin{align*}
        {{\rm Err}_\cD(f_{\rm ET})-{\rm Err}_\cD(f^\star)}\leq \eps.
      \end{align*}
    \end{proof}

    \subsubsection{{Generalization of \texorpdfstring{\cref{thm:main_result}}{Theorem 4.3}  to the Nasty Sample Noise Model}}\label{sec:main_result:for_stronger_model}
    In this section, we generalize \cref{thm:main_result} to the Nasty sample noise model.
    Precisely, we show that if \cref{asmp:1} holds with constant $\lambda>0$ and $\cF$ has VC dimension $d\in\N$,
    then for all perturbation rates $\eta\in (0,\nfrac\lambda{2})$,
    fairness thresholds $\tau\in (0,1]$,
    bounds on error $\eps> 2\eta$ and constraint violation $\nu >  \nfrac{8\eta\tau}{(\lambda-2\eta)}$,
    and confidence parameters $\delta\in (0,1)$,
    given sufficiently many perturbed samples from the $\eta$-\Ham{} model,
    with probability at least $1-\delta$, it holds that
    \begin{align*}
      {\rm Err}_\cD(f_{\rm ET}) - {\rm Err}_\cD(f^\star) \leq \eps
      \text{ and }
      \Omega_{\cD}(f_{\rm ET}) &\geq \tau -\nu,
    \end{align*}
    where $f_{\rm ET}$ is an optimal solution of \errtolerant{}.

    The proof of the above generalization is almost identical to the proof of \cref{thm:main_result}.
    Instead of repeating the entire proof, we highlight the changes required in the proof of \cref{thm:main_result}.
    The generalization requires two changes: (1) Proving an analogue of \cref{coro:whp_bounding_power} for the Nasty Sample Noise model, (2) generalizing the definition of $s$-stability (\cref{def:stable_class:formal}) to the Nasty Sample Noise model.
    These changes are sufficient because all other arguments use the guarantees of \cref{coro:whp_bounding_power} or $s$-stability, without using any properties of the $\eta$-\Ham{} model.
    Updating \cref{def:stable_class:formal} only requires changing the perturbation model from $\eta$-\Ham{} to $\eta$-Nasty Sample Noise; we omit the formal statement.
    Next, we prove \cref{coro:whp_bounding_power:nsn} which is the required analogue of \cref{coro:whp_bounding_power}.
    \begin{lemma}[\bf Bound on difference in means of bounded functions on $\cD$ and on $\hS$]\label{coro:whp_bounding_power:nsn}
      For any bounded function $\ell\colon \zo\times\zo\times[p]\to [0,1]$, constants $\Delta,\delta_0 \in (0,1)$, and adversaries $A$ admissible under the $\eta$-Nasty Sample Noise model,
      given $N \geq \Theta\big(\Delta^{-2}\cdot\big({\rm VC}(\cF)\cdot\log\inparen{\nfrac{{\rm VC}(\cF)}{\Delta}}+\log\inparen{\nfrac{1}{\delta_0}}\big)\big)$
      samples $S$ iid from $\cD$,
      and corresponding \mbox{perturbed samples $A(S) \coloneqq \{(\hx_i,\hy_i,\hz_i)\}_{i\in [N]}$, with probability at least $1-\delta_0$, it holds that}
      \begin{align*}
        \forall\ f\in\cF, \quad \abs{\Ex_{(\hX,\hY,\hZ)\sim \hD} g(f(\hX,\hZ), \hY, \hZ) - \Ex\nolimits_{(X,Y,Z)\sim\cD}\insquare{g(f(X,Z), Y, Z)}} \leq \Delta+\eta,
      \end{align*}
      where $\hD$ is the empirical distribution of $A(S)$.
    \end{lemma}
    \begin{proof}
      Let $S\coloneqq \inbrace{(x_i,z_i,y_i)}_{i\in [N]}$.
      Using the triangle inequality for absolute value, we have
      \begin{align*}
        &\abs{\Ex\nolimits_{(\hX,\hY,\hZ)\sim \hD} \sinsquare{g(f(\hX,\hZ), \hY, \hZ)} - \Ex\nolimits_{(X,Y,Z)\sim\cD}\sinsquare{g(f(X,Z), Y, Z)}}\\
        &\qquad\qquad\qquad  \leq \abs{\Ex\nolimits_{(X,Y,Z)\sim D} \sinsquare{g(f(X,Z), Y, Z)} - \Ex\nolimits_{(X,Y,Z)\sim\cD}\sinsquare{g(f(X,Z), Y, Z)}}\\
        &\qquad\qquad\qquad + \hspace{0.45mm} \abs{\Ex\nolimits_{(\hX,\hY,\hZ)\sim \hD} \sinsquare{g(f(\hX,\hZ), \hY, \hZ)} - \Ex\nolimits_{(X,Y,Z)\sim D} \sinsquare{g(f(X,Z), Y, Z)}}.\yesnum
        \label{eq:eq_1_in_whp_bounding_power:nsn}
      \end{align*}
      We can upper bound the first term in the right-hand side using \cref{lem:conc}; this is identical to the argument under the $\eta$-\Ham{} model.
      In particular, we have that with probability at least $1-\delta$, for all $f\in \cF$, it holds that
      \begin{align*}
        \abs{\Ex\nolimits_{(X,Y,Z)\sim D} \sinsquare{g(f(X,Z), Y, Z)} - \Ex\nolimits_{(X,Y,Z)\sim\cD}\insquare{g(f(X,Z), Y, Z)}}\leq \Delta.\yesnum\label{eq:eq_2_in_whp_bounding_power:nsn}
      \end{align*}
      (The proof of the upper bound on the second term in the right-hand side of Equation~\eqref{eq:eq_1_in_whp_bounding_power:nsn} is slightly different from the proof under the $\eta$-\Ham{} model.)
      For all $f\in \cF$, it holds that
      \begin{align*}
        &\abs{\Ex\nolimits_{(\hX,\hY,\hZ)\sim \hD} \sinsquare{g(f(\hX,\hZ), \hY, \hZ)} - \Ex\nolimits_{(X,Y,Z)\sim D} \sinsquare{g(f(X,Z), Y, Z)}}\\
        &\quad = \quad \frac1N\abs{\sum\nolimits_{i\in[N]} g(f(\hx_i, \hz_i), \hy_i,\hz_i) -g(f(x_i, z_i), y_i, z_i)}\\
        &\quad =\quad  \frac1N\abs{\sum\nolimits_{i\in[N]\colon (x_i,y_i,z_i)\neq (\hx_i,\hy_i,\hz_i)} g(f(\hx_i, \hz_i), \hy_i,\hz_i) - g(f(x_i, z_i), y_i, z_i)}
        \tag{For all $i\in [N]$, where $(x_i,y_i,z_i)=(\hx_i,\hy_i,\hz_i)$, $g(f(\hx_i, \hz_i), \hy_i,\hz_i) = g(f(x_i, z_i), y_i, z_i).$}\\
        &\quad \leq\quad  \frac1N\abs{\sum\nolimits_{i\in[N]\colon (x_i,y_i,z_i)\neq (\hx_i,\hy_i,\hz_i)} 1}\tag{Using that $g$ is bounded by $0$ and 1}\\
        &\quad \leq\quad  \eta.\yesnum\label{eq:eq_3_in_whp_bounding_power:nsn}
      \end{align*}
      Since with probability at least $1-\delta$, both Equations~\eqref{eq:eq_2_in_whp_bounding_power:nsn} and \eqref{eq:eq_3_in_whp_bounding_power:nsn} hold for all $f\in \cF$, substituting them in Equation~\eqref{eq:eq_1_in_whp_bounding_power:nsn} gives us the required bound.
    \end{proof}
    \noindent One can substitute \cref{coro:whp_bounding_power} by \cref{coro:whp_bounding_power:nsn} and repeat the proof of \cref{thm:main_result} for the $\eta$-Nasty Sample Noise model.
    \subsection{{Proof of \texorpdfstring{\cref{thm:no_stable_classifier}}{Theorem 4.4}}}\label{sec:proofof:thm:no_stable_classifier}
    \cref{thm:no_stable_classifier} assumes that $\cF$ shatters the set $\inbrace{x_A,x_B,x_C}\times [2]\subseteq \cX\times [p]$, the fairness threshold $\tau\in (\nfrac12,1)$, and statistical rate is the fairness metric; recall that the statistical rate of a classifier $f\in \cF$ is
    $$\Omega_\cD(f)\coloneqq \frac{ \min_{\ell\in[p]} \Pr_\cD[f=1\mid Z=\ell]}{\max_{\ell\in[p]} \Pr_\cD[f=1\mid Z=\ell]}.$$
    {Then, given parameters $\tau\in (\nfrac12,1)$ and $\delta\in [0,\nfrac12)$, our goal is to show that for any
    \begin{align*}
      \text{$\eps\in \left[0,\frac12\right)$ and $\nu\in \left[0,\tau-\frac12\right)$,}
    \end{align*}
    $\cF$ is not ($\eps,\nu$)-learnable with perturbation rate $\eta$ and confidence $\delta$.
    We prove a more general result:
    {Given parameters $\tau\in [0,1)$, $c\in (0,\min\sinbrace{\tau,
    \nfrac12})$, and $\delta\in [0,\nfrac12)$, we show that for any
    \begin{align*}
      \text{$\eps\in [0,c)$ and $\nu\in [0,\tau-c)$,}
    \end{align*}
    $\cF$ is not ($\eps,\nu$)-learnable with perturbation rate $\eta$ and confidence $\delta$.}
    {When $\tau>\nfrac12$, the original result follows by taking the limit as $c$ approaches $\nfrac12$.}

    \subsubsection{Proof of \texorpdfstring{\cref{thm:no_stable_classifier}}{Theorem 4.4}}
    \begin{proof}[Proof of \cref{thm:no_stable_classifier} (assuming \cref{lem:1,lem:2,lem:3})]
      Let $\cL$ be any learner.
      Set
      \begin{align*}
        \alpha &\coloneqq \min\inbrace{\frac{\eta}2, 1-\tau, \tau-c-\nu, c-\eps},\yesnum\label{eq:definition_alpha}
      \end{align*}
      and confidence parameter $\delta\in (0,\nfrac12)$ (in \cref{def:learning_model}).
      We construct three distributions $\cD_1,\cD_2,$ and $\cD_3$, parameterized by $\alpha$,
      that satisfy \cref{lem:1}. (See \cref{fig:1} for details of the distributions).
      \begin{lemma}[\bf Adversary can hide the true distribution]\label{lem:1}
        For all $\delta_{0}\in (0,1)$, $\alpha \in (0,\eta)$, and $\ell,k\in [3]$, given
        $$N\geq 3\cdot \ln \inparen{\frac1{\delta_{0}}} \cdot (\eta-\alpha)^{-2}$$
        iid samples, $S$, from $\cD_\ell$
        there is an adversary $A\in\cA(\eta)$ such that
        with probability $1-\delta_{0}$ (over the draw of $S$) the perturbed samples $\hS\coloneqq A(S)$ are distributed as iid draws from $\cD_k$.
      \end{lemma}
      \noindent Suppose $\cL$ is given $N$ samples, where
      \begin{align*}
        N &\geq 3\cdot \ln{2} \cdot (\eta-\alpha)^{-2}.\yesnum\label{eq:definition_N_theorem_4_4}
      \end{align*}
      Consider three cases, where in the $k$-th case ($k\in [3]$) the samples in $S$ are iid from $\cD_k$. %
      By \cref{lem:1}, in each case there is an $A\in \cA(\eta)$, such that with probability at least $\nfrac12$, $\hS\coloneqq A(S)$ have the distribution $\cD_1$.

      Thus, given $\hS$, with probability at least $\nfrac12$, $\cL$ cannot identify the distribution from which $S$ was drawn. %
      As a result, with probability at least $\nfrac12$, $\cL$ outputs the same classifier, say $f_{\rm Com}\in \cF$, in all three cases.
      We show that no $f_{\rm Com}\in \cF$  satisfies the accuracy and fairness guarantee in all three cases.
      \begin{lemma}[\bf No good classifier for all cases]\label{lem:2}
        There is no classifier $f\in \cF$ such that
        for all $k\in [3]$,
        \begin{align*}
          {\rm Err}_{\cD_k}(f) < c\cdot(1-\alpha)\quad \text{and} \quad\Omega_{\cD_k}(f) \geq  c+\alpha.
        \end{align*}
      \end{lemma}
      \begin{lemma}[\bf A good classifier for each case]\label{lem:3}
        For each $k\in[3]$, there is an $h_k\in \cF$ such that
        \begin{align*}
          {\rm Err}_{\cD_k}(h_k) < \frac{c\alpha}2\quad \text{and} \quad\Omega_{\cD_k}(h_k) > 1-\alpha.
        \end{align*}
      \end{lemma}
      \noindent Note that for each $k\in [3]$, $f_k^\star$ satisfies the fairness constraint because $\tau<1-\alpha$ (\cref{eq:definition_alpha}).
      Thus, the optimal fair classifier $f_k^\star$ for $\cD_k$ subject to having a statistical rate $\tau$ must satisfy
      \begin{align*}
        {\rm Err}_{\cD_k}\sinparen{f^\star_k} &< \frac{c\alpha}2.\yesnum\label{eq:boundONOPTClassifier}
      \end{align*}
      (Otherwise, we have a contradiction as $h_k$ satisfies the fairness constraints and has a smaller error than $f^\star_k$.)
      If $\cL$ is a ($\eps,\nu$)-learner, then in the $k$-th case, with probability at least $1-\delta>\nfrac12$,
      $\cL$ must output a classifier $f_k$ which satisfies
      \begin{align*}
        {\rm Err}_{\cD_k}(f_k) - {\rm Err}_{\cD_k}(f_k^\star) &\leq \eps \text{ and }\tau - \Omega_{\cD_k}(f_k) \leq\nu.
      \end{align*}
      But in all cases with probability at least $\nfrac12$, $\cL$ outputs $f_{\rm Com}$.
      Because $\nfrac12>\delta$, $f_{\rm Com}$ must satisfy:
      \begin{align*}
        \text{For all $k\in [3]$,}\quad
        {\rm Err}_{\cD_k}(f_{\rm Com}) - {\rm Err}_{\cD_k}(f_k^\star) &\leq \eps \text{ and }\tau - \Omega_{\cD_k}(f_{\rm Com}) \leq\nu.\yesnum\label{eq:lem:learnability}
      \end{align*}
      But from \cref{lem:2} we know that for each $k\in [3]$ either
      \begin{align*}
        {\rm Err}_{\cD_k}\sinparen{f_{\rm Com}} &\geq c\cdot(1-\alpha) \quad\text{or}\quad \Omega_{\cD_k}(f_{\rm Com}) < c+\alpha.\yesnum\label{eq:lem:1}
      \end{align*}
      \noindent{\bf (Case A) ${\rm Err}_{\cD_k}\sinparen{f_{\rm Com}} \geq c\cdot(1-\alpha)$:} In this case, from \cref{eq:lem:learnability}, we have
      \begin{align*}
        \eps &\geq c\cdot(1-\alpha)- {\rm Err}_{\cD_k}(f_k^\star)\\
        &\Stackrel{\eqref{eq:boundONOPTClassifier}}{>} c\cdot(1-\alpha)- \frac{c\alpha}2\\
        &\Stackrel{}{>} c-\alpha \tagnum{Using that $c<\nfrac12$ and  $\alpha>0$}\customlabel{eq:contradiction_1}{\theequation}
      \end{align*}
      But because $\alpha\leq c-\eps$, Equation~\eqref{eq:contradiction_1} cannot hold.

      \noindent{\bf (Case B) $\Omega_{\cD_k}(f_{\rm Com}) < c+\alpha$:} In this case, from \cref{eq:lem:learnability}, we have
      \begin{align*}
        \nu &> \tau - c - \alpha.\yesnum\label{eq:contradiction_2}%
      \end{align*}
      But because $\alpha\leq \tau - c -\nu$, Equation~\eqref{eq:contradiction_2} does not hold.

      Therefore, we have a contradiction.
      Hence, $\cL$ is not an ($\eps,\nu$)-learner for $\cF$.
      Since the choice of $\cL$ was arbitrary, we have shown that there is no learner which ($\eps,\nu$)-learns $\cF$.
      It remains to prove \cref{lem:1,lem:2,lem:3}.
    \end{proof}

    \subsubsection{Proof of \texorpdfstring{\cref{lem:1}}{Lemma 6.12}}\label{sec:proofof:extension_of_high_prop_result}
    Set $\cD_1$, $\cD_2$, and $\cD_3$ to be unique distributions with marginal distribution specified in \cref{fig:1}, such that, for any draw $(X,Y,Z)\sim \cD_k$ ($k\in [3]$) $Y$ takes the value $\mathds{I}\insquare{X=x_A}$, i.e.,
    \begin{align*}
      Y  = \begin{cases}
      1 & \text{if } X=x_A,\\
      0 & \text{otherwise}.
    \end{cases}\yesnum\label{eq:deterministic_y_lemma_C_1}
  \end{align*}

  \noindent In particular, the construction of $\cD_1$, $\cD_2$, and $\cD_3$ ensures that the total variation distance between any pair of distributions is less than $\alpha$.
  \begin{table}[t]
    \centering
    %
    %
    \subfigure[$\cD_1$: $\Pr\nolimits_{\cD_1}\insquare{(X,Z)=(r,s)}$ for $(r,s)\in\inbrace{x_A,x_B,x_C}\times \insquare{2}$.\white{$\bigg|$}]{
    \begin{tabular}{c>{\centering\arraybackslash}p{3cm}>{\centering\arraybackslash}p{3cm}>{\centering\arraybackslash}p{3cm}}
      \toprule
      \midsepremove{}
      & $x_A$ & $x_B$ & $x_C$\\
      \midrule{}
      $1$ & $c(1-\alpha)$ & $(1-c) (1-\alpha)$ & $\sfrac\alpha2$\\
      $2$ & $\sfrac{c\alpha}{2}$ & $\sfrac{\alpha(1-c)}{2}$ & $0$\\
      \bottomrule
    \end{tabular}
    }
    %
    \subfigure[$\cD_2$: $\Pr\nolimits_{\cD_2}\insquare{(X,Z)=(r,s)}$ for $(r,s)\in\inbrace{x_A,x_B,x_C}\times \insquare{2}$.\white{$\bigg|$}]{
    \begin{tabular}{c>{\centering\arraybackslash}p{3cm}>{\centering\arraybackslash}p{3cm}>{\centering\arraybackslash}p{3cm}}
      \toprule
      \midsepremove{}
      & $x_A$ & $x_B$ & $x_C$\\
      \midrule{}
      $1$ & $c(1-\alpha)$ & $(1-c)(1-\sfrac\alpha2)$ & $\sfrac{c\alpha}{2}$\\
      $2$ & $\sfrac{c\alpha}{2}$ & $0$ & $\sfrac{\alpha(1-c)}{2}$\\
      \bottomrule
    \end{tabular}
    }
    \subfigure[$\cD_3$: $\Pr\nolimits_{\cD_3}\insquare{(X,Z)=(r,s)}$ for $(r,s)\in\inbrace{x_A,x_B,x_C}\times \insquare{2}$.\white{$\bigg|$}]{
    \begin{tabular}{c>{\centering\arraybackslash}p{3cm}>{\centering\arraybackslash}p{3cm}>{\centering\arraybackslash}p{3cm}}
      \toprule
      \midsepremove{}
      & $x_A$ & $x_B$ & $x_C$\\
      \midrule
      $1$ & $c(1-\sfrac\alpha2)$ & $(1-c)(1-\alpha)$ & $\sfrac{\alpha(1-c)}{2}$\\
      $2$ & 0 & $\sfrac{\alpha(1-c)}{2}$ & $\sfrac{c\alpha}{2}$\\
      \bottomrule
    \end{tabular}
    }
    %
    \caption{
    \ifconf\small\fi
    Marginal distributions of $\cD_1$, $\cD_2$, $\cD_3$ over $\cX\times [2]$.
    Recall that for a sample $(X,Y,Z)\sim \cD_k$ ($k\in [3]$), $Y$ takes the value $\mathds{I}[X=x_A]$.
    }
    \label{fig:1}
  \end{table}
  In this section, we prove the following generalization of \cref{lem:1}.
  \newcommand{\cN}{\mathcal{N}}
  \newcommand{\cQ}{\mathcal{Q}}
  \begin{proposition}[\bf Adversary can hide the true distribution]\label{prop:to_lem_1}
    For all $\delta,\eta \in (0,1)$, $\alpha \in (0,\eta)$, and
    two distributions $\cP$ and $\cQ$ over $\cX\times \zo\times [p]$ that satisfy the following conditions:
    \begin{enumerate}[itemsep=\itemsepINTERNAL,leftmargin=\leftmarginINTERNAL]
      \item[(C1)] ${\rm TV}(\cP,\cQ)=\alpha$,
      \item[(C2)] $\cP$ and $\cQ$ have the same marginal on $\cX$, i.e., for all $T\subseteq \cX$, $$\Pr\nolimits_{(X,Y,Z)\sim \cP}[X\in T]=\Pr\nolimits_{(X,Y,Z)\sim \cQ}[X\in T],$$
      \item[(C3)] for a random sample $(X,Y,Z)$ drawn from $\cP$, the label $Y$ is independent of $Z$ conditioned on $X$, i.e., $Y\bot Z \mid X,$ similarly for a random sample $(X,Y,Z)$ drawn from $\cQ$, the label $Y$ is independent of $Z$ conditioned on $X$.
    \end{enumerate}
    Then, there is an adversary $A\in \cA(\eta)$, that given $N$
    iid samples $S$ from $\cP$, where
    $$N\geq 3\cdot \ln \inparen{\frac1\delta} \cdot (\eta-\alpha)^{-2},$$
    outputs a perturbed samples $\hS\coloneqq A(S)$ such that
    with probability $1-\delta$ (over the draw of $S$) samples in $\hS$ are distributed as iid draws from $\cQ$.
  \end{proposition}
  \noindent Note that \cref{lem:1} follows by substituting $\cP$ and $\cQ$ by $\cD_\ell$ and $\cD_k$ respectively.
  \begin{proof}[Proof of \cref{prop:to_lem_1}]
    Let $A\in \cA(\eta)$ use the following algorithm:\\

    \begin{mdframed}
      \begin{enumerate}[itemsep=1pt]
        \item {\bf For} $r,s\in \cX$ {\bf do:} {\bf Set} $p(r,s)\coloneqq \min\inbrace{ \frac{\Pr\nolimits_{(X,Y,Z)\sim \cQ}\sinsquare{(X,Z)=(r,s)}}{\Pr\nolimits_{(X,Y,Z)\sim \cP}\sinsquare{(X,Z)=(r,s)}}, 1 }$
        \white{.............................} \white{.} \hfill  \commentalg{Since $A$ knows the distributions $\cP$ and $\cQ$, it can compute $p(r,s)$}
        \item {\bf For} $i\in [N]$ {\bf do:}
        \begin{enumerate}[itemsep=0pt]
          \item {\bf Sample} a point $t_i$ uniformly at random from $[0,1]$
          \item {\bf If} $t_i\leq p(X_i,Z_i)$ {\bf then:} {\bf Set} $\wt{Z}_i\coloneqq Z_i$,
          \item {\bf Otherwise:} {\bf Set} $\wt{Z}_i\coloneqq {3-Z_i}$ \hfill \commentalg{If $Z_i=1$ set $\wt{Z}_i=2$, if $Z_i=2$ set $\wt{Z}_i=1$}
        \end{enumerate}
        \item {\bf If} $\sum_{i\in [N]} \mathds{I}\sinsquare{\wt{Z}_i\neq Z_i}<\delta\cdot N$ {\bf then:} {\bf return} $\sinbrace{(X_i,\wt{Z}_i,Y_i)}_{i\in [N]}$,
        \item {\bf Otherwise:} {\bf return} $\sinbrace{(X_i,Z_i,Y_i)}_{i\in [N]}$
      \end{enumerate}
    \end{mdframed}

    \noindent (Since $A$ knows the distributions $\cP$ and $\cQ$, it can compute $p(r,s)$.)

    \begin{lemma}\label{lem:1_1}
      If $$N\geq 3\cdot \ln \inparen{\frac1\delta} \cdot (\eta-\alpha)^{-2},$$ then with probability at least $1-\delta$,
      $$\sum\nolimits_{i\in [N]} \mathds{I}\sinsquare{\wt{Z}_i\neq Z_i}<\eta\cdot N.$$
    \end{lemma}
    \begin{proof}
      For all $i\in [N]$, let $C_i\in \zo$ be a random variable indicating if $\wt{Z}_i\neq Z_i$.
      Since for each $i\in [N]$ the sample $(X_i,Y_i,Z_i)$ and the point $t_i$ is drawn independently of others, it follows that the random variables $C_i$ are independent of each other.
      Suppose we can show that $\Pr[C_i]\leq \alpha$.
      Then, by linearity of expectation, it follows that $$\Ex\insquare{\sum\nolimits_{i\in [N]}C_i}\leq \alpha \cdot N.$$
      Thus, using the Chernoff bound, we get that $$\Pr\insquare{\sum\nolimits_{i\in [N]}C_i\leq \eta\cdot N} \geq 1-\delta.$$
      This completes the proof of \cref{lem:1_1}, up to proving $\Pr[C_i=1]\leq \alpha$. %
      Towards this, observe that
      \ifconf
      \begin{align*}
        &\hspace{-2mm}\Pr[C_i=1]\\
        &\ \ =\  \sum_{r\in \cX, s\in [p]}\Pr\insquare{(X_i,Z_i)=(r,s)}\cdot \Pr\insquare{C_i=1\mid (X_i,Z_i)=(r,s)}\\
        &\ \ =\  \sum_{r\in \cX, s\in [p]}\Pr\insquare{(X_i,Z_i)=(r,s)}\cdot \Pr\insquare{Z_i\neq \smash{\wt{Z}_i}\mid (X_i,Z_i)=(r,s)}\tag{Definition of $C_i$}\\
        &\ \ \Stackrel{}{=}\ \sum_{r\in \cX, s\in [p]}\Pr\insquare{(X_i,Z_i)=(r,s)}\cdot \inparen{1-p(r,s)}\tag{Definition of $p(r,s)$}\\
        &\ \ \Stackrel{}{=}\ \sum_{r\in \cX, s\in [p]}\Pr\insquare{(X_i,Z_i)=(r,s)}\cdot \max\inbrace{ 1 - \frac{\Pr\nolimits_{\cQ}\sinsquare{(X,Z)=(r,s)}}{\Pr\nolimits_{\cP}\sinsquare{(X,Z)=(r,s)}} ,0}\tag{Definition of $p(r,s)$}\\
        &\ \ \Stackrel{}{=}\ \sum_{r\in \cX, s\in [p]}\Pr\nolimits_{\cP}\insquare{(X,Z)=(r,s)}\cdot \max\inbrace{ 1 - \frac{\Pr\nolimits_{\cQ}\sinsquare{(X,Z)=(r,s)}}{\Pr\nolimits_{\cP}\sinsquare{(X,Z)=(r,s)}}, 0}
        \tag{Using that for each $i\in [N]$, $(X_i,Y_i,Z_i)\sim\cP$}\\
        &\ \ \Stackrel{}{=}\ \sum_{r\in \cX, s\in [p]} \max\inbrace{ \Pr\nolimits_{\cP}\insquare{(X,Z)=(r,s)} - \Pr\nolimits_{\cQ}\sinsquare{(X,Z)=(r,s)},0}\\
        &\ \ \Stackrel{}{=}\ {\rm TV}(\cP, \cQ)\\
        &\ \ \Stackrel{}{=}\ \ \alpha. \tag{Using that ${\rm TV}(\cP,\cQ)=\alpha$}
      \end{align*}
      \else
      \begin{align*}
        \Pr[C_i=1]
        &\ \ =\  \sum_{r\in \cX, s\in [p]}\Pr\insquare{(X_i,Z_i)=(r,s)}\cdot \Pr\insquare{C_i=1\mid (X_i,Z_i)=(r,s)} \\
        &\ \ =\  \sum_{r\in \cX, s\in [p]}\Pr\insquare{(X_i,Z_i)=(r,s)}\cdot \Pr\insquare{Z_i\neq \smash{\wt{Z}_i}\mid (X_i,Z_i)=(r,s)}\tag{Definition of $C_i$}\\
        &\ \ \Stackrel{}{=}\ \sum_{r\in \cX, s\in [p]}\Pr\insquare{(X_i,Z_i)=(r,s)}\cdot \inparen{1-p(r,s)}\tag{Definition of $p(r,s)$}\\
        &\ \ \Stackrel{}{=}\ \sum_{r\in \cX, s\in [p]}\Pr\insquare{(X_i,Z_i)=(r,s)}\cdot \max\inbrace{ 1 - \frac{\Pr\nolimits_{\cQ}\sinsquare{(X,Z)=(r,s)}}{\Pr\nolimits_{\cP}\sinsquare{(X,Z)=(r,s)}} ,0}\tag{Definition of $p(r,s)$}\\
        &\ \ \Stackrel{}{=}\ \sum_{r\in \cX, s\in [p]}\Pr\nolimits_{\cP}\insquare{(X,Z)=(r,s)}\cdot \max\inbrace{ 1 - \frac{\Pr\nolimits_{\cQ}\sinsquare{(X,Z)=(r,s)}}{\Pr\nolimits_{\cP}\sinsquare{(X,Z)=(r,s)}}, 0}
        \tag{Using that  for each $i\in [N]$, $(X_i,Y_i,Z_i)\sim\cP$}\\
        &\ \ \Stackrel{}{=}\ \sum_{r\in \cX, s\in [p]} \max\inbrace{ \Pr\nolimits_{\cP}\insquare{(X,Z)=(r,s)} - \Pr\nolimits_{\cQ}\sinsquare{(X,Z)=(r,s)},0}\\
        &\ \ \Stackrel{}{=}\ {\rm TV}(\cP, \cQ)\\
        &\ \ \Stackrel{}{=}\ \ \alpha. \tag{Using that ${\rm TV}(\cP,\cQ)=\alpha$}
      \end{align*}
      \fi
    \end{proof}

    \begin{lemma}\label{lem:1_2}
      Each sample in $\wt{S}\coloneqq \sinbrace{(X_i,\wt{Z}_i,Y_i)}_{i\in [N]}$ is independent of each other and is distributed according to $\cQ$.
    \end{lemma}
    \begin{proof}
      Since for each $i\in [N]$ the sample $(X_i,Y_i,Z_i)$ and the point $t_i$ is drawn independently of others, it follows that the samples $(X_i,\wt{Z}_i,Y_i)$ are independent of each other.

      To see that $(X_i,\wt{Z}_i,Y_i)\sim \cQ$, fix any $i\in [N]$, $r\in \cX$, and $s\in [2]$.
      It holds that
      \begin{align*}
        \Pr[X_i=r,\wt{Z}_i=s,Y_i=1] &\ \ \Stackrel{\eqref{eq:deterministic_y_lemma_C_1}}{=}\ \ \Pr[Y_i=1 \mid X_i = r]\cdot \Pr[X_i=r,\wt{Z}_i=s],\yesnum\label{eq:1}\\
        \Pr[X_i=r,\wt{Z}_i=s,Y_i=0] &\ \ \Stackrel{\eqref{eq:deterministic_y_lemma_C_1}}{=}\ \ \Pr[Y_i=0 \mid X_i = r]\cdot \Pr[X_i=r,\wt{Z}_i=s],\yesnum\label{eq:2}
      \end{align*}
      here we used the fact that $Y_i$ is independent of $Z_i$ (see \cref{eq:deterministic_y_lemma_C_1}).
      Suppose that
      \begin{align*}
        \Pr[X_i=r,\wt{Z}_i=s] = \Pr\nolimits_{\cQ}[(X,Z)=(r,s)]\yesnum\label{eq:left_to_prove}.
      \end{align*}
      Then, from Equation~\eqref{eq:deterministic_y_lemma_C_1} we have that
      \begin{align*}
        \Pr\nolimits_{\cQ}[X=r,\wt{Z}=s,Y=1] &\ \ \Stackrel{\eqref{eq:deterministic_y_lemma_C_1}}{=}\ \ \Pr[Y_i=1 \mid X_i = r]\cdot \Pr\nolimits_{\cQ}[(X,Z)=(r,s)],\yesnum\label{eq:3}\\
        \Pr\nolimits_{\cQ}[X=r,\wt{Z}=s,Y=0] &\ \ \Stackrel{\eqref{eq:deterministic_y_lemma_C_1}}{=}\ \ \Pr[Y_i=0 \mid X_i = r]\cdot \Pr\nolimits_{\cQ}[(X,Z)=(r,s)].
        \yesnum\label{eq:4}
      \end{align*}
      Further, combining \cref{eq:1,eq:2} and \cref{eq:3,eq:4}, we get
      for all $y\in \zo$
      \begin{align*}
        \Pr[X_i=r,\wt{Z}_i=s,Y_i=y] = \Pr\nolimits_{\cQ}[X=r,\wt{Z}=s,Y=y].
      \end{align*}
      It remains to prove Equation~\eqref{eq:left_to_prove}.
      Before proving it, we recall the following invariant from the statement of this proposition:
      For all $r\in \cX$, it holds that
      \begin{align*}
        \Pr\nolimits_{\cP}[X=r]=\Pr\nolimits_{\cQ}[X=r].\yesnum\label{eq:invariant}
      \end{align*}
      \noindent Consider $\Pr[X_i=r,\wt{Z}_i=1]$ for some $r\in \cX$.
      From the algorithm used by the adversary, we have
      \begin{align*}
        \Pr[X_i=r,\wt{Z}_i=1] &\ \ \Stackrel{}{=} \ \ \inparen{1-p(a,2)}\cdot \Pr[(X_i,Z_i)=(r,2)] + p(a,1)\cdot \Pr[(X_i,Z)=(r,1)]\\
        &\ \ = \ \ \inparen{1-p(a,2)} \cdot \Pr\nolimits_{\cP}[(X,Z)=(r,2)] + p(a,1) \cdot \Pr\nolimits_{\cP}[(X,Z)=(r,1)].
        \tagnum{For all $i\in [N]$, $(X_i,Y_i,Z_i)\sim\cP$}\customlabel{eq:intermediate}{\theequation}
      \end{align*}
      We consider two cases.\\

      \noindent {\bf (Case A) $\Pr\nolimits_{\cQ}[(X,Z)=(a,1)] \geq \Pr\nolimits_{\cP}[(X,Z)=(a,1)]$:}
      In this case, we have $p(a,1)=1.$
      \begin{align*}
        \Pr[X_i=r,\wt{Z}_i=1] &\
        \Stackrel{\eqref{eq:intermediate}}{=}\  \inparen{1-p(a,2)} \cdot \Pr\nolimits_{\cP}[(X,Z)=(r,2)] + p(a,1) \cdot \Pr\nolimits_{\cP}[(X,Z)=(r,1)]\\
        & \ \Stackrel{}{=}\  \Pr\nolimits_{\cP}[(X,Z)=(r,2)]
        + \Pr\nolimits_{\cP}[(X,Z)=(r,1)]
        \cdot
        \inparen{
        \ifconf\textstyle\else\fi
        1 - \frac{\Pr\nolimits_{\cQ}\sinsquare{(X,Z)=(r,1)}}{\Pr\nolimits_{\cP}\sinsquare{(X,Z)=(r,1)}}  }
        \tag{Definition of $p(r,s)$}\\
        & \  = \  \Pr\nolimits_{\cP}[X=r] - \Pr\nolimits_{\cQ}[(X,Z)=(r,2)]\\
        & \ \Stackrel{\eqref{eq:invariant}}{=}\  \Pr\nolimits_{\cQ}[X=r] - \Pr\nolimits_{\cQ}[(X,Z)=(r,2)]\\
        & \  = \   \Pr\nolimits_{\cQ}[(X,Z)=(r,1)]
      \end{align*}

      \noindent {\bf (Case B) $\Pr\nolimits_{\cQ}[(X,Z)=(a,1)] < \Pr\nolimits_{\cP}[(X,Z)=(a,1)]$:}
      In this case, we have $p(a,1)<1.$
      \begin{align*}
        \quad \Pr[X_i=r,\wt{Z}_i=1]
        &  \  \Stackrel{\eqref{eq:intermediate}}{=}\   \inparen{1-p(a,2)} \cdot \Pr\nolimits_{\cP}[(X,Z)=(r,2)] + p(a,1) \cdot \Pr\nolimits_{\cP}[(X,Z)=(r,1)]\\
        & \  \Stackrel{}{=}\    \Pr\nolimits_{\cP}[(X,Z)=(r,1)]
        \cdot{ \frac{\Pr\nolimits_{\cQ}\sinsquare{(X,Z)=(r,1)}}
        {\Pr\nolimits_{\cP}\sinsquare{(X,Z)=(r,1)}} }\tag{Definition of $p(r,s)$}\\
        & \   = \   \Pr\nolimits_{\cQ}\sinsquare{(X,Z)=(r,1)}.
      \end{align*}
      In both, cases, we have
      $\Pr[X_i=r,\wt{Z}_i=1] = \Pr\nolimits_{\cQ}\sinsquare{(X,Z)=(r,1)}$.
      By swapping the protected labels, we can show that
      $\Pr[X_i=r,\wt{Z}_i=2] = \Pr\nolimits_{\cQ}\sinsquare{(X,Z)=(r,2)}$.
      This proves Equation~\eqref{eq:left_to_prove}.
    \end{proof}
    \noindent From \cref{lem:1_1}, with probability at least $1-\delta$, $\hS\coloneqq \sinbrace{(X_i,\wt{Z}_i,Y_i)}_{i\in [N]}$.
    By \cref{lem:1_2}, the samples $\sinbrace{(X_i,\wt{Z}_i,Y_i)}_{i\in [N]}$ are iid from $\cQ$.
    Thus, \cref{prop:to_lem_1} follows.
  \end{proof}

  \subsubsection{Proof of \texorpdfstring{\cref{lem:2}}{Lemma 6.13}}
  \begin{proof}[Proof of \cref{lem:2}]
    Our goal is to show that for every $f\in\cF$, there exists a choice $k\in [3]$, such that, $f$ has error at least $c(1-\alpha)$ or statistical rate at most $c+\alpha$ with respect to $\cD_k$. %
    Since $\cD_1$, $\cD_2$, and $\cD_3$ are supported on subsets of $\inbrace{x_A,x_B,x_C}\times [2]$, it suffices to consider the restriction of $\cF$ on this domain.
    There at most $2^6$ classifiers in this restriction.
    We partition them into three cases.\\

    \noindent {\bf (Case A) $f(x_B,1)=1$:}
    For any $k\in[3]$, we have
    \begin{align*}
      {\rm Err}_{\cD_k}\sinparen{f} &\quad \ =\quad \    \sum_{r\in \cX, s\in [p]} \Pr\nolimits_{\cD_k}\insquare{f(X,Z)\neq Y\mid (X,Z)=(r,s)}\cdot\Pr\insquare{(X,Z)=(r,s)}\\
      &\quad \  \geq\quad \    \Pr\nolimits_{\cD_k}\insquare{f(X,Z)\neq Y\mid X=x_B,Z=1} \cdot\Pr\insquare{X=x_B,Z=1}\\
      &\quad \  \Stackrel{\eqref{eq:deterministic_y_lemma_C_1}}{\geq}\quad \    \Pr\insquare{X=x_B,Z=1}\tag{Using that, in this case, $f(x_B,1)=1$}\\
      &\quad \  \Stackrel{\rm\cref{fig:1}}{\geq}\quad \    (1-c)\cdot(1-\alpha)\\
      &\quad \  \Stackrel{}{>}\quad \    c(1-\alpha).\tag{Using that $c<\nfrac12$}
    \end{align*}
    Thus, in Case A, $f$ has an error larger than $c\cdot (1-\alpha)$ on each of $\cD_1,\cD_2,$ and $\cD_3$.\\

    \noindent {\bf (Case B) $f(x_A,1)=0$ and $f(x_B,1)=0$:}
    For any $k\in[3]$, we have
    \begin{align*}
      {\rm Err}_{\cD_k}\sinparen{f}
      &\quad\ =\quad\  \sum_{r\in \cX, s\in [p]} \Pr\nolimits_{\cD_k}\insquare{f(X,Z)\neq Y\mid (X,Z)=(r,s)}\cdot\Pr\insquare{(X,Z)=(r,s)}\\
      &\quad\ \geq\quad\  \Pr\nolimits_{\cD_k}\insquare{f(X,Z)\neq Y\mid X=x_A,Z=1} \cdot\Pr\insquare{X=x_A,Z=1}\\
      &\quad\ \Stackrel{\eqref{eq:deterministic_y_lemma_C_1}}{\geq}\quad\  \Pr\insquare{X=x_A,Z=1}\tag{Using that, in this case, $f(x_B,1)=1$}\\
      &\quad\ \Stackrel{\rm\cref{fig:1}}{\geq}\quad\  c(1-\alpha).
    \end{align*}
    Thus, in Case B, $f$ has an error larger than $c\cdot (1-\alpha)$ on each of $\cD_1,\cD_2,$ and $\cD_3$.\\

    \noindent {\bf (Case C) $f(x_A,1)=1$ and $f(x_B,1)=0$:}

    \noindent {\bf (Case C.1) $\sum_{r\in \cX} f(r,2)\geq 2$:}
    In this case, $f$ takes a value of 1 on at least two points in the tuple
    $$L= ((x_A,2), (x_B,2), (x_C,2)).$$
    If $f$ takes a value of 0 on a point in $L$, then fix $j\in [3]$ such that
    $h(L_j)=0$.
    Let $k\coloneqq 4-j$.
    Consider the distribution $\cD_k$.
    Notice that by our construction $\Pr\nolimits_{\cD_k}[L_i]\stackrel{}{=} 0$ (see \rm\cref{fig:1}).
    Since $L_i$ has measure $0$ for $\cD_k$, the value of $f$ at this point does not affect its accuracy or statistical rate on $\cD_k$.
    Thus, we can assume that $f(L_i)=1$.
    Or in other words, we can assume that
    \begin{align*}
      \text{For all $r\in \cX$,}\quad f(r,2)=1.\yesnum\label{eq:all_one}
    \end{align*}
    We compute the performance of $f$ on both protected groups $\ell\in[p]$.
    For $Z=2$, we have
    \begin{align*}
      \Ex\nolimits_{\cD_k}\insquare{f(X,Z)\mid Z=2}
      \quad\  &=\quad\  \sum_{r\in \cX} f(r,2)\cdot\Pr\nolimits_{\cD_k}\insquare{X=r\mid Z=2}\\
      &\Stackrel{\eqref{eq:all_one}}{=}\quad\  \sum_{r\in \cX} 1\cdot\Pr\nolimits_{\cD_k}\insquare{X=r\mid Z=2}\\
      &\Stackrel{\rm\cref{fig:1}}{=}\quad\   1.\yesnum\label{eq:stat_rate_den_1}
    \end{align*}
    For $Z=1$, we have
    \begin{align*}
      \hspace{-25mm}\Ex\nolimits_{\cD_k}\insquare{f(X,Z)\mid Z=1} \quad\  &=\quad\  \sum_{r\in \cX} f(r,1)\cdot\Pr\nolimits_{\cD_k}\insquare{X=r\mid Z=1}\\
      &=\quad\  1\cdot \Pr\nolimits_{\cD_k}\insquare{X=x_A\mid Z=1} + f(x_C, 1)\cdot \Pr\nolimits_{\cD_k}\insquare{X=x_C\mid Z=1}\tag{In this case, $f(x_A,1)=1$ and $f(x_B,1)=0$}
    \end{align*}
    \begin{align*}
      \hspace{+35mm}&\Stackrel{}{\leq}\quad\   1\cdot \Pr\nolimits_{\cD_k}\insquare{X=x_A\mid Z=1} + \Pr\nolimits_{\cD_k}\insquare{X=x_C\mid Z=1}
      \tag{Using that $f(x_C,1)\leq 1$}\\
      &\Stackrel{\rm\cref{fig:1}}{\leq}\quad\
      \max\inbrace{
      \ifconf\textstyle\else\fi
      \frac{c\cdot (1-\alpha)+\sfrac\alpha2}{1-\sfrac\alpha2}, \frac{c\cdot (1-\alpha)+c\sfrac\alpha2}{1-\sfrac\alpha2}, \frac{c\cdot(1-\sfrac\alpha2)+\sfrac{\alpha\cdot(1-c)}2}{1-\sfrac\alpha2}}\\
      &\Stackrel{}{=}\quad\   \frac{c\cdot (1-\alpha)+\sfrac\alpha2}{1-\sfrac\alpha2}\tag{Using $c,\alpha>0$}\\
      &\Stackrel{}{<}\quad\   c+\alpha.\tagnum{Using $c,\alpha>0$ and $\alpha\leq 1$}\customlabel{eq:stat_rate_num_1}{\theequation}
    \end{align*}
    Since $c<\nfrac12$ and $\alpha\leq \min\sinbrace{\eta,\nfrac12}$, we have
    \begin{align*}
      c+\alpha < 1.\yesnum\label{eq:bound_on_sum}
    \end{align*}
    Now, we can compute the statistical rate of $f$ using Equations~\eqref{eq:stat_rate_den_1}, \eqref{eq:stat_rate_num_1}, and \eqref{eq:bound_on_sum}.
    \begin{align*}
      \Omega_{\cD_k}(f) = \frac{\min_{\ell\in [p]} \Ex\nolimits_{\cD_k}\insquare{f(X,Z)\mid Z=\ell}  }{\max_{\ell\in [p]} \Ex\nolimits_{\cD_k}\insquare{f(X,Z)\mid Z=\ell}   }%
      \qquad\Stackrel{\eqref{eq:stat_rate_den_1}, \eqref{eq:stat_rate_num_1}, \eqref{eq:bound_on_sum}}{<}\qquad \frac{c+\alpha}{1}.
    \end{align*}
    Thus, in Case C.1, $f$ has a statistical rate smaller than $c+\alpha$ on distribution $\cD_k$.

    \noindent {\bf (Case C.2) $\sum_{r\in \cX} f(r,2)\leq 1$:}
    Thus, $f$ takes a value of 0 on at least two points in  the list
    $$L= ((x_A,2), (x_B,2), (x_C,2)).$$
    If $f$ takes a value of 1 on one of the points in $L$, then fix $j\in [3]$ such that
    $f(L_j)=0$.
    Let $k\coloneqq 4-k$
    Consider the distribution $\cD_k$.
    Notice that by our construction $\Pr\nolimits_{\cD_k}[L_j]\stackrel{}{=} 0$ (see \cref{fig:1}).
    Since $L_i$ has measure $0$ on $\cD_k$, the value of $f$ at this point does not affect its accuracy or statistical rate on $\cD_k$.
    Thus, we can assume that $f(L_j)=0$.
    Or in other words, we can assume that
    \begin{align*}
      \text{For all $r\in \cX$,}\quad f(r,2)=0.\yesnum\label{eq:all_one_2}
    \end{align*}
    We would like to compute the statistical rate of $f$.
    Toward this, we first compute the performance of $f$ on both protected groups.
    For $Z=2$, we have
    \begin{align*}
      \Ex\nolimits_{\cD_i}\insquare{f(X,Z)\mid Z=2}
      \ &=\ \sum_{r\in \cX} f(r,2)\cdot\Pr\insquare{X=r\mid Z=2}\\
      &\Stackrel{\eqref{eq:all_one_2}}{=}\ 0.\yesnum\label{eq:stat_rate_den_2}
    \end{align*}
    For $Z=1$, we have
    \begin{align*}
      \Ex\nolimits_{\cD_k}\insquare{f(X,Z)\mid Z=1}\quad \
      &=\quad \ \sum_{r\in \cX} f(r,1)\cdot\Pr\nolimits_{\cD_k}\insquare{X=r\mid Z=1}\\
      &=\quad \ 1\cdot \Pr\nolimits_{\cD_k}\insquare{X=x_A\mid Z=1} + f(x_C, 1)\cdot \Pr\nolimits_{\cD_k}\insquare{X=x_C\mid Z=1}\tag{Using that, in this case, $f(x_A,1)=1$ and $f(x_B,1)=0$}\\
      &\Stackrel{}{\geq}\quad \ \Pr\nolimits_{\cD_k}\insquare{X=x_A\mid Z=1} \tag{Using $f(x_C, 1)\cdot \Pr\insquare{X=x_C\mid Z=1}\geq 0$}\\
      &\Stackrel{\rm\cref{fig:1}}{\geq}\quad \  \max\inbrace{\frac{c\cdot (1-\alpha)}{1-\sfrac\alpha2}, \frac{c(1-\alpha)}{1-\sfrac\alpha2}, \frac{c\cdot(1-\sfrac\alpha2)}{1-\sfrac\alpha2}}\\
      &\Stackrel{}{=}\quad \  \frac{c\cdot(1-\alpha)}{1-\sfrac\alpha2}\tag{Using that $\alpha,c>0$}\\
      &\Stackrel{}{>}\quad \  c.\tagnum{Using that $\alpha,c>0$}\customlabel{eq:stat_rate_num_2}{\theequation}
    \end{align*}
    Now, we can compute the statistical rate of $f$ using Equations~\eqref{eq:stat_rate_num_2} and \eqref{eq:stat_rate_den_2}.
    \begin{align*}
      \Omega_{\cD_i}(f) = \frac{\min_{\ell\in [p]} \Ex\nolimits_{\cD_k}\insquare{f(X,Z)\mid Z=\ell}  }{\max_{\ell\in [p]} \Ex\nolimits_{\cD_k}\insquare{f(X,Z)\mid Z=\ell}   }%
      \ \qquad\Stackrel{\eqref{eq:stat_rate_num_2}, \eqref{eq:stat_rate_den_2}, (c>0)}{\leq}\qquad \ 0
    \end{align*}
    Thus, in Case C.2, $f$ has a statistical rate 0 on the distribution $\cD_k$.

    Across all cases, we proved that all $2^6$ classifiers in the restriction of $\cF$ to $\inbrace{x_A,x_B,x_C}\times [2]$, either have a error larger than $c\cdot (1-\alpha)$ or statistical rate smaller than $c+\alpha$ on one of $\cD_1$, $\cD_2$, or $\cD_3$.
  \end{proof}

  \subsubsection{Proof of \texorpdfstring{\cref{lem:3}}{Lemma 6.14}}
  \begin{proof}[Proof of \cref{lem:3}]
    For each distribution $\cD_1$, $\cD_2$, and $\cD_3$, we will give an classifier $f\in \cF$ which satisfies the condition in \cref{lem:3}.\\

    \noindent {\bf (Case A) $\cD_1$: }
    Define $f$ as $f(x,z)\coloneqq \mathds{I}[x=x_A].$
    Comparing this to Equation~\eqref{eq:deterministic_y_lemma_C_1}, we get that
    \begin{align*}
      {\rm Err}_{\cD_1}(f)\stackrel{}{=} 0.
    \end{align*}
    For $Z=2$, we have
    \begin{align*}
      \Ex\nolimits_{\cD_1}\insquare{f(X,Z)\mid Z=2}\quad \
      &=\quad\ \sum_{r\in \cX} f(r,2)\cdot\Pr\nolimits_{\cD_1}\insquare{X=r\mid Z=2}\\
      &\Stackrel{}{=}\quad\ \Pr\nolimits_{\cD_1}\insquare{X=x_A\mid Z=2}\\
      &\Stackrel{\rm\cref{fig:1}}{=}\quad\  c.\yesnum\label{eq:stat_rate_den_3}
    \end{align*}
    Similarly, for $Z=1$, we have
    \begin{align*}
      \Ex\nolimits_{\cD_1}\insquare{f(X,Z)\mid Z=1}\quad \
      &=\quad\  \sum_{r\in \cX} f(r,1)\cdot\Pr\nolimits_{\cD_1}\insquare{X=r\mid Z=1}\\
      &\Stackrel{}{=}\quad\  \Pr\nolimits_{\cD_1}\insquare{X=x_A\mid Z=1}\\
      &\Stackrel{\rm\cref{fig:1}}{=}\quad\   \frac{c(1-\alpha)}{(1-\nfrac\alpha2)}.\yesnum\label{eq:stat_rate_num_3}
    \end{align*}
    Thus, we have
    \begin{align*}
      \Omega_{\cD_1}(h) &\quad= \quad\frac{\min_{\ell\in [p]} \Ex\nolimits_{\cD_1}\insquare{f(X,Z)\mid Z=\ell}  }{\max_{\ell\in [p]} \Ex\nolimits_{\cD_1}\insquare{f(X,Z)\mid Z=\ell}   }\\
      &\quad\Stackrel{\eqref{eq:stat_rate_den_3}, \eqref{eq:stat_rate_num_3}}{=}\quad \frac{c(1-\alpha)}{(1-\sfrac\alpha2)}\cdot \frac{1}{c}\tag{Using that $c,\alpha>0$}\\
      &\quad\Stackrel{}{=}\quad 1-\frac{\alpha}{(2-\alpha)}\\
      &\quad\Stackrel{}{\geq}\quad 1-\alpha.\tag{Using that $\alpha\leq 1$}
    \end{align*}

    \noindent {\bf (Case B) $\cD_2$: }
    Define $f$ as $f(x,z)\coloneqq \mathds{I}[x=x_A].$
    Comparing this to Equation~\eqref{eq:deterministic_y_lemma_C_1}, we get that
    \begin{align*}
      {\rm Err}_{\cD_1}(f)\stackrel{}{=} 0.
    \end{align*}
    For $Z=2$, we have
    \begin{align*}
      \Ex\nolimits_{\cD_2}\insquare{f(X,Z)\mid Z=2}\quad\
      &=\quad\ \sum_{r\in \cX} f(r,2)\cdot\Pr\nolimits_{\cD_2}\insquare{X=r\mid Z=2}\\
      &\Stackrel{}{=}\quad\ \Pr\nolimits_{\cD_2}\insquare{X=x_A\mid Z=2}\\
      &\Stackrel{\rm\cref{fig:1}}{=}\quad\  c.\yesnum\label{eq:stat_rate_den_4}
    \end{align*}
    Similarly, for $Z=1$, we have
    \begin{align*}
      \Ex\nolimits_{\cD_2}\insquare{f(X,Z)\mid Z=1}\quad\
      &=\quad\ \sum_{r\in \cX} f(r,1)\cdot\Pr\nolimits_{\cD_2}\insquare{X=r\mid Z=1}\\
      &\Stackrel{}{=}\quad\ \Pr\nolimits_{\cD_2}\insquare{X=x_A\mid Z=1}\\
      &\Stackrel{\rm\cref{fig:1}}{=}\quad\ \frac{c\cdot (1-\alpha)}{(1-\nfrac\alpha2)}.\yesnum\label{eq:stat_rate_num_4}
    \end{align*}
    Thus, we have
    \begin{align*}
      \Omega_{\cD_1}(h) &\quad =\quad \frac{\min_{\ell\in [p]} \Ex\nolimits_{\cD_2}\insquare{f(X,Z)\mid Z=\ell}  }{\max_{\ell\in [p]} \Ex\nolimits_{\cD_2}\insquare{f(X,Z)\mid Z=\ell}   }
      \\
      &\quad\Stackrel{\eqref{eq:stat_rate_den_4}, \eqref{eq:stat_rate_num_4}}{=}\quad
      \frac{c\cdot (1-\alpha)}{(1-\sfrac\alpha2)}\cdot \frac{1}{c}\tag{Using $c,\alpha>0$}\\
      &\quad\Stackrel{}{\geq}\quad 1-\alpha.\tag{Using that $\alpha\leq 1$}
    \end{align*}

    \noindent {\bf (Case C) $\cD_3$: }
    Define $f\in \cF$ to be the following classifier
    \begin{align*}
      f(x,z)  \coloneqq \begin{cases}
      1 & \text{if } (x,z)\in \inbrace{(x_A,1),(x_C,2)},\\
      0 & \text{otherwise.}
    \end{cases}\yesnum\label{eq:def_h_2}
  \end{align*}
  Such an $f\in \cF$ exists because $\cF$ shatters the set $\inbrace{x_A,x_B,x_C}\times [2]$.
  We have that
  \begin{align*}
    {\rm Err}_{\cD_3}(f)\quad\
    &=\quad\ \sum_{r\in \cX,s\in [p]} \Pr\nolimits_{\cD_3}[f(r,s)\neq Y\mid (X,Z)=(r,s)]\cdot \Pr\nolimits_{\cD_3}[(X,Z)=(r,s)]\\
    &\Stackrel{\eqref{eq:deterministic_y_lemma_C_1}, \eqref{eq:def_h_2}}{=}\quad\ \Pr\nolimits_{\cD_3}[(X,Z)=(x_A,2)] + \Pr\nolimits_{\cD_3}[(X,Z)=(x_C,2)]\\
    &\Stackrel{\rm\cref{fig:1}}{=}\quad\ \frac{c\alpha}2.
  \end{align*}
  Further, for $Z=2$, we have
  \begin{align*}
    \Ex\nolimits_{\cD_3}\insquare{f(X,Z)\mid Z=2}\quad\
    &=\quad\ \sum_{r\in \cX} f(r,2)\cdot\Pr\insquare{X=r\mid Z=2}\\
    &\Stackrel{\eqref{eq:def_h_2}}{=}\quad\ \Pr\nolimits_{\cD_3}\insquare{X=x_C\mid Z=2}\\
    &\Stackrel{\rm\cref{fig:1}}{=}\quad\  c.\yesnum\label{eq:stat_rate_den_5}
  \end{align*}
  Similarly, for $Z=1$, we have
  \begin{align*}
    \Ex\nolimits_{\cD_3}\insquare{f(X,Z)\mid Z=1}\quad\
    &=\quad\ \sum_{r\in \cX} f(r,1)\cdot\Pr\insquare{X=r\mid Z=1}\\
    &\Stackrel{\eqref{eq:def_h_2}}{=}\quad\ \Pr\nolimits_{\cD_3}\insquare{X=x_A\mid Z=1}\\
    &\Stackrel{\rm\cref{fig:1}}{=}\quad\  \frac{c\cdot(1-\sfrac\alpha2)}{(1-\sfrac\alpha2)}\\
    &=\ \ c.\yesnum\label{eq:stat_rate_num_5}
  \end{align*}
  Thus, we can compute $\Omega_{\cD_3}(f)$ as follows
  \begin{align*}
    \Omega_{\cD_3}(f) = \frac{\min_{\ell\in [p]} \Ex\nolimits_{\cD_3}\insquare{f(X,Z)\mid Z=\ell}  }{\max_{\ell\in [p]} \Ex\nolimits_{\cD_3}\insquare{f(X,Z)\mid Z=\ell}   }%
    &\quad\Stackrel{\eqref{eq:stat_rate_den_5}, \eqref{eq:stat_rate_num_5}}{=}\quad \frac{c}{c} = 1.
  \end{align*}

  \noindent Thus, for each $\cD\in \inbrace{\cD_1,\cD_2,\cD_3}$, we give a classifier $f\in \cF$ such that that has error at most $\nfrac{(c\alpha)}2$ and a statistical rate at least $1-\alpha$.
\end{proof}

\subsection{{Proof of \texorpdfstring{\cref{thm:imposs_under_assumption_sr}}{Theorem 4.5}}}\label{sec:proofof:thm:imposs_under_assumption_sr}
In \cref{thm:imposs_under_assumption_sr}, the hypothesis class $\cF$ that shatters the set $\inbrace{x_A,x_B,x_C,x_D,x_E}\times [2]\subseteq \cX\times [p]$,
\cref{asmp:1} holds with a constant $\lambda\in (0,\nfrac14]$,
the fairness threshold $\tau=1$, and statistical rate is the fairness metric; recall that the statistical rate of a classifier $f\in \cF$ is
$$\Omega_\cD(f)\coloneqq \frac{ \min_{\ell\in[p]} \Pr_\cD[f=1\mid Z=\ell]}{\max_{\ell\in[p]} \Pr_\cD[f=1\mid Z=\ell]}.$$
{Then, given parameters $\eta\in (0,1]$ and $\delta\in [0,\nfrac12)$, our goal is to show that for any
\begin{align*}
  \text{$\eps<\frac{1}4 - \frac{2\eta}{5}$ and  $v<\frac{\eta}{10\lambda}\cdot (1-4\lambda)-O\inparen{{\frac{\eta^2}{\lambda^2}}}$,}\yesnum\label{eq:required_result}
\end{align*}
$\cF$ is not ($\eps,\nu$)-learnable with perturbation rate $\eta$ and confidence $\delta$.
We prove a more general result:
Given parameters $\eta\in (0,1]$, $\delta\in [0,\nfrac12)$, for any $c\in (0,\min\inbrace{\eta,\nfrac{2\lambda}{9}}]$, we show that for any
\begin{align*}
  \text{$0<\eps<\frac{1}2- \lambda-2c$ and $0<v<\frac{c(1-4\lambda)}{2\lambda}-\frac{3c^2}{4\lambda^2},$}\yesnum\label{eq:values_of_eps_nu}
\end{align*}
$\cF$ is not ($\eps,\nu$)-learnable with perturbation rate $\eta$ and confidence $\delta$.
Setting $c=\nfrac{2\eta}5$ recovers \cref{thm:imposs_under_assumption_sr}.

\begin{remark}
  The proof of \cref{thm:imposs_under_assumption_sr} has a similar structure to the proof of \cref{thm:no_stable_classifier}, but the specific distributions constructed are different from those in the proof of \cref{thm:no_stable_classifier}.
  The proof of \cref{thm:imposs_under_assumption_sr}  also borrows \cref{prop:to_lem_1} from the proof of \cref{thm:no_stable_classifier};
  \cref{prop:to_lem_1} is proved in \cref{sec:proofof:extension_of_high_prop_result}.
\end{remark}

\subsubsection{{Proof of \texorpdfstring{\cref{thm:imposs_under_assumption_sr}}{Theorem 4.5}}}
\begin{proof}[Proof of \cref{thm:imposs_under_assumption_sr}]
  Let $\cL$ be any learner.
  Fix any constant
  \begin{align*}
    c\in \left(0,\min\inbrace{\eta,\frac{2\lambda}{9}}\right],\yesnum\label{eq:definition_alpha_2}
  \end{align*}
  and confidence parameter $\delta\in (0,\nfrac12)$ (in \cref{def:learning_model}).
  We construct three distributions $\cD_1$, $\cD_2$, and $\cD_3$ (parameterized by $c$) that satisfy the requirements of \cref{prop:to_lem_1}:
  \begin{enumerate}[itemsep=\itemsepINTERNAL,leftmargin=\leftmarginINTERNAL]
    \item The total variation distance between any two distributions is bounded by $\eta$,
    \item the distributions have the same marginal on $\cX$, and
    \item the label $Y$ is independent of the protected attribute $Z$ conditioned on features $X$.
  \end{enumerate}
  Suppose $\cL$ is given $N$ samples, where
  \begin{align*}
    N &\geq 3\cdot\ln{2} \cdot (\eta-c)^{-2}.\yesnum\label{eq:definition_N_theorem_4_5}
  \end{align*}
  Consider three cases, depending on whether the samples in $S$ are iid from $\cD_1$, $\cD_2$, or $\cD_3$.

  By \cref{prop:to_lem_1}, in each case, there is an adversary $A\in \cA(\eta)$, that can ensure that, with probability at least $\nfrac12$, the perturbed samples $\hS\coloneqq A(S)$ have the same distribution as iid samples from $\cD_1$.
  Thus, given $\hS$, $\cL$ cannot differentiate between the three cases with probability at least $\nfrac12$.
  As a result, with probability at least $\nfrac12$, $\cL$ outputs the same classifier, say $f_{\rm Com}\in \cF$, in each case.
  We show that no $f_{\rm Com}\in \cF$  satisfies the accuracy and fairness guarantee in all three cases.
  \begin{lemma}[\bf No good classifier for all cases]\label{lem:2:ptb}
    There is no $f\in \cF$, such that,
    for all $k\in [3]$,
    \begin{align*}
      \Pr\nolimits_{\cD_k}\insquare{h(X,Z)\neq Y} < \frac12-\lambda-\frac{c}2
      \quad \text{and}\quad
      \Omega_{\cD_k}(f)<1-\frac{c(1-4\lambda)}{2\lambda}+\frac{3c^2}{4\lambda^2}.
    \end{align*}
  \end{lemma}
  \begin{lemma}[\bf A good classifier for each case]\label{lem:3:ptb}
    For each $k\in[3]$, there is an $f\in \cF$ such that
    \begin{align*}
      \Pr\nolimits_{\cD_k}\insquare{h(X,Z)\neq Y} \leq \frac{3\alpha}2
      \quad \text{and}\quad
      \Omega_\cD(h) =1.
    \end{align*}
  \end{lemma}
  \noindent Note that for each $k\in [3]$, $h_k$ satisfies the fairness constraint. %
  Thus, an optimal solution $f_k^\star\in \cF$ of Program~\eqref{prog:target_fair} for $\cD_k$ and $\tau=1$ must satisfy
  \begin{align*}
    {\rm Err}_{\cD_k}(f^\star_k) &< \frac{3c}{2}.\yesnum\label{eq:boundONOPTClassifier_2}
  \end{align*}
  (Otherwise, we have a contradiction as $h_k$ satisfies the fairness constraints and has a smaller error than $f^\star_k$.)
  If $\cL$ is a ($\eps,\nu$)-learner, then ibn the $k$-th case, with probability at least $1-\delta>\nfrac12$,
  $\cL$ must output a classifier $f_k$ which satisfies
  \begin{align*}
    {\rm Err}_{\cD_k}(f_k) - {\rm Err}_{\cD_k}(f_k^\star) &\leq \eps \text{ and }\tau - \Omega_{\cD_k}(f_k) \leq\nu.
  \end{align*}
  However, in all cases with probability at least $\nfrac12$, $\cL$ outputs $f_{\rm Com}$.
  Because $\nfrac12>\delta$, $f_{\rm Com}$ must satisfy:
  \begin{align*}
    \text{For all $k\in [3]$,}\quad
    {\rm Err}_{\cD_k}(f_{\rm Com}) - {\rm Err}_{\cD_k}(f_k^\star) &\leq \eps \text{ and }\tau - \Omega_{\cD_k}(f_{\rm Com}) \leq\nu.\yesnum\label{eq:lem:learnability_2}
  \end{align*}
  But from \cref{lem:2:ptb} we know that for each $k\in [3]$ either
  \begin{align*}
    {\rm Err}_{\cD_k}\sinparen{f_{\rm Com}} &\geq \frac12-\lambda-\frac{c}2 \quad\text{or}\quad \Omega_{\cD_k}(f_{\rm Com}) \leq 1-\frac{c(1-4\lambda)}{2\lambda}+\frac{3c^2}{4\lambda^2}.\yesnum\label{eq:lem:1:ptb}
  \end{align*}
  \noindent{\bf (Case A) ${\rm Err}_{\cD_k}\sinparen{f_{\rm Com}} \geq \nfrac12-\lambda-\nfrac{c}2$:} In this case, from \cref{eq:lem:learnability_2}, we have
  \begin{align*}
    \eps &\geq \frac12-\lambda-\frac{c}2- {\rm Err}_{\cD_k}(f_k^\star)\\
    &\Stackrel{\eqref{eq:boundONOPTClassifier_2}}{>} \frac12-\lambda-\frac{c}2- \frac{3c}2\\
    &\Stackrel{}{=} \frac12-\lambda-2c \tagnum{Using that $c<\nfrac12$ and  $\alpha>0$}\customlabel{eq:contradiction_1:ptb}{\theequation}
  \end{align*}
  But because $\eps\leq \nfrac12-\lambda-2c$ (see \cref{eq:values_of_eps_nu}), Equation~\eqref{eq:contradiction_1:ptb} cannot hold.

  \noindent{\bf (Case B) $\Omega_{\cD_k}(f_{\rm Com}) \leq 1-\nfrac{c(1-4\lambda)}{(2\lambda)}+\nfrac{3c^2}{(4\lambda^2)}$:} In this case, from \cref{eq:lem:learnability_2}, we have
  \begin{align*}
    \nu &\geq \frac{c(1-4\lambda)}{2\lambda} -\frac{3c^2}{4\lambda^2}.\yesnum\label{eq:contradiction_2:ptb}%
  \end{align*}
  But because $\nu< \nfrac{c(1-4\lambda)}{(2\lambda)} -\nfrac{3c^2}{(4\lambda^2)}$ (see \cref{eq:values_of_eps_nu}), Equation~\eqref{eq:contradiction_2:ptb} does not hold.
  Therefore, we have a contradiction.
  Hence, $\cL$ is not an ($\eps,\nu$)-learner for $\cF$.
  Since the choice of $\cL$ was arbitrary, we have shown that there is no learner which ($\eps,\nu$)-learns $\cF$.
  It remains to prove the \cref{lem:2:ptb,lem:3:ptb}.
\end{proof}

\subsubsection{Proof of \texorpdfstring{\cref{lem:2:ptb}}{Lemma 6.19}}
Fix any $$c\in \left(0,\min\inbrace{\eta,\frac{2\lambda}{9}}\right].$$
Set $\cD_1$, $\cD_2$, and $\cD_3$ to be unique distributions with marginal distribution specified in \cref{fig:2}, such that, for any draw $(X,Y,Z)\sim \cD_k$ ($k\in [3]$) $Y$ takes the value $\mathds{I}\insquare{X\neq x_E}$, i.e.,
\begin{align*}
  Y  = \begin{cases}
  1 & \text{if } X\neq x_E,\\
  0 & \text{otherwise}.
\end{cases}\yesnum\label{eq:deterministic_y_lemma_c_7}
\end{align*}
In particular, the construction of $\cD_1$, $\cD_2$, and $\cD_3$ ensures that the total variation distance between any pair of distributions is less than $c$.
\begin{table}[t]
  \centering
  %
  \subfigure[$\cD_1$: $\Pr\nolimits_{\cD_1}\insquare{(X,Z)=(r,s)}$ for $(r,s)\in\inbrace{x_A,x_B,x_C,x_D,x_E}\times \insquare{2}$.\white{$\bigg|$}]{
  \begin{tabular}{c>{\centering\arraybackslash}p{1.64cm}>{\centering\arraybackslash}p{1.64cm}>{\centering\arraybackslash}p{1.64cm}>{\centering\arraybackslash}p{1.64cm}>{\centering\arraybackslash}p{1.64cm}}
    \toprule
    \midsepremove{}
    & $x_A$ & $x_B$ & $x_C$ & $x_D$ & $x_E$\\
    \midrule{}
    $1$ & $0$ & $\sfrac{c}2$ & $\sfrac{c}2$ & $\nfrac12-\lambda-\nfrac{c}2$ & $\nfrac12-\lambda-\nfrac{c}2$\\
    $2$ & $c$ & $\sfrac{c}2$ & $\sfrac{c}2$ & $\lambda-c$ & $\lambda-c$\\
    \bottomrule{}
    %
  \end{tabular}
  }
  %
  \subfigure[$\cD_2$: $\Pr\nolimits_{\cD_2}\insquare{(X,Z)=(r,s)}$ for $(r,s)\in\inbrace{x_A,x_B,x_C,x_D,x_E}\times \insquare{2}$.\white{$\bigg|$}]{
  \begin{tabular}{c>{\centering\arraybackslash}p{1.64cm}>{\centering\arraybackslash}p{1.64cm}>{\centering\arraybackslash}p{1.64cm}>{\centering\arraybackslash}p{1.64cm}>{\centering\arraybackslash}p{1.64cm}}
    \toprule
    \midsepremove{}
    & $x_A$ & $x_B$ & $x_C$ & $x_D$ & $x_E$\\
    \midrule
    $1$ & $\sfrac{c}2$ & $0$ & $\sfrac{c}2$ & $\nfrac12-\lambda-\nfrac{c}2$ & $\nfrac12-\lambda-\nfrac{c}2$\\
    $2$ & $\sfrac{c}2$ & $c$ & $\sfrac{c}2$ & $\lambda-c$ & $\lambda-c$\\
    \bottomrule{}
  \end{tabular}
  }
  \subfigure[$\cD_3$: $\Pr\nolimits_{\cD_3}\insquare{(X,Z)=(r,s)}$ for $(r,s)\in\inbrace{x_A,x_B,x_C,x_D,x_E}\times \insquare{2}$.\white{$\bigg|$}]{
  \begin{tabular}{c>{\centering\arraybackslash}p{1.64cm}>{\centering\arraybackslash}p{1.64cm}>{\centering\arraybackslash}p{1.64cm}>{\centering\arraybackslash}p{1.64cm}>{\centering\arraybackslash}p{1.64cm}}
    \toprule
    \midsepremove{}
    & $x_A$ & $x_B$ & $x_C$ & $x_D$ & $x_E$\\
    \midrule
    $1$ & $\sfrac{c}2$ & $\sfrac{c}2$ & $0$ & $\nfrac12-\lambda-\nfrac{c}2$ & $\nfrac12-\lambda-\nfrac{c}2$\\
    $2$ & $\sfrac{c}2$ & $\sfrac{c}2$ & $c$ & $\lambda-c$ & $\lambda-c$\\
    \bottomrule
  \end{tabular}
  }
  \caption{
  \ifconf\small\fi
  Marginal distributions of $\cD_1$, $\cD_2$, $\cD_3$ over $\cX\times [2]$.
  Recall that for a sample $(X,Y,Z)\sim \cD_k$ ($k\in [3]$), $Y$ takes the value $\mathds{I}[X \neq x_E]$.
  }
  \label{fig:2}
\end{table}
\begin{proof}[Proof of \cref{lem:2:ptb}]
  Our goal is to show that for every $f\in\cF$, there is a $k\in [3]$, such that, $f$ has error at least $\nfrac12-\lambda-\nfrac{c}2$ or statistical rate at most $$1+\frac{c(1-4\lambda)}{2\lambda}+\frac{3c^2}{4\lambda^2}$$ with respect to $\cD_k$. %
  Since $\cD_1$, $\cD_2$, and $\cD_3$ are supported on subsets of $\inbrace{x_A,x_B,x_C,x_D,x_E}\times [2]$, it suffices to consider the restriction of $\cF$ to $\inbrace{x_A,x_B,x_C,x_D,x_E}\times [2]$.
  There at most $2^{10}$ classifiers in this restriction.
  We partition them into four cases.\\

  \noindent {\bf (Case A) $f(x_D,1)=0$ or $f(x_E,1)=1$:}

  \noindent {\bf (Case A.1) $f(x_D,1)=0$:}
  For any $k\in[3]$ it holds that

  \begin{align*}
    \hspace{-25mm}{\rm Err}_{\cD_k}(f)
    \quad\  =\quad\   \sum_{r\in \cX, s\in [p]} \Pr\nolimits_{\cD_k}\insquare{f(X,Z)\neq Y\mid (X,Z)=(r,s)}\cdot\Pr\insquare{(X,Z)=(r,s)}
  \end{align*}
  \begin{align*}
    \hspace{15mm}&\geq\quad\   \Pr\nolimits_{\cD_k}\insquare{f(X,Z)\neq Y\mid X=x_D,Z=1} \cdot\Pr\insquare{X=x_D,Z=1}\tag{$\forall\ x\in \cX, f(x,2)\geq 0$}\\
    &\Stackrel{\eqref{eq:deterministic_y_lemma_c_7}}{\geq}\quad\   \Pr\insquare{X=x_D,Z=1}\tag{Using that, in this case, $f(x_D,1)=0$}\\
    &\Stackrel{\rm\cref{fig:2}}{=}\quad\   \frac12-\lambda-\frac{c}2.
  \end{align*}
  Thus, in Case A.1, $f$ has error at least $\nfrac12-\lambda-\nfrac{c}2$ on each of $\cD_1,\cD_2,$ and $\cD_3$.

  \noindent {\bf (Case A.2) $f(x_E,1)=1$:}
  For any $k\in[3]$ it holds that
  \begin{align*}
    {\rm Err}_{\cD_k}(f)
    \quad\ &=\quad\  \sum_{r\in \cX, s\in [p]} \Pr\nolimits_{\cD_k}\insquare{f(X,Z)\neq Y\mid (X,Z)=(r,s)}\cdot\Pr\insquare{(X,Z)=(r,s)}\\
    &\geq\quad\  \Pr\nolimits_{\cD_k}\insquare{f(X,Z)\neq Y\mid X=x_E,Z=1} \cdot\Pr\insquare{X=x_E,Z=1}\tag{$\forall\ x\in \cX, f(x,2)\geq 0$}\\
    &\Stackrel{\eqref{eq:deterministic_y_lemma_c_7}}{\geq}\quad\  \Pr\insquare{X=x_E,Z=1}\tag{Using that, in this case, $f(x_E,1)=1$}\\
    &\Stackrel{\rm\cref{fig:2}}{=}\quad\  \frac12-\lambda-\frac{c}2.
  \end{align*}
  Thus, in Case A.2, $f$ has error at least $\nfrac12-\lambda-\nfrac{c}2$ on each of $\cD_1,\cD_2,$ and $\cD_3$.\\

  \noindent In the rest of the cases, assume that $f(x_D,1)=1$ and $f(x_E,1)=0$.\\

  \noindent {\bf (Case B) $f(x_D,2)=f(x_E,2)$:}

  \noindent {\bf (Case B.1) $f(x_D,2)=f(x_E,2)=0$:}
  For any $k\in[3]$ and $Z=2$ it holds that
  \begin{align*}
    \Ex\nolimits_{\cD_k}\insquare{f(X,Z)\mid Z=2}
    \quad\ &=\quad\ \sum_{r\in \cX} f(r,2)\cdot\Pr\nolimits_{\cD_k}\insquare{X=r\mid Z=2}\\
    &\stackrel{}{=}\quad\ \hspace{-4mm}\sum_{r\in \cX\backslash \inbrace{x_D,x_E}} f(r,2)\cdot\Pr\nolimits_{\cD_k}\insquare{X=r\mid Z=2}\tag{Using that, in this case, $f(x_D,2)=f(x_E,2)=0$}\\
    &\stackrel{}{\leq}\quad\ \sum_{r\in \cX\backslash \inbrace{x_D,x_E}} \Pr\nolimits_{\cD_k}\insquare{X=r\mid Z=2} \tag{Using that $\forall\ r\in \cX$, $f(r,2)\leq 1$}\\
    &\Stackrel{\rm\cref{fig:2}}{\leq}\quad\  \frac{2c}{2\lambda}.\yesnum\label{eq:stat_rate_den_1:ptb}
  \end{align*}
  For $Z=1$, we have
  \begin{align*}
    \Ex\nolimits_{\cD_k}\insquare{f(X,Z)\mid Z=1}
    \ \ &=\ \ \sum_{r\in \cX} f(r,1)\cdot\Pr\nolimits_{\cD_k}\insquare{X=r\mid Z=1}\\
    &\geq \ \ 1\cdot \Pr\nolimits_{\cD_k}\insquare{X=x_D\mid Z=1}\tag{From Case A, $f(x_D,1)=1$}\\
    &\Stackrel{}{=}\ \  \frac{1-2\lambda-c}{2(1-2\lambda)}\\
    &\Stackrel{}{\geq}\ \  \frac{\nfrac56-2\lambda}{2} \tag{Using that $c\leq \nfrac16$ and $\lambda>0$}\\
    &\Stackrel{}{\geq}\ \  \nfrac{1}{6}. \tagnum{Using that $\lambda\leq \nfrac14$}\customlabel{eq:stat_rate_den_2:ptb}{\theequation}
  \end{align*}
  \noindent Since $c<\lambda/2$, %
  the statistical rate of $f$ is as follows
  \begin{align*}
    \Omega_{\cD_k}(f) = \frac{\min_{\ell\in [p]} \Ex\nolimits_{\cD_k}\insquare{f(X,Z)\mid Z=\ell}  }{\max_{\ell\in [p]} \Ex\nolimits_{\cD_k}\insquare{f(X,Z)\mid Z=\ell}   }%
    \quad\Stackrel{\eqref{eq:stat_rate_den_1:ptb}, \eqref{eq:stat_rate_den_2:ptb}}{\leq}\quad \frac{6c}{\lambda}.
  \end{align*}
  Thus, in Case B.1, $f$ has a statistical rate at most $\nfrac{(6c)}{\lambda}$ on each of $\cD_1,\cD_2,$ and $\cD_3$.

  \noindent {\bf (Case B.2) $f(x_D,2)=f(x_E,2)=1$:}
  For any $k\in[3]$ and $Z=2$ it holds that
  \begin{align*}
    \Ex\nolimits_{\cD_k}\insquare{f(X,Z)\mid Z=2}
    \quad\  &=\quad\  \sum_{r\in \cX} f(r,2)\cdot\Pr\nolimits_{\cD_k}\insquare{X=r\mid Z=2}\\
    &\stackrel{}{\geq }\quad\  \sum_{r\in \inbrace{x_D,x_E}} f(r,2)\cdot\Pr\nolimits_{\cD_k}\insquare{X=r\mid Z=2}\tag{Using that $f(x_D,2)=f(x_E,2)=1$ and for all $r\in \cX$, $f(r,2)\geq 0$}\\
    &\Stackrel{\rm\cref{fig:2}}{=}\quad\  1-\frac{c}{\lambda}.\yesnum\label{eq:stat_rate_den_case_c2}
  \end{align*}
  For $Z=1$, we have
  \begin{align*}
    \Ex\nolimits_{\cD_k}\insquare{f(X,Z)\mid Z=1}
    \quad\ &=\quad\ \sum_{r\in \cX} f(r,1)\cdot\Pr\nolimits_{\cD_k}\insquare{X=r\mid Z=1}\\
    \quad\ &=\quad\ \sum_{r\in \cX\backslash \sinbrace{x_E}} f(r,1)\cdot\Pr\nolimits_{\cD_k}\insquare{X=r\mid Z=1}\tag{From Case B, $f(x_E,1)=0$}\\
    \quad\ &\leq \quad\ \sum_{r\in \cX\backslash \sinbrace{x_E}} 1\cdot\Pr\nolimits_{\cD_k}\insquare{X=r\mid Z=1}\tag{Using that $\forall\ r\in \cX$, $f(r,2)\leq 1$}\\
    &\Stackrel{\rm\cref{fig:2}}{=}\quad\  \frac12 - \frac{c}{2(1-2\lambda)}.\yesnum\label{eq:stat_rate_num_case_c2}
  \end{align*}
  Using this, we can compute an upper bound on the statistical rate of $f$ as follows
  \begin{align*}
    \Omega_{\cD_k}(f) \quad &=\quad  \frac{\min_{\ell\in [p]} \Ex\nolimits_{\cD_k}\insquare{f(X,Z)\mid Z=\ell}  }{\max_{\ell\in [p]} \Ex\nolimits_{\cD_k}\insquare{f(X,Z)\mid Z=\ell}   }\\
    &\Stackrel{\eqref{eq:stat_rate_num_case_c2}, \eqref{eq:stat_rate_den_case_c2}}{\leq}\quad \frac{\frac12 - \frac{c}{2(1-2\lambda)}} { 1-\frac{c}{\lambda} }\\
    &=\quad \frac12 \cdot \frac{1 - \frac{c}{(1-2\lambda)}} { 1-\frac{c}{\lambda} }\cdot \frac{1+\frac{3c}{\lambda}}{1+\frac{3c}{\lambda}}\\
    &=\quad \frac12 \cdot \frac{1 - \frac{c}{(1-2\lambda)}} { 1+\frac{2c}{\lambda}-\frac{3c^2}{\lambda^2} }\cdot \inparen{1+\frac{3c}{\lambda}}\\
    &\leq\quad \frac12 +\frac{3c}{2\lambda}.\tag{Using that for all $0<c,\lambda\leq\frac14$, if $\frac{c}{\lambda}\leq \frac13$, then $\frac{c}{1-2\lambda}\geq  \frac{3c^2}{\lambda^2} -\frac{2c}{\lambda}$}
  \end{align*}
  Thus, in Case B.2, $f$ has a statistical rate at most $(\nfrac12) +\nfrac{(3c)}{2\lambda}$ on each of $\cD_1,\cD_2,$ and $\cD_3$.

  \noindent In the rest of the cases, we assume that $f(x_D,2)\neq f(x_E,2).$\\

  \noindent {\bf (Case C) $\sum_{r\in \inbrace{x_A,x_B,x_C}} f(r,2)\geq 2$:}
  In Case C, $f$ outputs 0 on at most one point in the tuple
  $$L= ((x_A,2), (x_B,2), (x_C,2)).$$
  If $f$ takes a value of 0 on a point, say $L_j$, in $L$, then fix $k\in [3]$, such that $\cD_k(L_j)=\nfrac{c}2$. (Such a distribution always exists by construction).
  Otherwise, fix any $k\in [3]$.
  For $Z=2$, we have
  \begin{align*}
    \Ex\nolimits_{\cD_k}\insquare{f(X,Z)\mid Z=2}
    &= \sum_{r\in \cX} f(r,2)\cdot\Pr\nolimits_{\cD_k}\insquare{X=r\mid Z=2}\\
    &\stackrel{}{\geq } \frac{\lambda+c}{2\lambda}. \tagnum{By our choice of $k$}
    \customlabel{eq:stat_rate_den_case_c2_2}{\theequation}
  \end{align*}
  For $Z=1$, we have
  \ifconf
  \begin{align*}
    &\Ex\nolimits_{\cD_k}\insquare{f(X,Z)\mid Z=1}\\
    &\quad=\quad\ \sum_{r\in \cX} f(r,1)\cdot\Pr\nolimits_{\cD_k}\insquare{X=r\mid Z=1}\\
    &\quad=\quad\ \Pr\nolimits_{\cD_k}\insquare{X=x_D\mid Z=1}+ \sum_{r\in \inbrace{x_A,x_B,x_C}} f(r,1)\cdot\Pr\nolimits_{\cD_k}\insquare{X=r\mid Z=1}\tag{From previous cases, $f(x_D,1)=1$ and $f(x_E,1)=0$}\\
    &\quad\Stackrel{\rm\cref{fig:2}}{\leq}\quad \
    \frac{(1-2\lambda+c)}{2(1-2\lambda)}.\yesnum\label{eq:stat_rate_num_case_c2_2}
  \end{align*}
  \else
  \begin{align*}
    \Ex\nolimits_{\cD_k}\insquare{f(X,Z)\mid Z=1}
    &\quad=\quad\ \sum_{r\in \cX} f(r,1)\cdot\Pr\nolimits_{\cD_k}\insquare{X=r\mid Z=1}\\
    &\quad=\quad\ \Pr\nolimits_{\cD_k}\insquare{X=x_D\mid Z=1}+ \sum_{r\in \inbrace{x_A,x_B,x_C}} f(r,1)\cdot\Pr\nolimits_{\cD_k}\insquare{X=r\mid Z=1}\tag{From previous cases, $f(x_D,1)=1$ and $f(x_E,1)=0$}
  \end{align*}
  \begin{align*}
    &\hspace{-44mm}\quad\Stackrel{\rm\cref{fig:2}}{\leq}\quad \
    \frac{(1-2\lambda+c)}{2(1-2\lambda)}.\yesnum\label{eq:stat_rate_num_case_c2_2}
  \end{align*}
  \fi
  We can compute the statistical rate of $f$ using Equations~\eqref{eq:stat_rate_den_case_c2_2} and \eqref{eq:stat_rate_num_case_c2_2}
  \begin{align*}
    \Omega_{\cD_k}(f) \quad&=\quad \frac{\min_{\ell\in [p]} \Ex\nolimits_{\cD_k}\insquare{f(X,Z)\mid Z=\ell}  }{\max_{\ell\in [p]} \Ex\nolimits_{\cD_k}\insquare{f(X,Z)\mid Z=\ell}   }\\
    &\Stackrel{\eqref{eq:stat_rate_num_case_c2_2}, \eqref{eq:stat_rate_den_case_c2_2}}{\leq}
    \quad \frac{1+\frac{c}{1-2\lambda}}{ 1+\frac{c}{2\lambda} }\\
    \quad &\leq\quad  \inparen{1+\frac{c}{1-2\lambda}}\cdot \inparen{ 1-\frac{c}{2\lambda} + \frac{c^2}{4\lambda^2}}
    \tag{Using that for all $x\geq 0$, $\frac1{1+x}\leq 1-x+x^2$}\\
    \quad &=\quad  1+\frac{c}{1-2\lambda}-\frac{c}{2\lambda}
    + \frac{c^2}{4\lambda^2}\cdot \inparen{1+\frac{c}{1-2\lambda}} - \frac{c^2}{2\lambda(1-2\lambda)}\\
    \quad &\leq \quad  1-\frac{c(1-4\lambda)}{2\lambda(1-2\lambda)} + \frac{3c^2}{4\lambda^2}
    \tag{Using that $0<\lambda\leq \frac14$ and $c>0$}\\
    \quad &\leq \quad  1-\frac{c}{2\lambda}\cdot (1-4\lambda) + \frac{3c^2}{4\lambda^2} \tag{Using that $\lambda\geq 0$}.
  \end{align*}
  Thus, in Case C, $f$ has a statistical rate at most $1-\frac{c}{2\lambda}\cdot (1-4\lambda) + \frac{3c^2}{4\lambda^2}$ on the chosen distribution $\cD_k$.\\

  \noindent {\bf (Case D) $\sum_{r\in \inbrace{x_A,x_B,x_C}} f(r,2)\leq 1$:}
  In Case D, $f$ outputs of 1 on at most one point in the tuple
  $$L= ((x_A,2), (x_B,2), (x_C,2)).$$
  If $f$ takes a value of 1 on a point, say $L_j$, in $L$, then fix $k\in [3]$, such that $\cD_k(L_j)=\nfrac{c}2$. (Such a distribution always exists by construction).
  Otherwise, fix any $k\in [3]$.
  For $Z=2$, we have
  \begin{align*}
    \Ex\nolimits_{\cD_k}\insquare{f(X,Z)\mid Z=2}
    &= \sum_{r\in \cX} f(r,2)\cdot\Pr\nolimits_{\cD_k}\insquare{X=r\mid Z=2}
    \stackrel{}{\leq } \frac{\lambda-c}{2\lambda}. \tagnum{By our choice of $k$}
    \customlabel{eq:stat_rate_den_case_c2_3}{\theequation}
  \end{align*}
  For $Z=1$, we have
  \ifconf
  \begin{align*}
    &\Ex\nolimits_{\cD_k}\insquare{f(X,Z)\mid Z=1}\\
    &\quad=\quad\  \sum_{r\in \cX} f(r,1)\cdot\Pr\nolimits_{\cD_k}\insquare{X=r\mid Z=1}\\
    &\quad=\quad\  \Pr\nolimits_{\cD_k}\insquare{X=x_D\mid Z=1}+ \sum_{r\in \inbrace{x_A,x_B,x_C}} f(r,1)\cdot\Pr\nolimits_{\cD_k}\insquare{X=r\mid Z=1}\tag{From previous cases, $f(x_D,1)=1$ and $f(x_E,1)=0$}\\
    &\quad\Stackrel{\rm\cref{fig:2}}{\geq}\quad\
    \frac{(1-2\lambda-c)}{2(1-2\lambda)}.\yesnum\label{eq:stat_rate_num_case_c2_3}
  \end{align*}
  \else
  \begin{align*}
    \Ex\nolimits_{\cD_k}\insquare{f(X,Z)\mid Z=1}
    &\quad=\quad\  \sum_{r\in \cX} f(r,1)\cdot\Pr\nolimits_{\cD_k}\insquare{X=r\mid Z=1}\\
    &\quad=\quad\  \Pr\nolimits_{\cD_k}\insquare{X=x_D\mid Z=1}+ \sum_{r\in \inbrace{x_A,x_B,x_C}} f(r,1)\cdot\Pr\nolimits_{\cD_k}\insquare{X=r\mid Z=1}\tag{From previous cases, $f(x_D,1)=1$ and $f(x_E,1)=0$}\\
    &\quad\Stackrel{\rm\cref{fig:2}}{\geq}\quad\
    \frac{(1-2\lambda-c)}{2(1-2\lambda)}.\yesnum\label{eq:stat_rate_num_case_c2_3}
  \end{align*}
  \fi
  We can compute the statistical rate of $f$ using Equations~\eqref{eq:stat_rate_den_case_c2_3} and  \eqref{eq:stat_rate_num_case_c2_3}
  \begin{align*}
    \Omega_{\cD_k}(f) \quad&=\quad \frac{\min_{\ell\in [p]}
    \Ex\nolimits_{\cD_k}\insquare{f(X,Z)\mid Z=\ell}  }{\max_{\ell\in [p]} \Ex\nolimits_{\cD_k}\insquare{f(X,Z)\mid Z=\ell}   }\\
    &\Stackrel{\eqref{eq:stat_rate_num_case_c2_3}, \eqref{eq:stat_rate_den_case_c2_3}}{\leq}\quad
    \frac{ 1-\frac{c}{2\lambda} }{1-\frac{c}{1-2\lambda}}\\
    \quad &\leq\quad  \inparen{1-\frac{c}{2\lambda}}\cdot \inparen{ 1+\frac{c}{1-2\lambda}
    + \frac{2c^2}{(1-2\lambda)^2}}\tag{Using that for all $x\leq \frac12$, $\frac1{1-x}\leq 1+x+2x^2$}\\
    \quad &=\quad  1+\frac{c}{1-2\lambda}-\frac{c}{2\lambda}
    + \frac{2c^2}{(1-2\lambda)^2}\cdot \inparen{1-\frac{c}{2\lambda}} - \frac{c^3}{2\lambda(1-2\lambda)^2}\\
    \quad &\leq \quad  1-\frac{c(1-4\lambda)}{2\lambda(1-2\lambda)} + 8c^2
    \tag{Using that $0<\lambda\leq \frac14$ and $c>0$}\\
    \quad &\leq \quad 1-\frac{c}{2\lambda}\cdot (1-4\lambda) + \frac{3c^2}{4\lambda^2}.
    \tag{Using that $0\leq \lambda\leq \frac14$}
  \end{align*}
  {Thus, in this case, $f$ has a statistical rate at most $ 1-\frac{c}{2\lambda}\cdot (1-4\lambda) + \frac{3c^2}{4\lambda^2}$ on the chosen distribution $\cD_k$.}
\end{proof}

\subsubsection{Proof of \texorpdfstring{\cref{lem:3:ptb}}{Lemma 6.20}}
\begin{proof}[Proof of \cref{lem:3:ptb}]
  For each distribution $\cD_1$, $\cD_2$, and $\cD_3$, we will give classifiers $f_1,f_2,f_3\in \cF$, such that, for each $k\in [3]$,
  $f_k$ has error at most $\nfrac{3\alpha}2$ and a statistical rate 1 with respect to $\cD_k$.
  The idea is to choose a classifier $f_k$ such that, conditioned on a value of $Z$, they label exactly $\nfrac12$-fraction of the samples as positive on $\cD_k$.
  Thus, they have a statistical rate of 1.
  Then, by the construction it follows that these classifiers also have a small error.
  We give the classifier $f_1$ for distribution $\cD_1$.
  Classifiers $f_2$ and $f_3$ follow by symmetry.
  Define $f_1\in \cF$ to be the following classifier
  \begin{align*}
    f_1(x,z)\coloneqq \begin{cases}
    1 & \text{if } x\in \inbrace{x_B,x_D},\\
    1 & \text{if } x=x_C \text{ and } z=2,\\
    0 & \text{otherwise.}
  \end{cases}
\end{align*}
Using Equation~\eqref{eq:deterministic_y_lemma_c_7} and \cref{fig:2}, we get that
\begin{align*}
  {\rm Err}_{\cD_1}(f_1)\stackrel{}{=} \frac{3c}{2}.
\end{align*}
For $Z=2$, we have
\begin{align*}
  \Ex\nolimits_{\cD_1}\insquare{f_1(X,Z)\mid Z=2}\quad\
  &=\quad\  \sum_{r\in \cX} f_1(r,2)\cdot\Pr\nolimits_{\cD_1}\insquare{X=r\mid Z=2}\\
  &\Stackrel{}{=}\quad\  \Pr\nolimits_{\cD_1}\insquare{X\in \inbrace{x_B, x_C, x_D}\mid Z=2}\\
  &\Stackrel{\rm\cref{fig:2}}{=}\quad\   \frac12.\yesnum\label{eq:stat_rate_den_3_ptb}
\end{align*}
For $Z=1$, we have
\begin{align*}
  \Ex\nolimits_{\cD_1}\insquare{f(X,Z)\mid Z=1}\quad\
  &=\quad\  \sum_{r\in \cX} f_1(r,1)\cdot\Pr\nolimits_{\cD_1}\insquare{X=r\mid Z=1}\\
  &\Stackrel{}{=}\quad\  \Pr\nolimits_{\cD_1}\insquare{X\in \inbrace{x_B, x_D}\mid Z=1}\\
  &\Stackrel{\rm\cref{fig:2}}{=}\quad\   \frac12.\yesnum\label{eq:stat_rate_num_3_ptb}
\end{align*}
Thus, we have
\begin{align*}
  \Omega_{\cD_1}(h) = \frac{\min_{\ell\in [p]} \Ex\nolimits_{\cD_1}\insquare{f(X,Z)\mid Z=\ell}  }{\max_{\ell\in [p]} \Ex\nolimits_{\cD_1}\insquare{f(X,Z)\mid Z=\ell}   }%
  &\quad\Stackrel{\eqref{eq:stat_rate_den_3_ptb}, \eqref{eq:stat_rate_num_3_ptb}}{=}\quad  1.
\end{align*}
\end{proof}

%

%
\subsection{{Proof of \texorpdfstring{\cref{thm:imposs_under_assumption_sr_2}}{Theorem 6.21}}}\label{sec:additional_imposs_result}
\cref{thm:imposs_under_assumption_sr_2} assumes that: (1) $\cF$ shatters a set $\inbrace{x_A,x_B,x_C}\times [2]\subseteq \cX\times [p]$, (2) \cref{asmp:1} holds with a constant $\lambda\in (0,\nfrac14]$, (3) $\tau=1$, and (4) statistical rate is the fairness metric.
Recall that statistical rate of $f\in \cF$ with respect to distribution $\cD$ is
$$\Omega_\cD(f)\coloneqq \frac{\min_{\ell\in[p]}\Pr_\cD[f=1\mid Z=\ell]}{\max_{\ell\in[p]} \Pr_\cD[f=1\mid Z=\ell]}.$$
Then, given parameters $\eta\in (0,\nfrac12]$, $\lambda\in (0,\nfrac14]$, and $\delta\in [0,\nfrac12)$
our goal is to prove that for any
\begin{align*}
  \text{$0<\eps<\eta$ and $\nu\in [0,1],$}
\end{align*}
$\cF$ is not ($\eps,\nu$)-learnable with perturbation rate $\eta$ and confidence $\delta$.

Our proof is inspired by the approach of \cite[Theorem  1]{KearnsL93} and \cite[Theorem 1]{bshouty2002pac}.
It has a similar structure to the proof of \cref{thm:no_stable_classifier}; {but uses a different construction in which label $Y$ depends on $Z$ conditioned on $X$.}
(For a reader who has read the proof of \cref{thm:no_stable_classifier}, this means that condition (C3) in \cref{prop:to_lem_1} does not hold; and we have to give a different algorithm for the adversary to ``hide'' the true distribution.)

\subsubsection{{Proof of \texorpdfstring{\cref{thm:imposs_under_assumption_sr_2}}{Theorem 6.21}}}
\begin{proof}[Proof of \cref{thm:imposs_under_assumption_sr_2}]
  Let $\cL$ be any learner.
  We construct two distributions $\cP$ and $\cQ$ (parameterized by $\eta$) such that \cref{lem:1:ptc} holds (see \cref{fig:thmd1,fig:thmd1:1} for a description of $\cP$ and $\cQ$).
  \begin{lemma}[\bf Adversary can hide the true distribution]\label{lem:1:ptc}
    For both $\cD\in\inbrace{\cP,\cQ}$, given
    $N\in \N$ iid samples $S$ from $\cD$,
    if $\eta\cdot N$ is integral, then there is a distribution $\cD_{\rm Mix}$ over $\cX\times \zo\times [p]$ and an adversary $A\in\cA(\eta)$
    such that
    with probability at least $\nfrac12$ (over the draw of $S$)
    samples in $\hS\coloneqq A(S)$ are distributed as iid draws from $\cD_{\rm Mix}$.
  \end{lemma}
  \noindent Suppose $\cL$ is given $N\in \N$ samples $S$; where $N\cdot \eta$ is integral.
  Consider two cases:
  In the first case, $S$ is iid from $\cP$ and in the second case, $S$ is iid from $\cQ$.
  By \cref{lem:1:ptc}, in each case
  there is an $A\in \cA(\eta)$ such that, with probability at least $\nfrac12$,
  the perturbed samples $\hS\coloneqq A(S)$ are have the distribution $\cD_{\rm Mix}$.

  Thus, given $\hS$, with probability at least $\nfrac12$, the learner $\cL$ cannot identify the distribution from which $S$ was drawn.
  As a result, with probability at least $\nfrac12$, $\cL$ must to output a common classifier, say $f_{\rm Com}\in \cF$, which
  satisfies the accuracy and fairness guarantee
  for both $\cP$ and $\cQ$.
  We show that no $f_{\rm Com}\in \cF$  satisfies the accuracy and fairness guarantee for both $\cP$ and $\cQ$.
  \begin{lemma}[\bf No good classifier for both $\cP$ and $\cQ$]\label{lem:2:ptc}
    There is no $f\in \cF$, such that
    $${\rm Err}_{\cP}[f] < \eta\quad\text{ and }\quad{\rm Err}_{\cQ}[f] < \eta.$$
  \end{lemma}
  \begin{lemma}[\bf A good classifier for each $\cP$ and $\cQ$]\label{lem:3:ptc}
    For each $\cD\in\inbrace{\cP,\cQ}$, there is $f_{\cD}^\star\in \cF$ such that
    $${\rm Err}_{\cD}[f_{\cD}^\star]=0\quad\text{ and }\quad\Omega_\cD(f_\cD) =1.$$
  \end{lemma}
  \noindent Note that $f_\cP^\star$ (respectively $f_\cQ^\star$) has perfect accuracy and satisfies the fairness constraints with respect to  $\cP$ (respectively $\cQ$).
  If $\cL$ is a ($\eps,\nu$)-learner, then the first case, with probability at least $1-\delta>\nfrac12$,
  $\cL$ must output a classifier $f_1$ which satisfies
  \begin{align*}
    {\rm Err}_\cP(f_1) - {\rm Err}_\cP(f_\cP^\star) &\leq \eps \quad\text{ and }\quad\tau - \Omega_{\cP}(f_1) \leq\nu.
  \end{align*}
  And in the second case, with probability at least $1-\delta$,
  $\cL$ must output a classifier $f_2$ which satisfies
  \begin{align*}
    {\rm Err}_\cQ(f_2) - {\rm Err}_\cQ(f_\cQ^\star) &\leq \eps \quad\text{ and }\quad\tau - \Omega_{\cQ}(f_2) \leq\nu.
  \end{align*}
  However, in both cases, with probability at least $\nfrac12$, $\cL$ outputs $f_{\rm Com}$.
  Because $1-\delta>\nfrac12$, $f_{\rm Com}$ must satisfy
  \begin{align*}
    {\rm Err}_\cP(f_{\rm Com}) - {\rm Err}_\cP(f_\cP^\star) &\leq \eps \quad\text{ and }\quad {\rm Err}_\cQ(f_{\rm Com}) - {\rm Err}_\cQ(f_\cQ^\star)\leq \eps\yesnum\label{eq:contradiction_thm_imposs_2}\\
    \tau - \Omega_{\cP}(f_{\rm Com}) &\leq \nu \quad\text{ and }\quad \tau - \Omega_{\cQ}(f_{\rm Com})\leq \nu\yesnum
  \end{align*}
  But since $\eps<\eta$, Equation~\eqref{eq:contradiction_thm_imposs_2} contradicts \cref{lem:1:ptc}.
  Thus, $\cL$ is not an ($\eps,\nu$)-learner for $\cF$.
  Since the choice of $\cL$ was arbitrary, we have shown that there is no learner which ($\eps,\nu$)-learns $\cF$.

  It remains to prove the \cref{lem:1:ptc}, \cref{lem:2:ptc}, and \cref{lem:3:ptc}.
\end{proof}
\subsubsection{{Proof of \texorpdfstring{\cref{lem:1:ptc}}{Lemma 6.22}}}
Set $\cP$ and $\cQ$ to be the unique distributions satisfying the following properties:

\begin{figure}[t]
  \centering
  \begin{tabular}{c>{\centering\arraybackslash}p{1.75cm}>{\centering\arraybackslash}p{1.75cm}>{\centering\arraybackslash}p{1.75cm}}
    \toprule
    \midsepremove{}
    & $x_A$ & $x_B$ & $x_C$\\
    \midrule
    $1$ & $\nfrac{\eta}2$ & $\nfrac{1}2-\eta$ & $\nfrac{\eta}2$\\
    $2$ & $\nfrac{\eta}2$ & $\nfrac{1}2-\eta$ & $\nfrac{\eta}2$\\
    \bottomrule
  \end{tabular}
  \caption{
  \ifconf\small\fi
  Marginal distributions of $\cP$ and $\cQ$ over $\cX\times [2]$.
  For $(r,s)\in\inbrace{x_A,x_B,x_C}\times \insquare{2}$, the corresponding cell denotes $\Pr\nolimits_{\cP}\insquare{(X,Z)=(r,s)}=\Pr\nolimits_{\cQ}\insquare{(X,Z)=(r,s)}.$
  }
  \vspace{6mm}
  \label{fig:thmd1}
\end{figure}
\begin{figure}[t]
  \centering
  \subfigure[{\em Distribution  $\cP$: }For $(r,s)\in\inbrace{x_A,x_B,x_C}\times \insquare{2}$, the corresponding cell denotes the value of $Y$ conditioned on $X$ and $Z$, when sample $(X,Y,Z)\sim \cP$.\label{fig:thmd1:2}]
  {
  \hspace{20mm}\begin{tabular}{c>{\centering\arraybackslash}p{1.75cm}>{\centering\arraybackslash}p{1.75cm}>{\centering\arraybackslash}p{1.75cm}}
  \toprule
  \midsepremove{}
  & $x_A$ & $x_B$ & $x_C$\\
  \midrule
  $1$ & $1$ & $1$ & $0$\\
  $2$ & $0$ & $1$ & $1$\\
  \bottomrule
  \vspace{0.25mm}
\end{tabular}
\vspace{0.25mm}
\hspace{20mm}
}
\hspace{5mm}
\subfigure[{\em Distribution  $\cQ$: }For $(r,s)\in\inbrace{x_A,x_B,x_C}\times \insquare{2}$, the corresponding cell denotes the value of $Y$ conditioned on $X$ and $Z$, when  sample $(X,Y,Z)\sim \cQ$.\label{fig:thmd1:3}]
{
\hspace{20mm}\begin{tabular}{c>{\centering\arraybackslash}p{1.75cm}>{\centering\arraybackslash}p{1.75cm}>{\centering\arraybackslash}p{1.75cm}}
\toprule
\midsepremove{}
& $x_A$ & $x_B$ & $x_C$\\
\midrule
$1$ & $0$ & $1$ & $1$\\
$2$ & $1$ & $1$ & $0$\\
\bottomrule
\vspace{0.25mm}
\end{tabular}
\vspace{0.25mm}
\hspace{20mm}
}
\caption{
\ifconf\small\fi
\cref{fig:thmd1:2} denotes the conditional distribution of $Y$ when $(X,Y,Z)\sim \cP$.
\cref{fig:thmd1:3} denotes the conditional distribution of $Y$ when $(X,Y,Z)\sim \cQ$.
\label{fig:thmd1:1}
}
\end{figure}

\begin{enumerate}[itemsep=\itemsepINTERNAL,leftmargin=\leftmarginINTERNAL]
  \item[(P1)] $\cP$ and $\cQ$ have the same marginal distribution on $\cX\times [p]$; their marginal distribution on $\cX\times [p]$ is given in \cref{fig:thmd1}.
  \item[(P2)] \conftrue for a sample $(X,Y,Z)\sim \cP$, \cref{fig:thmd1:2} denotes the distribution of $Y$ conditioned on $X$ and $Z$, and for a sample $(X,Y,Z)\sim \cQ$, {\cref{fig:thmd1:3} denotes the distribution of $Y$ conditioned on $X$ and $Z$.}
\end{enumerate}
From \cref{fig:thmd1,fig:thmd1:1}, it follows that, conditional on $X=x_B$ the samples $(X,Y,Z)\sim\cP$ and $(X,Y,Z)\sim\cQ$ follow the same distribution.
Then, the goal of the adversary $A$ (in \cref{lem:1:ptc}) is to specifically perturb the samples with $X\neq x_B$ so that the perturbed sample $(X,Y,\hZ)$ follows the same distribution irrespective of whether the original sample $(X,Y,Z)$ is drawn from $\cP$ or $\cQ$.

\begin{proof}[Proof of \cref{lem:1:ptc}]
  Let $S\coloneqq \sinbrace{(X_i,Z_i,Y_i)}_{i\in [N]}$.
  Let $A\in \cA(\eta)$ execute the following algorithm:
  \begin{mdframed}
    \begin{enumerate}[itemsep=\itemsepINTERNAL]
      \item {\bf For} $i\in [N]$ {\bf do:}
      \begin{enumerate}[itemsep=\itemsepINTERNAL,leftmargin=\leftmarginINTERNAL]
        \item {\bf Sample} $t_i$ uniformly at random from $[0,1]$
        \item {\bf If} $X_i=x_B$ {\bf or} $t_i\leq \nfrac12$ {\bf then:} {\bf Set} $\wt{Z}_i\coloneqq Z_i$,
        \item {\bf Otherwise:} {\bf Set} $\wt{Z}_i\coloneqq {3-Z_i}$ \hfill\commentalg{If $Z_i=1$ set $\wt{Z}_i=2$, if $Z_i=2$ set $\wt{Z}_i=1$}
      \end{enumerate}
      \item {\bf If} $\sum_{i\in [N]} \mathds{I}\sinsquare{\wt{Z}_i\neq Z_i}\leq \eta\cdot N$ {\bf then:} {\bf return} $\sinbrace{(X_i,\wt{Z}_i,Y_i)}_{i\in [N]}$,
      \item {\bf Otherwise:} {\bf return} $\sinbrace{(X_i,Z_i,Y_i)}_{i\in [N]}$
    \end{enumerate}
  \end{mdframed}
  {Intuitively, $A$ perturbs the samples with $X\neq x_B$, such that whether $S$ is drawn from $\cP$ or $\cQ$, the following invariant holds:
  For any $s\in [2]$,}
  \begin{align*}
    \Pr[Y=0\mid X\neq x_B, \wt{Z}=s]=\Pr[Y=1\mid X\neq x_B, \wt{Z}=s]=\nfrac12.
  \end{align*}
  We will show that combining this with the facts that $\cP$ and $\cQ$ have the same marginal distribution on $\cX\times [p]$, and that conditioned on $X=x_B$ samples from $\cP$ and $\cQ$ have the same distribution suffices to prove \cref{lem:1:ptc}.

  The check in Step 2 ensures that $A$ never perturbs more than  $\eta\cdot N$ samples; thus, $A$ is admissible in the $\eta$-Hamming model.
  \cref{lem:1_1_c} shows that the check in Step 2 is true with probability at least $\nfrac12$.
  \begin{lemma}\label{lem:1_1_c}
    If $\eta\cdot N$ is integral, then
    with probability at least $\nfrac12$,
    $$\sum\nolimits_{i\in [N]} \mathds{I}\sinsquare{\wt{Z}_i\neq Z_i}<\eta\cdot N.$$
  \end{lemma}
  \begin{proof}
    For all $i\in [N]$, let $C_i\in \zo$ be a random variable indicating if $\wt{Z}_i\neq Z_i$.
    It suffices to show that with probability at least $\nfrac12$, $$\sum\nolimits_{i\in [N]}C_i\leq \eta\cdot N.$$
    Towards this, we first compute $\Pr[C_i=1]$ as follows
    \begin{align*}
      \Pr[C_i=1] &= \Pr[Z_i\neq \wt{Z}_i]\tag{Definition of $C_i$}\\
      &= \Pr[X_i\neq x_B, t_i> \nfrac12]\tag{Using Step 1(b) and Step 1(c)}\\
      &= \Pr[X_i\neq x_B] \cdot \Pr[t_i> \nfrac12]\tag{$(X_i,Y_i,Z_i)$ and $t_i$ are independent}\\
      &= 2\eta \cdot \frac12\tag{Using $t_i$ is uniformly at random from $[0,1]$ and \cref{fig:thmd1}}.
    \end{align*}
    Since for each $i$, the sample $(X_i,Y_i,Z_i)$ and the point $t_i$ is drawn independently of each other, other samples, and other points, it follows that the random variables $\inbrace{C_i}_{i}$ are independent.
    Hence, $\sum_{i\in [N]}C_i$ follows a binomial distribution with mean $\eta\cdot N$.
    Since $\eta\cdot N$ is integral, we have that the median of the distribution of $\sum_{i\in [N]}C_i$ is $\eta\cdot N$.
    Thus, we get that $$\Pr\insquare{\sum\nolimits_{i\in [N]}C_i\leq \eta\cdot N}= \nfrac12.$$
  \end{proof}
  \noindent Thus, \cref{lem:1_1_c} shows that with probability at least $\nfrac12$, $\hS\coloneqq {\sinbrace{(X_i,\wt{Z}_i,Y_i)}_{i\in [N]}}$.
  \begin{lemma}\label{lem:1_2_c}
    There is a distribution $\cD_{\rm Mix}$, such that,
    for all $\cD\in \inbrace{\cP,\cQ}$,
    given $S$ iid from $\cD$, the
    samples $\wt{S}\coloneqq \sinbrace{(X_i,\wt{Z}_i,Y_i)}_{i\in [N]}$ (computed in the algorithm of $A$) are independent of each other and is distributed according $\cD_{\rm Mix}$.
  \end{lemma}
  \begin{proof}
    Since for each $i\in [N]$ the sample $(X_i,Y_i,Z_i)$ and the point $t_i$ is drawn independently of others, it follows that the samples $(X_i,\wt{Z}_i,Y_i)$ are independent of each other.

    Next, we show that $\wt{S}$ follows the same distribution whether $S$ is iid from $\cP$ or from $\cQ$.
    The distribution $\cD_{\rm Mix}$ will be implicit in the proof.
    Since samples in $S$ are iid and the algorithm of $A$ does not depend on $i$, it follows that all samples in $\wt{S}$ follow the same distribution.
    Thus, it suffices to consider any one sample in $\wt{S}$.
    Fix any $i\in[N]$.
    Consider the $i$-th sample  $(X_i,\wt{Z}_i,Y_i)$.
    For the rest of the proof we drop the subscript $i$ in $(X_i,\wt{Z}_i,Y_i)$, $(X_i,{Z}_i,Y_i)$, and $t_i$.

    Fix any $y\in \zo$ and $s\in [2]$.
    Then, it holds that
    \begin{align*}
      \Pr[X=x_A,Y=y,\wt{Z}=s]
      &=\quad \Pr[X=x_A,Y=y,\wt{Z}=s\mid t\leq \nfrac12]\cdot \Pr[t\leq \nfrac12]\\
      &\quad+\Pr[X=x_A,Y=y,\wt{Z}=s\mid t> \nfrac12]\cdot \Pr[t> \nfrac12]\\
      &=\quad \Pr[X=x_A,Y=y,Z=s\mid t\leq \nfrac12]\cdot \Pr[t\leq \nfrac12]\\
      &\quad+\Pr[X=x_A,Y=y,Z=3-s\mid t> \nfrac12]\cdot \Pr[t> \nfrac12]\\
      &=\quad \Pr[X=x_A,Y=y,Z=s]\cdot\frac12\\
      &\quad+\Pr[X=x_A,Y=y,Z=3-s]\cdot\frac12\\
      &=\quad \Pr[X=x_A,Y=y]\cdot\frac12. \tag{Using that $Z\in \zo$}
    \end{align*}
    Now, from \cref{fig:thmd1} it follows that $\Pr[X=x_A,Y=y]=\nfrac{\eta}2$ both when $(X,Y,Z)\sim \cP$ or when $(X,Y,Z)\sim \cQ.$
    Thus, we get that
    \begin{align*}
      \Pr[X=x_A,Y=y,\wt{Z}=s] = \frac{\eta}4.
    \end{align*}
    Replacing $x_A$ by $x_C$ in the above argument, we get that
    \begin{align*}
      \Pr[X=x_C,Y=y,\wt{Z}=s] = \frac{\eta}4.
    \end{align*}
    Finally, since $A$ does not perturb samples with $X=x_B$ (see Step 1(b)), from \cref{fig:thmd1,fig:thmd1:1} it follows that
    \begin{align*}
      \Pr[X=x_B,Y=y,\wt{Z}=s]
      &= \Pr[X=x_B,Y=y,{Z}=s]\\
      &= \mathds{I}[y=1]\cdot\inparen{\frac{1}2-\eta}.
    \end{align*}
    Since the choice of $i\in [N]$, $y\in \zo$, and $s\in [2]$ was arbitrary.
    We get that all samples in $\wt{S}$ follow the same distribution whether $S$ is iid from $\cP$ or from $\cQ$.
  \end{proof}
  \noindent \cref{lem:1:ptc} follows
  because with probability at least $\nfrac12$, $\hS$ is ${\sinbrace{(X_i,\wt{Z}_i,Y_i)}_{i\in [N]}}$ (by \cref{lem:1_1_c}) and the samples ${\sinbrace{(X_i,\wt{Z}_i,Y_i)}_{i\in [N]}}$ are independent and distributed according to $\cD_{\rm Mix}$ (by \cref{lem:1_2}).
\end{proof}

\subsubsection{{Proof of \texorpdfstring{\cref{lem:2:ptc}}{Lemma 6.23}}}
\begin{proof}[Proof of \cref{lem:2:ptc}]
  Fix any $f\in \cF$.
  \ifconf
  \begin{align*}
    &{\rm Err}_\cP(f) + {\rm Err}_\cQ(f)\\
    &\quad = \Pr\nolimits_{(X,Y,Z)\sim \cP}[f(X,Z)\neq Y]+\Pr\nolimits_{(X,Y,Z)\sim \cQ}[f(X,Z)\neq Y]\tag{Using the definition of ${\rm Err}$}\\
    &\quad = \sum_{r\in \cX, s\in [2]} \Pr\nolimits_{(X,Y,Z)\sim \cP}[X=r,Z=s]\cdot\Pr\nolimits_{(X,Y,Z)\sim \cP}[f(X,Z)\neq Y\mid X=r,Z=s]\\
    &\qquad\qquad\quad  +\
    \Pr\nolimits_{(X,Y,Z)\sim \cQ}[X=r,Z=s]\cdot\Pr\nolimits_{(X,Y,Z)\sim \cQ}[f(X,Z)\neq Y\mid X=r,Z=s]\\
    &\quad \geq \sum_{r\in \sinbrace{x_A,x_C}, s\in [2]} \frac{\eta}2\cdot\Pr\nolimits_{(X,Y,Z)\sim \cP}[f(X,Z)\neq Y\mid X=r,Z=s]\\
    &\qquad\qquad\quad  +\qquad\
    \frac{\eta}2\cdot\Pr\nolimits_{(X,Y,Z)\sim \cQ}[f(X,Z)\neq Y\mid X=r,Z=s]\\
    &\quad = \sum_{r\in \sinbrace{x_A,x_C}, s\in [2]} \frac{\eta}2\cdot\Pr\nolimits_{(X,Y,Z)\sim \cP}[f(X,Z)\neq 0\mid X=r,Z=s]\\
    &\qquad\qquad\quad +\qquad\
    \frac{\eta}2\cdot\Pr\nolimits_{(X,Y,Z)\sim \cQ}[f(X,Z)\neq 1\mid X=r,Z=s]\tag{Using Property (P2) of $\cP$ and $\cQ$; see \cref{fig:thmd1:1}}\\
    &\quad = \sum_{r\in \sinbrace{x_A,x_C}, s\in [2]} \frac{\eta}2\cdot\Pr\nolimits_{(X,Y,Z)\sim \cP}[f(X,Z)\neq 0\mid X=r,Z=s]\\
    &\qquad\qquad\quad +\qquad\
    \frac{\eta}2\cdot\Pr\nolimits_{(X,Y,Z)\sim \cP}[f(X,Z)\neq 1\mid X=r,Z=s]\tag{Using Property (P1) of $\cP$ and $\cQ$; see \cref{fig:thmd1}}\\
    &\quad = \sum_{r\in \sinbrace{x_A,x_C}, s\in [2]} \frac{\eta}2\\
    &\quad = 2\eta.
  \end{align*}
  \else
  \begin{align*}
    {\rm Err}_\cP(f) + {\rm Err}_\cQ(f)
    &= \sum_{\cD\in \inbrace{\cP,\cQ}}\Pr\nolimits_{\cD}[f(X,Z)\neq Y]\tag{Using the definition of ${\rm Err}$}\\
    &= \sum_{r\in \cX, s\in [2]}\sum_{\cD\in \inbrace{\cP,\cQ}} \Pr\nolimits_{\cD}[X=r,Z=s]\cdot\Pr\nolimits_{\cP}[f(X,Z)\neq Y\mid X=r,Z=s]\\
    & \geq \sum_{r\in \sinbrace{x_A,x_C}, s\in [2]} \frac{\eta}2\cdot\inparen{\Pr\nolimits_{\cP}[f(X,Z)\neq Y\mid X=r,Z=s]+\Pr\nolimits_{\cQ}[f(X,Z)\neq Y\mid X=r,Z=s]}\\
    & = \sum_{r\in \sinbrace{x_A,x_C}, s\in [2]} \frac{\eta}2\cdot\inparen{\Pr\nolimits_{\cP}[f(X,Z)\neq 0\mid X=r,Z=s] + \Pr\nolimits_{\cQ}[f(X,Z)\neq 1\mid X=r,Z=s]} \tag{Using Property (P2) of $\cP$ and $\cQ$; see \cref{fig:thmd1:1}}\\
    & = \sum_{r\in \sinbrace{x_A,x_C}, s\in [2]} \frac{\eta}2\cdot\inparen{\Pr\nolimits_{\cP}[f(X,Z)\neq 0\mid X=r,Z=s] + \Pr\nolimits_{\cP}[f(X,Z)\neq 1\mid X=r,Z=s]} \tag{Using Property (P1) of $\cP$ and $\cQ$; see \cref{fig:thmd1}}\\
    & = \sum_{r\in \sinbrace{x_A,x_C}, s\in [2]} \frac{\eta}2\\
    & = 2\eta.
  \end{align*}
  \fi
  Since ${\rm Err}_\cP(f),{\rm Err}_\cQ(f)\geq 0$,
  by the Pigeonhole principle either ${\rm Err}_\cP(f)\geq \eta$ or ${\rm Err}_\cQ(f)\geq \eta$.
\end{proof}

\subsubsection{{Proof of \texorpdfstring{\cref{lem:3:ptc}}{Lemma 6.24}}}
\begin{proof}[Proof of \cref{lem:3:ptc}]
  Consider a sample $(X,Y,Z)\sim \cP$
  Notice that, for some $r\in \cX$ and $s\in [2]$, conditioned on $X=r$ and $Z=s$, the label $Y$ is uniquely identified by \cref{fig:thmd1:2}.
  Let $f_{\cP}^\star$ be the classifier than given $(r,s)$ predicts the value in the corresponding cell of \cref{fig:thmd1:2}.
  Clearly,  ${\rm Err}_\cP(f_{\cP}^\star)=0.$
  Further, by our assumption that $\cF$ shatters the set $\inbrace{x_A,x_B,x_C}\times [2]\subseteq \cX\times [p]$, it follows that $f_{\cP}^\star\in \cF$.

  The construction of $f_{\cQ}^\star$ follows symmetrically by using \cref{fig:thmd1:3}.
\end{proof}

\section{Extensions of Theoretical Results}\label{sec:extended_thy_results}
\subsection{{Theoretical Results with Multiple Protected Attributes and Fairness Metrics}}\label{sec:multiple_fairness_metrics}
\newcommand{\genetol}{Program~\ref{prog:errtolerant:multiple_attributes}}
In this section, we extend \errtolerant{} to the general case with $m\in \N$ fairness constraints $\Omega^{\sexp{1}}, \Omega^{\sexp{2}},\dots,\Omega^{\sexp{k}}$ with respect to $m$ (not necessarily distinct) protected attributes $Z^{\sexp{1}},Z^{\sexp{2}},\dots,Z^{\sexp{m}}$. %

\begin{definition}[\bf General Error-tolerant program]
  Given a perturbation rate $\eta\in [0,1]$, constants $\lambda,\Delta\in (0,1]$, and
  for each $r\in [m]$, given a fairness constraint $\Omega^{\sexp{r}}$, corresponding events $\cE^{\sexp{r}}$ and $(\cE^\prime)^{\sexp{r}}$ (as in \cref{def:performance_metrics}),
  and protected attribute $Z^{\sexp{r}}\in [p_r]$,
  we define the general error-tolerant program for perturbed samples $\hS$, with empirical distribution is $\smash{\hD}$, as
  \begin{align*}
    \min\nolimits_{f\in \cF}&\quad {\rm Err}_{\hD}(f),\quad\tagnum{General-ErrTolerant}\customlabel{prog:errtolerant:multiple_attributes}{(General-ErrTolerant)}\\
    \st,&\quad \text{$\forall\ r\in [m]$,}
    \qquad\qquad\qquad\ \Omega^{\sexp{r}}(f, \hS) \geq \tau \cdot \inparen{\frac{1- (\eta+\Delta)/\lambda}  {  1 + (\eta+\Delta)/\lambda}}^2,
    \yesnum\label{eq:fairness_general_case}\\
    &\quad\text{$\forall\ r\in [m]$ and $\ell\in [p_{r}],$}\quad
    \Pr\nolimits_{\hD}\insquare{\cE^{\sexp{r}}(f),(\cE^\prime)^{\sexp{r}}(f), \hZ=\ell} \geq \lambda - \eta - \Delta. \yesnum\label{eq:lowerbound_general_case}
  \end{align*}
\end{definition}
\noindent Let $f^\star\in \cF$ have the lowest error subject to satisfying all fairness constraints with respect to $\cD$:
\begin{align*}
  f^\star \coloneqq \argmin\nolimits_{f\in \cF} {\rm Err}_\cD(f) \quad \st,\quad \text{for all $r\in [m]$,}\quad \Omega_{\cD}^{\sexp{r}}(f)\geq \tau.
\end{align*}
We need the following generalization of \cref{asmp:1}:
\begin{assumption}\label{asmp:2}
  There is a known constant $\lambda>0$ such that $$\min_{r\in [m]}\min_{\ell\in [p_{r}]}\Pr\nolimits_{\cD}[\cE^{\sexp{r}}(f^\star), (\cE^\prime)^{\sexp{r}}(f^\star), Z=\ell]\geq \lambda.$$
\end{assumption}

\begin{theorem}[\textbf{Extending \cref{thm:main_result} to multiple protected attributes and fairness constraints}]\label{thm:multiple_fairness_metrics}
  Suppose \cref{asmp:2} holds with constant $\lambda>0$ and $\cF$ has VC dimension $d\in\N$.
  Then, for all perturbation rates $\eta\in (0,\nfrac\lambda{2})$,
  fairness thresholds $\tau\in (0,1]$,
  bounds on error $\eps> 2\eta$ and constraint violation $\nu >  \nfrac{8\eta\tau}{(\lambda-2\eta)}$,
  and confidence parameters $\delta\in (0,1)$
  with probability at least $1-\delta$,
  the optimal solution $f_{\rm ET}\in \cF$ of \genetol{} with parameters $\eta$, $\lambda$, and $\Delta\coloneqq O\inparen{\min\inbrace{\eps-2\eta, \nu-\nfrac{8\eta\tau}{(\lambda-2\eta)},\lambda-2\eta}}$, and
  $N=\poly(d, \nfrac{1}{\Delta},  \log(\delta^{-1}\cdot \sum_{i\in [m]}p_i))$
  perturbed samples from the $\eta$-\Ham{} model satisfies
  \begin{align*}
    {\rm Err}_\cD(f_{\rm ET}) - {\rm Err}_\cD(f^\star) &\leq \eps,\yesnum\label{eq:1:exnd_results}\\
    \text{for all $r\in [m]$,}\quad \Omega_{\cD}^{\sexp{r}}(f_{\rm ET}) &\geq \tau -\nu.\yesnum\label{eq:2:exnd_results}
  \end{align*}
\end{theorem}
\noindent The proof of \cref{thm:multiple_fairness_metrics} is similar to the proof of \cref{thm:main_result}.
Instead of repeating the entire proof, we highlight the differences between the two proofs.
\begin{proof}
  Set $$N\coloneqq \Theta\inparen{\frac{1}{\Delta^2\cdot(\lambda-2\eta)^4}\cdot\inparen{d\log\inparen{\frac{d}{\Delta^2\cdot(\lambda-2\eta)^4}}+\log\inparen{\delta^{-1}\cdot \sum\nolimits_{r\in [m]}p_r}}}.$$
  \noindent Fix any fairness constraint $r\in [m]$.
  Applying \cref{lem:suff_cond_for_stable} to the $r$-th fairness constraint, we get that any $f\in \cF$ feasible for \genetol{} is $\sinparen{\frac{1- (\eta+\Delta)/(\lambda-2\eta)}  {  1 + (\eta+\Delta)/(\lambda-2\eta)}}^2$-stable with respect to the fairness constraint $\Omega^{\sexp{r}}$.
  This satisfies the first condition in \cref{lem:approx_feasible}.
  The second condition holds because any $f\in \cF$ feasible for \genetol{} satisfies the fairness constraint in Equation~\eqref{eq:fairness_general_case}.
  This allows us to use \cref{lem:approx_feasible};
  we get that with probability at least $1-\delta\cdot \frac{p_r}{\sum\nolimits_{i}p_i},$ any $f\in \cF$ feasible for \genetol{} satisfies that
  \begin{align*}
    \Omega_{\cD}^{\sexp{r}}(f_{\rm ET}) &\geq \tau -\nu.
  \end{align*}
  Using the union bound over $r\in [m]$, implies that Equation~\eqref{eq:2:exnd_results} holds with probability at least $1-\delta$.

  If we can show that $f^\star$ is feasible for \genetol{}, then \cref{lem:conditionalAccuracyGuarantee} implies Equation~\eqref{eq:1:exnd_results}.
  Using \cref{asmp:2} and \cref{lem:overview1}, it follows that $f^\star$ satisfies the lower bounds in \cref{eq:lowerbound_general_case} with probability at least $1- \frac{\delta}{\sum\nolimits_{i}p_i}.$
  This, along with \cref{lem:suff_cond_for_stable}, implies that $f^\star$ is $\inparen{\frac{1- (\eta+\Delta)/\lambda}  {  1 + (\eta+\Delta)/\lambda}}^2$-stable for all fairness metrics.
  Then, using \cref{coro:optf_feasible}, it follows that $f^\star$ satisfies the constraints for a particular $r\in [m]$ with probability at least $1-\delta\cdot \frac{p_r}{\sum\nolimits_{i}p_i}.$
  Taking the union bound over all $r\in [m]$, it follows that $f^\star$ is feasible for \genetol{} with probability at least $1-\delta$.
\end{proof}

\begin{remark}
  Like \errtolerant{}, in general, \genetol{} is also a nonconvex optimization program.
  But, for any arbitrarily small $\alpha>0$, the techniques from \cite{celis2019classification} can be used to find an $f\in \cF$ that has the optimal objective value for \genetol{} and that additively violates its fairness {constraint \eqref{eq:fairness_general_case}} by at most $\alpha$ by solving a set of $O\inparen{(\lambda\alpha)^{-m}}$ convex programs (see \cref{sec:efficient_algorithms_for_errtol} for details on an analogous argument for \errtolerant{}).
\end{remark}

\subsection{{Theoretical Results for Program~\ref{prog:errtolerant_extnd}}}\label{sec:guarantees_on_errTol_plus}
\newcommand{\extended}{Program~\ref{prog:errtolerant_extnd}}
In this section, we show that \extended{} offers a better fairness guarantee than \errtolerant{} (up to a constant) if the classifiers in $\cF$ do not use the protected attributes for prediction.
In particular, we prove \cref{thm:guarantees_for_extended}.
\begin{theorem}[\textbf{Guarantees for Program~\ref{prog:errtolerant_extnd}}]\label{thm:guarantees_for_extended}
  Suppose for each $\ell\in [p]$
  \begin{align*}
    \lambda_\ell\coloneqq\Pr\nolimits_{\cD}\insquare{\cE(f^\star), \cE^\prime(f^\star), Z=\ell}
    \quad\text{and}\quad
    \gamma_\ell \coloneqq \Pr\nolimits_\cD[\cE^\prime(f^\star), Z=\ell],
  \end{align*}
  $\cF$ has VC dimension $d\in\N$, and classifiers in $\cF$ do not use the protected attributes for prediction.
  Let $s$ be the optimal value of Program~\eqref{eq:...} and $\lambda\coloneqq \min_{\ell\in [p]} \lambda_\ell$.
  Then, for all perturbation rates $\eta\in (0,\nfrac\lambda{2})$,
  fairness thresholds $\tau\in (0,1]$,
  bounds on error $\eps> 2\eta$ and constraint violation $$\nu >  \tau\cdot \inparen{1-s+\frac{4\eta}{\lambda-2\eta}},$$
  and confidence parameters $\delta\in (0,1)$
  with probability at least $1-\delta$,
  the optimal solution $f_{\rm ET}\in \cF$ of \errtolerant{} with parameters $\eta$, $\lambda$, $$\Delta\coloneqq O\inparen{\min\inbrace{\eps-2\eta, \nu-\tau\cdot \inparen{1-s+\frac{4\eta}{\lambda-2\eta}},\lambda-2\eta}}\quad \text{and}\quad N\coloneqq \poly\inparen{d, \frac{1}{\Delta},  \log\inparen{\frac{p}\delta}} $$
  perturbed samples from an $\eta$-\Ham{} adversary satisfies
  \begin{align*}
    {\rm Err}_\cD(f_{\rm ET}) - {\rm Err}_\cD(f^\star) \leq \eps
    \quad \text{and}\quad
    \Omega_{\cD}(f_{\rm ET}) \geq \tau -\nu.\yesnum\label{eq:program_errtol+}
  \end{align*}
\end{theorem}

\subsubsection{Proof of \cref{thm:guarantees_for_extended}}
The proof of \cref{thm:guarantees_for_extended} is similar to the proof of \cref{thm:main_result}.
Instead of repeating the entire proof, we highlight the differences from the proof of \cref{thm:main_result}.

\begin{proof}[Proof of \cref{thm:guarantees_for_extended}]
  The proof of \cref{thm:main_result} has three main steps:

  \begin{enumerate}[itemsep=\itemsepINTERNAL,leftmargin=\leftmarginINTERNAL]
    \item	Step 1 proves that if $f^\star$ is feasible for \errtolerant{},
    then the $f_{\rm ET}$ has accuracy $2\eta$-close to the accuracy of $f^\star$ (\cref{lem:conditionalAccuracyGuarantee}).
    \item Step 2 proves that any $f\in \cF$ feasible for \errtolerant{} satisfies the fairness guarantee in \cref{thm:main_result} with high probability (\cref{lem:approx_feasible}).
    The main substep in the proof of \cref{lem:approx_feasible} is to show that any $f\in \cF$ feasible for \errtolerant{} is
    $\sinparen{\frac{1- (\eta+\Delta)/\lambda}  {  1 + (\eta+\Delta)/\lambda}}^2$-stable.
    \item Step 3 proves that, with high probability, $f^\star$ is feasible for \errtolerant{} (\cref{coro:optf_feasible}).
    Combining this with Step 1 shows that $f_{\rm ET}$ satisfies the accuracy guarantee in \cref{thm:main_result} with high probability.
  \end{enumerate}

  \noindent The proof of Step 1 (\cref{lem:conditionalAccuracyGuarantee}) only depends on the perturbation model and the objective of \errtolerant{}.
  Thus, it also generalizes to \extended{}.
  The proof of Step 2 (\cref{lem:approx_feasible}), follows because any $f\in \cF$ feasible for \errtolerant{} satisfies the following inequality:
  \begin{align*}
    \Pr\nolimits_{\hD}\insquare{\cE(f),\cE^\prime(f), \hZ=\ell} \geq \lambda - \eta - \Delta.\yesnum\label{eq:new_eq}
  \end{align*}
  Since for all $\ell\in [p]$, $\lambda_\ell\geq \lambda$,
  any $f\in \cF$ feasible for \extended{} also satisfies \cref{eq:new_eq}.
  Thus, \cref{lem:approx_feasible} also holds for \extended{}.
  This shows that any $f\in \cF$ feasible for \extended{} satisfies
  \begin{align*}
    \Omega_{\cD}(f_{\rm ET}) \geq \tau- \frac{8\eta\tau}{\lambda - 2\eta}.
  \end{align*}
  However, \cref{thm:guarantees_for_extended} has a tighter fairness guarantee.\footnote{Because for all $\ell\in [p]$, $\lambda_\ell\geq \lambda$ and $\gamma_\ell\geq \lambda$, it can be shown that $s\geq \sinparen{\frac{1- (\eta+\Delta)/\lambda}  {  1 + (\eta+\Delta)/\lambda}}^2$. Thus, the guarantee in \cref{thm:guarantees_for_extended} is stronger than the guarantee in \cref{thm:main_result}.}
  The tighter guarantee follows by lower bounding $\Omega(f,\hS)$ by $\tau\cdot s$ in \cref{eq:used_to_prove_guarantee_for_extension} in the proof of \cref{lem:approx_feasible}.

  The main difference in the proofs of \cref{thm:main_result} and \cref{thm:guarantees_for_extended} is in Step 3.
  Using \cref{asmp:1} and \cref{coro:whp_bounding_power}, we can show that $f^\star$ satisfies: $$\Pr\nolimits_{\hD}\insquare{\cE(f),\cE^\prime(f), \hZ=\ell} \geq \lambda_\ell - \eta - \Delta.$$
  It remains to show that $f^\star$ satisfies the fairness constraint:
  $$\Omega_{}(f, \hS) \geq \tau\cdot s.$$

  \noindent In the proof of \cref{coro:optf_feasible} we show a weaker result: $f^\star$ satisfies the fairness constraint:
  $$\Omega_{}(f, \hS) \geq \tau\cdot \inparen{\frac{1- (\eta+\Delta)/\lambda}  {  1 + (\eta+\Delta)/\lambda}}^2.$$
  This follows because $f^\star$ is $\sinparen{\frac{1- (\eta+\Delta)/\lambda}  {  1 + (\eta+\Delta)/\lambda}}^2$-stable.
  Instead, here, we prove that with high probability $f^\star$ satisfies the following inequality
  \begin{align*}
    \frac{\Omega(f^\star, \hS)}{\Omega_\cD(f^\star)}\geq s.\yesnum\label{eq:one_sided_stability}
  \end{align*}
  (\cref{eq:one_sided_stability} does not imply to $s$-stability because $s$-stability also requires an upper bound of $\nfrac1s$.)
  This suffices to show that $f^\star$ is feasible for Program~\ref{prog:errtolerant_extnd} with high probability;
  by our discussion so far, it also proves \cref{thm:guarantees_for_extended}.
\end{proof}

\noindent It remains to prove \cref{eq:one_sided_stability}.
Formally, we prove \cref{lem:new_lemma}.

\begin{lemma}\label{lem:new_lemma}
  Let $s$ be the optimal value of Program~\eqref{eq:...}, then with probability at least $1-\delta_0$ %
  \begin{align*}
    \frac{\Omega(f^\star, \hS)}{\Omega_\cD(f^\star)}\geq s,\yesnum\label{eq:one_sided_stability_2}
  \end{align*}
  where $f^\star$ is an optimal solution of Program~\eqref{prog:target_fair} and $\delta_0$ is as defined in \cref{eq:define_params}.
\end{lemma}
\begin{proof}
  Let the unperturbed samples be $S\coloneqq \inbrace{(x_i,y_i,z_i)}_{i\in [N]}$ and the perturbs samples be $\hS\coloneqq \inbrace{(x_i,y_i,\hz_i)}_{i\in [N]}$.
  We will prove that $$\frac{\Omega(f,\hS)}{\Omega(f,S)}\geq s.$$
  Then the result follows as for $N=\poly(d, \nfrac{1}{\Delta},  \log(\nfrac{p}\delta_0))$ iid samples from $\cD$,
  it holds that with probability at least $1-\delta_0$,
  $$\frac{\Omega(f,S)}{\Omega_\cD(f)} \geq 1-\Delta.$$

  \noindent Let $\cE$ and $\cE^\prime$ be the events defining the fairness metric $\Omega$ (\cref{def:performance_metrics}).
  For each $\ell\in [p]$, we define
  \begin{align*}
    \eta^1_\ell
    &\coloneqq
    \Pr\nolimits_{\hD}[\cE(f^\star(X)), \cE^\prime(f^\star(X)), \hZ=\ell]
    -
    \Pr\nolimits_D [\cE(f^\star(X)), \cE^\prime(f^\star(X)), Z=\ell],\\
    \eta^2_\ell
    &\coloneqq
    \Pr\nolimits_{\hD}[\lnot\cE(f^\star(X)), \cE^\prime(f^\star(X)), \hZ=\ell]
    -
    \Pr\nolimits_D [\lnot\cE(f^\star(X)), \cE^\prime(f^\star(X)), Z=\ell].
  \end{align*}
  Substituting the values of $\lambda_\ell$ and $\gamma_\ell$, we get that\footnote{Here, we implicitly use the fact that $f\in \cF$ does not use the protected attribute $Z$ for prediction.}
  \begin{align*}
    \eta^1_\ell &\coloneqq
    \Pr\nolimits_{\hD}[\cE(f^\star(X)), \cE^\prime(f^\star(X)), \hZ=\ell] - \lambda_\ell,
    \yesnum\label{eq:def_eta_1}\\
    \eta^2_\ell &\coloneqq
    \Pr\nolimits_{\hD}[\lnot\cE(f^\star(X)), \cE^\prime(f^\star(X)), \hZ=\ell] - \inparen{\gamma_\ell - \lambda_\ell}.
    \yesnum\label{eq:def_eta_2}
  \end{align*}
  Intuitively, $\eta^1_\ell$ is the number of samples with $\mathds{I}[\cE(f^\star(X)), \cE^\prime(f^\star(X))]=1$ added to protected group $Z=\ell$ and
  $\eta^2_\ell$ is the number of samples with $\mathds{I}[\lnot\cE(f^\star(X)), \cE^\prime(f^\star(X))]=1$ added to protected group $Z=\ell$.
  (Note that the event $\cE(f^\star(X))$ and $\lnot\cE(f^\star(X))$ are disjoint.)

  The values $\sinbrace{\eta^1_\ell,\eta^2_\ell}_{\ell}$ satisfy several conditions:
  Because the total number of samples added or removed from all protected groups is 0, it holds that $$\sum\nolimits_{\ell\in [p]} {\eta^1_\ell} + {\eta^2_\ell} = 0.$$
  Moreover, because $f$ does not use $Z$ for prediction, perturbing $Z$ does not change the value of $\mathds{I}[\cE(f^\star(X))$, $\cE^\prime(f^\star(X))]$, and hence, it holds that
  \begin{align*}
    \sum\nolimits_{\ell\in [p]} {\eta^1_\ell}= 0\quad\text{and}\quad\sum\nolimits_{\ell\in [p]} {\eta^2_\ell}= 0.\yesnum\label{eq:sum0bound}
  \end{align*}
  Next, because any $\eta$-\Ham{} adversary perturbs at most $\eta$-fraction of the samples, we have that
  \begin{align*}
    \sum\nolimits_{\ell\in [p]} \sabs{\eta^1_\ell} + \sabs{\eta^2_\ell} \leq 2\eta.\label{eq:2etabound}\yesnum
  \end{align*}
  Finally, as the probability in the right-hand side of \cref{eq:def_eta_2} is nonnegative, it follows that for all $\ell\in [p]$
  \begin{align*}
    \eta^2_\ell \geq - \inparen{\gamma_\ell - \lambda_\ell}.\yesnum\label{eq:bound_on_eta}
  \end{align*}
  Now we are ready to prove the result
  \begin{align*}
    \Omega(f^\star, \hS) \coloneqq \min_{\ell,k\in[p]}\frac{\Pr\nolimits_{\hD}[\cE^\prime(f^\star(X)), \wh{Z}=\ell]}{\Pr\nolimits_{\hD}[\cE(f^\star(X)), \cE^\prime(f^\star(X)), \wh{Z}=\ell]}
    \cdot  \frac{\Pr\nolimits_{\hD}[\cE(f^\star(X)), \cE^\prime(f^\star(X)), \wh{Z}=k]}{\Pr\nolimits_{\hD}[\cE^\prime(f^\star(X)), \wh{Z}=k]}.\yesnum\label{eq:tmp}
  \end{align*}
  Fix any $\ell,k\in [p]$.
  From \cref{eq:def_eta_1,eq:def_eta_2} we have that
  \begin{align*}
    {\frac{\Pr\nolimits_{\hD}[\cE^\prime(f^\star(X)), \wh{Z}=\ell]}{\Pr\nolimits_{\hD}[\cE(f^\star(X)), \cE^\prime(f^\star(X)), \wh{Z}=\ell]}
    \ifconf\else\cdot\fi
    \frac{\Pr\nolimits_{\hD}[\cE(f^\star(X)), \cE^\prime(f^\star(X)), \wh{Z}=k]}{\Pr\nolimits_{\hD}[\cE^\prime(f^\star(X)), \wh{Z}=k]}}
    =
    {
    \ifconf\textstyle\else\fi \frac{\lambda_\ell+\eta^1_\ell}{\gamma_\ell+\eta^1_\ell+\eta^2_\ell}\cdot \frac{\gamma_k+\eta^1_k+\eta^2_k}{\lambda_k+\eta^1_k}.}
  \end{align*}
  Our goal is to lower bound the left-hand side of the above equation because it implies a lower bound on $\Omega(f^\star, \hS)$ by \cref{eq:tmp}.
  Towards this, given vectors $\eta^1,\eta^2\in \R^p$, define the following objective:
  \begin{align*}
    {\rm Obj}(\eta^1,\eta^2)\coloneqq \frac{\lambda_\ell+\eta^1_\ell}{\gamma_\ell+\eta^1_\ell+\eta^2_\ell}\cdot \frac{\gamma_k+\eta^1_k+\eta^2_k}{\lambda_k+\eta^1_k}.
  \end{align*}
  We would like to lower bound ${\rm Obj}(\eta^1,\eta^2)$ subject to  \cref{eq:sum0bound,eq:2etabound,eq:bound_on_eta},
  i.e., we would like to lower bound the value of the following program
  \begin{align*}
    &{\rm Obj}(\eta^1,\eta^2),\yesnum\label{prog:1}\\
    \st,\quad & {\rm eqs.}~\eqref{eq:sum0bound},\ \eqref{eq:2etabound} \text{ and } \eqref{eq:bound_on_eta} {\rm\ hold}
  \end{align*}
  We claim that for all $i\not\in\inbrace{\ell,k}$ it is optimal to set $\eta^1_i=0$ in Program~\eqref{prog:1}.
  To see this, note that given any solution $\eta^1,\eta^2\in \R^p$ to Program~\eqref{prog:1}, with $\alpha\coloneqq \sum_{i\neq \ell,k}\eta^1_i$,
  we can construct another solution $\delta^1,\delta^2\in \R^p$ that has a better objective while satisfying $\delta^1_i=0$ for all $i\not\in\inbrace{\ell,k}$:
  For all $i\not\in \inbrace{\ell,k}$,
  set $\delta^1_i\coloneqq 0$.
  Next, if $\alpha>0$, set $\delta^1_k\coloneqq \eta^1_k+\alpha$, and otherwise, set $\delta^1_\ell\coloneqq \eta^1_\ell+\alpha$.
  Finally, let all other variables remain unchanged.
  (Under Equation~\eqref{eq:bound_on_eta}, $(\delta^1,\delta^2)$ has a smaller objective than $(\eta^1,\eta^2)$.)
  Thus, we get that\\[1mm]
  \begin{align*}
    \min_{\eta^1,\eta^2\in \R^p}&{\rm Obj}(\eta^1,\eta^2),\hspace{10mm}\white{.}
    &=\qquad
    \min_{\eta^1,\eta^2\in \R^p}&{\rm Obj}(\eta^1,\eta^2),\yesnum\label{prog:2}\\
    \st,\quad& {\rm eqs.}~\eqref{eq:sum0bound},\ \eqref{eq:2etabound} \text{ and } \eqref{eq:bound_on_eta} {\rm\ hold}
    &\quad\ \ \st,\quad& {\rm eqs.}~\eqref{eq:sum0bound},\ \eqref{eq:2etabound} \text{ and } \eqref{eq:bound_on_eta} {\rm\ hold}\\
    & & & \text{for all $i\not\in \inbrace{\ell,k}$, $\eta^1_i=0$}
  \end{align*}
  Next, dropping \cref{eq:bound_on_eta} from the constraint, which only improves the objective, we get that:
  \begin{align*}
    \min_{\eta^1,\eta^2\in \R^p} &{\rm Obj}(\eta^1,\eta^2),
    &\geq\qquad
    \min_{\eta^1,\eta^2\in \R^p} &{\rm Obj}(\eta^1,\eta^2),\yesnum\label{prog:3}\\
    \st,\quad& {\rm eqs.}~\eqref{eq:sum0bound},\ \eqref{eq:2etabound} \text{ and } \eqref{eq:bound_on_eta} {\rm\ hold}\hspace{-10mm}
    &\quad \st,\quad& {{\rm eqs.}~\eqref{eq:sum0bound} \text{ and } \eqref{eq:2etabound}} {\rm\ hold}\\
    &\text{for all $i\not\in \inbrace{\ell,k}$, $\eta^1_i=0$} & & \text{for all $i\not\in \inbrace{\ell,k}$, $\eta^1_i=0$}.
  \end{align*}
  We claim that for all $i\not\in\inbrace{\ell,k}$ it is optimal to set $\eta^2_i\coloneqq 0$ in the program in the right-hand side of \cref{prog:3}.
  This follows by a similar construction used to prove \cref{prog:2}.
  To see this, note that given any solution $\eta^1,\eta^2\in \R^p$ to the program in the right-hand side of \cref{prog:3}, with $\alpha\coloneqq \sum_{i\neq \ell,k}\eta^2_i$,
  there is another solution $\delta^1,\delta^2\in \R^p$ that has a better objective value while satisfying $\delta^2_i=0$ for all $i\not\in\inbrace{\ell,k}$:
  For all $i\not\in \inbrace{\ell,k}$, set $\delta^1_i\coloneqq 0$.
  Next, if $\alpha>0$, set $\delta^2_\ell\coloneqq \eta^2_\ell+\alpha$, otherwise set $\delta^2_k\coloneqq \eta^2_k+\alpha$.
  Finally, let all other variables remain unchanged.
  ($(\delta^1,\delta^2)$ always has a smaller objective than $(\eta^1, \eta^2)$.)
  Thus, we have
  \begin{align*}
    \min_{\eta^1,\eta^2\in \R^p} &{\rm Obj}(\eta^1,\eta^2),\hspace{10mm}\white{.}
    &=\qquad
    \min_{\eta^1,\eta^2\in \R^p} &{\rm Obj}(\eta^1,\eta^2),\yesnum\label{prog:4}\\
    \st,\quad& {\rm eqs.}~\eqref{eq:sum0bound} \text{ and } \eqref{eq:2etabound} {\rm\ hold}
    &\quad\ \ \st,\quad& {\rm eqs.}~\eqref{eq:sum0bound} \text{ and } \eqref{eq:2etabound} {\rm\ hold}\\
    & \text{for all $i\not\in \inbrace{\ell,k}$, $\eta^1_i=0$} & & \text{for all $i\not\in \inbrace{\ell,k}$, $\eta^1_i=0$, $\eta^2_i=0$}.
  \end{align*}
  Rewriting the program in the right-hand side of \cref{prog:4}, by dropping the always 0 variables, we get
  \begin{align*}
    \hspace{-5mm}\min_{\eta^1,\eta^2\in \R^p} &{\rm Obj}(\eta^1,\eta^2),\hspace{10mm}\white{.}
    &=\qquad\hspace{-5mm}
    \min_{\eta^1_\ell,\eta^1_k,\eta^2_\ell,\eta^2_k\in \R}\ \ & \frac{\lambda_\ell+\eta^1_\ell}{\gamma_\ell+\eta^1_\ell+\eta^2_\ell}\cdot \frac{\gamma_k+\eta^1_k+\eta^2_k}{\lambda_k+\eta^1_k},\yesnum\label{prog:5}\\
    \st,\quad& {\rm eqs.}~\eqref{eq:sum0bound} \text{ and } \eqref{eq:2etabound} {\rm\ hold}
    &\quad\ \ \st,\quad& \eta^1_\ell=-\eta^1_k\text{ and }\eta^2_\ell=-\eta^2_k\\
    & \text{for all $i\not\in \inbrace{\ell,k}$, $\eta^1_i=0$, $\eta^2_i=0$} & &
    \abs{\eta^1_\ell}+\abs{\eta^2_\ell}\leq \eta
  \end{align*}
  Rewriting the program in the right-hand side of \cref{prog:5}, we get
  \begin{align*}
    \min_{\eta^1_\ell,\eta^1_k,\eta^2_\ell,\eta^2_k\in \R}\ \ & \frac{\lambda_\ell+\eta^1_\ell}{\gamma_\ell+\eta^1_\ell+\eta^2_\ell}\cdot \frac{\gamma_k+\eta^1_k+\eta^2_k}{\lambda_k+\eta^1_k},
    &=\quad
    \min_{\eta_\ell,\eta_k\in \R} & \ \ \frac{\lambda_\ell-\eta_\ell}{\gamma_\ell-\eta_\ell+\eta_k}\cdot \frac{\gamma_k+\eta_\ell-\eta_k}{\lambda_k+\eta_\ell},\yesnum\label{prog:6}\\
    \st,\quad& \eta^1_\ell=-\eta^1_k\text{ and }\eta^2_\ell=-\eta^2_k,
    &\quad\ \ \st,\quad& \abs{\eta^1_\ell}+\abs{\eta^2_\ell}\leq \eta\\
    & \abs{\eta^1_\ell}+\abs{\eta^2_\ell}\leq \eta & &
  \end{align*}
  In the program in the right-hand side of \cref{prog:6}, the optimal $\eta_\ell,\eta_k$ satisfy: $\eta_\ell,\eta_k\geq 0.$
  This simplifies the constraint
  $$\sabs{\eta^1_\ell}+\sabs{\eta^2_\ell}\leq \eta\quad\text{to}\quad{\eta^1_\ell}+{\eta^2_\ell}\leq \eta.$$
  The sequence of equations, \cref{prog:2,prog:3,prog:4,prog:5,prog:6}, implies that
  \begin{align*}
    \Omega(f^\star, \hS)\ \
    &\geq\ \
    \min_{\ell,k \in [p]}\min_{\eta_\ell,\eta_k\geq 0} \frac{\lambda_\ell-\eta_\ell}{\gamma_\ell-\eta_\ell+\eta_k}\cdot \frac{\gamma_k+\eta_\ell-\eta_k}{\lambda_k+\eta_\ell},\quad \st,\quad  {\eta^1_\ell}+{\eta^2_\ell}\leq \eta \\
    &=\ \
    \min_{\ell,k \in [p]} \frac{\lambda_\ell\cdot \gamma_k}{\gamma_\ell\cdot \lambda_k}\cdot
    \min_{\eta_\ell,\eta_k\geq 0}
    \frac{1-\sfrac{\eta_\ell}{\lambda_\ell}}{1-\sfrac{(\eta_\ell-\eta_k)}{\gamma_\ell}}\cdot \frac{1+\sfrac{(\eta_\ell-\eta_k)}{\gamma_k}}{1+\sfrac{\eta_\ell}{\lambda_k}},\ \ \st,\ \  {\eta^1_\ell}+{\eta^2_\ell}\leq \eta \\
    &\geq\ \
    \inparen{\min_{\ell,k \in [p]} \frac{\lambda_\ell\cdot \gamma_k}{\gamma_\ell\cdot \lambda_k}}\cdot
    \inparen{\min_{\ell,k \in [p]}\min_{\eta_\ell,\eta_k\geq 0}
    \frac{1-\sfrac{\eta_\ell}{\lambda_\ell}}{1-\sfrac{(\eta_\ell-\eta_k)}{\gamma_\ell}}\cdot \frac{1+\sfrac{(\eta_\ell-\eta_k)}{\gamma_k}}{1+\sfrac{\eta_\ell}{\lambda_k}},\ \ \st,\ \  {\eta^1_\ell}+{\eta^2_\ell}\leq \eta}\\
    &=\ \
    \Omega(f^\star,S)\cdot
    \min_{\ell,k \in [p]}\min_{\eta_\ell,\eta_k\geq 0}
    \frac{1-\sfrac{\eta_\ell}{\lambda_\ell}}{1-\sfrac{(\eta_\ell-\eta_k)}{\gamma_\ell}}\cdot \frac{1+\sfrac{(\eta_\ell-\eta_k)}{\gamma_k}}{1+\sfrac{\eta_\ell}{\lambda_k}},\ \ \st,\ \  {\eta^1_\ell}+{\eta^2_\ell}\leq \eta \tag{By the definition of $\lambda_\ell$ and $\gamma_\ell$}\\
    &=\ \
    \Omega(f^\star,S)\cdot s.\tag{By the definition of $s$}
  \end{align*}
\end{proof}

\section{Limitations and Conclusion}\label{sec:conclusion}

This work extends fair classification to real-world settings where perturbations in the protected attributes may be correlated or affect arbitrary subsets of samples.
We consider the $\eta$-Hamming model and give a framework that outputs classifiers with provable guarantees on both fairness and accuracy;
this framework works for categorical protected attributes and the class of linear-fractional fairness metrics.
We show near-tightness of our framework's guarantee and extend it to the Nasty Sample Noise model, which can perturb both labels and features.
Empirically, classifiers produced by our framework achieve high fairness at a small cost to accuracy and outperform existing approaches.

Compared to existing frameworks for fair classification with stochastic perturbations, our framework requires less information about the perturbations.
That said, in a few applications, e.g.,  the randomized response procedure~\cite{warner}, where the perturbations are  independent across samples and identically distributed according to a {\em known} distribution, frameworks for fair classification with stochastic perturbations can perform better.
Further, like existing frameworks, our framework's efficacy will depend on an appropriate choice of parameters;
e.g., an overly conservative $\lambda$ can decrease accuracy and an optimistic $\lambda$ can decrease fairness.
A careful assessment both pre- and post-deployment would be important to avoid negative social implications in a misguided attempt to do good~\cite{liu2018delayed}.

Finally, we note that discrimination is a systematic problem and our work only addresses one part of it; this work would be effective as one piece of a broader approach to mitigate and rectify biases.

\section*{Acknowledgements}
This research was supported in part by an {NSF CAREER Award (IIS-2045951)}, a J.P. Morgan Faculty Award, and an AWS MLRA Award.

\newpage

\bibliographystyle{plain}
\bibliography{bib-v1.bib}

\appendix

\newpage

\section{Further Comparison to Related Work}

\subsection{{Other Related Work}}\label{sec:other_related_work}

\paragraph{Fair classification without perturbations.}
A large body of work studies fair classification.
Here, several works frame fair classification as a constrained optimization program and develop algorithms to solve these programs~\cite{zhang2018mitigating,ZafarVGG17,zafar17,menon2018the,goel2018non,AgarwalBD0W18,celis2019classification}.
A different approach is to alter the decision boundary of a given classifier to improve its fairness~\cite{fish2016confidence,hardt2016equality,goh2016satisfying,pleiss2017on,Woodworth2017learning,DworkIKL18} (possibly with different alterations for different protected groups).
Furthermore, some works preprocesses the training data to ``correct'' for its bias~\cite{kamiran2009classifying,luong2011k,kamiran2012data,zemel13,feldman2015certifying,krasanakis2018adaptive}.
However, these works require the protected attributes in the training samples to be known exactly, whereas in this paper we study the setting where fraction of the protected attributes are arbitrarily corrupted.

\paragraph{Missing protected attributes.}
Some works have studied fair classification in the absence of protected attributes--using auxiliary data.
For example, \cite{Gupta2018proxy} use other variables as proxies for protected attributes and \cite{CostonRWVSMC19}
augment their dataset with ``related data'' (that includes protect attributes) to control fairness. %
In the absence of auxiliary data, \cite{HashimotoSNL18} use distributionally robust optimization to minimize the maximum empirical risk across the protected groups, and
{\cite{LahotiBCLPT0C20} use a neural network to identify ``potential'' protected groups.}
However, these approaches do not offer provable guarantees on accuracy (with respect to $f^\star$).
In contrast, our approach uses perturbed protected attributes and comes with provable guarantees on fairness and accuracy (with respect to $f^\star$).

\paragraph{Stochastic perturbations in labels.}
\cite{blum2020recovering,wang2020label} study fair classification with perturbations in the labels:
\cite{blum2020recovering} consider a model where perturbations arise due to bias in the training samples.
They show that, under some models of bias, adding fairness constraints can improve the accuracy of the classifier on the unbiased data.
\cite{wang2020label} consider a model where the labels in each protected group are perturbed to a different value independently with a known (group-dependent) probability;
they give a framework for a non-binary protected attribute that provably outputs a classifier with near-optimal accuracy that nearly satisfies the fair constraint with respect to equalized odds, true-positive rate, or false-positive rate fairness constraints.
In contrast, we focus on adversarial perturbations in the features, and our framework can be extended to adversarial perturbations in both features and labels (see \cref{sec:main_result:for_stronger_model}).
Finally, our framework works for a large class of linear-fractional fairness metrics (which include true-positive rate and false-positive rate fairness constraints, and can ensure equalized odds fairness).

\subsection{Performance of Prior Frameworks under the \texorpdfstring{$\eta$}{eta}-\Ham{} Model}\label{sec:prior_work_comparison}
In this section, we present examples showing that prior frameworks for fair classification can have low accuracy and fairness compared to our framework under the $\eta$-\Ham{} model.

\subsubsection{{\texorpdfstring{\cite{bshouty2002pac}}{[11]}'s Framework Can Output Classifiers with Low Statistical Rate}}\label{sec:bshouty_poor}
In this section, for any $\delta\in (0,\nfrac14)$,
we give an example (\cref{example:bshouty_poor}) where with high probability
\cite{bshouty2002pac}'s framework outputs a classifier $f_{\rm OPT}$ that has perfect accuracy and 0 statistical rate.
On the same example, an optimal solution $f_{\rm ET}$ of \errtolerant{} has accuracy $1-\delta$  and perfect statistical rate $1$.
\begin{example}\label{example:bshouty_poor}
  Fix $\cX$ to be any set with at least two distinct points, say $x_A$ and $x_B$.
  Let $\cF$ be any hypothesis class that shatters the set $\inbrace{x_A,x_B}\times [2]\subseteq \cX\times [p]$.
  Define the distribution $\cD$ as follows
  \begin{align*}
    \Pr\nolimits_\cD[X=x,Y=y,Z=z]\coloneqq \begin{cases}
    \nfrac13-\nfrac\delta3 &\text{if } x=x_A, y=1, z=1,\\
    \nfrac13-\nfrac\delta3 &\text{if } x=x_B, y=0, z=1,\\
    \delta &\text{if } x=x_A, y=0, z=2,\\
    \nfrac13-\nfrac\delta3 &\text{if } x=x_B, y=0, z=2,\\
    0 &\text{otherwise},
  \end{cases}
\end{align*}
where $\delta$ is some constant smaller than $\nfrac14$.
Note that for a sample $(X,Y,Z)\sim \cD$, conditioned on $X=x$ and $Z=z$, $Y$ takes the value $\mathds{I}\insquare{x= x_A,z=1}$.
Thus, the classifier $$f_{\rm OPT}(x,z)\coloneqq\mathds{I}\insquare{x=x_A,z=1}$$ has 0 predictive error.
One can verify that $f_{\rm OPT}$ has a statistical rate of 0 with respect to $\cD$.
Since $\cF$ shatters $\inbrace{x_A,x_B}\times [2]$, $\cF$ contains $f_{\rm OPT}(x,z)$.
\mbox{Further, any other classifier in $\cF$ has an error at least $\delta$ with respect to $\cD$.}

\cite{bshouty2002pac}'s framework outputs the classifier with the minimum empirical risk on the given samples.
Suppose the perturbation rate is $\eta\coloneqq 0$.
Then, given a sufficient number of samples from $\cD$, with high probability, $f_{\rm OPT}\in \cF$ has the minimum empirical error, and hence, is output by \cite{bshouty2002pac}'s framework;
$f_{\rm OPT}$ satisfies
$${\rm Err}_\cD(f_{\rm OPT}) = 0\quad\text{and}\quad \Omega_\cD(f_{\rm OPT})=0,$$
where $\Omega$ is the statistical rate fairness metric.

Next, we show that on this example, \errtolerant{} outputs a classifier with a large statistical rate.
Set the fairness threshold to be any value $\tau<1$.
Fix any $\lambda\leq \delta$ (this ensures that  \cref{asmp:1} is satisfied).
Fix any $\Delta>0$.
Finally, as mentioned, $\eta\coloneqq 0$.
One can verify that $$f_{\rm ET}(x,z) \coloneqq \mathds{I}\insquare{x=x_A}$$ has error ${\rm Err}_\cD(f_{\rm ET})=\delta$ and a statistical rate of 1 with respect to $\cD$.
In contrast, any other classifier with statistical rate at least $1-2\delta$ has error at least $\nfrac13-\nfrac\delta3>\delta$ with respect to $\cD$.
(O1) Using this, one can show that, given a sufficient number of iid samples from $\cD$, with high probability, any other classifier feasible for Equation~\eqref{eq:fairness_constraint_in_err_tol} in \errtolerant{} has an error larger than the error of $f_{\rm ET}$ (on the given samples).
(O2) Further, because $\lambda\leq \delta$, one can verify that given a sufficient number of iid samples from $\cD$, with high probability, $f_{\rm ET}$ satisfies Equation~\eqref{eq:lower_bound} in \errtolerant{}; thus, with high probability, $f_{\rm ET}$ is feasible for \errtolerant{}.

Combining observations (O1) and (O2), we get that: Given a sufficient number of iid samples from $\cD$, with high probability, $f_{\rm ET}$ is the optimal solution \errtolerant{} with parameters $\tau=1-\delta$, $\lambda=\delta$, $\eta=0$, and $\Delta>0$; $f_{\rm ET}$ satisfies
$${\rm Err}_\cD(f_{\rm ET}) = \delta\quad\text{and}\quad \Omega_\cD(f_{\rm ET})=1.$$
\end{example}

\subsubsection{{\texorpdfstring{\cite{LamyZ19}}{[44]}'s and \texorpdfstring{\cite{celis2020fairclassification}}{[14]}'s Frameworks Can Output Classifiers with Low Accuracy}}\label{sec:comparison_to_CHKV_LAMY}
\newcommand{\eeqref}[1]{Equation~\eqref{#1}}
In this section, for any $\eta\in (0,\nfrac12)$,
we give an example where with high probability
\cite{LamyZ19}'s and \cite{celis2020fairclassification}'s frameworks output classifiers $f_L$ and $f_C$ (respectively) whose error is at least $\nfrac14$ higher than the error of $f^\star$ (\cref{thm:controlling_stoc_noise_is_insufficient}); where $f^\star$ is an optimal solution to Program~\eqref{prog:target_fair}.
On the same example, an optimal solution of \errtolerant{} has error within $2\eta$ of the \mbox{error of $f^\star$ and violates the fairness constraint by at most $O(\eta)$.}

\cite{LamyZ19} and \cite{celis2020fairclassification} take parameters $\delta_L,\tau\in [0,1]$ as input; these parameters control the desired fairness, where decreasing $\delta_L$ or increasing $\tau$ increases the desired fairness.
\cite{celis2020fairclassification} also takes the constant $\lambda$ from \cref{asmp:1} as input.
In addition, both \cite{LamyZ19} and \cite{celis2020fairclassification} require group specific perturbation rates as input: for each pair $\ell,k\in[p]$, they require  $P{\ell k} \coloneqq \Pr\nolimits_D[\hZ=k \mid Z=\ell]$.

Let $P\in[0,1]^{p\times p}$ denote the resulting matrix.
To give a meaningful estimate of $P$ with adversarial noise, we define the following restriction of the Hamming adversary.
\begin{definition}[\bf $P$-restricted Hamming adversary]\label{def:p_hamming_adv}
  Given a matrix $P\in [0,1]^{p\times p}$ and $N\in \N$ samples $\inbrace{(x_i,y_i,z_i)}_{i\in [N]}$,
  for each $\ell\in [p]$, let $G_\ell\coloneqq \inbrace{i\in [N]\mid z_i=\ell}$ be the set of samples with protected attribute $\ell$. %
  For each $\ell,k\in [p]$, the $P$-restricted Hamming adversary $A_{RH}$ chooses ${P_{\ell k}\cdot |G_\ell|}$ samples $i\in [N]$ from $G_\ell$,
  and perturbs their protected attribute $z_i$ from $\ell$ to $\hz_i=k$.\footnote{We assume that $P{\ell k}\cdot N_\ell$ is integral for all $\ell,k\in[p]$. This can be ensured by slightly increasing $P{\ell k}$ or $N$.}

\end{definition}
\noindent The modifier ``$P$-restricted'' refers to the restriction placed by the matrix $P$ on the adversary.
Let $\cA_{RH}(P)$ be the set of all $P$-restricted Hamming adversaries.
Then one can show that $\cA_{RH}(P)\subseteq \cA(\eta)$ for any $$\eta\geq \max_{\ell\in [p]}\sum\nolimits_{k\in [p]}P_{\ell k}.$$

\begin{theorem}\label{thm:controlling_stoc_noise_is_insufficient}
  Suppose that there are two protected groups ($p\coloneqq 2$) and $\cX$ contains at least two distinct points.
  Then, there is a family of hypothesis classes $\cF$ such that
  for all fairness thresholds $\tau\in (0,1]$ and perturbation rates $\eta\in (0,\nfrac12)$, there is
  \begin{enumerate}[itemsep=\itemsepINTERNAL,leftmargin=\leftmarginINTERNAL]
    \item
    a distribution $\cD$ over $\cX\times \zo\times [2]$ that satisfies \cref{asmp:1} with $\lambda\coloneqq \nfrac\tau{4}$,
    \item
    a matrix $P\in [0,1]^{2\times 2}$ such that $\cA_{RH}(P)\subseteq \cA(\eta)$, and
    \item
    an adversary $A_{RH}\in \cA_{RH}(P)$ that perturbs at most $\eta$-fraction of the samples, %
  \end{enumerate}
  such that, if the fairness metric is statistical rate, then for a draw of $N\in \N$ iid samples $S$ from $\cD$,
  with probability at least {$1-e^{-\Omega(\eta^2\tau^2 N)}$} (over the draw of $S$), it holds that
  the optimal classifiers
  \begin{enumerate}[itemsep=\itemsepINTERNAL,leftmargin=\leftmarginINTERNAL]
    \item $f_{C}\negsp\in\negsp\cF$ of \cite{celis2020fairclassification}'s program with parameters $P$, $\lambda,$ and $\tau$ and samples $A_{RH}(S)$, %
    \item $f_{L}\negsp\in\negsp\cF$ of \cite{LamyZ19}'s program with parameters $P$ and $\delta_L\coloneqq \nfrac12 - \nfrac\tau{2}$ and samples $A_{RH}(S)$, and\footnote{{In this example, $(\nfrac12) - (\nfrac\tau{2})$ is the minimum value of $\delta_L$ needed to ensure that $f^\star$, an optimal solution of Program~\eqref{prog:target_fair}, is feasible for \cite{LamyZ19}'s program with $\eta=0$.}}
    \item $f_{\rm ET}\negsp\in\negsp\cF$ of \errtolerant{} with parameters $\eta,\lambda,$ and $\tau$ and samples $A_{RH}(S)$
  \end{enumerate}
  have errors
  \begin{align*}
    \mathrm{Err}_\cD(f_{C}) - \mathrm{Err}_\cD(f^\star) &\geq \frac1{4},\yesnum\label{eq:thm:appendix_eq_1}\\
    \mathrm{Err}_\cD(f_{L}) - \mathrm{Err}_\cD(f^\star) &\geq \frac1{4},\yesnum\label{eq:thm:appendix_eq_2}\\
    \mathrm{Err}_\cD(f_{\rm ET}) - \mathrm{Err}_\cD(f^\star) &\leq 2\eta.\yesnum\label{eq:thm:appendix_eq_3}
  \end{align*}
  Further, $f_{\rm ET}$ has statistical rates at least $\tau - O\inparen{\nfrac{\eta}{\tau}}$ with respect to $\cD$, i.e., $\Omega_\cD(f_{\rm ET})\geq \tau - O\inparen{\nfrac{\eta}{\tau}}$.
\end{theorem}
\vspace{-4.0mm}

\paragraph{Proof for \texorpdfstring{\cref{thm:controlling_stoc_noise_is_insufficient}}{Theorem B.3}.}
Let $S\coloneqq \inbrace{(x_i,y_i,z_i)}_{i\in [N]}$ denote $N$ iid samples from $\cD$. %

\paragraph{Setting \texorpdfstring{$P, A$, and $\cD$}{P, A, and D}.}
We let $$P\coloneqq{\begin{bmatrix}
1-\eta_1 & \eta_1\\ \eta_2 & 1-\eta_2\\
\end{bmatrix}},$$ where $\eta_1\coloneqq 0$ and $\eta_2\coloneqq \eta$.
Since $\eta=\max_{\ell\in [p]}\sum_{k\in[p]}P{\ell k}$, we can verify that $\cA_{\rm RH}(P)\subseteq \cA(\eta)$.
We fix $A\in \cA_{\rm RH}(P)$ to be the following algorithm.
\begin{mdframed}
  \begin{enumerate}[leftmargin=38pt,itemsep=1pt]
    \item[\bf Input.] A perturbation rate $\eta>0$, matrix $P\coloneqq\insquare{\begin{smallmatrix}
    1-\eta_1 & \eta_1\\ \eta_2 & 1-\eta_2\\
    \end{smallmatrix}}$, where $\eta_1,\eta_2\in [0,1]$, and samples $S\coloneqq \inbrace{(x_i,y_i,z_i)}_{i\in [N]}$ %
    \item[\bf Output.] Samples $\hS$\\\hrule
    \item {\bf For $\ell\in [2]$ do:}
    \begin{enumerate}[leftmargin=11pt,itemsep=0pt]
      \item {\bf Set} $N_\ell\coloneqq {\eta_\ell \cdot \sum_{i\in [N]} \mathds{I}[z_i=\ell]}$ %
      \item {\bf Set} $G_{A}\coloneqq\inbrace{i\in [N]\colon z_i=\ell, x_i=x_A}$ and $G_{B}\coloneqq\inbrace{i\in [N]\colon z_i=\ell, x_i=x_B}$
      \item {\bf Initialize} $C = \emptyset$\hfill \commentalg{~Corrupted samples}
      \item Pick any $\min\{N_\ell,|G_{B}|\}$ items from $G_{B}$ and add them to $C$ %
      \item Pick any $N_\ell-\min\{N_\ell,|R_B|\}$ items from $G_{A}$ and add them to $C$
      \item {\bf For $i\in C$ do:} {\bf Set} $\hz_i = 3-\ell$\hfill \commentalg{\ If $z_i=1$ the $\hz_i=2$, and if $z_i=2$ then $\hz_i=1$}
      \item {\bf For $i\in (G_A\cup G_B) \backslash C$ do:} {\bf Set} $\hz_i = \ell$
    \end{enumerate}
    \item {\bf return} $\hS\coloneqq \inbrace{(x_i,y_i,\hz_i)}_{i\in [N]}$
  \end{enumerate}
\end{mdframed}

One can verify that $A$ perturbs exactly $P{\ell k}\cdot |G_\ell|$ samples with protected attribute $\ell$ to protected attribute $k$.
Hence, $A$ is a $P$-restricted \Ham{} adversary.
Further, as $\eta_1+\eta_2= \eta$, it also follows that $A$ perturbs at most $\eta$-fraction of samples, and hence, is an $\eta$-\Ham{} adversary.

Fix $\cX$ to be any set with at least two distinct points, say $x_A$ and $x_B$.
Let $\cF$ be any hypothesis class that shatters the set $\inbrace{x_A,x_B}\times [2]$.
We define the distribution $\cD$ as follows
\begin{align*}
  \Pr\nolimits_\cD[X=x,Y=y,Z=z]\coloneqq \begin{cases}
  \nfrac{\tau}4 &\text{if } x=x_A, y=1, z=1,\\
  \nfrac14 &\text{if } x=x_A, y=1, z=2,\\
  \nfrac12 - \nfrac{\tau}4 &\text{if } x=x_B, y=0, z=1,\\
  \nfrac14 &\text{if } x=x_B, y=0, z=2,\\
  0 &\text{otherwise}.
\end{cases}\yesnum\label{eq:example_def_cd}
\end{align*}
Note that for a sample $(X,Y,Z)\sim \cD$, conditioned on $X$, the value of $Y$ is $\mathds{I}\insquare{X= x_A}$.
Thus, the classifier $f^\star(x,z)\coloneqq\mathds{I}\insquare{X= x_A}$ has 0 predictive error.
We use \cref{lem:example_2} in the proof of \cref{thm:controlling_stoc_noise_is_insufficient}.
\begin{lemma}[\bf Estimates of statistic on perturbed samples]\label{lem:example_2}
  For all $\delta\in (0,1)$,
  with probability at least $1-e^{-\Omega(\delta^2\cdot N)}$ (over the draw of $S$), the following bounds hold
  \begin{align*}
    \abs{\Pr\nolimits_{\hD}[f^\star=1, Z=1] - \Pr\nolimits_{\cD}[f^\star=1, Z=1]}\leq \delta,\yesnum\label{eq:example_bound_1}\\
    \abs{\Pr\nolimits_{\hD}[f^\star=1, Z=2] - \Pr\nolimits_{\cD}[f^\star=1, Z=2]}\leq \delta,\yesnum\label{eq:example_bound_2}\\
    \abs{\Pr\nolimits_{\hD}[Z=1] - \inparen{\Pr\nolimits_{\cD}[Z=1]+\eta\cdot \Pr\nolimits_{\cD}[Z=2]}}\leq \delta,\yesnum\label{eq:example_bound_3}\\
    \abs{\Pr\nolimits_{\hD}[Z=2] - \Pr\nolimits_{\cD}[Z=2]\cdot (1-\eta)}\leq \delta.\yesnum\label{eq:example_bound_4}
  \end{align*}
  Equivalently substituting the statistics on $\cD$ in \cref{eq:example_bound_1,eq:example_bound_2,eq:example_bound_3,eq:example_bound_4}, we get
  \begin{align*}
    \abs{\Pr\nolimits_{\hD}[f^\star=1, Z=1] - \frac{\tau}4}\leq \delta,\yesnum\label{eq:example_bound_5}\\
    \abs{\Pr\nolimits_{\hD}[f^\star=1, Z=2] - \frac{1}4}\leq \delta,\yesnum\label{eq:example_bound_6}\\
    \abs{\Pr\nolimits_{\hD}[Z=1] - \frac{1+\eta}2}\leq \delta,\yesnum\label{eq:example_bound_7}\\
    \abs{\Pr\nolimits_{\hD}[Z=2] - \frac{1-\eta}2}\leq \delta.\yesnum\label{eq:example_bound_8}
  \end{align*}
\end{lemma}
\noindent The proof of \cref{lem:example_2} follows by analyzing the algorithm of $A$ and using the Chernoff bound.
The proof of \cref{lem:example_2} appears at the end of this section.
\begin{proof}[Proof of \cref{thm:controlling_stoc_noise_is_insufficient}]
  Since the distribution $\cD$ is supported on 4 points, namely $\inbrace{x_A,x_B}\times [2]$ (see \cref{eq:example_def_cd}), we only need to consider hypothesis in the restriction of $\cF$ on the set $\inbrace{x_A,x_B}\times [2]$; this restriction has $2^4$ hypothesis.

  The first observation is that $\mathds{I}[X=x_A]$ is an optimal solution for Program~\eqref{prog:target_fair} and satisfies.
  ${\rm Err}_\cD(f^\star)=0$ and $\Omega_\cD(f^\star)=1$, where $\Omega$ is the statistical rate fairness metric. %
  This is because, $\mathds{I}[X=x_A]$ satisfies the constraints of Program~\eqref{prog:target_fair} and has perfect accuracy.

  The second observation is that $f^\star$ is not feasible for \cite{celis2020fairclassification} and \cite{LamyZ19}'s programs.
  \paragraph{$f^\star$ is not feasible for \cite{celis2020fairclassification}'s program.}
  \cite{celis2020fairclassification} express their constraints (for statistical rate) in terms of vectors $u(f),w\in [0,1]^2$ (where $u(f)$ depends on $f\in \cF$).
  They define $u(f)$ and $w$ as follows
  \begin{align*}
    u(f)&\coloneqq (P^T)^{-1}\cdot \insquare{\begin{matrix}
    \Pr\nolimits_{\hD}[f=1, Z=1]\\
    \Pr\nolimits_{\hD}[f=1, Z=2]\end{matrix}},\yesnum\label{eq:abcd1}
  \end{align*}
  \begin{align*}
    w &\coloneqq (P^T)^{-1}\cdot \insquare{\begin{matrix}\Pr\nolimits_{\hD}[Z=1]\\ \Pr\nolimits_{\hD}[Z=2]\end{matrix}}.\yesnum\label{eq:abcd2}
  \end{align*}
  \cite{celis2020fairclassification} impose the following constraint
  \begin{align*}
    \frac{\min_{\ell\in [p]}\sfrac{u(f)_\ell}{w_\ell}}
    {\max_{\ell\in [p]}\sfrac{u(f)_\ell}{w_\ell}} \geq \tau.\yesnum
    \label{eq:CHKVconst}
  \end{align*}
  In our example, $$(P^T)^{-1}\coloneqq \frac{1}{1-\eta}\cdot{\begin{bmatrix} 1-\eta & -\eta \\ 0 & 1\end{bmatrix}}.$$
  Set
  \begin{align*}
    \delta\coloneqq \frac{\eta\cdot \tau}{64}.
  \end{align*}
  Substituting the value of $(P^T)^{-1}$ in \cref{eq:abcd1,eq:abcd2}, and then using \cref{lem:example_2}, we get that with probability at least $1-e^{-\Omega(\delta^2\cdot N)}$, $u(f^\star)$ and $w$ satisfy the following bounds
  \begin{align*}
    \abs{u(f^\star)_1 -  \inparen{\frac{\tau}4 -  \frac{\eta}{4(1-\eta)}}} &\leq \frac{\delta}{1-\eta},\yesnum\label{eq:just1}\\
    \abs{u(f^\star)_2 -  {\frac{1}{4(1-\eta)}}} &\leq \delta,\yesnum\label{eq:just2}\\
    \abs{w_1 -  \frac12} &\leq \frac{\delta}{1-\eta},\yesnum\label{eq:just3}\\
    \abs{w_2 -  \frac12} &\leq \delta.\yesnum\label{eq:just4}
  \end{align*}
  Suppose the \cref{eq:just1,eq:just2,eq:just3,eq:just4} hold.
  Toward computing the constraint in \cref{eq:CHKVconst}, we compute bounds for $\nfrac{u(f^\star)_1}{w_1}$ and $\nfrac{u(f^\star)_2}{w_2}$.
  \begin{align*}
    \frac{u(f^\star)_1}{w_1} &\Stackrel{}{\leq} \frac{\frac{\tau}{4}-\frac{\eta}{4(1-\eta)}-\frac{\delta}{1-\eta}}{\frac{1}{2}+\frac{\delta}{1-\eta}}\tag{Using \cref{eq:just1,eq:just3}}\\
    &\Stackrel{}{\leq} \frac{\frac{\tau}{4}-\frac{\eta}{8(1-\eta)}}{\frac{1}{2}\cdot(1-\frac{2\delta}{1-\eta})}
    \tag{Using that $\delta\leq \nfrac\eta8$}\\
    &\Stackrel{}{=} \frac{\tau}{2}\cdot \frac{1-\frac{\eta}{2\tau(1-\eta)}}{1-\frac{2\delta}{1-\eta}}\\
    &\Stackrel{}{<}\frac{\tau}{2},
    \tagnum{Using that $\delta\leq \nfrac{\eta}{(4\tau)}$}\customlabel{eq:example2:conclusion1}{\theequation}\\
    \frac{u(f^\star)_2}{w_2} &\Stackrel{}{\geq} \frac{\frac{1}{4(1-\eta)}-\delta}{\frac{1}{2}+\delta}\tag{Using \cref{eq:just2,eq:just4}}\\
    &\Stackrel{}{=} \frac{1}{2(1-\eta)}\cdot \frac{1-4\delta(1-\eta)}{1+2\delta}\\
    &\Stackrel{}{\geq} \frac{1}{2(1-\eta)}\cdot \inparen{1-4\delta(1-\eta)}\cdot\inparen{1-2\delta}\tag{Using that for all $x\in \R$, $\frac{1}{1+x}\geq 1-x.$}\\
    &\Stackrel{}{\geq} \frac{1}{2(1-\eta)}\cdot\inparen{1-6\delta}\tag{Using that $\delta,\eta>0$}\\
    &\Stackrel{}{>}\frac{1}{2}.\tagnum{Using that $\delta < \nfrac\eta6$}\customlabel{eq:example2:conclusion2}{\theequation}
  \end{align*}
  Substituting Equations~\eqref{eq:example2:conclusion1} and \eqref{eq:example2:conclusion2} in \cref{eq:CHKVconst}, we get that
  \begin{align*}
    \frac{\min_{\ell\in [p]}\sfrac{u(f)_\ell}{w_\ell}}
    {\max_{\ell\in [p]}\sfrac{u(f)_\ell}{w_\ell}}
    \quad &\leq\quad \frac{\sfrac{u(f)_1}{w_1}}{\sfrac{u(f)_2}{w_2}}
    \qquad \Stackrel{\rm \eqref{eq:example2:conclusion1},\eqref{eq:example2:conclusion2}}{<}\qquad \tau.
  \end{align*}
  Thus, $f^\star$ is not feasible for \cite{celis2020fairclassification}'s optimization program.

  \paragraph{$f^\star$ is not feasible for \cite{LamyZ19}'s program.}
  For any $f\in \cF$, \cite{LamyZ19} impose the constraint
  \begin{align*}
    \abs{\Pr\nolimits_{\hD}[f=1\mid Z=1]-\Pr\nolimits_{\hD}[f=1\mid Z=2]}\leq \delta_L\cdot (1-\alpha-\beta),\yesnum\label{eq:lamyconst}
  \end{align*}
  where $\alpha,\beta\in [0,1]$ are some function of $\eta_1$ and $\eta_2$.
  In particular, it holds that if $\eta_1>0$ (respectively $\eta_2>0$) then $\alpha>0$ (respectively $\beta>0$), otherwise $\alpha=0$ (respectively $\beta=0$)
  In our example, $\eta_1=0$ and $\eta_2=\eta>0$.
  Thus, $\alpha=0$ and $\beta>0$.
  Recall that $$\delta_L\coloneqq \frac12-\frac\tau2.$$
  To show that $f^\star$ does not satisfy \cref{eq:lamyconst}, we bound $\Pr\nolimits_{\hD}[f=1\mid Z=1]$ and $\Pr\nolimits_{\hD}[f=1\mid Z=2]$.
  \begin{align*}
    \Pr\nolimits_{\hD}[f=1\mid Z=1]
    &\leq \frac{\frac{\tau}4 + \delta }{\frac{1+\eta}{2}-\delta}\tag{Using \cref{eq:example_bound_5,eq:example_bound_7}}\\
    &= \frac{\tau}{2(1+\eta)}\cdot \frac{1 + \frac{4\delta}{\tau} }{1-\frac{2\delta}{1+\eta}}\\
    &\leq \frac{\tau}{2(1+\eta)}\cdot (1 + \frac{\eta}{16})\cdot\inparen{1+\frac{4\delta}{1+\eta}}\tag{Using that $\delta\coloneqq \frac{\eta\tau}{64}$ and $\frac{4\delta}{1+\eta}\in [0,\nfrac12]$}\\
    &\leq \frac{\tau}{2(1+\eta)}\cdot \inparen{1 + \frac{\eta}{8}}^2\tag{Using that $\delta\leq \frac{\eta}{4}$}\\
    &< \frac\tau2,\tagnum{Using that $\eta\leq 1$}\customlabel{eq:boundLamy1}{\theequation}\\
    \Pr\nolimits_{\hD}[f=1\mid Z=2]
    &\geq \frac{\frac{1}4 - \delta }{\frac{1-\eta}{2}+\delta}\tag{Using \cref{eq:example_bound_6,eq:example_bound_8}}\\
    &= \frac{1}{2(1-\eta)}\cdot \frac{1 - 4\delta}{1+\frac{2\delta}{1+\eta}}\\
    &\geq \frac{1}{2(1-\eta)}\cdot \inparen{1 - \frac{\eta}{8} }\cdot\inparen{1-\frac{2\delta}{1+\eta}}\tag{Using that $4\delta\leq \nfrac\eta8$ and for all $x\in \R$, $(1+x)^{-1}\geq 1-x$}\\
    &\geq \frac{1}{2(1-\eta)}\cdot \inparen{1 - \frac{\eta}{8}}^2\tag{Using that $2\delta\leq \nfrac\eta8$}\\
    &\geq \frac{1-\frac{\eta}4}{2(1-\eta)}\\
    &> \frac12.\tagnum{Using that $\eta>0$}\customlabel{eq:boundLamy2}{\theequation}
  \end{align*}
  Thus, combining Equations~\eqref{eq:boundLamy1} and \eqref{eq:boundLamy2} and the fact that $\beta>0$ and $\alpha=0$, it follows that $f^\star$  is not feasible for \cref{eq:lamyconst}.

  \paragraph{\cite{LamyZ19}'s and \cite{celis2020fairclassification}'s frameworks output a classifier with large error.}
  Since $f^\star$ is not feasible for \cite{LamyZ19}'s and \cite{celis2020fairclassification}'s programs, they must output some other classifier $f_{\rm Alt}\in \cF$.
  Toward a contradiction, suppose that $${\rm Err}_\cD(f_{\rm Alt})<\nfrac14.$$
  Consider the set $$U\coloneqq \inbrace{x_A,x_B}\times [2]\backslash\inbrace{(x_A,1)}.$$
  Each point in the $U$ has probability mass at least $1/4$.
  Thus, if $f_{\rm Alt}$ has different outcome than $f^\star$ on the set $U$, then $${\rm Err}_\cD(f_{\rm Alt})\geq \nfrac14.$$
  So we must have $$f_{\rm Alt}(r,s)=f^\star(r,s)$$ for all $(r,s)\in U.$
  Because $f_{\rm Alt}$ is different than $f^\star$, its outcome must differ from $f^\star$ on at least one point in the support of $\cD$.
  The only remaining point is $(x_A,1)$.
  Thus, we must have $f_{\rm Alt}(x_A,1) = 0$.
  However, in this case, one can show that
  \begin{align*}
    \Pr\nolimits_{\hD}[f_{\rm Alt}=1,Z=1]=0\quad \text{and} \quad \Pr\nolimits_{\hD}[f_{\rm Alt}=1,Z=1]>0.
  \end{align*}
  Substituting this in \cref{eq:CHKVconst,eq:lamyconst} we get, that $f_{\rm Alt}$ is not feasible for  \cite{celis2020fairclassification}'s and \cite{LamyZ19}'s optimization programs.
  This is a contradiction since we assumed that \cite{celis2020fairclassification} and \cite{LamyZ19} output $f_{\rm Alt}$.
  This proves Equations~\eqref{eq:thm:appendix_eq_1} and \eqref{eq:thm:appendix_eq_2}.

  Finally, Equations~\eqref{eq:thm:appendix_eq_3} and the bound on $\Omega_\cD(f_{ET})$ follows from \cref{thm:main_result} because $f^\star$ satisfies \cref{asmp:1} with constant $\lambda=\nfrac{\tau}4$.
\end{proof}
\begin{remark}[\bf Choice of $P$]
  In \cref{thm:controlling_stoc_noise_is_insufficient}, we fix $$P\coloneqq \begin{bmatrix}
  1 & 0\\
  \eta & 1-\eta
  \end{bmatrix}.$$
  However, we can show that \cref{thm:controlling_stoc_noise_is_insufficient} holds for  $P\coloneqq \begin{bmatrix}
  1-\eta_1 & \eta_1\\
  \eta_2 & 1-\eta_2
  \end{bmatrix}$ where $0\leq \eta_1<\eta_2<\eta$.
  The only change is that
  the high probability guarantee changes from
  \begin{align*}
    1-e^{-\Omega(\min\{\tau,\eta\}\cdot N)}
    \quad \text{to}\quad
    1-e^{-\Omega(\min\{\tau,\abs{\eta_2-\eta_1}\}\cdot N)}
  \end{align*}
  Note that the distribution $\cD$ does not change.
\end{remark}

\paragraph{Proof of \cref{lem:example_2}.}
\begin{proof}[Proof of \cref{lem:example_2}]
  We give a proof of \cref{eq:example_bound_1,eq:example_bound_3}.
  The proofs of
  \cref{eq:example_bound_2,eq:example_bound_4} follow by replacing $Z=1$ by $Z=2$ in the following argument.

  Since $A$ only flips samples with feature $x_B$ and $f^\star(x_B,z)=0$ for all $z\in [2]$, we have that
  \begin{align*}
    \Pr\nolimits_{\hD}[f^\star=1,Z=1] &= \Pr\nolimits_{D}[f^\star=1,Z=1]\yesnum\label{eq:2above1}
  \end{align*}
  Using the Chernoff bound, it follows that the next inequality holds with probability at least $1-2e^{-\frac{16}{3\tau^2}\cdot\delta^2N}$
  \begin{align*}
    \abs{\Pr\nolimits_{D}[f^\star=1,Z=1]-\Pr\nolimits_{\cD}[f^\star=1,Z=1]} &\leq \delta\yesnum\label{eq:2above3}
  \end{align*}
  \cref{eq:example_bound_1} follows from \cref{eq:2above1,eq:2above3} by using the triangle inequality for the absolute value function.
  Since $A$ flips $\eta$-fraction of the samples with $Z=2$ to $Z=1$, we have that
  \begin{align*}
    \Pr\nolimits_{\hD}[Z=1] &= \Pr\nolimits_{D}[Z=1]+\eta\cdot \Pr\nolimits_{D}[Z=2],\yesnum\label{eq:above1}\\
    \Pr\nolimits_{\hD}[Z=2] &= \Pr\nolimits_{D}[Z=2]\cdot(1-\eta).\yesnum\label{eq:above2}
  \end{align*}
  Using the Chernoff bound, it follows that the next inequality holds with probability at least $1-4e^{-\frac{16}{3}\cdot\delta^2N}$
  \begin{align*}
    \abs{\Pr\nolimits_{D}[Z=1]-\Pr\nolimits_{\cD}[Z=1]} &\leq \delta\yesnum\label{eq:above3}
  \end{align*}
  \cref{eq:example_bound_3} follows from \cref{eq:above1,eq:above3} by using the triangle inequality for the absolute value function.
\end{proof}

\subsubsection{{\texorpdfstring{\cite{wang2020robust}}{[59]}'s Distributionally Robust Framework Can Output Classifiers with Low Accuracy}}\label{sec:comparison_to_wang}
In this section, for any $\eta\in (0,\nfrac14)$,
we give an example where with high probability
\cite{wang2020robust}'s distributionally robust optimization (DRO) framework outputs a classifier $f_{\rm DRO}\in \cF$ whose error is at least {$\nfrac12-\nfrac{\eta}{2}$} (\cref{rem:comparing_to_wang})
On the same example, an optimal solution of {\errtolerant{} has error at most $2\eta$ and additively violates the fairness constraint by at most $O(\eta)$.}

\cite{wang2020robust}, in their distributionally robust approach, assume that for each protected group its feature and label distributions in the true data and the perturbed data are a known total variation distance away from each other.
Formally, given a distribution $\cP$, for each $\ell\in [p]$, let $\cP_\ell$ be the distribution of features and labels in group $\ell$ when the data is drawn from $\cP$, i.e., $\cP_\ell$ is the distribution of $(X,Y)\mid Z=\ell$ when $(X,Y,Z)\sim \cP$. %
Let $\cD$ be the true distribution of samples and
let $\hD$ be the (empirical) distribution of perturbed samples. %
Define a vector $\gamma\in [0,1]^p$ as follows: For all $\ell\in [p]$
\begin{align*}
  \gamma_\ell\coloneqq {\rm TV}(\cD_\ell,\hD_\ell).\yesnum\label{def:gamma}
\end{align*}
\cite{wang2020robust} assume that an upper bound on $\gamma$ is known.
One can show that in presence of an $\eta$-\Ham{} adversary a tight upper bound on $\gamma_\ell$ is $\frac{\eta}{\Pr[Z=\ell]}$; It is achieved by the adversary that, given $N$ samples, perturbs the protected attribute of $\eta\cdot N$ samples with protected attribute $Z=\ell$.

Let $\mathfrak{D}(\gamma)$ be the set of all distributions $\cP$ which satisfy that:
\begin{align*}
  \text{for all $\ell\in [p]$,}\quad {\rm TV}(\hD_\ell, \cP_\ell)\leq \gamma_\ell.\yesnum\label{eq:def_of_set_of_dist_wang}
\end{align*}
Then, \cite{wang2020robust}'s output a classifier $f_{\rm DRO}\in \cF$ which has the highest accuracy on $\hD$ such that it satisfies their fairness constraints
for all distributions $\cP\in \mathfrak{D}(\gamma)$;
for statistical rate, they solve
\begin{align*}
  \min_{f\in \cF}\ \  {\rm Err}_{\hD}(f),\quad \st, \quad \text{for all $ \cP\in\mathfrak{D}(\gamma)$}, \quad \Ex\nolimits_\cP[f=1\mid Z=\ell]= \Ex\nolimits_\cP[f=1].\yesnum\label{prog:wang}
\end{align*}

\begin{theorem}\label{rem:comparing_to_wang}
  Suppose that there are two protected groups ($p\coloneqq 2$) and $\cX$ contains at least three distinct points.
  Then, there is a family of hypothesis classes $\cF$ such that
  for all perturbation rates $\eta\in (0,\nfrac14)$, there is an adversary $A\in \cA(\eta)$ and a distribution $\cD$ over $\cX\times \zo\times [2]$ that satisfies
  \begin{enumerate}[itemsep=\itemsepINTERNAL,leftmargin=\leftmarginINTERNAL]
    \item \cref{asmp:1} with $\lambda\coloneqq \nfrac1{4}$, and
    \item $\Pr_\cD[Z=1]=\Pr_\cD[Z=2]=\nfrac12$,
  \end{enumerate}
  such that,
  for a draw of $N\in \N$ iid samples $S$ from $\cD$,
  with probability at least {$1-O(e^{-N})$} (over the draw of $S$), it holds that
  the optimal solution {$f_{\rm DRO}\negsp\in\negsp\cF$ of Program~\eqref{prog:wang} with parameter $\gamma=(2\eta,2\eta)$\footnote{Following the fact mentioned earlier, for each $\ell\in [2]$, we set $\gamma_\ell\coloneqq \frac{\eta}{\Pr[Z=\ell]}$.} and samples $A(S)$ has error}
  \begin{align*}
    \mathrm{Err}_\cD(f_{\rm DRO}) &\geq \frac1{2}-\eta,\yesnum\label{eq:thm:appendix_eq_1:3}
  \end{align*}
  and the optimal solution $f_{\rm ET}\negsp\in\negsp\cF$ of \errtolerant{} with parameters $\eta,\lambda,$ and $\tau=1$ and samples $A(S)$ has error
  \begin{align*}
    \mathrm{Err}_\cD(f_{\rm ET}) &\leq 2\eta.\yesnum\label{eq:thm:appendix_eq_2:3}
  \end{align*}
  While $$\Omega_\cD(f_{\rm DRO})=1\quad\text{and}\quad\Omega_\cD(f_{\rm ET})\geq 1 - O\inparen{\eta}.$$
\end{theorem}

\paragraph{{Proofs for \texorpdfstring{\cref{sec:comparison_to_wang}}{Section B.2.3}}.}
\paragraph{Setting $P$, $\cD$, and $A$.}
Fix $\cX$ to be any set with at least three distinct points, say $x_A,x_B$ and $x_C$.
Let $\cF$ be any hypothesis class that shatters the set $\inbrace{x_A,x_B,x_C}\times [2]\subseteq \cX\times [p]$.
Define $\cD$ as the unique distribution such that for a draw $(X,Y,Z)\sim \cP$, $Y$ takes the value $\mathds{I}[X\neq x_C]$ and that has the following marginal distribution:
\begin{center}
  \vspace{2mm}
  \begin{tabular}{c>{\centering\arraybackslash}p{1.75cm}>{\centering\arraybackslash}p{1.75cm}>{\centering\arraybackslash}p{1.75cm}}
    \toprule
    \midsepremove{}
    & $x_A$ & $x_B$ & $x_C$\\
    \midrule
    $1$ & $\eta$ & $\nfrac{1}4$ & $\nfrac14- \eta$\\
    $2$ & $0$ & $\nfrac{1}4$ & $\nfrac14$\\
    \bottomrule
  \end{tabular}
  \vspace{2mm}
\end{center}
Where for each $(r,s)\in\inbrace{x_A,x_B,x_C}\times \insquare{2}$, the corresponding cell denotes $\Pr\nolimits_{(X,Y,Z)\sim\cP}\insquare{(X,Z)=(r,s)}.$
Because $Y$ takes the value $\mathds{I}[X\neq x_C]$, the classifier $f^\star(x,z)\coloneqq\mathds{I}\insquare{X\neq x_C}$ has 0 predictive error.
We fix $A\in \cA(\eta)$ to be the adversary that does not perturb any samples.
(This suffices to prove the claim in \cref{rem:comparing_to_wang}, but one can also other adversaries in $\cA(\eta)$.)
\paragraph{Supporting lemmas.}
We use the following lemmas in the proof of \cref{rem:comparing_to_wang}.
\begin{lemma}\label{lem:opt_fair_classifier:example3}
  The classifier $f^\star(X,Z)=\mathds{I}[X=x_B]$ is an optimal solution for Program~\eqref{prog:target_fair} with $\tau=1$:
  \begin{align*}
    {\rm Err}_\cD(f^\star)=\eta\quad\text{and}\quad\Omega_\cD(f^\star)=1,
  \end{align*}
  where $\Omega$ is the statistical rate fairness metric. %
\end{lemma}

\begin{proof}
  One can verify that the classifier with the perfect accuracy, $f_{\rm OPT}\coloneqq \mathds{I}[X\neq x_C]$, has a statistical rate strictly smaller than 1.
  So it is not feasible for Program~\eqref{prog:target_fair} for $\tau=1$.
  Any feasible classifier $f\in \cF$ must differ from $f_{\rm OPT}$ on some point in the support of $\cD$.
  Since (by construction) all points in the support of $\cD$ have a probability mass at least $\eta$,
  it follows that any $f\in \cF$ feasible Program~\eqref{prog:target_fair} must have an error at least $\eta$.
  Now the result follows since $f^\star\coloneqq \mathds{I}[X=x_A]$ is feasible for  Program~\eqref{prog:target_fair} and has the optimal error, ${\rm Err}_\cD(f^\star)=\eta$.
\end{proof}
\begin{lemma}\label{lem:aboveDroexample}
  Consider the distribution $\cP$, such that, for a draw $(X,Y,Z)\sim \cP$, $Y\coloneqq \mathds{I}[X\neq x_C]$ and that has the following marginal distribution:
  \begin{center}
    \begin{tabular}{c>{\centering\arraybackslash}p{1.75cm}>{\centering\arraybackslash}p{1.75cm}>{\centering\arraybackslash}p{1.75cm}}
      \toprule
      \midsepremove{}
      & $x_A$ & $x_B$ & $x_C$\\
      \midrule
      $1$ & $\eta$ & $\nfrac{1}4- \eta$ & $\nfrac14 $\\
      $2$ & $0$ & $\nfrac{1}4$ & $\nfrac14$\\
      \bottomrule
    \end{tabular}
  \end{center}
  Where for each $(r,s)\in\inbrace{x_A,x_B,x_C}\times \insquare{2}$, the corresponding cell denotes $\Pr\nolimits_{(X,Y,Z)\sim\cP}\insquare{(X,Z)=(r,s)}.$
  Given a draw of $N$ iid samples from $\cD$, with probability at least $1-e^{-\Omega(N)}$, it holds that $\cP\in \mathfrak{D}(\gamma)$ and $\cD\in \mathfrak{D}(\gamma)$.
\end{lemma}
\begin{proof}
  Given a sufficient number of iid samples $S$ from $\cD$, one can show that with high probability, the empirical distribution $D$ of $S$ satisfies that:
  \begin{align*}
    {\rm TV}(D_1,\cD_1)\leq\eta\quad\text{and}\quad {\rm TV}(D_2,\cD_2)\leq\eta.
  \end{align*}
  Since $\gamma\coloneqq (2\eta,2\eta)$, this implies that with high probability $\cD\in \mathfrak{D}(\gamma)$.
  Further, the construction in \cref{lem:aboveDroexample} ensures that ${\rm TV}(\cP_1,\cD_1)=\eta$ and ${\rm TV}(\cP_2,\cD_2)=\eta$.
  By using the triangle inequality of the total variation distance, it follows that with high probability,
  %
  %
  \begin{align*}
    {\rm TV}(\cP_1,\cD_1)\leq2\eta\quad\text{and}\quad {\rm TV}(\cP_2,\cD_2)\leq2\eta.
  \end{align*}
  Since $\gamma\coloneqq (2\eta,2\eta)$, it follows that with high probability $\cP\in \mathfrak{D}(\gamma)$.
\end{proof}

\begin{lemma}\label{lem:no_DR_classifier}
  Any $f\in \cF$ that satisfies the following equalities
  \begin{align*}
    \Ex\nolimits_\cD[f=1\mid Z=\ell] =  \Ex\nolimits_\cD[f=1],\yesnum\label{cond:1}\\
    \Ex\nolimits_\cP[f=1\mid Z=\ell] =  \Ex\nolimits_\cP[f=1],\yesnum\label{cond:2}
  \end{align*}
  where $\cD$ is true distribution and $\cP$ is the distribution from \cref{lem:aboveDroexample}
  must have an error
  \begin{align*}
    {\rm Err}_\cD(f)\geq \frac12 - \eta.
  \end{align*}
\end{lemma}
\noindent Using \cref{lem:opt_fair_classifier:example3,lem:aboveDroexample,lem:no_DR_classifier}, \cref{rem:comparing_to_wang} follows as a corollary.
\begin{proof}[Proof of \cref{rem:comparing_to_wang}]
  From \cref{lem:aboveDroexample} know that with probability at least $1-e^{-\Omega(N)}$, $\cP\in \mathfrak{D}$ and $\cD\in \mathfrak{D}(\gamma)$.
  Suppose that this event happens.
  Assume that $f_{\rm DRO}\in \cF$ is the optimal solution of Program~\eqref{prog:wang}.
  Since $f_{\rm DRO}$ is feasible for Program~\eqref{prog:wang}, it must satisfy that
  \begin{align*}
    \Ex\nolimits_\cD[f_{\rm DRO}=1\mid Z=\ell]=  \Ex\nolimits_\cP[f_{\rm DRO}=1],\\
    \Ex\nolimits_\cP[f_{\rm DRO}=1\mid Z=\ell]=  \Ex\nolimits_\cP[f_{\rm DRO}=1],
  \end{align*}
  Then \cref{lem:no_DR_classifier} tells us that ${\rm Err}_\cD(f_{\rm DRO})\geq \frac12 - \eta.$

  Finally, one can verify that when statistical rate is the fairness metric, $f^\star$ (from \cref{lem:opt_fair_classifier:example3}) satisfies \cref{asmp:1} with $\lambda=\frac12$.
  Thus, \cref{eq:thm:appendix_eq_2:3} and the inequality $\Omega_\cD(f_{\rm ET})\geq 1 - O\inparen{\eta}$ follow from \cref{thm:main_result}.
\end{proof}

\begin{proof}[Proof of \cref{lem:no_DR_classifier}]
  Since that both $\cD$ and $\cP$ are supported on subsets of $\inbrace{x_A,x_B,x_C}\times [2]$,
  it suffices to consider the restriction of $\cF$ on this domain.
  Consider any classifier $f\in \cF$ and define the following variables
  \begin{align*}
    f_{A1}\coloneqq f(x_A,1),\quad &f_{B1}\coloneqq f(x_B,1),\quad f_{C1}\coloneqq f(x_C,1)\\
    f_{A2}\coloneqq f(x_A,2),\quad &f_{B2}\coloneqq f(x_B,2),\quad f_{C2}\coloneqq f(x_C,2),
  \end{align*}
  denoting the predictions of $f$ on $\inbrace{x_A,x_B,x_C}\times [2]$.
  Since $f$ satisfies \cref{cond:1}, we must have
  \begin{align*}
    2\cdot \inparen{\eta f_{A1}+\frac14 f_{A2}+ \inparen{\frac14-\eta}f_{A3} }
    &= \Ex\nolimits_\cD[f=1\mid Z=\ell] \\
    &= \Ex\nolimits_\cD[f=1]\\
    &= \eta f_{A1}+\frac14 f_{A2}+ \inparen{\frac14-\eta}f_{A3}  + \frac14 f_{B2}+ \frac14 f_{B3}.\yesnum\label{eq:condition_1}
  \end{align*}
  Similarly, since $f$ satisfies \cref{cond:2}, we must have
  \begin{align*}
    2\cdot \inparen{\eta f_{A1}+\inparen{\frac14-\eta}f_{A2}+\frac14 f_{A3} }
    &= \Ex\nolimits_\cP[f=1\mid Z=\ell] \\
    &= \Ex\nolimits_\cP[f=1]\\
    &= \eta f_{A1}+\inparen{\frac14-\eta}f_{A2}+\frac14 f_{A3} + \frac14 f_{B2}+ \frac14 f_{B3}.\yesnum\label{eq:condition_2}
  \end{align*}
  Combining \cref{eq:condition_1,eq:condition_2}, we get
  \begin{align*}
    \eta f_{A1}+\inparen{\frac14-\eta}f_{A2}+\frac14 f_{A3} &= \frac14 f_{B2}+ \frac14 f_{B3}\\
    &= \eta f_{A1}+\frac14 f_{A2}+ \inparen{\frac14-\eta}f_{A3}.
    \yesnum\label{eq:intermediate_eq}
  \end{align*}
  On canceling the like terms in the left-hand side and right-hand side, and using that $\eta>0$, we get
  \begin{align*}
    f_{A2} =f_{A3}.
  \end{align*}
  We consider two cases.\\

  \noindent {\bf (Case A) $f_{A2}=f_{A3}=1$:}
  Substituting $f_{A2}=f_{A3}=1$ in \cref{eq:intermediate_eq}, we get
  \begin{align*}
    \eta f_{A1}+{\frac12-\eta}&= \frac14 f_{B2}+ \frac14 f_{B3}.
  \end{align*}
  Here, the right-hand side can only take values $\inbrace{0,\nfrac14,\nfrac12}$ and the left-hand side can only take values $\inbrace{\nfrac12-\eta,\nfrac12}$.
  Thus, the unique solution is $f_{A1}=f_{B2}=f_{B3}=1$.
  One can verify that the unique resulting classifier has error
  ${\rm Err}_\cD(f)=\nfrac12-\eta.$

  \noindent {\bf (Case B) $f_{A2}=f_{A3}=0$:}
  Substituting $f_{A2}=f_{A3}=0$ in \cref{eq:intermediate_eq}, we get
  \begin{align*}
    \eta f_{A1}&= \frac14 f_{B2}+ \frac14 f_{B3}.
  \end{align*}
  Here, the right-hand side can only take values $\inbrace{0,\nfrac14,\nfrac12}$ and the left-hand side can only take values $\inbrace{0,\eta}$.
  Thus, the unique solution is $f_{A1}=f_{B2}=f_{B3}=0$.
  One can verify that the unique resulting classifier has error
  ${\rm Err}_\cD(f)=\nfrac12.$

  Thus, any $f\in \cF$ satisfying \cref{cond:1,cond:2}, must have an error at least $\nfrac12-\eta$.
\end{proof}

\section{{Additional Remarks about the \texorpdfstring{$\eta$}{eta}-\Ham{} Model and Theoretical Results}}\label{sec:ham_captures}

\noindent In \cref{example:ham_adv_can_reduce_fairness}, we show that the ratio of the fairness of a classifier $f\in \cF$ with respect to the perturbed samples $\hS$ and with respect to the unperturbed samples $S$ can be 0.
\begin{example}[\bf Ratio of fairness of a classifier on perturbed and unperturbed samples]\label{example:ham_adv_can_reduce_fairness}
  Let $S\coloneqq\smash{\inbrace{(x_i,y_i,z_i)}_{i\in [N]}}$ denote the $N$ unperturbed samples.
  Suppose that the fairness metric $\Omega$ is statistical rate and $S$ has an equal number of samples from each protected group (i.e., for all $\ell\in [p]$, $\sum_{i\in [N]}\mathds{I}[z_i=\ell]=N/p$.)
  Consider a classifier $f\in \cF$ that has exactly $\eta\cdot N$ positive predictions on each protected group $\ell\in [p]$, i.e., for all $\ell\in [p]$, $$\abs{\inbrace{i\in [N]\mid f(x_i,z_i)=1\text{ and } z_i=\ell}}=\eta\cdot N.$$
  This implies that $\Omega(f,S)=1$.
  Fix any protected group $\ell\in [p]$.
  An adversary $A\in \cA(\eta)$, can perturb the protected attributes of all $\eta\cdot N$ samples in the set $$\inbrace{i\in [N]\mid f(x_i,z_i)=1 \text{ and } z_i=\ell}.$$
  In this case, $\Pr\nolimits_{\hS}[f=1\mid Z=\ell]=0$.
  This implies that $\Omega(f,\hS) = 0$.
  Thus, in this example, $$\frac{\Omega(f,\hS)}{\Omega(f,S)}=0.$$
\end{example}

\noindent In \cref{rem:our_assumption_on_CF} we give two example hypothesis classes that satisfy the assumptions in \cref{thm:no_stable_classifier,thm:imposs_under_assumption_sr,thm:imposs_under_assumption_sr_2}.
\cref{thm:imposs_under_assumption_sr} assumes that there exist five distinct points ${x_A,x_B,x_C,x_D,x_E}\in \cX$ such that $\cF$ shatters the set of points $P\coloneqq \inbrace{x_A,x_B,x_C,x_D,x_E}\times [2]\subseteq \cX$.
\cref{thm:no_stable_classifier,thm:imposs_under_assumption_sr_2} makes the weaker assumption that $\cF$ shatters the set $\inbrace{x_A,x_B,x_C}\times [2]\subseteq \cX$.

\begin{remark}[]\label{rem:our_assumption_on_CF}\hfill
  \begin{enumerate}[leftmargin=\leftmarginINTERNAL,itemsep=\itemsepINTERNAL]
    \item {\bf Decision trees.}
    Suppose $\cX\coloneqq \R$.
    Then, the following hypothesis class of two-layer decision trees shatters $P$:
    On the first layer, the decision tree splits the root node into five nodes by thresholding $X\in \cX$.
    On the second layer, it further splits each node in the first layer into two leaves depending on whether $Z=1$ or $Z=2$.
    The resulting tree has 10 leaves.
    This hypothesis class shatters $P$:
    One can choose the thresholds in the first layer so that each of $x_A$, $x_B$, $x_C$, $x_D$, and $x_E$ belong to different nodes on the first layer.
    Then, the $2^{10}$ hypothesis generated by assigning different outcomes to the leaves shatter $P$.
    \item {\bf SVM with kernels.}
    Suppose $\cX\coloneqq \R^5$, and $x_A=e_1$, $x_B=e_2$, $x_C=e_3$, $x_D=e_4$, $x_E=e_5$, where $e_i$ is the $i$-th standard basis in $\R^5$.
    Consider the hypothesis class of SVM classifiers on the feature space $\cX^2$; where given a sample $(x,z)\in \cX\times [p]$, we map it to $\cX^2$ using the map $\psi\colon \cX\times [p]\to \cX^2$ defined as follows: $\psi(x,z)\coloneqq (x,z\cdot x)\in \cX^2$.
    One can verify that the hypothesis class of SVM classifiers in this feature space shatter the set $P$ (more precisely, it shatters the image of $P$ under map $\psi$).
  \end{enumerate}
\end{remark}

\begin{remark}[\bf Test errors]
  In the paper, we focus on the setting where given perturbed samples as input, the learner's goal is to find a classifier that satisfies a given set of fairness constraints
  with the highest accuracy, where both accuracy and fairness are measured with respect to {\em true distribution} $\cD$.
  In some applications, for example, when protected attributes are self-reported after deployment, the test samples can also have perturbations.
  Given a number $\Delta>0$, a sufficient number of test samples $T$, and an $\eta$-\Ham{} adversary $A$, one can show that
  $${\rm Err}_{\hD}(f_{\rm ET}) - {\rm Err}_\cD(f^\star) \leq 3\eta+\Delta\quad \text{and}\quad \Omega(f_{\rm ET},A(T)) \geq \tau - \frac{12\eta\tau}{\lambda-2\eta}-\Delta,$$
  where $A(T)$ denotes the perturbed test samples and $\hD$ is the empirical distribution of $A(T)$.
\end{remark}
\section{Reduction from \errtolerant{} to a Set of Convex Programs}\label{sec:efficient_algorithms_for_errtol}
In general, \errtolerant{} is a nonconvex optimization program.
But, we can reduce \errtolerant{} to a set of convex programs.
Formally, for any arbitrarily small $\alpha>0$, we can find an $f\in \cF$ that has the optimal objective value for \errtolerant{} and that additively violates its fairness constraint (\cref{eq:fairness_constraint_in_err_tol}) by at most $\alpha$, by solving a set of $O(\nfrac1{(\lambda\alpha)})$ convex programs.
In this section, we present this reduction.
It largely follows from \cite{celis2019classification}, but is included for completeness.

Recall that given a fairness metric $\Omega$ and corresponding events $\cE$ and $\cE^\prime$ (as in \cref{def:performance_metrics}),
perturbed samples $\hS$, whose empirical distribution is $\smash{\hD}$,
a perturbation rate $\eta\in [0,1]$, and constants $\lambda,\Delta\in (0,1]$,
\errtolerant{} is the following program:
\begin{align*}
  \min\nolimits_{f\in \cF}&\quad {\rm Err}_{\hD}(f),\quad\yesnum\\
  \st,&\quad
  \Omega_{}(f, \hS) \geq \tau \cdot \inparen{\frac{1- (\eta+\Delta)/\lambda}  {  1 + (\eta+\Delta)/\lambda}}^2,
  \\
  &\quad
  \text{$\forall\ \ell\in [p]$,\ } \Pr\nolimits_{\hD}\insquare{\cE(f),\cE^\prime(f), \hZ=\ell} \geq \lambda - \eta - \Delta.
\end{align*}
Equivalently defining scalars $\wh{\tau}\coloneqq \inparen{\frac{1- (\eta+\Delta)/\lambda}  {  1 + (\eta+\Delta)/\lambda}}^2$ and $\wh{\lambda}\coloneqq \lambda - \eta - \Delta$, our goal is to solve
\begin{align*}
  \min\nolimits_{f\in \cF}&\quad {\rm Err}_{\hD}(f),\quad\yesnum\label{prog:errTol:app}\\
  \st,&\quad
  \Omega_{}(f, \hS) \geq \wh{\tau},
  \yesnum\label{eq:fairness_const:app}\\
  &\quad
  \text{$\forall\ \ell\in [p]$,\ } \Pr\nolimits_{\hD}\insquare{\cE(f),\cE^\prime(f), \hZ=\ell} \geq \wh{\lambda}.\yesnum\label{eq:lb_const:app}
\end{align*}

\begin{remark}
  All references to the results in \cite{celis2019classification} are to its  \href{https://arxiv.org/abs/1806.06055}{arXiv version}.
\end{remark}

\begin{remark}
  In this section, all probabilities and expectations are with respect to the draw of perturbed samples $(X,Y,\hZ)$.
  Given $\hD$, the empirical distribution over $\wh{S}$, we use
  $\Pr_{\hD}[\cdot]$ to denote
  $\Pr_{(X,Y,\wh{Z})\sim \wh{D}}[\cdot]$ and $\Ex_{\hD}[\cdot]$ to denote
  $\Ex_{(X,Y,\wh{Z})\sim \wh{D}}[\cdot]$.
\end{remark}

\subsection{Performance Metrics in \texorpdfstring{\cref{def:performance_metrics}}{Definition 3.1} Are a Special Case of the Metrics in \texorpdfstring{\cite{celis2019classification}}{[13]}}
To use the results in \cite{celis2019classification}, we need to show that \cref{def:performance_metrics} is a special case of \cite[Definition 2.3]{celis2019classification}.
\begin{lemma}\label{lem:special_case}
  Suppose $\cF\coloneqq \zo^{\cX\times [p]}$.
  For all events $\cE$ and $\cE^\prime$, that can depend on $f$, the corresponding metric $q$ is a ``performance function'' as defined in \cite[Definition 2.3]{celis2019classification}.
\end{lemma}
\begin{proof}
  Observe that
  \begin{align*}
    q_\ell(f)\coloneqq \Pr\nolimits_{\wh{D}}\sinsquare{\cE(f)\mid \cE^\prime(f), \wh{Z}=\ell}
    = \frac{\Pr_{\wh{D}}\sinsquare{\cE(f), \cE^\prime(f)\mid \wh{Z}=\ell}}{\Pr_{\wh{D}}\sinsquare{\cE^\prime(f) \mid \wh{Z}=\ell}}.
  \end{align*}
  To simplify the notation, below, we use $f$ to denote $f(X,\hZ)$.
  We can rewrite the denominator as:
  \begin{align*}
    \Pr\nolimits_{\hD}\sinsquare{\cE^\prime(f)\mid \hZ=\ell}
    &= \Pr\nolimits_{\hD}\sinsquare{\cE^\prime(f)\mid \hZ=\ell}\\
    &= \Pr\nolimits_{\hD}\sinsquare{\cE^\prime(f)\mid f=0, \hZ=\ell}\cdot \Pr\nolimits_{\hD}\sinsquare{f=0\mid \hZ=\ell}\\
    & \hspace{0.5mm} +  \Pr\nolimits_{\hD}\sinsquare{\cE^\prime(f)\mid f=1, \hZ=\ell}\cdot \Pr\nolimits_{\hD}\sinsquare{f=1\mid \hZ=\ell}\\
    &= c_0\cdot \Pr\nolimits_{\hD}\sinsquare{f=0\mid \hZ=\ell} +  c_1\cdot \Pr\nolimits_{\hD}\sinsquare{f=1\mid \hZ=\ell},
  \end{align*}
  where we defined
  \begin{align*}
    c_0\coloneqq \Pr\nolimits_{\hD}\sinsquare{\cE^\prime(f)\mid f=0, \hZ=\ell},\\
    c_1\coloneqq \Pr\nolimits_{\hD}\sinsquare{\cE^\prime(f)\mid f=1, \hZ=\ell}.
  \end{align*}
  Let $\alpha_0\coloneqq c_0$ and $\alpha_1\coloneqq c_1-c_0$.
  Then, we have
  \begin{align*}
    \Pr\nolimits_{\hD}\sinsquare{\cE^\prime(f)\mid \hZ=\ell} &= c_0\cdot \inparen{1-\Pr\nolimits_{\hD}\sinsquare{f=1\mid \hZ=\ell}} +  c_1\cdot \Pr\nolimits_{\hD}\sinsquare{f=1\mid \hZ=\ell}\\
    &= \alpha_0 + \alpha_1\cdot \Pr\nolimits_{\hD}\sinsquare{f=1\mid \hZ=\ell}
    \tag{Using $\alpha_0\coloneqq c_0$ and $\alpha_1\coloneqq c_1-c_0$}\\
    &= \alpha_0 + \alpha_1\cdot \Pr\nolimits_{\hD}\sinsquare{f(X,\hZ)=1\mid \hZ=\ell}.
    \yesnum\label{eq:simplifying_q_1}
  \end{align*}
  Where the last equality follows due to our notation that the event $f=1$ denotes $f(X,\hZ)=1$ for random draws $(X,Y,\hZ)\sim \hD$.
  Next, by replacing $\cE^\prime(f)$ by $\cE(f)\land \cE^\prime(f)$, we get that
  \begin{align*}
    \Pr\nolimits_{\hD}\sinsquare{\cE(f),\cE^\prime(f)\mid \hZ=\ell}
    &= \beta_0 + \beta_1\cdot \Pr\nolimits_{\hD}\sinsquare{f(X,\hZ)=1\mid\hZ=\ell}.
    \yesnum\label{eq:simplifying_q_2}
  \end{align*}
  for some $0\leq \beta_0\leq 1$ and $-1\leq \beta_1\leq 1$.
  Comparing Equations~\eqref{eq:simplifying_q_1} and \eqref{eq:simplifying_q_2} with \cite[Definition 2.3]{celis2019classification}, it follows that $q_\ell(f)$ is a special case of the performance functions in \cite[Definition 2.3]{celis2019classification}.
\end{proof}

\subsection{Reduction from \errtolerant{} to a Set of Convex Programs}
Before stating the result,  we need some additional notation.
\begin{definition}
  Given a fairness metric $\Omega$, the corresponding performance metric $q_{\ell}$ (from \cref{def:performance_metrics}),
  desired fairness threshold $\tau\in (0,1]$,
  approximation parameter $\alpha\in (0,1]$,
  and lower and upper bounds $L,U\in [0,1]$,
  define the sets $K{(\tau,\alpha)},P{(L,U)}\subseteq \cF$ as
  \begin{align*}
    K{(\tau,\alpha)}
    &\coloneqq \inbrace{f\in \cF\colon \min\nolimits_{\ell\in [p]}q_\ell(f)\geq \tau\cdot \max\nolimits_{\ell\in [p]} q_\ell(f)-\alpha},\\
    P{(L,U)}
    &\coloneqq \inbrace{f\in \cF\colon\ \text{for all }\ell\in[p],\ L\leq q_\ell(f)\leq U}.
  \end{align*}
\end{definition}
\noindent Note that $K{(\tau,0)}$ (i.e., setting $\alpha=0$) is the set of classifiers that satisfy the fairness constraint $\Omega(f)\geq \tau$ exactly.
$K{(\tau,\alpha)}$ ($\alpha>0$) is the set of classifiers that satisfy a relaxation of this constraint.
{Formally, for any $\alpha>0$, the set of classifiers $\alpha$-feasible for \prog{prog:errTol:app} are all $f$ in $K{(\tau,\alpha)}$ that also satisfy \cref{eq:lb_const:app}.}
Under \cref{asmp:1},  any $\alpha$-feasible classifier $f\in \cF$ additively violates the fairness constraint in \prog{prog:errTol:app} by at most $\nfrac{\alpha}{\lambda}$.
To see this, suppose $f\in \cF$ is $\alpha$-stable, then
\begin{align*}
  \min\nolimits_{\ell\in [p]}q_\ell(f)
  &\geq \tau\cdot \max\nolimits_{\ell\in [p]}q_\ell(f)-\alpha\\
  &= \inparen{\tau-\frac{\alpha}{\max\nolimits_{\ell\in [p]}q_\ell(f)}}\cdot \max\nolimits_{\ell\in [p]}q_\ell(f)\\
  &\geq  \inparen{\tau-\frac{\alpha}{\lambda}}\cdot \max\nolimits_{\ell\in [p]}q_\ell(f),
  \tag{Using that $\max\nolimits_{\ell\in [p]}q_\ell(f)\geq \lambda$}
\end{align*}
and hence,
\begin{align*}
  \Omega_{\hD}(f)
  =
  \frac{\min\nolimits_{\ell\in [p]}q_\ell(f)}{\min\nolimits_{\ell\in [p]}q_\ell(f)}
  \geq \tau-\frac{\alpha}{\lambda}.\yesnum\label{eq:connection_bet_fea_and_vio}
\end{align*}
Using the above notation, we can write \prog{prog:errTol:app} as follows:
\begin{align*}
  \min\nolimits_{f\in \cF}&\quad {\rm Err}_{\hD}(f),\quad\\
  \st,&\quad
  f\in K{(\wh{\tau},0)},\\
  &\quad
  \text{$\forall\ \ell\in [p]$,\ } \Pr\nolimits_{\hD}\insquare{\cE(f),\cE^\prime(f), \hZ=\ell} \geq \wh{\lambda}.
\end{align*}
Here, $K{(\wh{\tau},0)}$ can be a nonconvex set.
But as \cite{celis2019classification} show, it can be approximated as a union of convex sets.
In particular, they approximate  $K{(\wh{\tau},0)}$ as the union $\bigcup_{j=1}^{J} P{(L_j,U_j)}$ for some $J\in\N$ and vectors $L,U\in [0,1]^J$.
(One can prove that for all $L,U\in [0,1]$, $P{(L,U)}$ is a convex set \cite{celis2019classification}.)
\begin{theorem}[\textbf{Implicit in \cite[Theorem 3.1]{celis2019classification}}]\label{thm:reduction}
  Given constants $\tau,\alpha\in (0,1]$,
  let $J\coloneqq \ceil{\nfrac{\tau}{\alpha}}$,
  and for all $j\in [J]$, let $L_j\coloneqq (j-1)\alpha$ and $U_j\coloneqq \nfrac{(j\alpha)}{\tau}$.
  For all fairness metrics $\Omega$ and corresponding performance metric $q_{\ell}$ (as defined in \cref{def:performance_metrics})
  it holds that
  \begin{align*}
    K{(\tau,0)}\subseteq \bigcup_{j=1}^{J} P{(L_j,U_j)} \subseteq  K{(\tau,\alpha)}.
  \end{align*}
\end{theorem}
\noindent \cref{thm:reduction} allows one to reduce the problem of finding an $\alpha$-feasible classifier to solving a set of $\ceil{\nfrac{\tau}{\alpha}}$ convex programs of the following form:
For some $L,U\in [0,1]$
\begin{align*}
  \min\nolimits_{f\in \cF}&\quad {\rm Err}_{\hD}(f),\label{prog:progA}\yesnum\\
  \st,&\quad
  f\in P{(L,U)},\yesnum\label{eq:const12}\\
  &\quad
  \text{$\forall\ \ell\in [p]$,\ } \Pr\nolimits_{\hD}\insquare{\cE(f),\cE^\prime(f), \hZ=\ell} \geq \wh{\lambda}.\yesnum\label{eq:const23}
\end{align*}
\begin{theorem}\label{thm:our_result}
  Given constants $\tau,\alpha\in (0,1]$,
  let $J\coloneqq \ceil{\nfrac{\tau}{\alpha}}$,
  and for all $j\in [J]$, let $L_j\coloneqq (j-1)\alpha$ and $U_j\coloneqq \nfrac{(j\alpha)}{\tau}$.
  Further, let $f_j$ be the optimal solution of \cref{prog:progA} with $L\coloneqq L_j$ and $U\coloneqq U_j$.
  Then, $f_{\alpha}\coloneqq \argmin_{f_j}{\rm Err}(f_j,\hS)$ has the optimal accuracy for \prog{prog:errTol:app} and is $\alpha$-feasible for \prog{prog:errTol:app}.
\end{theorem}
\noindent If \cref{asmp:1} holds, then using \cref{thm:our_result} and \cref{eq:connection_bet_fea_and_vio}, we can to find an $f\in \cF$ that has the optimal objective value for \errtolerant{} and that additively violates its fairness constraint \eqref{eq:fairness_constraint_in_err_tol} by at most $\alpha$ by solving a set of  $\ceil{\nfrac{\smash{\wh{\tau}}}{(\lambda\alpha)}}$ convex programs.
\begin{proof}[Proof of \cref{thm:our_result}]
  \white{.}

  \paragraph{\bf Fairness guarantee.} Since $f_{\alpha}$ is an optimal solution of \prog{prog:progA} with $L\coloneqq L_j$ and $U\coloneqq U_j$ for some $r\in [J]$, $f_{\alpha}\in P{(L_r,U_r)}$, and hence,
  $f\in \bigcup_{j=1}^{J} P{(L_j,U_j)}$.
  Therefore, by \cref{thm:reduction}, $f_{\alpha}\in K{(\tau,\alpha)}$.
  Since $f_{\alpha}\in K{(\tau,\alpha)}$ and $f_\alpha$ satisfies \cref{eq:const23}, $f_{\alpha}$ is an $\alpha$-feasible solution for \prog{prog:errTol:app}.

  \paragraph{\bf Accuracy guarantee.}
  Let $f_{\rm ET}\in \cF$ be the optimal solution of \prog{prog:errTol:app}.
  Since $f_{\rm ET}\in K{(\tau,0)}$ and by \cref{thm:reduction} the inclusion $\bigcup_{j=1}^{J} P{(L_j,U_j)}\supseteq K{(\tau,0)}$ holds, $f_{\rm ET}\in \bigcup_{j=1}^{J} P{(L_j,U_j)}$.
  Further, since $f_{\alpha}$ minimizes ${\rm Err}(\cdot,\hS)$ over $\bigcup_{j=1}^{J} P{(L_j,U_j)}$ and by \cref{thm:reduction} the containment $\bigcup_{j=1}^{J} P{(L_j,U_j)}\supseteq K{(\tau,0)}$ holds, it follows that ${\rm Err}(f_{\alpha},\hS)\leq {\rm Err}(f_{\rm ET}, \hS).$
\end{proof}

\renewcommand{\folder}{./figures/synthetic-data}
\newcommand{\chkvfpr}{{\bf CHKV-FPR}}
\newcommand{\wangdro}{{\bf WGN+DRO}}
\newcommand{\wangsw}{{\bf WGN+SW}}

\section{Implementation Details and Additional Empirical Results}\label{sec:more_empirical_results}
In this section, we give an implementation of our optimization framework using the logistic loss function (\cref{sec:log_loss_implementation}),
list the all hyper-parameters used for our approach (\cref{sec:log_loss_implementation}) and baselines (\cref{sec:base_line_hyper_params}), and present some additional empirical results (\cref{sec:additional_empirics}).

\paragraph{Code.} The code for all the simulations is available at \url{https://github.com/AnayMehrotra/Fair-classification-with-adversarial-perturbations}.

\subsection{{Implementation Details}}\label{sec:more_empirical_results:implementation_details}
\subsubsection{Implementation of Our Framework Using Logistic Regression}\label{sec:log_loss_implementation}
As an illustration, we implement our optimization framework (Program~\ref{prog:errtolerant_extnd}) using the logistic regression framework.
For some $\theta\in \R^d$, let $f_\theta$ be the linear-classifier that given an input $x\in \R^d$ predicts $$\mathds{I}\insquare{\frac{1}{{1+e^{-\inangle{x,\theta}}}} \geq 0.5}$$
(or equivalently that predicts $\mathds{I}[\inangle{x,\theta}\geq 0]$).
Then, the logistic regression framework considers the following hypothesis class: $$\cF_{\rm LR}\coloneqq \inbrace{ f_\theta \mid \theta \in \R^d },$$ parameterized by $\theta$;
see \cite{shalev2014understanding} for more details.
Several baselines (e.g., \chkv{} and \lmzv{}) do not use protected attributes for prediction.
For a fair comparison, in this implementation we do not use protected attributes for prediction.
Let the domain of the features $\cX$ satisfy $\cX\subseteq \R^t$ for some $t\in \N$.
In this case, we have $d\coloneqq t$, where $d$ is the dimension of $\theta$.
(To use protected attributes for prediction, one can set $d\coloneqq t+1$ and append the protected attribute to the features.)

Recall that Program~\ref{prog:errtolerant_extnd} takes the following values as input:
the perturbation rate $\eta\in[0,1]$, fairness threshold $\tau\in [0,1]$, and for each $\ell\in [p]$, $\lambda_\ell\in [0,1]$ and $\gamma_\ell \in[0,1]$.
Given these, we solve the following problem and initialize $s$ to be its solution
\begin{align*}
  \min_{\eta_1,\eta_2,\dots,\eta_p\geq 0}\ \min_{\ell,k\in [p]}\frac{1-\sfrac{\eta_\ell}{\lambda_\ell}}{1+\sfrac{(\eta_k-\eta_\ell)}{\gamma_\ell}}
  \cdot
  \frac{1 +\sfrac{(\eta_\ell-\eta_k)}{\gamma_k}}{1+\sfrac{\eta_\ell}{\lambda_k}},
  \quad \st, \quad \sum_{\ell\in [p]}\eta_\ell \leq \eta+\Delta.\yesnum\label{eq:internal_min}
\end{align*}
We solve Program~\eqref{eq:internal_min} once to initialize $s$, then the same value $s$ is used for all runs of Program~\ref{prog:errtolerant_extnd}.
Let $\cE(f)$ and $\cE^\prime(f)$ denote the events defining the relevant linear-fractional fairness metric (\cref{def:performance_metrics}).
Let $\hS\coloneqq \inbrace{(x_i,y_i,z_i)}_{i\in [N]}$ be the perturbed samples.
Then we solve the following constrained optimization program
\begin{align*}
  \min_{\theta\in \R^d}\quad  & \frac{1}{N}\cdot {\sum_{i=1}^N y_i\cdot \log{f_\theta(x_i)} + (1-y_i)\cdot \log{(1-f_\theta(x_i))} }\,\quad
  \tagnum{ErrTolerant+}\\
  \st,\quad
  &\qquad\qquad \min_{\ell,k\in [p]} \frac{q_{\ell}(f)}{q_k(f)} \geq \tau\cdot s\\
  &\quad\text{$\forall\ \ell\in [p]$,\ } \quad \frac1N \sum_{i\in [N]}\mathds{I}[\cE(f(x_i)), \cE^\prime(f(x_i)), Z=\ell]
  \geq \lambda_\ell - \eta - \Delta,
\end{align*}
where for each $\ell\in [p]$,
$$q_{\ell}(f)\coloneqq \frac{\sum_{i\in [N]: z_i=\ell}\mathds{I}[\cE(f(x_i)), \cE^\prime(f(x_i))]}{\sum_{i\in [N]: z_i=\ell}\mathds{I}[\cE^\prime(f(x_i))]}.$$

\noindent In particular, for statistical rate, for each $i\in [N]$,
$$\cE(f(x_i))\coloneqq \mathds{I}[f(x_i)=1] \text{ and } \cE^\prime(f(x_i))\coloneqq 1,$$
and for false-positives rate, for each $i\in [N]$,
$$\cE(f(x_i))\coloneqq \mathds{I}[f(x_i)=1] \text{ and } \cE^\prime(f(x_i))\coloneqq \mathds{I}[y_i=0].$$
By substituting the appropriate $\cE$ and $\cE^\prime$, one can extend this implementation to any linear-fractional fairness metric (\cref{def:performance_metrics}).
\paragraph{Hyper-parameters.}
As a heuristic, given $\cE$ and $\cE^\prime$, in our simulations for each $\ell\in [p]$,
we set $$\gamma_\ell=\lambda_\ell\coloneqq \Pr\nolimits_{(X,Y,\hZ)\sim\hD}[\cE^\prime(Y), \hZ=\ell],$$ where $\smash{\hD}$ is the empirical distribution of $\hS$ and $Y$ is the label in perturbed data $\widehat{S}$.
We find that these estimates suffice, and expect that a more refined approach would only improve the performance of \errtol{}.
For all simulations, we set $$\Delta\coloneqq 10^{-2}.$$

\subsubsection{SLSQP Parameters}\label{sec:slsqp_params}
For simplicity, we do not implement the algorithm mentioned in \cref{sec:efficient_algorithms_for_errtol}, and instead use existing optimization packages in our implementation.
We solve both Program~\ref{prog:errtolerant_extnd} and Program~\eqref{eq:...} using the SLSQP solver~\cite{kraft1988software} in SciPy~\cite{scipy}.
For each optimization problem, we run the solver for 1000 iterations with parameters $$\texttt{ftol}=10^{-4}\quad\text{and} \quad \texttt{eps}=10^{-4},$$ starting at a point chosen uniformly at random.
If the solver fails to find a feasible solution, we rerun the solver for up to 10 iterations.
If it does not find a feasible solution after 10 iterations, we return the infeasible point reached.

\subsubsection{Implementation Details of the Adversaries}
Across our simulations we consider three $\eta$-Hamming adversaries (which we call $A_{\rm TN}$, $A_{\rm FN}$, and $A_{\rm FP}$).
Each adversary has access to the true samples $S$, the fairness metric $\Omega$, and the desired fairness threshold $\tau$.
Using these, the adversary computes the ``optimal fair classifier'' $f^\star$ that has the highest accuracy (on $S$) subject to satisfying $\Omega(f,S)\geq \tau$.
$f^\star$ is an optimal solution of  Program~\eqref{prog:target_fair}; note that Program~\eqref{prog:target_fair} is a special case of \errtolerant{} (with $\lambda,\eta,\Delta\to 0$).
In practice, we compute $f^\star$ by solving Program~\eqref{prog:target_fair} on the unperturbed data $S$, using the SLSQP solver in the SciPy package to heuristically solve Program~\eqref{prog:target_fair} (with the same parameters as described in \cref{sec:slsqp_params}).

After computing $f^\star$, $A_{\rm TN}$ considers the set of all true negatives of $f^\star$ that have protected attribute $Z=1$, selects the $\eta \cdot |S|$ samples that are furthest from the decision boundary of $f^\star$, and perturbs their protected attribute to $\hZ=2$.
$A_{\rm FN}$ and $A_{\rm FP}$ are identical, except that they consider the set of all false negatives and false positives of $f^\star$ respectively.

\subsubsection{Baseline Parameters and Implementation}\label{sec:base_line_hyper_params}

\paragraph{\lmzv{}.}
We use the implementation of the \lmzv{} at \url{https://github.com/AIasd/noise_fairlearn} provided by \cite{LamyZ19}; where the base classifier is by \cite{AgarwalBD0W18}.
\lmzv{} takes group-specific perturbation-rates, for each $\ell\in[p]$, $\eta_\ell \coloneqq \Pr\nolimits_D[\hZ\neq Z\mid Z=\ell]$, as input, and controls for additive statistical rate.
The desired level of fairness is controlled by $\delta_L\in [0,1]$, where smaller $\delta_L$ corresponds to higher fairness.
We refer the reader to \cite{LamyZ19} for a description of these parameters.
In our simulations, we vary $\delta_L$ over $\sinbrace{10^{-2},4\cdot 10^{-2},10^{-1}}$.
We fix all other hyper-parameters to the ones suggested by the authors for COMPAS.

\paragraph{\akm{}.}
We use the implementation of \akm{} at \url{https://github.com/matthklein/equalized_odds_under_perturbation} provided by \cite{awasthi2020equalized}; \akm{} is the equalized-odds postprocessing method of \cite{hardt2016equality}.
It takes the unconstrained optimal classifier (\uncons{}) as input, and post-processes its outputs to control for equalized-odds constraints.

\paragraph{\wangdro{}.}
This is the distributionally robust framework of \cite{wang2020robust};
we use the implementation of \wangdro{} at \url{https://github.com/wenshuoguo/robust-fairness-code} provided by \cite{wang2020robust}, which controls for additive false-positive rate.
It takes true and perturbed protected attributes as input and computes the required bound on the total variation distance.
We use the following learning rates for \wangdro{}{\bf :}
\begin{align*}
  \eta_{\theta}\in \inbrace{10^{-3},10^{-2},10^{-1}},\quad
  \eta_{\lambda}\in \inbrace{\nfrac14, \nfrac12, 1, 2},\quad and \quad
  \eta_{\wt{p}_j}\in \inbrace{10^{-3},10^{-2},10^{-1}};
\end{align*}
these are the same as the learning rates used by the authors (see \cite[Table 2]{wang2020robust}).
We refer the reader to \cite{wang2020robust} for the details of the parameters.
The implementation runs \wangdro{} for all combinations of learning rates and outputs the
classifier that has the best training objective and satisfies their constraints.

\paragraph{\wangsw{}.}
This is the ``soft-weights'' framework of \cite{wang2020robust};
we use the implementation of \wangsw{} at \url{https://github.com/wenshuoguo/robust-fairness-code} provided by \cite{wang2020robust}.
It takes true and perturbed protected attributes as input and controls for additive false-positive rate.
We use the following learning rates for \wangsw{}{\bf :}
\begin{align*}
  \eta_{\theta}\in \inbrace{10^{-3},10^{-2},10^{-1}},\quad
  \eta_{\lambda}\in \inbrace{\nfrac14, \nfrac12, 1, 2},\quad and \quad
  \eta_{w}\in \inbrace{10^{-3},10^{-2},10^{-1}};
\end{align*}
these are the same as the learning rates used by the authors (see \cite[Table 2]{wang2020robust}).
See \cite{wang2020robust} for a discussion on the parameters.
Their implementation runs \wangsw{} for all combinations of learning rates and outputs the
classifier that has the best training objective and satisfies their constraints.

\paragraph{\chkv{} and \chkvfpr{}.}
We use the implementation of the \cite{celis2020fairclassification}'s framework at \url{https://github.com/vijaykeswani/Noisy-Fair-Classification} provided by \cite{celis2020fairclassification}.
We use the implementations for statistical rate and false-positive rate, which we refer to these as \chkv{} and \chkvfpr{} respectively.
Both implementations take group specific perturbation-rates, for each $\ell\in[p]$,
\begin{align*}
  \eta_\ell \coloneqq \Pr\nolimits_D[\hZ\neq Z\mid Z=\ell],
\end{align*}
as input.
The desired level of fairness is controlled by $\tau\in [0,1]$, where a larger $\tau$ corresponds to higher fairness.
In our simulations, we vary $\tau$ over $\inbrace{0.7,0.8,0.9,0.95,1.0}$; other hyper-parameters were the same as those suggested by the authors for COMPAS.

\paragraph{\kl{}.}
This is the framework of \cite{konstantinov2021fairness} which controls for true-positive rate.
It takes the perturbation rate $\eta$ and for each $\ell\in [p]$, the probability
\begin{align*}
  p_{1\ell}\coloneqq \Pr\nolimits_D[Z=\ell,Y=1]
\end{align*}
as input; where $D$ is the empirical distribution of $S$.
\cite{konstantinov2021fairness} do not provide an implementation of \kl{}.
We implement \kl{} using the logistic loss function.
In particular, we solve the following optimization problem
\begin{align*}
  \min_{\theta\in \R^d}\quad  & \frac{1}{N}\cdot {\sum_{i=1}^N y_i\cdot \log{f_\theta(x_i)} + (1-y_i)\cdot \log{(1-f_\theta(x_i))} }\,\quad
  \yesnum\label{eq:kl_21_algorithm}\\
  \st,\quad
  &\quad\text{$\forall\ \ell\in [p]$,\ } \quad
  \frac
  {\sum_{i\in [N]} \mathds{I}[f(x_i)=0, y_i=1, \hz_i=\ell]}
  {\sum_{i\in [N]}\mathds{I}[y_i=1, \hz_i=\ell]}
  \leq \frac{6\eta}{\min_{\ell\in [p]}{p_{1\ell}}+3\eta}
\end{align*}
We solve Problem~\eqref{eq:kl_21_algorithm} using the standard implementation of the SLSQP solver in SciPy~\cite{scipy}; with the same parameters \cref{sec:slsqp_params}.

\subsubsection{Computational Resources Used}\label{sec:compute_used}
All simulations were run on a \texttt{t3a.2xlarge} instance, with 8 vCPUs and 32 Gb RAM, on Amazon's Elastic Compute Cloud (EC2).
\subsection{Visualization of Synthetic Data}\label{sec:additional_empirics}

\begin{figure}[h!]
  \centering
  \hspace{-4mm}\subfigure[Samples with protected attribute $Z=1$\label{fig:syndata:1}]{
  \begin{tikzpicture}
    \node (image) at (0,0) {\includegraphics[width=0.48\linewidth, trim={0cm 0cm 0cm 1cm},clip]{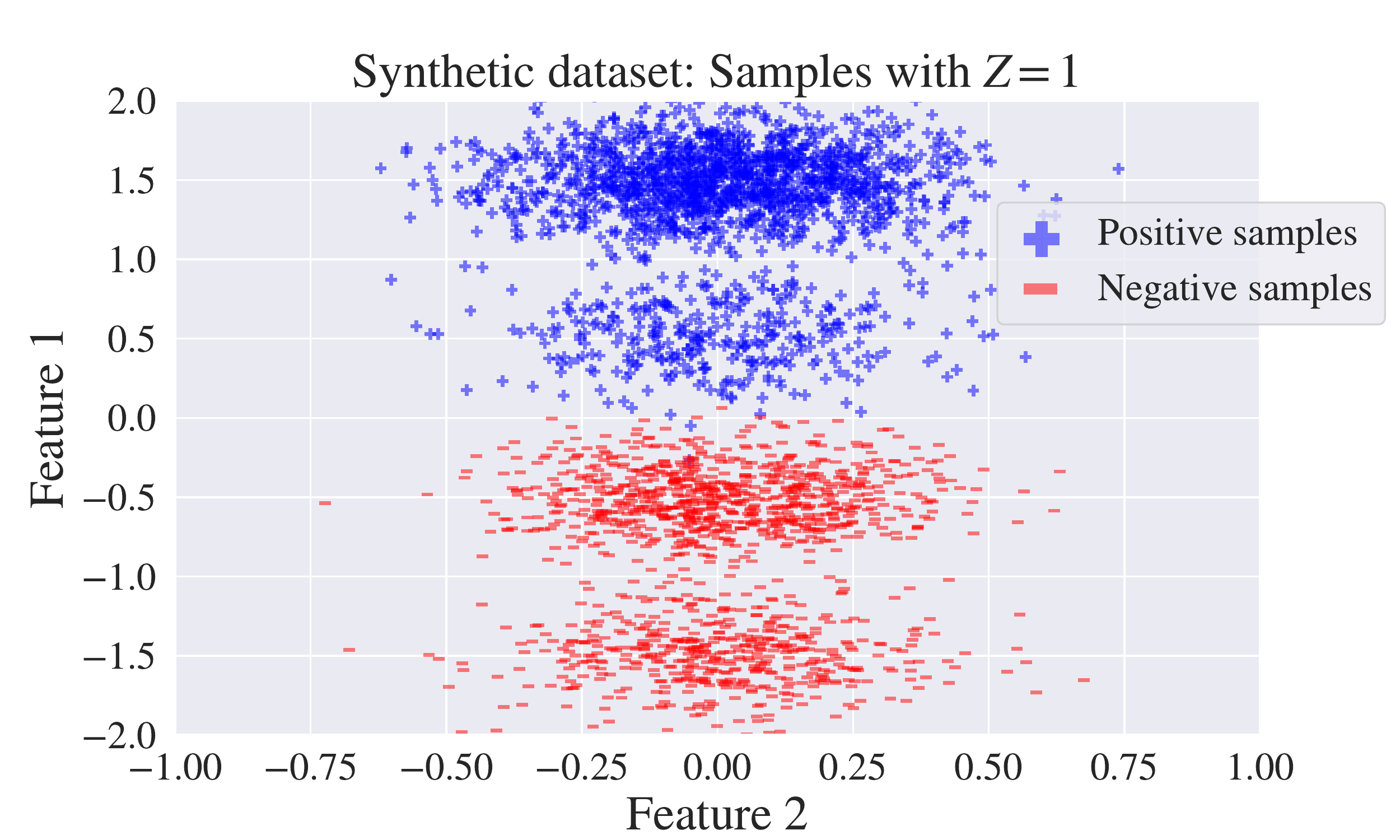}};
  \end{tikzpicture}
  }
  \hspace{-3mm}
  \subfigure[Samples with protected attribute $Z=2$\label{fig:syndata:2}]{
  \begin{tikzpicture}
    \node (image) at (0,0) {\includegraphics[width=0.48\linewidth, trim={0cm 0cm 0cm 1cm},clip]{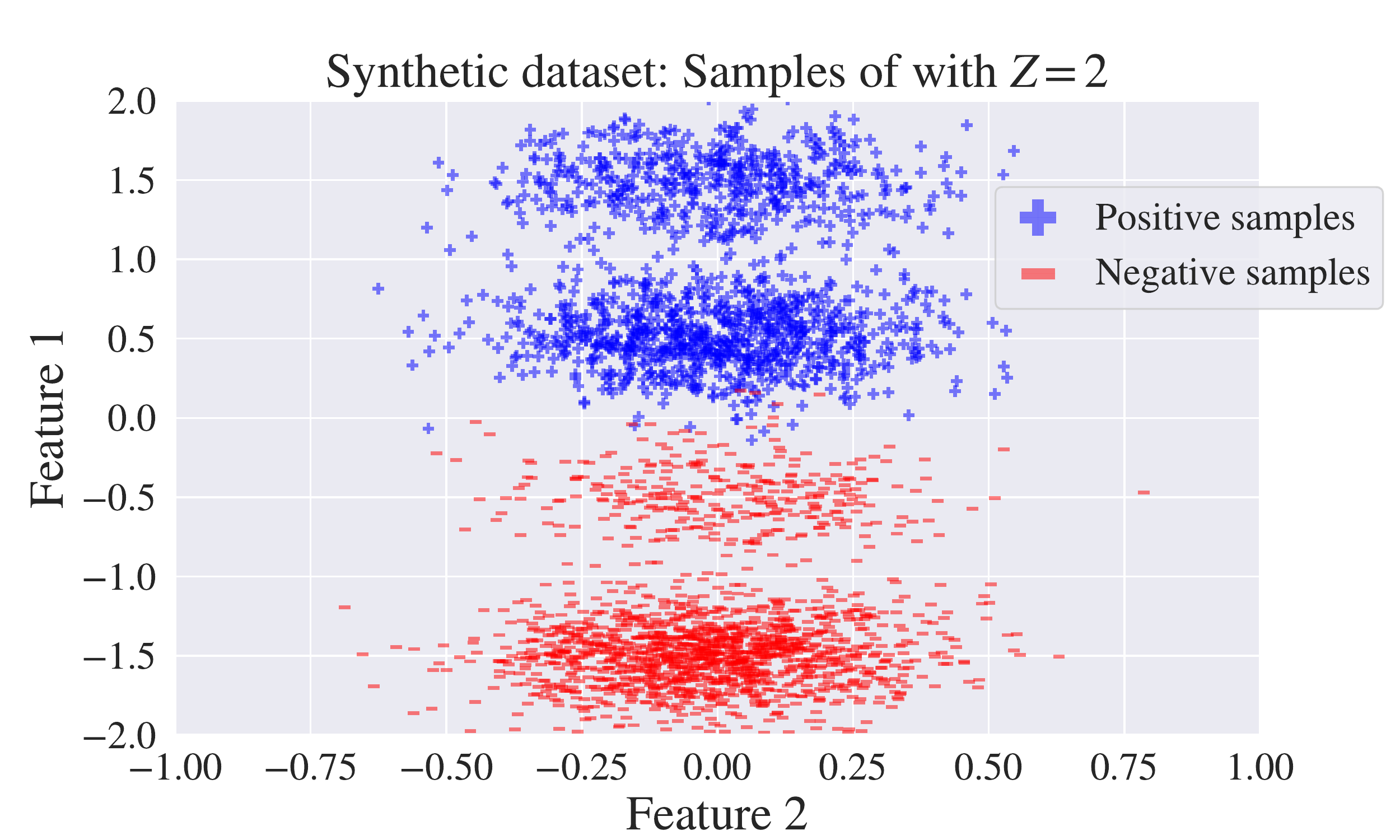}};
  \end{tikzpicture}
  }
  \hspace{-4mm}
  \caption{
  \ifconf\small\fi
  {\em {Samples from the synthetic data (\cref{sec:empirics}).}}
  \cref{fig:syndata:1} shows the samples with protected attribute $Z=1$ and \cref{fig:syndata:2} shows the samples with protected attribute $Z=2$.
  We consider synthetic data with 1,000 samples with two equally-sized protected groups; each sample has a binary protected attribute, two continuous features $x_1,x_2\in \R$, and a binary label.
  Conditioned on the protected attribute, ($x_1,x_2$) are independent draws from a mixture of four 2D Gaussians. %
  The distribution of the labels and features is such that 1) the protected group $Z=1$ has a higher likelihood of a positive label than the protected group $Z=2$, and 2) \uncons{} {has a near-perfect accuracy ($>99\%$) and a statistical rate of $0.8$ on $S$.}
  %
  }
  \label{fig:synthetic_dataset}
\end{figure}

\subsection{Additional Empirical Results}

\subsubsection{Simulation Varying the Size of the Protected Groups on Synthetic Data}\label{sec:vary_size}
In this simulation, we study the effect of varying the relative sizes of the protected groups in the synthetic data on the results of the simulation in  \ifexpand\cref{sec:empirics:synthetic}\else\cref{sec:empirics}\fi.
We vary the size of one protected group $\alpha$ from $0.15$ to $0.85$ (and the size of the other from $0.85$ to $0.15$.\footnote{15\% and 85\% are (roughly) the smallest and largest group sizes for which $f^\star$ has a sufficient number of true negatives to use $A_{\mathrm{TN}}$ with $\eta=5$\%.}
Recall that the synthetic data in \ifexpand\cref{sec:empirics:synthetic}\else\cref{sec:empirics}\fi\ is generated by sampling 500 samples with $Z=1$ from a distribution $\cD_1$ and sampling 500 samples with $Z=2$ from different distribution $\cD_2$ (where both $\cD_1$ and $\cD_2$ are mixtures of four 2D Gaussians).
(See \cref{fig:synthetic_dataset} for a plot of samples from $\cD_1$ and $\cD_2$).
To vary the size of the protected groups, given an $\alpha\in [0,1]$,
we draw $1000\cdot \alpha$ samples with $Z=1$ from $\cD_1$ and
$1000\cdot (1-\alpha)$ samples with $Z=2$ from $\cD_2$.
For each value of $\alpha$, we rerun the simulation from \ifexpand\cref{sec:empirics:synthetic}\else\cref{sec:empirics}\fi\  on the resulting data.

{\bf Results.}
We observe that varying $\alpha$ from from 15\% to 70\% does not change the accuracy and statistical rate of the algorithms significantly (the accuracy changed by <5\% and the statistical rate changed by <1\%).
However, as $\alpha$ approaches 85\%, we observe that \chkv{}\textbf{’s} accuracy reduced by 15\% (to $\approx$0.68) and its statistical rate increased by 4\% (to $\approx$0.86). In contrast, \errtol{} continued to have high accuracy (>0.99) and statistical rate ($\geq$0.77), without large changes in either (its accuracy changed by <1\% and statistical rate changed by <2\%).
Overall, we observe that prior approaches can fail to satisfy their guarantees under the $\eta$-\Ham{} model, and their deviation from the accuracy guarantee increases as the size of one protected group decreases (i.e., when $\alpha$ approaches 0.85).

\renewcommand{\folder}{./figures/synthetic-vary-eta}
\begin{figure}[t]
  \centering
  \tikzmath{\s = 1.2;}
  \hspace{-5mm}\includegraphics[width=0.9\linewidth, trim={0cm 0cm 5cm 0cm},clip]{\folder/../legend-medium.pdf}
  \par \vspace{-4mm}
  \subfigure[$A_{\rm TN}$ with $\eta=1.75\%$.\label{fig:vary_size:a}]{
  \begin{tikzpicture}
    \node (image) at (0,-0.17) {\includegraphics[width=0.43\linewidth, trim={1.4cm 0cm 1.5cm 1cm},clip]{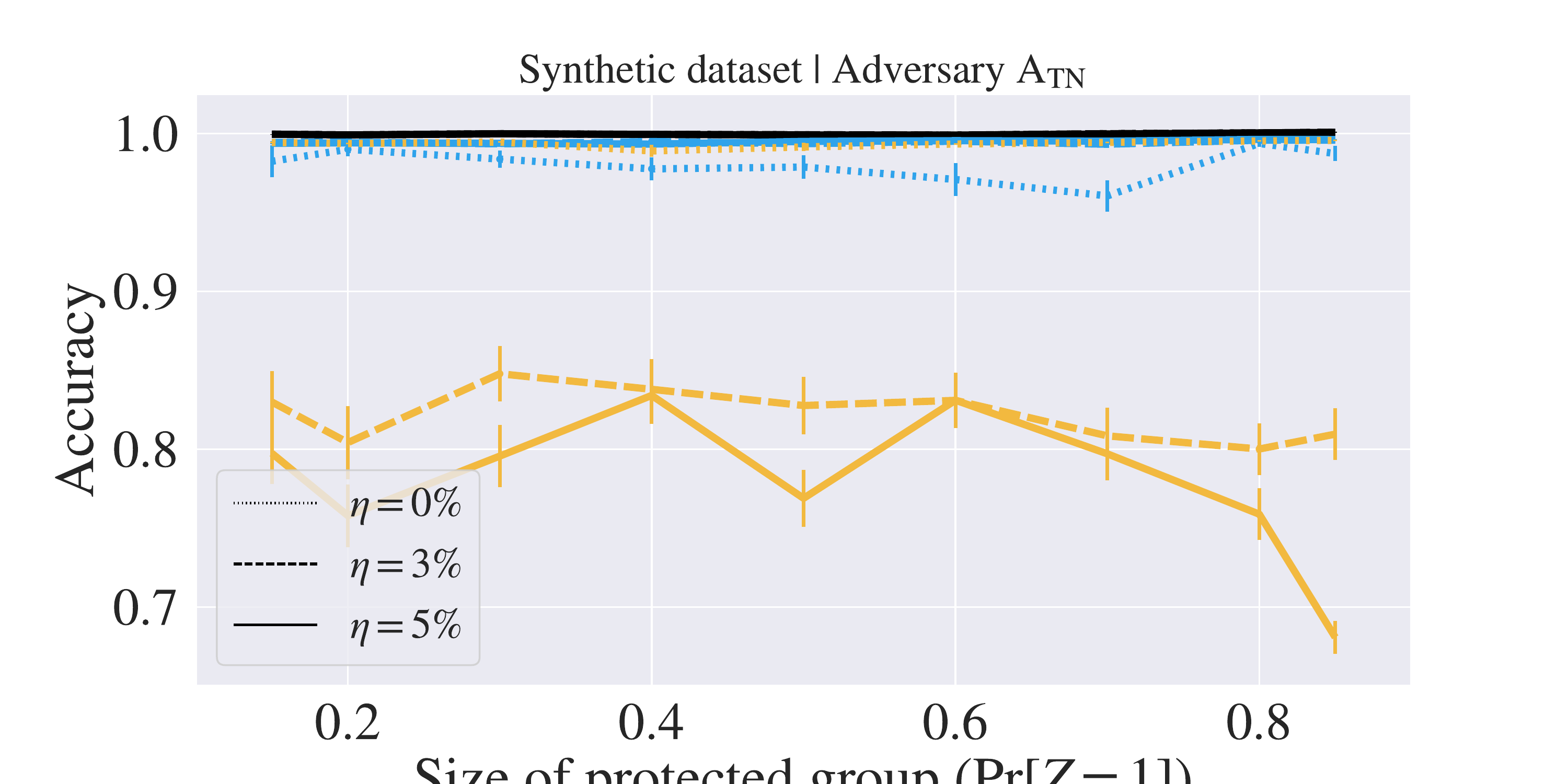}};
    \node[rotate=90, fill=white] at (-3.55,0) {\white{\large aaaaaaaaaaaaa}};
    \node[rotate=90] at (-3.9, 0) {Accuracy};
    \draw[draw=white, fill=white] (-2, -2.2) rectangle ++(4,0.35);
    \draw[draw=white, fill=white] (-2, 1.45) rectangle ++(4,0.25);
    \node[rotate=0, fill=white] at (0+0.05, 0.15-3.3*5/8-0.1+0.15)  {Size of protected group ($\Pr[Z=1]$)};
  \end{tikzpicture}
  }
  \subfigure[$A_{\rm FN}$ with $\eta=1.75\%$.\label{fig:vary_size:b}]{
  \begin{tikzpicture}
    \node (image) at (0,-0.17) {\includegraphics[width=0.43\linewidth, trim={1.4cm 0cm 1.5cm 1cm},clip]{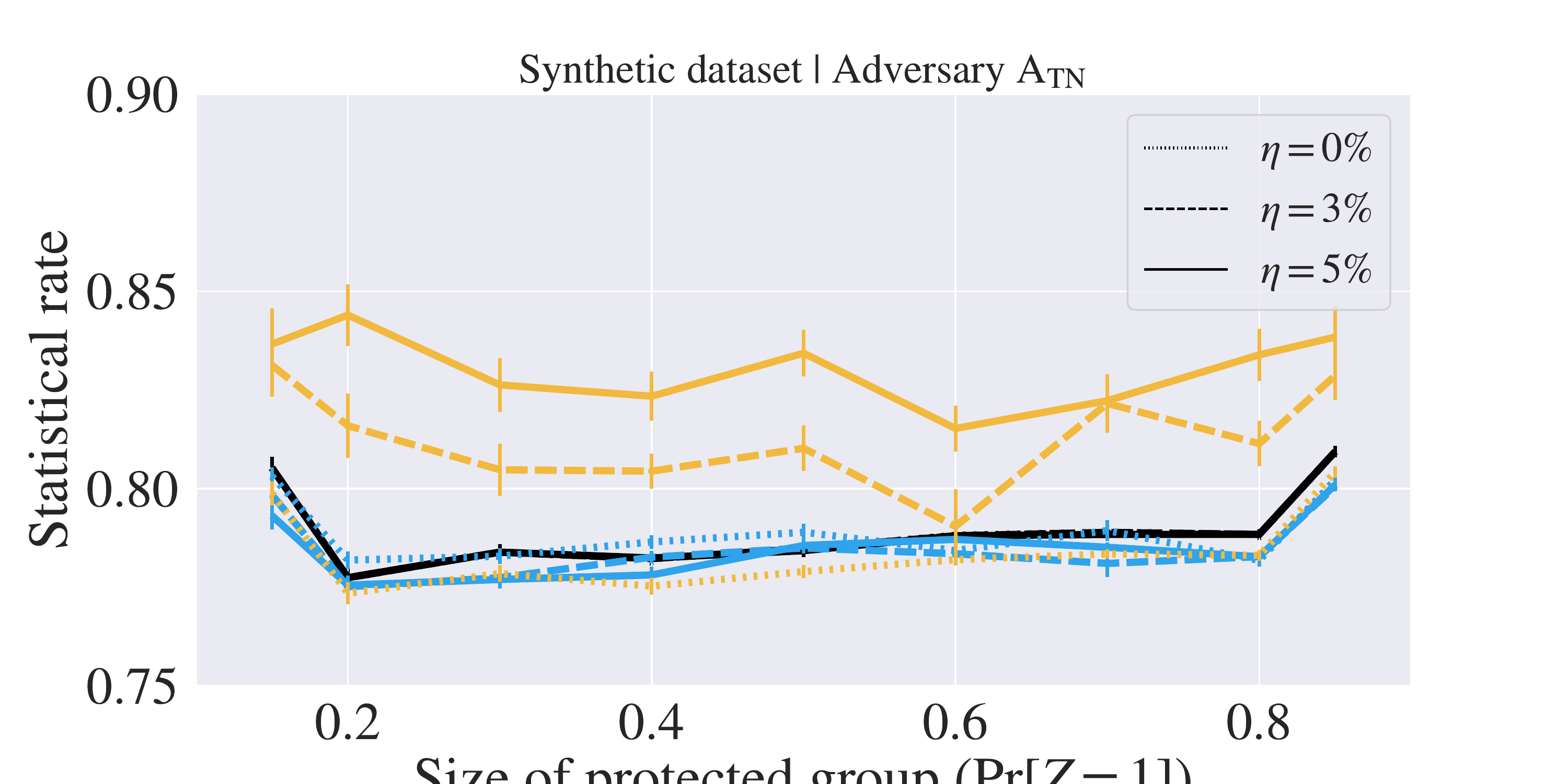}};
    \node[rotate=90, fill=white] at (-3.75,0) {\white{\large aaaaaaaaaaaaa}};
    \node[rotate=90] at (-3.9, 0) {Accuracy};
    \draw[draw=white, fill=white] (-2, -2.2) rectangle ++(4,0.35);
    \draw[draw=white, fill=white] (-2, 1.45) rectangle ++(4,0.25);
    \node[rotate=0, fill=white] at (0+0.05, 0.15-3.3*5/8-0.1+0.15)  {Size of protected group ($\Pr[Z=1]$)};
  \end{tikzpicture}
  }
  %
  \caption{
  \ifconf\small\fi
  {\em {Simulation varying the size of the protected groups on the synthetic data (see \cref{sec:vary_size}):}}
  We vary the perturbation rate $\eta$ over $\inbrace{0,0.03,0.05}$ and the size $\alpha$ of the protected group denoted by $Z=1$ from 15\% to 85\% (see \cref{sec:vary_size}).
  For each pair of $\alpha$ and $\eta$, we run \chkv{} and \errtol{} with $\tau=0.8$ and report the average accuracy and statistical rate plots \cref{fig:vary_size:a,fig:vary_size:b} (respectively).
  (In the plot, the color of the line identifies the algorithm and its style identifies the value of $\eta$).
  The $y$-axis depicts accuracy or statistical rate and the $x$-axis depicts $\alpha$; both the accuracy and the statistical rate are computed over the unperturbed test set (we refer the reader to  \cref{sec:empirics} for further details).
  We observe that prior approaches can fail to satisfy their guarantees under the $\eta$-\Ham{} model, and their deviation from the accuracy guarantee increases as the size of one protected group decreases (i.e., when $\alpha$ approaches 0.85).
  Error bars represent the standard error of the mean over 100 iterations.
  }

  \label{fig:vary_size}
\end{figure}

\subsubsection{Simulations with Stochastic Perturbations on the COMPAS Data}\label{sec:simulations_with_flipping_noise}
In this simulation, we evaluate our framework under stochastic perturbations on the COMPAS data, and show that it has a similar statistical rate and accuracy trade-off as approaches tailored for stochastic perturbations (e.g., \cite{LamyZ19} and \cite{celis2020fairclassification}).
Concretely, we consider a binary protected attribute and the perturbation model studied by \cite{celis2020fairclassification}:
Suppose we have a single binary protected attribute (i.e., $p=2$).
Given values $\eta_1,\eta_2\in [0,1]$, the protected attributes of each item with protected attribute $Z=1$ change to $\hZ=2$ with probability $\eta_1$ (independently), and similarly, the protected attributes of each item with protected attribute $Z=2$ change to $\hZ=1$ with probability $\eta_2$ (independently).
We consider the COMPAS data as preprocessed by \cite{ibm_aif360}, and consider gender (coded as binary) as the protected attribute.
We consider four values of $(\eta_1,\eta_2)$: $(0\%,0\%)$, $(0\%,3.5\%)$, $(3.5\%,0\%)$, and $(3.5\%, 3.5\%)$.

\paragraph{Results.}
The accuracy and statistical rate of \errtol{} and baselines for $\tau\in [0.7,1]$ and $\delta_L\in [0,0.1]$, averaged over 100 iterations, are reported in \cref{fig:stochastic_noise}.
For all settings of $\eta_1$ and $\eta_2$, \errtol{} attains a better statistical rate than the unconstrained classifier (\uncons{}) for a small trade-off in accuracy.
Further, \errtol{} has a similar statistical rate and accuracy trade-off as \chkv{} and \lmzv{}{\bf .}
In all cases, \akm{} a better statistical rate and accuracy trade-off than all other approaches.
Understanding why \akm{} has a better trade-off than other approaches with flipping noise requires further study.
But it is likely because \akm{} does not need to make pessimistic assumptions on the data (as it does not account for perturbations) and outputs a classifier from a richer hypothesis class compared to other approaches.
However, we recall that, when the perturbations are adversarial, \errtol{} has a better accuracy and statistical rate trade-off than \akm{} (see \cref{fig:adversarial_noise_main_body}).

\renewcommand{\folder}{./figures/stochastic-noise}
\begin{figure}[t]
  \centering
  \hspace{-5mm}\includegraphics[width=0.88\linewidth, trim={0cm 0cm 0cm 0cm},clip]{\folder/../legend-medium.pdf}
  \par\vspace{-2.5mm}
  \centering
  \par
  \hspace{-7mm}\subfigure[$(\eta_1,\eta_2)=(3.5\%,3.5\%)$\label{5a}]{
  \begin{tikzpicture}
    \node (image) at (0,-0.17) {\includegraphics[width=0.43\linewidth, trim={1.4cm 0cm 1.5cm  1cm},clip]{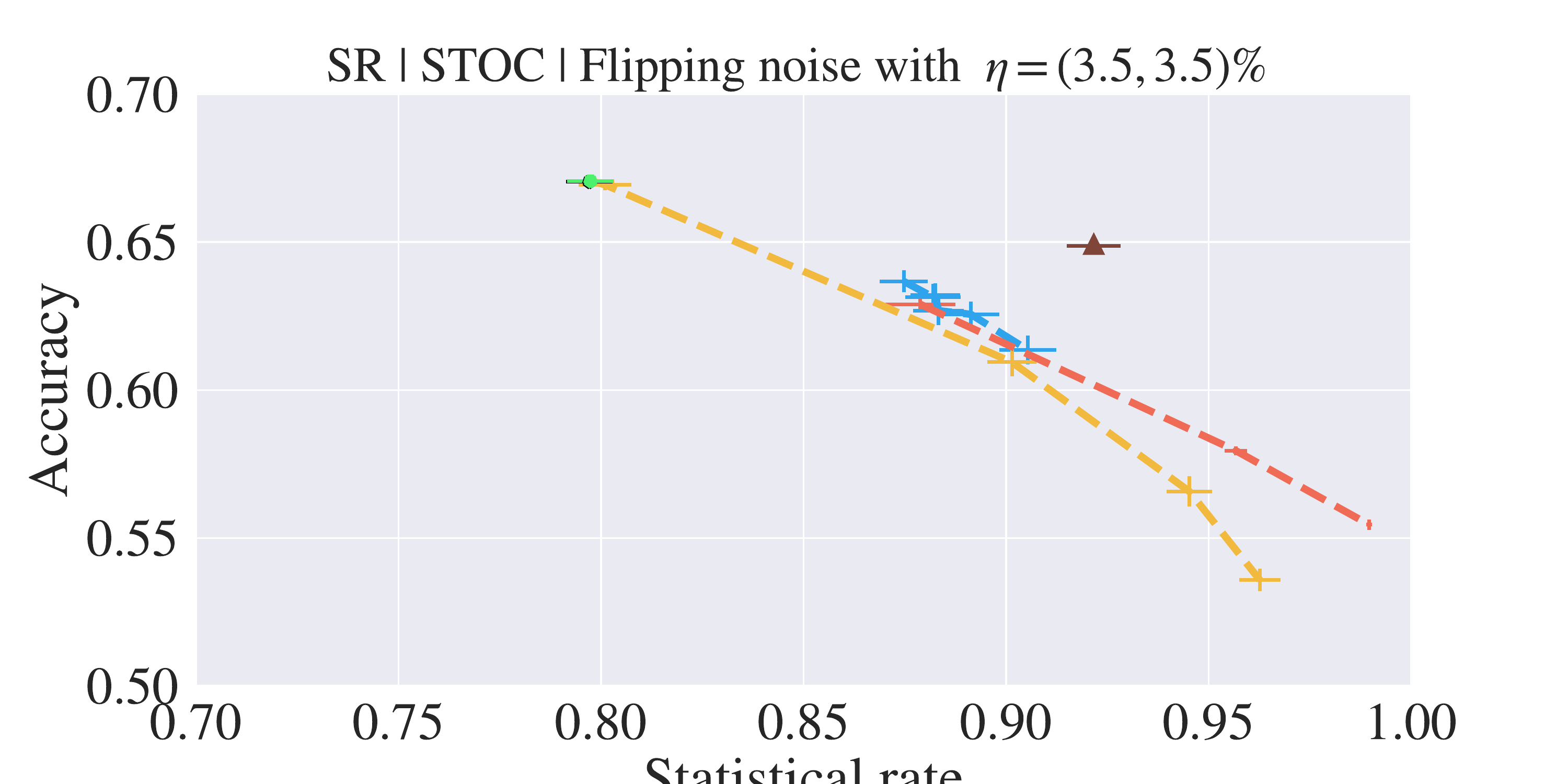}};
    \node[rotate=90, fill=white] at (-3.73,0) {\white{\large aaaaaaaaaaaaa}};
    \node[rotate=90] at (-3.9, 0) {Accuracy};
    \draw[draw=white, fill=white] (-2, -2.2) rectangle ++(4,0.35);
    \draw[draw=white, fill=white] (-3, 1.45) rectangle ++(6,0.25);
    \node[rotate=0, fill=white] at (0.05, -2.2)  {Statistical rate};
    \node[rotate=0] at (2.35,-2.2) {\textit{\small(more fair)}};
    \node[rotate=0] at (-2.25,-2.2) {\textit{\small(less fair)}};
  \end{tikzpicture}
  }
  \hspace{8mm}
  \hspace{-12mm}\subfigure[$(\eta_1,\eta_2)=(3.5\%,0\%)$\label{5b}]{
  \begin{tikzpicture}
    \node (image) at (0,-0.17) {\includegraphics[width=0.43\linewidth, trim={1.4cm 0cm 1.5cm  1cm},clip]{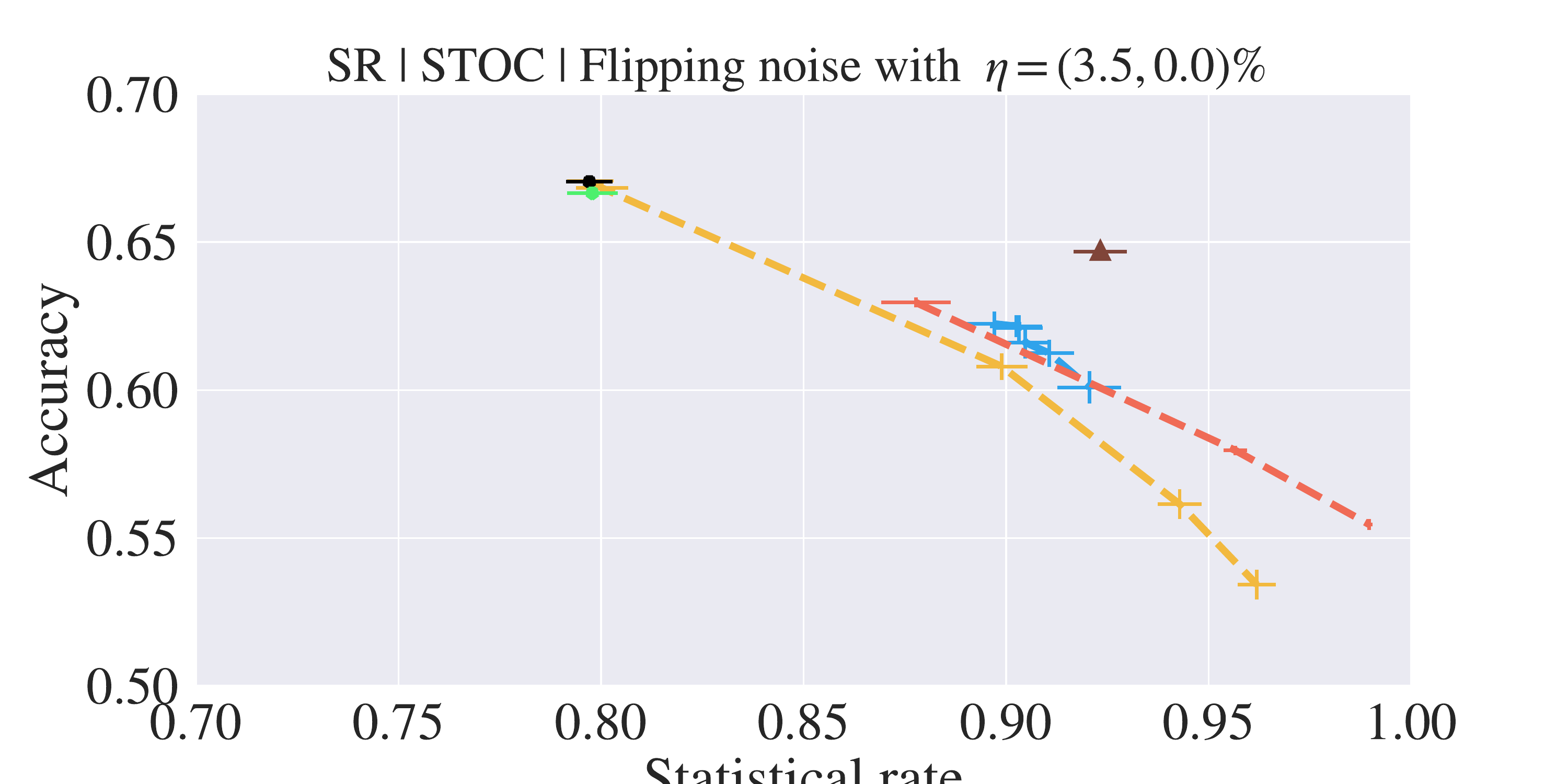}};
    \node[rotate=90, fill=white] at (-3.73,0) {\white{\large aaaaaaaaaaaaa}};
    \node[rotate=90] at (-3.9, 0) {Accuracy};
    \draw[draw=white, fill=white] (-2, -2.2) rectangle ++(4,0.35);
    \draw[draw=white, fill=white] (-3, 1.45) rectangle ++(6,0.25);
    \node[rotate=0, fill=white] at (0.05, -2.2)  {Statistical rate};
    \node[rotate=0] at (2.35,-2.2) {\textit{\small(more fair)}};
    \node[rotate=0] at (-2.25,-2.2) {\textit{\small(less fair)}};
  \end{tikzpicture}
  }
  \hspace{-5mm}
  \par
  \hspace{-7mm}\subfigure[$(\eta_1,\eta_2)=(0\%,3.5\%)$\label{5c}]{
  \begin{tikzpicture}
    \node (image) at (0,-0.17) {\includegraphics[width=0.43\linewidth, trim={1.4cm 0cm 1.5cm  1cm},clip]{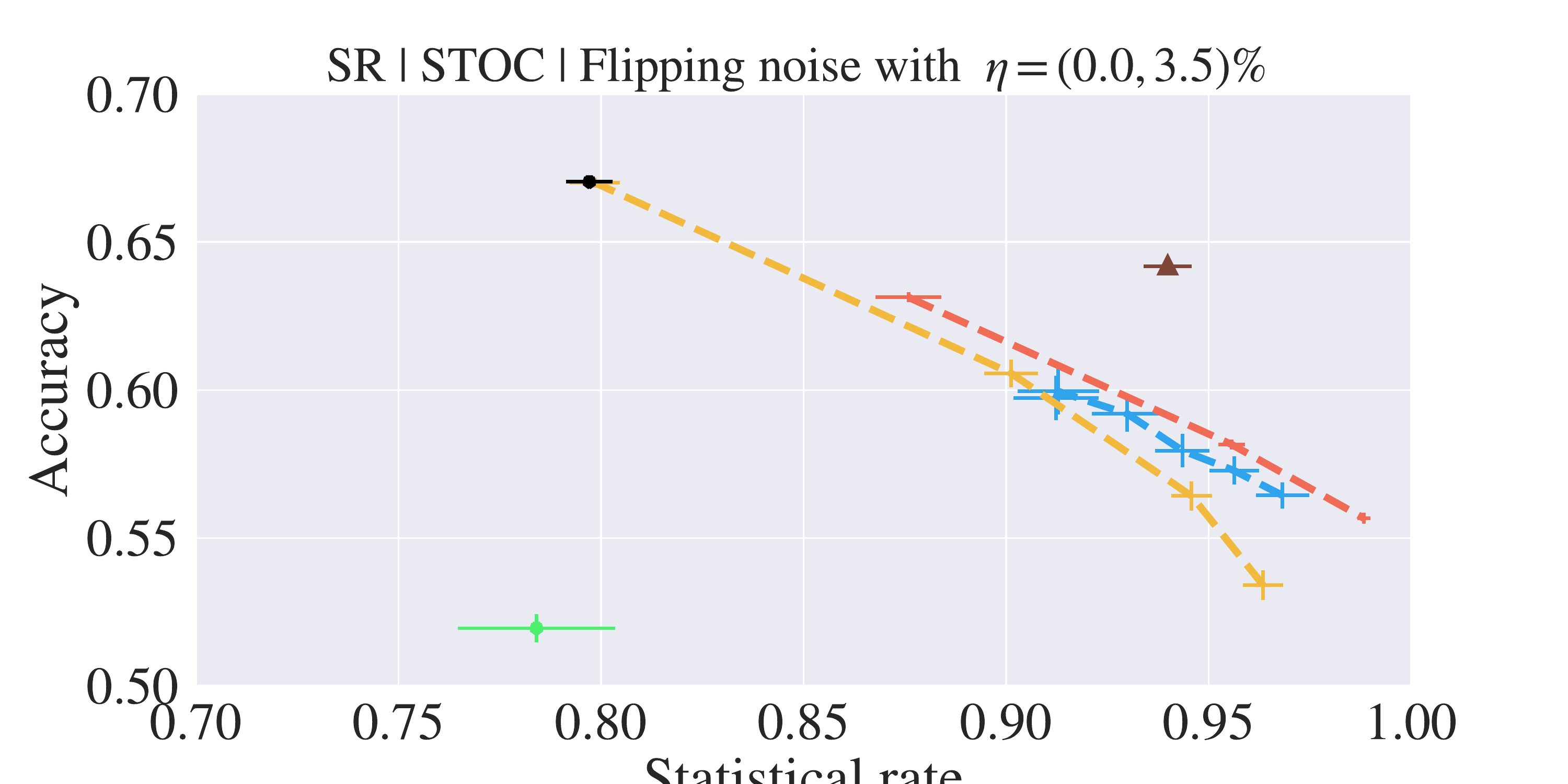}};
    \node[rotate=90, fill=white] at (-3.73,0) {\white{\large aaaaaaaaaaaaa}};
    \node[rotate=90] at (-3.9, 0) {Accuracy};
    \draw[draw=white, fill=white] (-2, -2.2) rectangle ++(4,0.35);
    \draw[draw=white, fill=white] (-3, 1.45) rectangle ++(6,0.25);
    \node[rotate=0, fill=white] at (0.05, -2.2)  {Statistical rate};
    \node[rotate=0] at (2.35,-2.2) {\textit{\small(more fair)}};
    \node[rotate=0] at (-2.25,-2.2) {\textit{\small(less fair)}};
  \end{tikzpicture}
  }
  \hspace{8mm}
  \hspace{-12mm}\subfigure[$(\eta_1,\eta_2)=(0\%,0\%)$\label{5d}]{
  \begin{tikzpicture}
    \node (image) at (0,-0.17) {\includegraphics[width=0.43\linewidth, trim={1.4cm 0cm 1.5cm  1cm},clip]{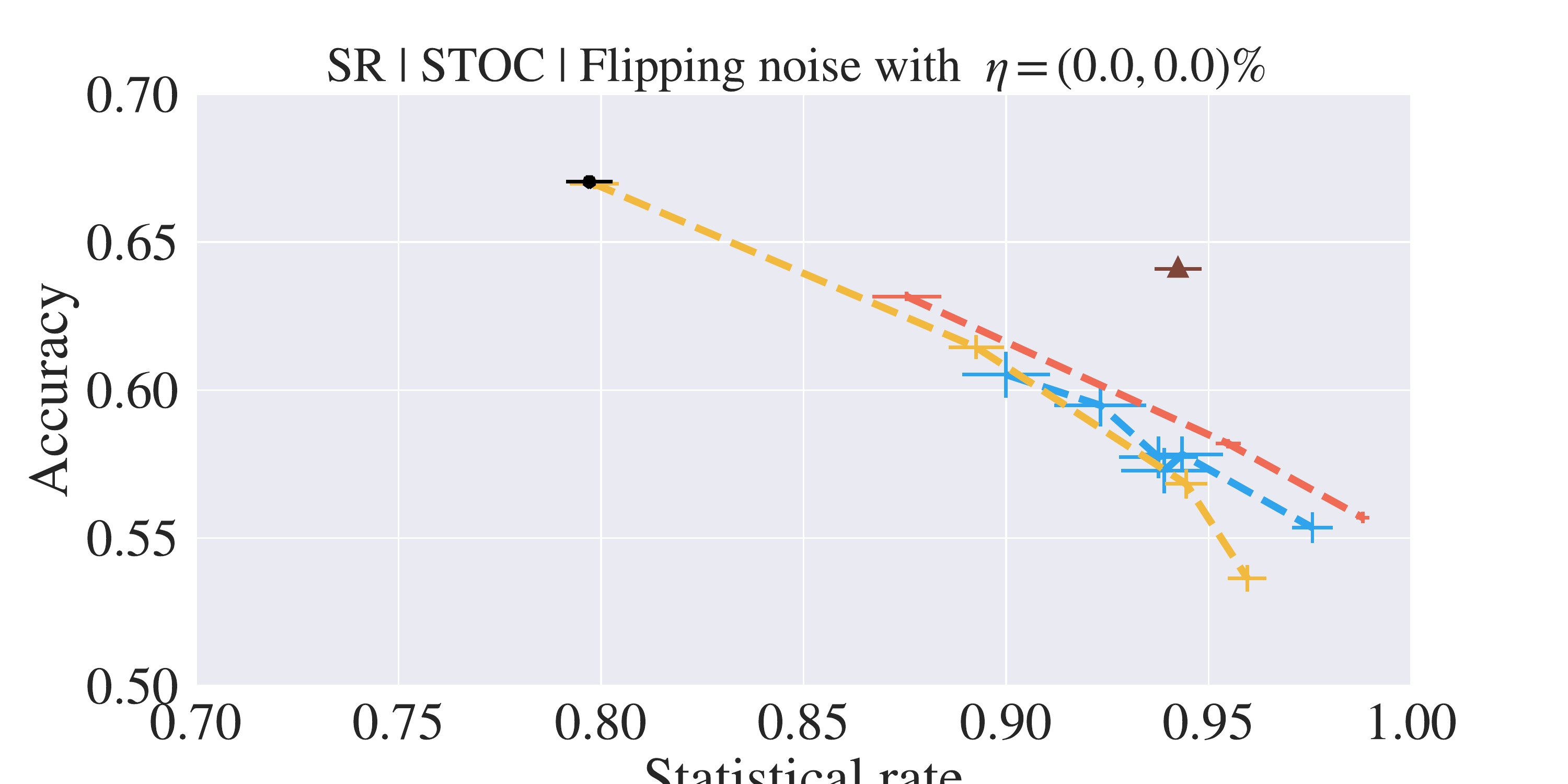}};
    \node[rotate=90, fill=white] at (-3.73,0) {\white{\large aaaaaaaaaaaaa}};
    \node[rotate=90] at (-3.9, 0) {Accuracy};
    \draw[draw=white, fill=white] (-2, -2.2) rectangle ++(4,0.35);
    \draw[draw=white, fill=white] (-3, 1.45) rectangle ++(6,0.25);
    \node[rotate=0, fill=white] at (0.05, -2.2)  {Statistical rate};
    \node[rotate=0] at (2.35,-2.2) {\textit{\small(more fair)}};
    \node[rotate=0] at (-2.25,-2.2) {\textit{\small(less fair)}};
  \end{tikzpicture}
  }
  \hspace{-5mm}
  \caption{
  \ifconf\small\fi
  {\em {Simulations on COMPAS data with flipping noise (see \cref{sec:simulations_with_flipping_noise}):}}
  Perturbed data is generated using the flipping noise model of \cite[Definition 2.3]{celis2020fairclassification} with $\eta_1$ and $\eta_2$ mentioned with each subfigure.
  All algorithms are run the perturbed data varying the fairness parameters ($\tau\in [0.7,1]$ and $\delta_L\in [0,0.1]$).
  The $y$-axis depicts the accuracy and the $x$-axis depicts statistical rate; both values are computed over the unperturbed test set.
  We observe that in each case, our approach, \errtol{}, has a similar fairness-accuracy trade-off as approaches tailored for flipping noise \cite{celis2020fairclassification,LamyZ19}
  Error bars represent the standard error of the mean over 100 iterations.
  }

  \label{fig:stochastic_noise}
\end{figure}

\subsubsection{Simulations with Adversarial Perturbations and False-Positive Rate on Compas Data}\label{sec:simulations_with_false_positive_rate}
In this simulation, we evaluate our framework for false-positive rate fairness metric with adversarial perturbations against state-of-the-art fair classification frameworks for false-positive rate under stochastic perturbations: \wangsw{} \cite{wang2020robust}, \wangdro{} \cite{wang2020robust}, and \chkvfpr{} \cite{celis2020fairclassification}.
We also compare against \kl{} \cite{konstantinov2021fairness}, which controls for true-positive rate (TPR) in the presence of a Malicious adversary, and \AKM{} \cite{awasthi2020equalized} that is the post-processing method of \cite{hardt2016equality} and controls for equalized-odds fairness constraints.

\newcommand{\displayAllAlgo}{}
\renewcommand{\folder}{./figures/adversarial-noise-FPR-final-100-iter/}
\begin{figure}[t]
  \centering
  \tikzmath{\s = 1.2;}
  \hspace{-5mm}\includegraphics[width=0.9\linewidth, trim={0cm 0cm 5cm 0cm},clip]{\folder/../legend-long.png}
  \par \vspace{-2mm}
  \subfigure[$A_{\rm TN}$ with $\eta=1.75\%$.\label{fig:fpr_main_results:a}]{
  \begin{tikzpicture}
    \node (image) at (0,-0.17) {\includegraphics[width=0.43\linewidth, trim={1.4cm 0cm 1.5cm 1cm},clip]{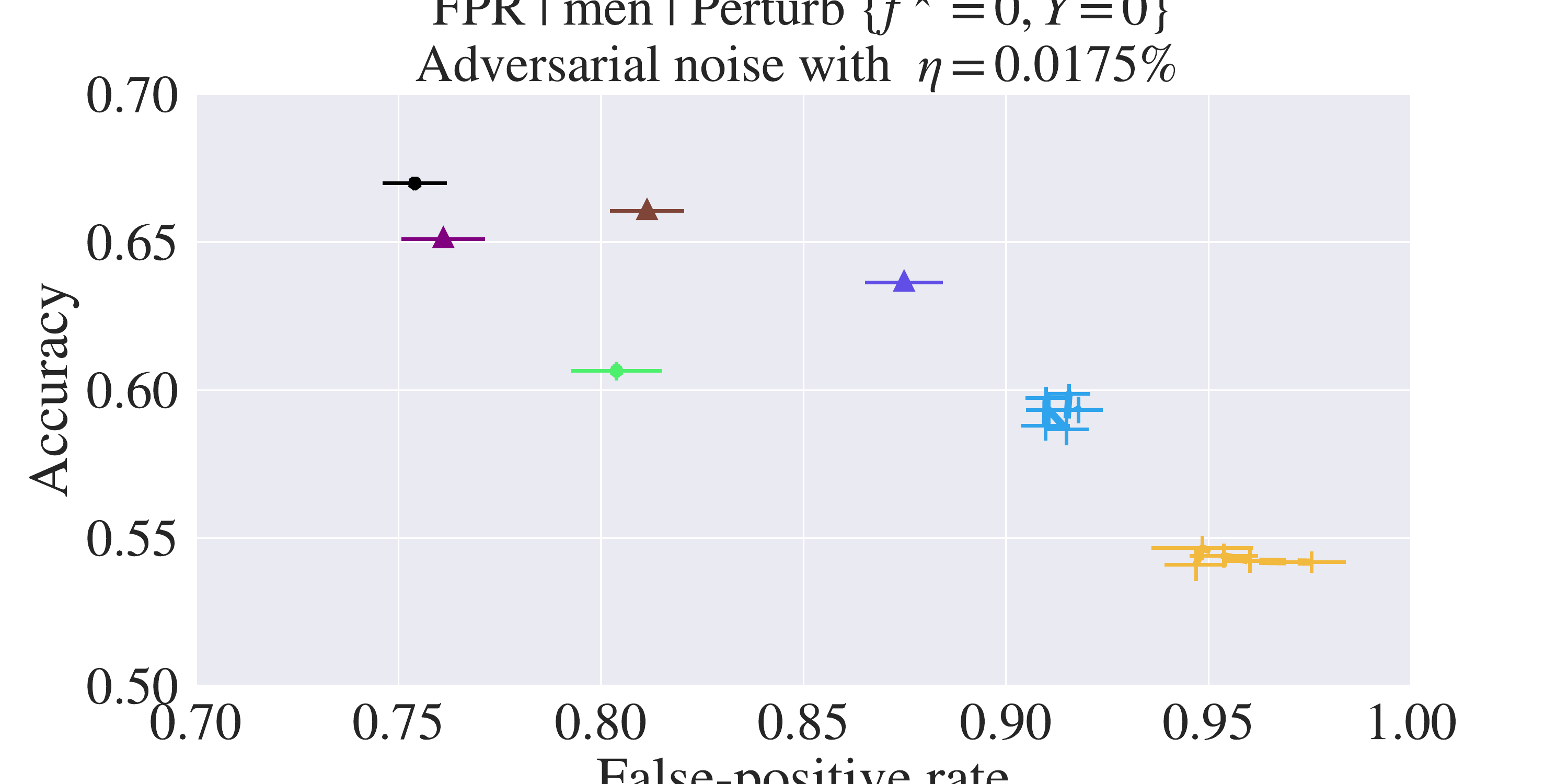}};
    \node[rotate=90, fill=white] at (-3.71,0) {\white{\large aaaaaaaaaaaaa}};
    \node[rotate=90] at (-3.9, 0) {Accuracy};
    \draw[draw=white, fill=white] (-2, -2.2) rectangle ++(4,0.35);
    \draw[draw=white, fill=white] (-3, 1.45) rectangle ++(6,0.25);
    \node[rotate=0, fill=white] at (0.05, -2.2)  {False positive rate};
    \node[rotate=0] at (2.45,-2.2) {\textit{\small(more fair)}};
    \node[rotate=0] at (-2.35,-2.2) {\textit{\small(less fair)}};
  \end{tikzpicture}
  }
  %
  %
  \subfigure[$A_{\rm FN}$ with $\eta=1.75\%$.\label{fig:fpr_main_results:b}]{
  \begin{tikzpicture}
    \node (image) at (0,-0.17) {\includegraphics[width=0.43\linewidth, trim={1.4cm 0cm 1.5cm 1cm},clip]{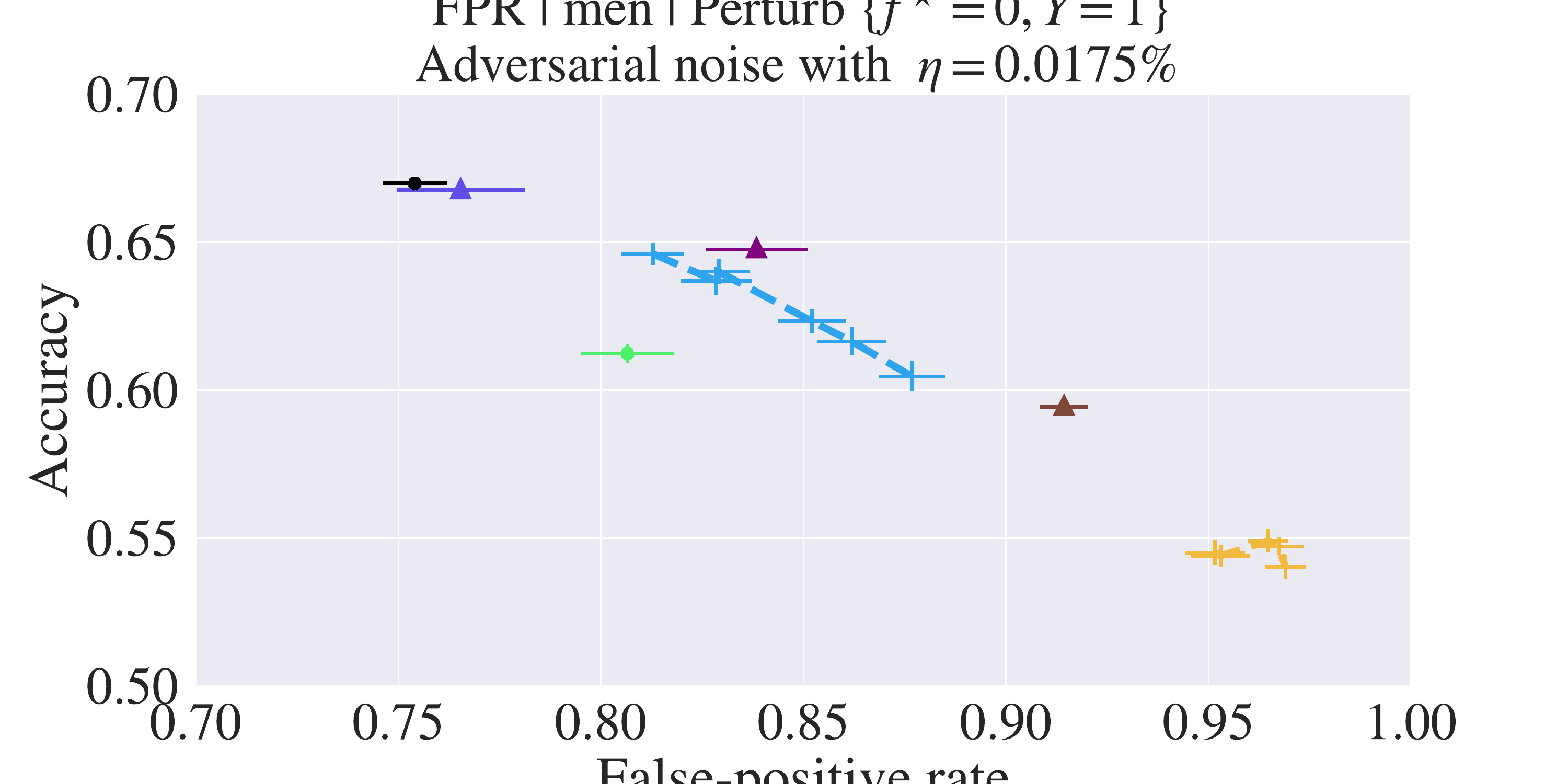}};
    \node[rotate=90, fill=white] at (-3.71,0) {\white{\large aaaaaaaaaaaaa}};
    \node[rotate=90] at (-3.9, 0) {Accuracy};
    \draw[draw=white, fill=white] (-2, -2.2) rectangle ++(4,0.35);
    \draw[draw=white, fill=white] (-3, 1.45) rectangle ++(6,0.25);
    \node[rotate=0, fill=white] at (0.05, -2.2)  {False positive rate};
    \node[rotate=0] at (2.45,-2.2) {\textit{\small(more fair)}};
    \node[rotate=0] at (-2.35,-2.2) {\textit{\small(less fair)}};
  \end{tikzpicture}
  }
  \caption{
  \ifconf\small\fi
  {\em {Simulations on COMPAS with false-positive rate (see \cref{sec:simulations_with_false_positive_rate}):}}
  Perturbed data is generated using adversaries $A_{\rm TN}$ (a) and $A_{\rm FN}$ as described in \cref{sec:more_empirical_results} with $\eta=1.75\%$.
  All algorithms are run on the perturbed data by varying the fairness parameters ($\tau\in [0.7,1]$).
  The $y$-axis depicts accuracy and the $x$-axis depicts false-positive rate computed over the unperturbed test set. %
  We observe that for all adversaries our approach, \errtol{}, attains a better fairness than the unconstrained classifier (\uncons{}) with a natural trade-off in accuracy.
  Further, \errtol{} achieves a higher fairness than each baseline except \chkv{} on at least one of (a) or (b).
  \chkv{} attains a higher fairness but has a low accuracy ($\sim 55\%$) (because it outputs the always-positive classifier).
  Error bars represent the standard error of the mean over 100 iterations.
  }


  \label{fig:fpr_main_results}
\end{figure}

Similar to the simulation with statistical rate fairness metric (see \cref{sec:empirics}), \errtol{} is given the perturbation rate $\eta$.
To advantage the baselines in our comparison, we provide them with more information as needed by their approaches:
\begin{enumerate}[itemsep=\itemsepINTERNAL,leftmargin=\leftmarginINTERNAL]
  \item \wangsw{} is given both the true and perturbed protected attributes as input; it internally generates the auxiliary data needed by \cite{wang2020robust}'s ``soft-weights'' approach.
  \item \wangdro{} is given both the true and perturbed protected attributes as input; it internally computes the total variation distances needed by \cite{wang2020robust}'s distributionally robust approach.
  \item \chkvfpr{} is given group-specific perturbation rates: $\forall$ $\ell\in[p]$, $\eta_\ell \coloneqq \Pr_D[\hZ\neq Z\mid Z=\ell]$.
  \item \kl{} is given $\eta$ and for each $\ell\in [p]$, the probability $\Pr_D[Z=\ell,Y=1]$; where $D$ is the empirical distribution of $S$.
\end{enumerate}
\errtol{} implements Program~\ref{prog:errtolerant_extnd} which requires estimates of $\lambda_\ell$ and $\gamma_\ell$ for all $\ell \in [p]$.
As a heuristic, we set $$\gamma_\ell=\lambda_\ell\coloneqq \Pr\nolimits_{\hD}[Y=1,Z=\ell],$$ where $\smash{\hD}$ is the empirical distribution of $\hS$.
More generally, for a general linear-fractional fairness metric, given by $\cE$ and $\cE^\prime$, the heuristic is to set $$\gamma_\ell=\lambda_\ell\coloneqq \Pr\nolimits_{\hD}[\cE^\prime(Y), Z=\ell],$$ where $\smash{\hD}$ is the empirical distribution of $\widehat{S}$ and $Y$ is the label in perturbed data $\widehat{S}$.

\paragraph{Adversaries.}
We consider the same adversaries as in \cref{sec:empirics} (which we call $A_{\rm TN}$ and $A_{\rm FN}$).
We consider a perturbation rate of $\eta=1.75\%$.
Again, $1.75\%$ is roughly the smallest value for $\eta$ necessary to ensure that the optimal fair classifier $f^\star$ for $\tau=0.9$ (on $S$) has a false-positive rate less than the false-positive rate of \uncons{} on the perturbed data.
This ``threshold'' perturbation rate is smaller for false-positive rate than for statistical rate because to the number of false positives of a classifier $f$ is smaller than the number of positive predictions of $f$; hence, an adversary perturbing the same number of samples, perturbs a larger fraction of false positives than the fraction of positive predictions of $f$.

\paragraph{Results.}
{The accuracy and statistical rate of \errtol{} and baselines for $\tau\in [0.7,1]$ are reported in \cref{fig:fpr_main_results}.}
For both adversaries, \errtol{} attains a better false-positive rate than the unconstrained classifier (\uncons{}) for a small trade-off in accuracy.
For adversary $A_{\rm TN}$ (\cref{fig:fpr_main_results}(a)), \uncons{} has false-positive rate ($0.75$) and accuracy ($0.65$).
In contrast, \errtol{} achieves a significantly higher false-positive rate ($0.92$) with  accuracy {($0.60$)}.
{In comparison, \chkvfpr{} has a higher false-positive rate ($0.97$) but lower accuracy {($0.55$)}; this accuracy is close to the accuracy of the all always-positive classifier.}
Compared to \errtol{}, \wangsw{} has a higher accuracy ($0.64$) but a lower false-positive rate ($0.87$), and
other baselines have an even lower false-positive rate ($\leq 0.82$) with accuracy comparable to \wangsw{}{\bf .}
For adversary $A_{\rm FN}$ (\cref{fig:adversarial_noise_main_body}(b)), \uncons{} has false-positive rate {($0.75$)} and accuracy ($0.67$),  while \errtol{} has a high higher false-positive rate {($0.87$)} and accuracy {($0.61$)}.
This significantly outperforms \wangsw{} which has false-positive rate {($0.76$)} and accuracy {($0.64$)}.
\akm{} achieves the higher false-positive rate ($0.92$) with a natural reduction in accuracy to {$0.59$}.
\chkv{} achieves the higher false-positive rate ($0.94$) but with a lower accuracy {($0.55$)}.
Meanwhile, \wangdro{} has a comparable false-positive rate ($0.84$) and a comparable accuracy at the same false-positive rate ($0.65$), and
\kl{} has a lower false-positive rate ($0.81$) and lower accuracy ($0.62$).

\renewcommand{\folder}{./figures/varying-eta}
\begin{figure}[t!]
  \centering
  \includegraphics[width=0.8\linewidth, trim={0cm 0cm 0cm 0cm},clip]{\folder/../legend-medium.pdf}
  \par \vspace{-2.5mm}
  \begin{tikzpicture}
    \node (image) at (0,-0.17) {\includegraphics[width=0.8\linewidth, trim={1.4cm 0cm 1.5cm 1cm},clip]{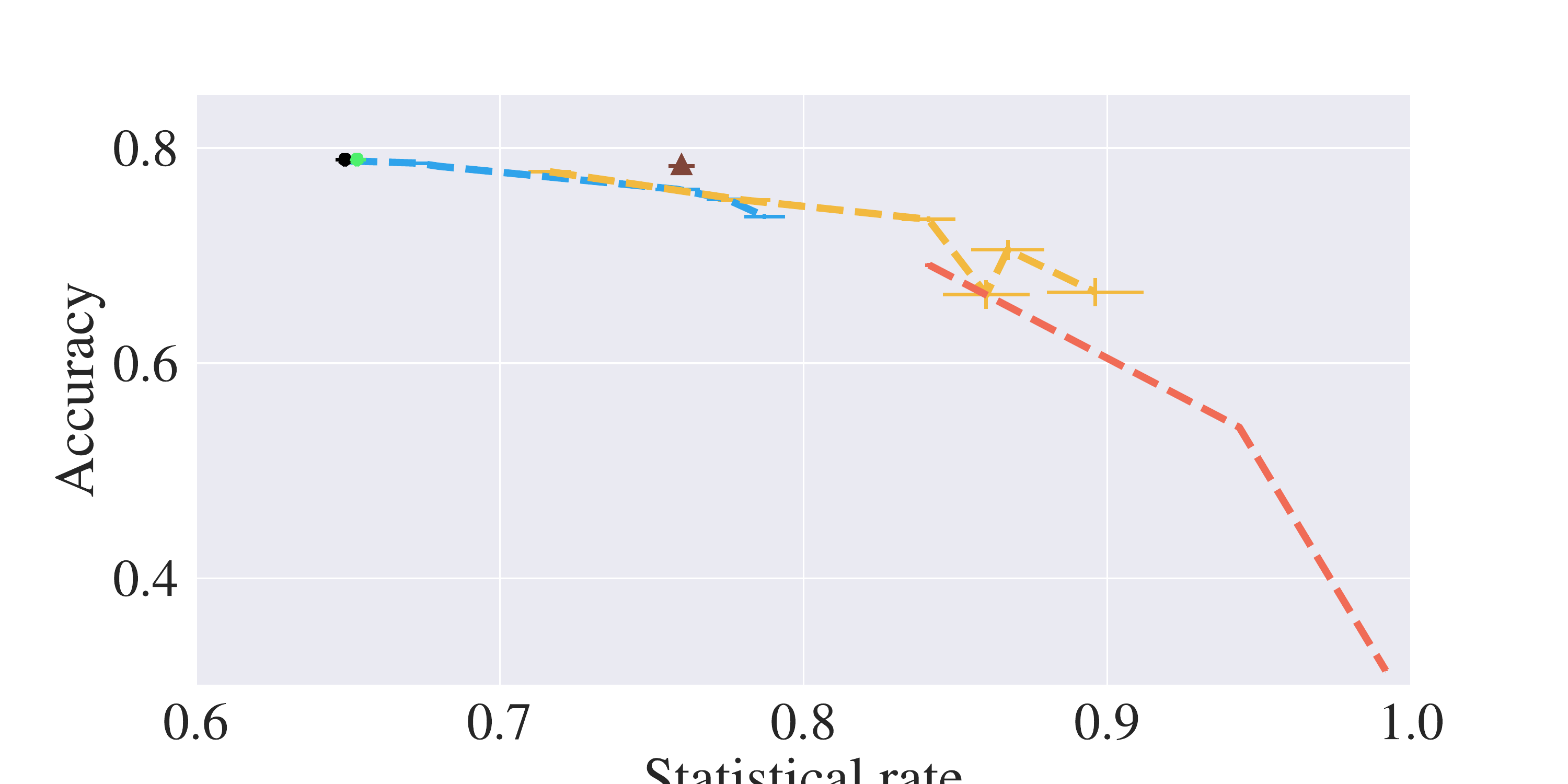}};
    \node[rotate=90, fill=white] at (-6.55, 0) {\Large Accuracy};
    \draw[draw=white, fill=white] (-2, -3.55) rectangle ++(4,0.35);
    \node[rotate=0, fill=white] at (0.1, -3.5)  {\Large Statistical rate};
    \node[rotate=0] at (4.55,-3.55) {\textit{\large(more fair)}};
    \node[rotate=0] at (-4.5+0.1,-3.55) {\textit{\large(less fair)}};
  \end{tikzpicture}
  %
  %
  \caption{
  \ifconf\small\fi
  {\em {Simulations on Adult data with gender as the protected attribute:}}
  Let $A_{\rm FP}$ be the adversary that is the same as $A_{\rm TN}$, except that instead of perturbing the true negatives of $f^\star$, it perturbs the false positives of $f^\star$.
  The perturbed data is generated using adversary $A_{\rm FP}$ (as described in \cref{sec:simulations_with_adult}) with $\eta=1\%$.
  All algorithms are run on the perturbed data varying the fairness parameters ($\tau\in [0.7,1]$ and $\delta_L\in [0,0.1]$).
  The $y$-axis depicts accuracy and the $x$-axis depicts statistical rate (SR); both values are computed over the unperturbed test set.
  (Error bars represent the standard error of the mean over 100 iterations.)
  We observe that for our approach \errtol{} attains a better statistical rate than the unconstrained classifier \uncons{} with a small natural  trade-off in accuracy.
  Further, \errtol{} achieves a better fairness-accuracy trade-off than \kl{} and achieves a similar fairness-accuracy trade-off as \akm{}\textbf{,} \chkv{}\textbf{,} \lmzv{}{\bf .}
  }
  \label{fig:adult:adversarial_noise_appendix}
\end{figure}
\subsubsection{Simulations with Adversarial Perturbations on the Adult Data}\label{sec:simulations_with_adult}
In this simulation, we evaluate our framework on the Adult data~\cite{adult} with the statistical rate fairness metric.
The Adult data consists of rows corresponding to approximately 45,000 individuals, with 18 binary features and a binary class label that is 1 is the individual has an income greater than \$50,000 USD and 0 otherwise.
Among the binary features, we use gender as the protected attribute.

\renewcommand{\folder}{./figures/varying-eta/}
\begin{figure}[h!]
  \centering
  \hspace{-5mm}\includegraphics[width=0.88\linewidth, trim={0cm 0cm 0cm 0cm},clip]{\folder/../legend-medium.pdf}
  \par\vspace{-2.5mm}
  \centering
  \par
  \subfigure[$\eta=0.5\%$\label{vary_etaa}]{
  \begin{tikzpicture}
    \node (image) at (0,-0.17) {\includegraphics[width=0.50\linewidth, trim={1.4cm 0cm 1.5cm  1cm},clip]{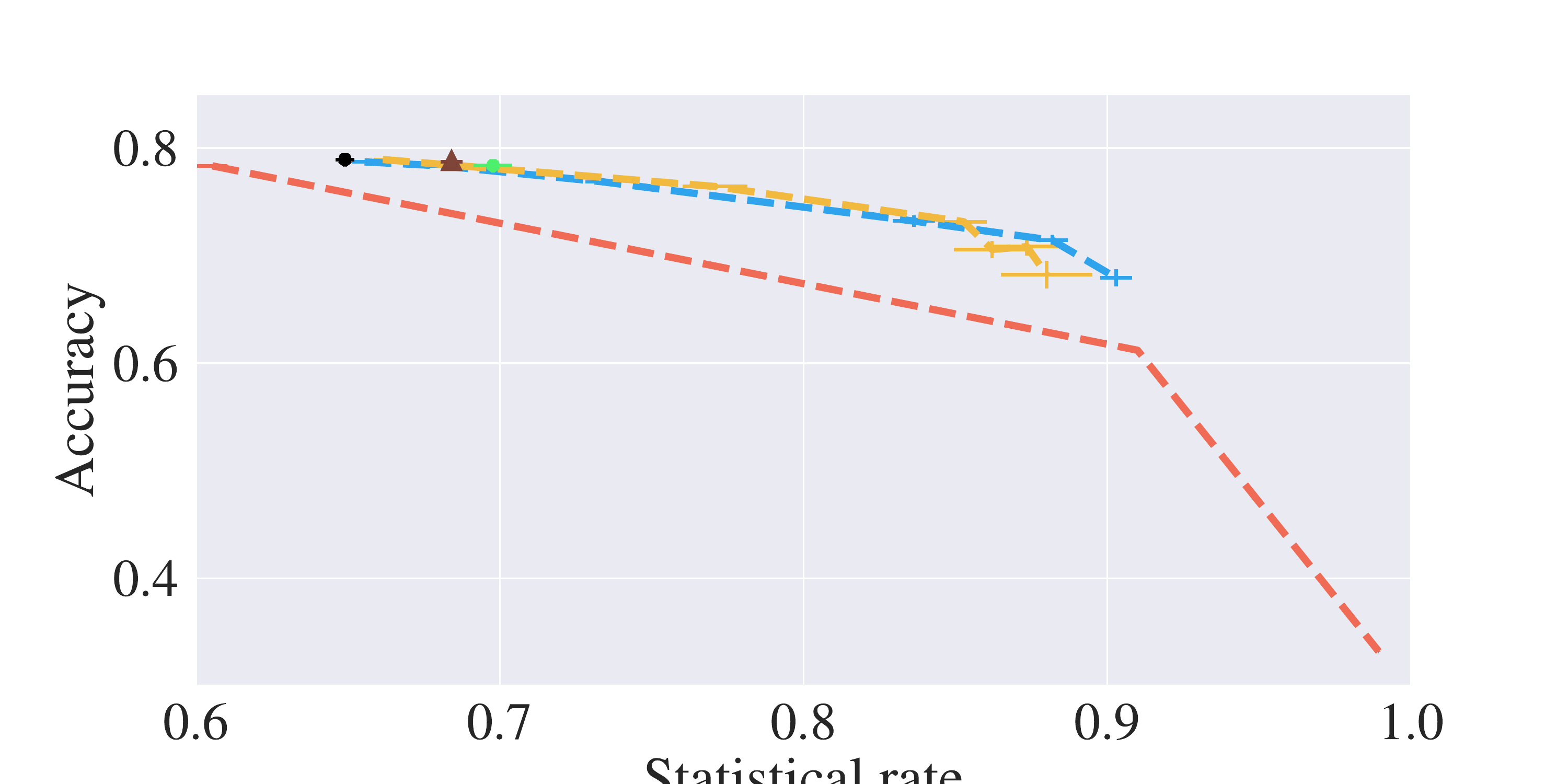}};
    \node[rotate=90, fill=white] at (-3.75*5/8*0.58/0.43-1.14*0.58/0.43+0.45*0.58/0.43,0) {\white{\Large aaaaaaaaaaaaa}};
    \node[rotate=90] at (-3.75*5/8*0.58/0.43-1.15*0.58/0.43+0.5*0.58/0.43,0) {Accuracy};
    \draw[draw=white, fill=white] (-2*0.58/0.43, -3.3*5/8*0.58/0.43+0.36*0.58/0.43-0.05*0.58/0.43-0.01) rectangle ++(4*0.58/0.43,0.15*0.58/0.43+0.15);
    \draw[draw=white, fill=white] (-2*0.58/0.43, 3.3*5/8*0.58/0.43-0.66*0.58/0.43-0.17*0.58/0.43+0.02) rectangle ++(4*0.58/0.43,0.15*0.58/0.43+0.13);
    \node[rotate=0, fill=white] at (0+0.05*0.58/0.43, 0.15*0.58/0.43-3.3*5/8*0.58/0.43-0.1*0.58/0.43+0.15*0.58/0.43-0.1+0.10+0.04)  {Statistical rate};
    \node[rotate=0] at (4*5/8*0.58/0.43-0.15*0.58/0.43,0.15*0.58/0.43-3.3*5/8*0.58/0.43-0.1*0.58/0.43+0.15*0.58/0.43+0.06+0.04) {\textit{\small(more fair)}};\node[rotate=0] at (-3.75*5/8*0.58/0.43+0.2*0.58/0.43, 0.15*0.58/0.43-3.3*5/8*0.58/0.43-0.1*0.58/0.43+0.15*0.58/0.43+0.06+0.04) {\textit{\small(less fair)}};
  \end{tikzpicture}
  }
  \par
  %
  \subfigure[$\eta=1\%$\label{vary_etac}]{
  \begin{tikzpicture}
    \node (image) at (0,-0.17) {\includegraphics[width=0.50\linewidth, trim={1.4cm 0cm 1.5cm  1cm},clip]{\folder/adult-afp-eta100.pdf}};
    \node[rotate=90, fill=white] at (-3.75*5/8*0.58/0.43-1.14*0.58/0.43+0.45*0.58/0.43,0) {\white{\Large aaaaaaaaaaaaa}};
    \node[rotate=90] at (-3.75*5/8*0.58/0.43-1.15*0.58/0.43+0.5*0.58/0.43,0) {Accuracy};
    \draw[draw=white, fill=white] (-2*0.58/0.43, -3.3*5/8*0.58/0.43+0.36*0.58/0.43-0.05*0.58/0.43-0.01) rectangle ++(4*0.58/0.43,0.15*0.58/0.43+0.15);
    \draw[draw=white, fill=white] (-2*0.58/0.43, 3.3*5/8*0.58/0.43-0.66*0.58/0.43-0.17*0.58/0.43+0.02) rectangle ++(4*0.58/0.43,0.15*0.58/0.43+0.13);
    \node[rotate=0, fill=white] at (0+0.05*0.58/0.43, 0.15*0.58/0.43-3.3*5/8*0.58/0.43-0.1*0.58/0.43+0.15*0.58/0.43-0.1+0.10+0.04)  {Statistical rate};
    \node[rotate=0] at (4*5/8*0.58/0.43-0.15*0.58/0.43,0.15*0.58/0.43-3.3*5/8*0.58/0.43-0.1*0.58/0.43+0.15*0.58/0.43+0.06+0.04) {\textit{\small(more fair)}};\node[rotate=0] at (-3.75*5/8*0.58/0.43+0.2*0.58/0.43, 0.15*0.58/0.43-3.3*5/8*0.58/0.43-0.1*0.58/0.43+0.15*0.58/0.43+0.06+0.04) {\textit{\small(less fair)}};
  \end{tikzpicture}
  }
  \par
  %
  \subfigure[$\eta=1.5\%$\label{vary_etad}]{
  \begin{tikzpicture}
    \node (image) at (0,-0.17) {\includegraphics[width=0.50\linewidth, trim={1.4cm 0cm 1.5cm  1cm},clip]{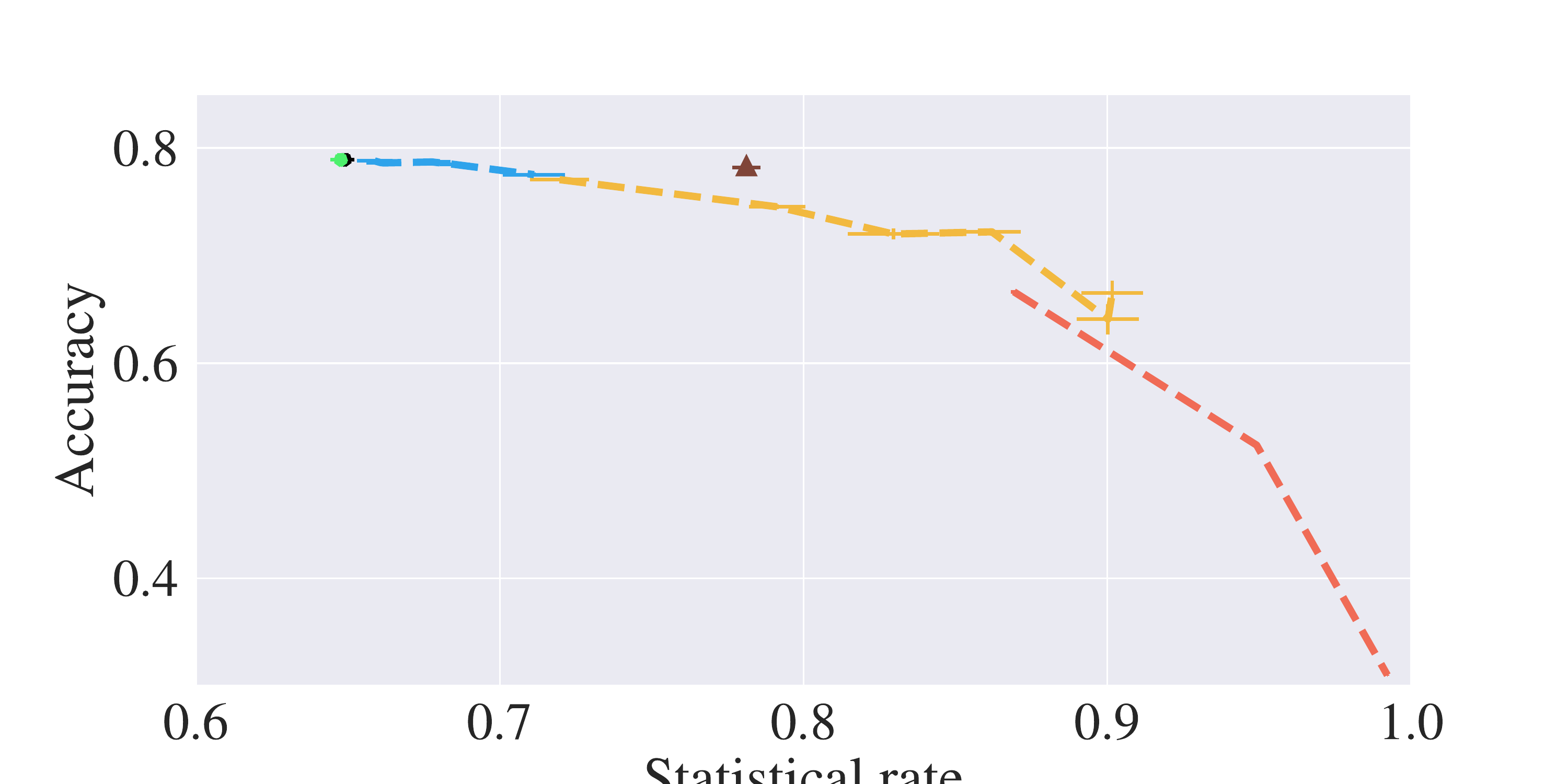}};
    \node[rotate=90, fill=white] at (-3.75*5/8*0.58/0.43-1.14*0.58/0.43+0.45*0.58/0.43,0) {\white{\Large aaaaaaaaaaaaa}};
    \node[rotate=90] at (-3.75*5/8*0.58/0.43-1.15*0.58/0.43+0.5*0.58/0.43,0) {Accuracy};
    \draw[draw=white, fill=white] (-2*0.58/0.43, -3.3*5/8*0.58/0.43+0.36*0.58/0.43-0.05*0.58/0.43-0.01) rectangle ++(4*0.58/0.43,0.15*0.58/0.43+0.15);
    \draw[draw=white, fill=white] (-2*0.58/0.43, 3.3*5/8*0.58/0.43-0.66*0.58/0.43-0.17*0.58/0.43+0.02) rectangle ++(4*0.58/0.43,0.15*0.58/0.43+0.13);
    \node[rotate=0, fill=white] at (0+0.05*0.58/0.43, 0.15*0.58/0.43-3.3*5/8*0.58/0.43-0.1*0.58/0.43+0.15*0.58/0.43-0.1+0.10+0.04)  {Statistical rate};
    \node[rotate=0] at (4*5/8*0.58/0.43-0.15*0.58/0.43,0.15*0.58/0.43-3.3*5/8*0.58/0.43-0.1*0.58/0.43+0.15*0.58/0.43+0.06+0.04) {\textit{\small(more fair)}};\node[rotate=0] at (-3.75*5/8*0.58/0.43+0.2*0.58/0.43, 0.15*0.58/0.43-3.3*5/8*0.58/0.43-0.1*0.58/0.43+0.15*0.58/0.43+0.06+0.04) {\textit{\small(less fair)}};
  \end{tikzpicture}
  }
  %
  \caption{
  \ifconf\small\fi
  {\em {Simulations on Adult data with adversarial noise on varying $\eta$ (see \cref{sec:simulations_with_flipping_noise}):}}
  In this the results obtained by repeating the simulation in \cref{sec:simulations_with_adult} by varying $\eta$ over $\inbrace{0.5\%, 1\%, 1.5\%}$.
  The $y$-axis depicts the accuracy and the $x$-axis depicts statistical rate; both values are computed over the unperturbed test set.
  (Error bars represent the standard error of the mean over 50 iterations.)
  We observe that for all values of $\eta\in \inbrace{0.5\%, 1\%, 1.5\%}$, \errtol{} achieves a higher statistical rate than the unconstrained classifier \uncons{} at a natural tradeoff to accuracy and \errtol{} has a similar (or better) fairness-accuracy tradeoff than other baselines.
  The only exception is \akm{}, which has a better fairness-accuracy tradeoff than \errtol{} at $\eta=1.5\%$.
  For further discussion of the results, we refer the reader to \cref{rem:vary_eta}.
  }


  \label{fig:adult:vary_eta}
\end{figure}

\paragraph{Baselines.}
Like the simulation with the statistical rate fairness metric on the COMPAS data (\cref{sec:empirics}), we compare our framework with:
\begin{enumerate}[itemsep=\itemsepINTERNAL,leftmargin=\leftmarginINTERNAL]
  \item State-of-the-art fair classification frameworks for statistical rate under stochastic perturbations: \lmzv{} \cite{LamyZ19} and \chkv{} \cite{celis2020fairclassification}.
  \item \kl{} \cite{konstantinov2021fairness}, which controls for true-positive rate in the presence of a Malicious adversary.
  \item \AKM{} \cite{awasthi2020equalized} that is the post-processing method of \cite{hardt2016equality} and controls for equalized-odds fairness.
  \item The optimal unconstrained classifier, \uncons{}\textbf{;} this is the same as \cite{bshouty2002pac}'s algorithm for PAC-learning in the Nasty Sample Noise Model without fairness constraints.
\end{enumerate}

\paragraph{Adversaries and implementation details.}
We set $\eta=1\%$ and consider an adversary $A_{\rm FP}$ that is the same as $A_{\rm TN}$, except that instead of perturbing the true negatives of $f^\star$, it perturbs the false positives of $f^\star$.\footnote{We were unable to implement the analogous adversary $A_{\rm TP}$, that perturbs the true positives of $f^\star$, because on the Adult data, $f^\star$ does not have sufficient number of true positives with $Z=1$.}
Note that positive labels are rare in the Adult data (e.g., less than 4\% of the total samples are positive and annotated as women).
Thus, an adversary with $\eta\geq 4\%$ can remove all positive samples from one protected group--thereby, changing the statistical rate on the perturbed data by an arbitrary amount.
This suggests that corrupting the positive labels is a hard case for learning fair classifiers on the Adult data.
The adversary tries to reduce the performance of $f^\star$ on $Z=1$ (which is the rarer than $Z=2$) in $\hS$ by removing the samples that $f^\star$ predicts as positive.
Thus, decreasing $f^\star$'s statistical rate on $\hS$.
All other implementation details were identical to the simulation with the COMPAS data in \cref{sec:empirics}.

\paragraph{Observations.}
The accuracy and statistical rate (SR) of \errtol{} and baselines for $\tau\in [0.7,1]$ and $\delta_L\in [0,0.1]$  and averaged over 100 iterations are reported in \cref{fig:adult:adversarial_noise_appendix}.
We observe that \errtol{} attains better fairness than the unconstrained classifier \uncons{} at a small tradeoff to accuracy.
Further, \errtol{} has a fairness-accuracy tradeoff that is better than \kl{} and at least as good as \akm{}\textbf{,} \chkv{}\textbf{,} and \lmzv{}\textbf{.}
These observations are consistent with our observations on the COMPAS data in \cref{sec:empirics}.

\begin{remark}\label{rem:vary_eta}
  We also explored the effect of changing the perturbation rate $\eta$ over $\inbrace{0.5\%, 1\%, 1.5\%}$.
  We report the results from this simulation in \cref{fig:adult:vary_eta}.
  We observe that for all values of $\eta\in \inbrace{0.5\%, 1\%, 1.5\%}$, \errtol{} achieves a higher statistical rate than the unconstrained classifier \uncons{} at a natural tradeoff to accuracy and \errtol{} has a similar (or better) fairness-accuracy tradeoff than other baselines.
  The only exception is \akm{}, which has a better fairness-accuracy tradeoff than \errtol{} at $\eta=1.5\%$.
  Further, we observe that the highest statistical rate of achieved by \errtol{} decreases as $\eta$ increases and this decrease is larger than the corresponding decrease in the statistical rate of \akm{}{\bf ,} \chkv{}{\bf ,} and \lmzv{}{\bf .}
  We believe this is because $A_{\rm FP}$ is not the worst-case adversary (for this data), and hence, our approach, which ``protects'' against the worst-case adversary, outputs a ``more robust'' classifier which happens to have a low statistical rate on Adult data.
  In contrast, the prior works do not correct for the worst-case adversaries and are able to output ``less robust'' classifiers which happen to have a high statistical rate for this adversary, but may perform poorly with the worst-case adversary.
\end{remark}

\end{document}